\documentclass{article} 

\usepackage{hyperref}
\usepackage{url}

\usepackage{float}     
\RequirePackage{algorithm}
\RequirePackage{algorithmic}
\usepackage[english]{babel}
\usepackage[authoryear]{natbib}
\usepackage{microtype}
\usepackage{graphicx}
\usepackage{subcaption}
\usepackage{booktabs}
\usepackage{wrapfig}
\usepackage{amsmath}
\usepackage{amssymb}
\usepackage{mathtools}
\usepackage{amsthm}
\usepackage{thmtools}
\usepackage{thm-restate}
\usepackage{xcolor}
\usepackage{listings}
\usepackage{caption}
\usepackage{tcolorbox}
\usepackage{titletoc}
\usepackage{placeins}
\theoremstyle{plain}
\newtheorem{theorem}{Theorem}[section]
\newtheorem{proposition}[theorem]{Proposition}
\newtheorem{lemma}[theorem]{Lemma}
\newtheorem{corollary}[theorem]{Corollary}
\theoremstyle{definition}
\newtheorem{definition}[theorem]{Definition}
\newtheorem{assumption}[theorem]{Assumption}
\theoremstyle{remark}
\newtheorem{remark}[theorem]{Remark}
\usepackage[capitalize,noabbrev,nameinlink]{cleveref}
\crefname{assumption}{Assumption}{Assumptions}
\Crefname{assumption}{Assumption}{Assumptions}
\newcommand{\R}{\mathbb R}
\newcommand{\E}{\mathbb E}
\newcommand{\Sph}{\mathbb S^{d-1}}
\newcommand{\inner}[2]{\langle #1,#2\rangle}
\newcommand{\norm}[1]{\left\lVert #1\right\rVert}

\newcommand{\Unif}{\mathrm{Unif}}

\newcommand{\Sym}{\mathrm{Sym}}

\definecolor{codeblue}{rgb}{0.25,0.5,0.5}
\definecolor{codekw}{rgb}{0.85, 0.18, 0.50}

\definecolor{codegreen}{rgb}{0,0.6,0}
\definecolor{codegray}{rgb}{0.5,0.5,0.5}
\definecolor{codepurple}{rgb}{0.58,0,0.82}
\definecolor{backcolour}{rgb}{0.95,0.95,0.92}
\definecolor{codepurple}{rgb}{0.58,0,0.82}
\definecolor{papercolor}{HTML}{0668E1}
\definecolor{darkred}{rgb}{0.68,0.05,0.0}
\definecolor{tab_blue}{RGB}{230,245,255} 
\definecolor{tab_purple}{RGB}{245,230,255} 
\definecolor{mygreen}{RGB}{66,150,83}

\usepackage{ssArxiv}

\hypersetup{
    colorlinks,
    linkcolor={mygreen},
    citecolor={darkblue},
    urlcolor={codepurple}
}
\newcommand{\samecolorfootnote}[1]{\textsuperscript{\textcolor{darkred}{#1}}}

\title{Canonicalizing
Multimodal Contrastive \\Representation Learning}

\author{
\begin{tabular}{c@{\hspace{1.2cm}}c@{\hspace{1.2cm}}c}
{Sharut Gupta\samecolorfootnote{*}} &
{Sanyam Kansal\samecolorfootnote{*}} &
{Stefanie Jegelka} \\
MIT CSAIL &
IIT Kanpur &
MIT CSAIL, TU Munich  \\
\texttt{sharut@csail.mit.edu} &
\texttt{sanyamka23@iitk.ac.in} &
\texttt{stefje@csail.mit.edu} \\
\\[3.0ex]
\multicolumn{3}{c}{%
\begin{tabular}{c@{\hspace{2.2cm}}c}
\textbf{Phillip Isola} &
\textbf{Vikas Garg} \\
MIT CSAIL &
Aalto University, YaiYai Ltd \\
\texttt{phillipi@csail.mit.edu} &
\texttt{vgarg@csail.mit.edu}
\end{tabular}
}
\end{tabular}
}

\begin{document}

\maketitle

\begin{abstract}
As models and data scale, independently trained networks often induce analogous notions of similarity. But, matching similarities is weaker than establishing an explicit correspondence between the representation spaces, especially for multimodal models, where consistency must hold not only within each modality, but also for the learned image–text coupling. We therefore ask: given two \emph{independently} trained multimodal contrastive models (with encoders $(f, g)$ and $(\tilde f,\tilde g)$)---trained on different distributions and with different architectures---does a systematic geometric relationship exist between their embedding spaces? If so, what form does it take, and does it hold uniformly across modalities? In this work, we show that across model families such as CLIP, SigLIP, and FLAVA, this geometric relationship is well approximated by an orthogonal map (up to a global mean shift), i.e., there exists an orthogonal map $Q$ where $Q^\top Q = I$ such that $\tilde f(x)\approx Q f(x)$ for paired images $x$. Strikingly, the \emph{same} $Q$ simultaneously aligns the text encoders i.e., $\tilde g(y)\approx Q g(y)$ for texts $y$. Theoretically, we prove that if the multimodal kernel agrees across models on a small anchor set i.e. $\langle f(x), g(y)\rangle \approx \langle \tilde f(x), \tilde g(y)\rangle$, then the two models must be related by a \emph{single orthogonal map} $Q$ and the same $Q$ maps images and text across models. More broadly, this finding enables backward-compatible model upgrades, avoiding costly re-embedding, and has implications for the privacy of learned representations.\looseness=-1 

Our project page: \url{https://canonical-multimodal.github.io/} 
\end{abstract}
\footnotetext{Equal contribution.}

\section{Introduction}

\begin{figure*}[!htb]
    \centering
    \includegraphics[width=1.\linewidth]{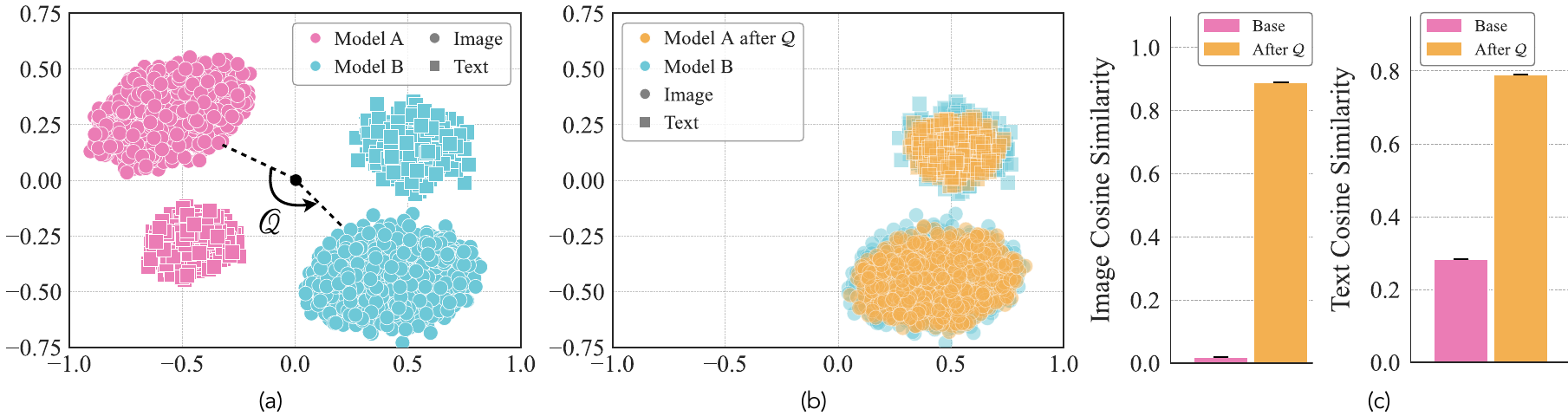}
    \caption{\textit{A single orthogonal map aligns two independent contrastive models across modalities.} (a) Illustrative PCA schematic on synthetic embeddings from two models (A and B) shows that the spaces are a priori incomparable; (b) A single orthogonal map $\mathcal Q$, fit only using image embeddings of CLIP (OpenAI) and CLIP (LAION), almost perfectly aligns image embeddings of the two models (c, left) and simultaneously aligns their text embeddings, as evidenced by a large gain in text-text pointwise cosine similarity (c, right).\looseness=-1}
    \label{fig:teaser}
\end{figure*}

A recurring question in modern representation learning is \emph{convergence}: as models and data scale, do independently trained networks---across datasets, architectures, and training runs---recover similar internal representations? A growing body of evidence suggests they often do, in the sense that different systems can induce similar notions of similarity over inputs or contain universal neurons and ``circuits" \citep{huh2024prh,merullo2022linearly,gupta2025better,chughtai2023toy,gurnee2024universal,zimmermann2021invert}. This idea is central to the \emph{Platonic Representation Hypothesis} (PRH), which posits that, at a sufficiently large scale, learned embeddings converge towards a shared representation that reflects the underlying structure of the world \citep{huh2024prh}. Empirically, this convergence is commonly studied through representational similarity analyses such as SVCCA and CKA \citep{raghu2017svcca,kornblith2019cka,huh2024prh} over unimodal co-occurrence kernels.\looseness=-1

\par However, much of this literature abstracts away parameterization. SVCCA is invariant to affine transformations \citep{raghu2017svcca}, and CKA compares induced similarity structure rather than the precise geometric correspondence between them \citep{kornblith2019cka}. Two models may agree on many tasks while their internal embedding spaces remain related only by complex, sample-dependent distortions. From a geometric standpoint, the stronger and far more consequential question is whether independently trained models recover representations that are equivalent up to simple transformations. This question becomes especially important for \emph{multimodal} models, which couple image and text through a contrastive objective while keeping the two modalities at arm's length in the learned space, a phenomenon often referred to as the \emph{modality gap}. To make this precise, let $\mathcal M\coloneq(f,g)$ and $\tilde{\mathcal M}\coloneq(\tilde f,\tilde g)$ denote two distinct multimodal models, with image encoders $f,\tilde f$ and text encoders $g,\tilde g$ mapping inputs to embedding spaces. Here, it no longer suffices to ask whether image representations $f$ and $\tilde f$ of $\mathcal M$ and $\tilde{\mathcal M}$ converge in isolation (or text encoders $g$ and $\tilde g$ of the two models converge in isolation). Instead, the central question becomes:\looseness=-1

\begin{center}
\begin{tcolorbox}[colback=gray!5!white,
                  colframe=black!50,
                  boxrule=0.9pt,
                  arc=2pt,
                  width=1.\linewidth]
\emph{Given two independently trained multimodal models, does a systematic geometric relationship exist between their embedding spaces? If so, what is its form, and how does it differ across modalities?\looseness=-1
}
\end{tcolorbox}
\end{center}

In this work, we study this question for multimodal \emph{contrastive} models and show that two independently trained instances---with different embedding dimensions, training distributions and modeling choices---exhibit a remarkably rigid, \emph{modality-invariant} geometric relationship. 
Concretely, across model families such as CLIP~\citep{radford2021clip}, SigLIP~\citep{zhai2023siglip}, and FLAVA~\citep{singh2022flava}, we find that the inter-model relationship is well-approximated by a \emph{single orthogonal} map i.e., there exists an orthogonal map $Q$ where $Q^\top Q = I$ such that $\tilde f(x)\approx Q f(x)$ for paired images $x$ and mean-centered\footnote{Even \emph{without} the mean-centering, this alignment holds up to semantic boundaries i.e. class-level retrieval and decision geometry; the mean shift primarily improves pointwise cosine agreement.} encoders $f$ and $\tilde f$. Moreover, the \emph{same} $Q$ simultaneously aligns mean-centered text encoders i.e. $\tilde g(y) \approx Q g(y)$ for texts $y$ (as shown in~\Cref{fig:teaser}). This induces a commuting correspondence between encoders; once $Q$ is learned, any embedding produced by one model---image or text---can be mapped into the other model’s coordinate system and back and compared meaningfully with any embedding there.

Empirically, despite never using text to estimate $Q$, applying this map to text substantially improves cross-model agreement in the target text space, as measured both by (i) the mean cosine similarity between matched text embeddings after mapping, and (ii) prompt retrieval, in which each mapped prompt is matched to its nearest neighbor among the target model’s class prompts and scored by whether it selects the correct class. In the same aligned space, nearest-neighbor image classification using mapped source embeddings matches the target model’s performance, indicating that $Q$ preserves semantic details and task-relevant geometry. Moreover, this transfer is data-efficient, requiring only about $\sim 30\%$ data to learn $\mathcal Q$ reliably. Finally, $Q$ learned on one dataset transfers to others without re-fitting, and is consistent under composition, consistent with a global geometric relationship rather than instance- or class-specific tuning.\looseness=-1

\noindent We complement these findings by theoretically characterizing when this coupling is guaranteed. At the population level, we derive the optimal multimodal contrastive critic and show that, on a fixed target domain, agreement of cross-modal similarity kernel i.e $\langle f, g \rangle = \langle \tilde f, \tilde g \rangle$ on a small set of anchor pairs, forces a shared orthogonal map across modalities: the map that aligns images simultaneously determines the induced alignment of text. We further move beyond the exact regime, proving stability bounds that quantify how approximate cross-modal kernel alignment $\langle f, g \rangle \approx \langle \tilde f, \tilde g \rangle$ translates into reliable alignment.

\noindent To summarize, the key contributions of our work are: 
\begin{itemize}
\item We show that independently trained multimodal contrastive models can be closely approximated by a \emph{single orthogonal} map. Additionally, this map is shared across modalities, i.e., estimating the map from images alone aligns text, and vice versa.\looseness=-1
\item Theoretically, we prove that matching multimodal kernels on a small anchor set across two distinct models forces a shared orthogonal alignment across modalities and derive stability bounds in the approximate regime. \looseness=-1
\item  We validate these claims across five benchmarks and multiple model pairs, with extensive ablations showing that this map transfers across datasets without re-fitting and remains consistent under composition, yielding the most reliable cross-model, cross-modal transfer.\looseness=-1
\end{itemize}

\section{Related Work}
\emph{For an extended discussion of related work, see~\Cref{sec:related_work_appendix}.\looseness=-1}

\noindent \textbf{Representational Convergence and Functional Interoperability.} A central question in deep learning is whether independently trained models converge to identical representations. The Platonic Representation Hypothesis~\citep{huh2024prh} argues that large models across different modalities are converging
towards the same representations, given the vast amounts of training data that are used
by these models. Standard embedding similarity tools such as CKA \citep{kornblith2019cka} or SVCCA \citep{raghu2017svcca} measure this convergence up to broad equivalence classes (e.g., invertible linear maps), but do not provide an explicit coordinate mapping. A stronger operational test is \emph{model stitching}, which connects representations via simple learnable transformations to enable zero-shot interchangeability \citep{lenc2015equivariance,bansal2021stitching,merullo2022linearly}. However, these approaches are limited to coarse transfer metrics (e.g., image captioning), whereas we measure a stronger notion of alignment---tight pointwise agreement via cosine similarity---while simultaneously verifying that semantic structure is preserved through retrieval performance. Prior alignment results largely operate on \emph{unimodal} marginals, including word embeddings~\citep{mikolov2013exploiting,dev2021closed} and vision features~\citep{maystre2025embeddingmodelsmeet,merullo2022linearly}. Aligning marginals, however, does not in general identify the \emph{joint} distribution: multiple distinct joint geometries may be consistent with the same unimodal alignments. We therefore study a strictly stronger question---whether the \emph{joint} image-text geometry of two multimodal models is identifiable up to a \emph{single, rigid} isometry \emph{shared across modalities}, rather than allowing separate, unconstrained maps.\looseness=-1

\noindent \textbf{The Modality Gap and Geometry of Contrastive Representations.} Our study is grounded in the geometry of contrastive vision-language models \citep{radford2021clip,jia2021align}. Empirically, these models exhibit a \emph{modality gap}, where image and text embeddings cluster in distinct cones \citep{liang2022modalitygap,udandarao2022understanding,shi2023towards}. This separation renders alignment non-trivial, as naive mappings can easily collapse the gap or distort the intra-modal structure. Theoretically, prior work has analyzed conditions under which contrastive objectives identify the underlying latent factors up to linear or affine transformations \citep{zimmermann2021invert, roeder21a}. While general linear maps include rotations, they also permit shear and anisotropic scaling, which are poorly constrained and undesirable for preserving semantic structure. In contrast, our setting restricts attention to \emph{isometries}, which preserve angles, norms, and neighborhood relations. Even though images and text remain separated within each model, we prove that multimodal kernel agreement on a \emph{small} anchor set suffices to recover a \emph{shared} isometry $Q$ that aligns the two models across \emph{both} modalities. As a result, instead of retraining a model or learning complex transfer operators, one can anchor a new model to a reference model using a modest number of examples and obtain transfer to another modality \emph{for free}. This yields a substantially more economical alternative to retraining, general linear transfer, or optimal-transport-based adaptations.\looseness=-1

\section{Problem Formulation} \label{sec:problem_formulation}
We consider the standard dual-encoder framework where data consists of co-occurring pairs $(x,y)$ (e.g., images and text). A contrastive model consists of two encoders, $f: \mathcal{X} \to \mathbb{S}^{d-1}$ and $g: \mathcal{Y} \to \mathbb{S}^{d-1}$, which map inputs to the unit hypersphere in $\mathbb{R}^d$. The training objective maximizes the cosine similarity $\langle f(x), g(y) \rangle$ for semantically matched pairs while minimizing it for mismatched ones.

\noindent  Consider two contrastive models, trained in complete isolation on different datasets, with different architectures, initializations, and modeling choices: a \emph{source} model $\mathcal{M} = (f, g)$ and a \emph{target} model $\tilde{\mathcal{M}} = (\tilde{f}, \tilde{g})$ mapping to dimensions $d$ and $\tilde d$ respectively (without loss of generality, assume $d \leq \tilde d$).
Due to optimization stochasticity and training differences, the embedding spaces of $\mathcal M$ and $\tilde{\mathcal M}$ are a priori incomparable. Even with identical data, the contrastive objective depends only on within-model dot products $\langle f(x), g(y)\rangle$, so jointly rotating both embeddings by any orthogonal matrix leaves the loss and all within-model similarities unchanged. Architectural mismatch, finite-sample noise, and optimization effects further amplify this ambiguity~\citep{saunshi22contrastive,robinson21short}; under distribution shift, the models may not even share the same population optimum. As a result, cross-model dot products and nearest-neighbor queries between $\mathcal M$ and $\tilde{\mathcal M}$ are ill-defined unless we first learn a map between their spaces.
We therefore seek a map $\psi:\R^d\to\R^{\tilde d}$ that transports embeddings from $\mathcal M$ into $\tilde{\mathcal M}$. The key question is whether a consistent geometric relationship exists between the two models, and how it depends on modality: must one learn separate maps for images and text, or does a single $\psi$ align both modalities? Concretely, if $\psi$ aligns images so that $\tilde f \approx \psi(f)$, does the same $\psi$ also align text, i.e., $\tilde g \approx \psi(g)$? Further, if such a map exists, what is its functional form: nonlinear, linear, or orthogonal?

\section{Modality Gap in Contrastive Representations}
\begin{wrapfigure}{r}{0.5\textwidth}
    \vspace{-6mm}
    \centering
    \includegraphics[width=1.\linewidth]{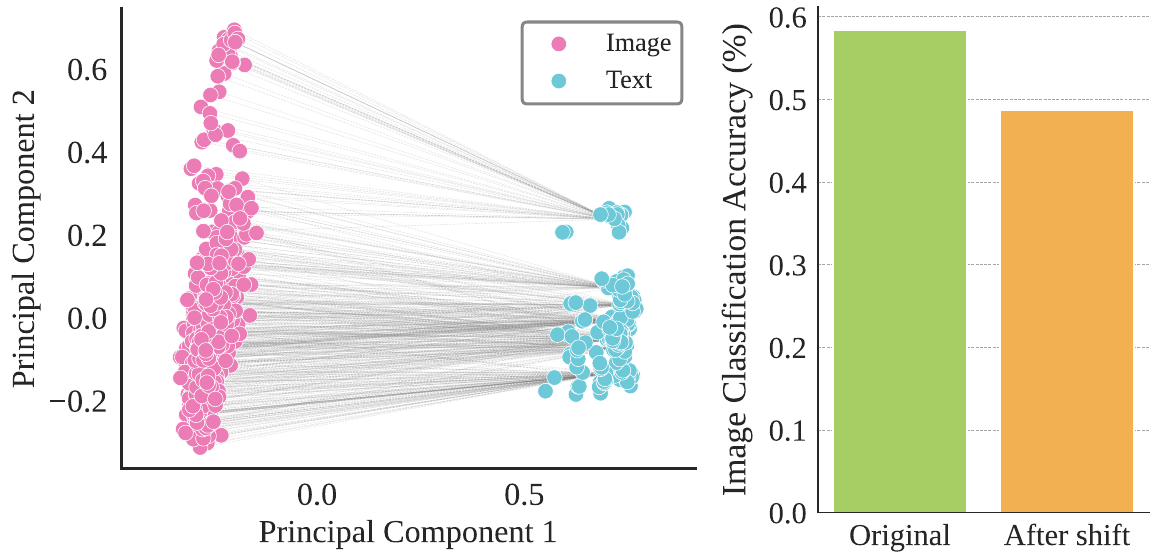}
    \caption{(left) PCA visualization of embeddings from \textsc{Clip} (OpenAI) on CIFAR 100, showing a large modality gap. (right) Applying a translation to close this gap distorts the embedding geometry and significantly degrades classification accuracy.\looseness=-1}
    \label{fig:mainpaper_modality_gap}
    \vspace{-8mm}
\end{wrapfigure}

This question is nontrivial due to the intrinsic geometry of contrastive representations. If matched image-text pairs collapsed to approximately the same point on the hypersphere (i.e., $f(x)\approx g(y)$), then aligning the image manifold would be equivalent to aligning the text manifold. In practice, however, contrastive models exhibit a pronounced \emph{modality gap} where image and text embeddings occupy largely disjoint regions of the sphere~\citep{liang2022modalitygap,shi2023towards,udandarao2022understanding}. Prior work also suggests that naïvely ``closing'' this gap can harm downstream performance and fairness~\citep{liang2022modalitygap}, as shown in~\Cref{fig:mainpaper_modality_gap}. 

\noindent Since the gap reflects systematic modality-specific structure rather than noise~\citep{schrodi2024two}, it is \emph{not} obvious that a map $\psi$ learned solely on images will extend correctly to text. In principle, infinitely many maps can agree on the image manifold yet behave arbitrarily on the text manifold. Our work specifically tests whether, despite this gap, the \emph{relative} geometry remains stable enough to permit transfer.

\begin{figure*}[!htb]
    \centering
    \includegraphics[width=1.\linewidth]{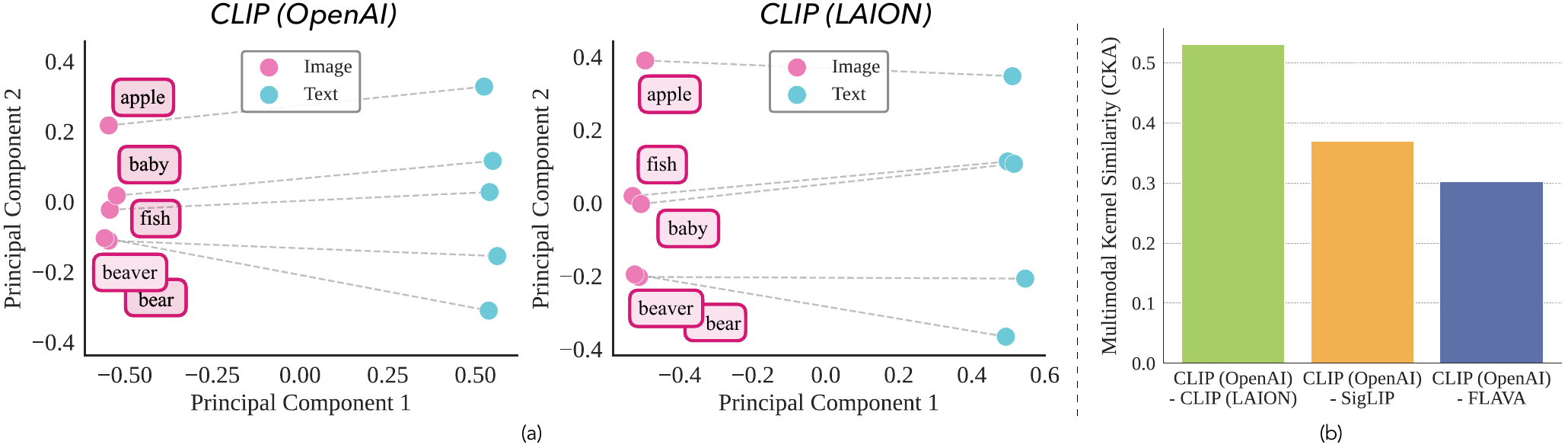}
    \caption{(a) Across CLIP variants, the multimodal kernel $\langle f(x), g(y)\rangle$ (relative angles between image and text embeddings) is strongly preserved (dashed lines), unlike the unimodal kernel $\langle f(x), \tilde f(x')\rangle$; (b) CKA on multimodal kernels shows high alignment across models.\looseness=-1}
    \label{fig:multimodal_kernels}
\end{figure*}

\noindent  Despite the modality gap and disjoint supports of the two models, we argue that the alignment problem is indeed solvable because \emph{relative} geometry is remarkably stable. As shown in \Cref{fig:multimodal_kernels}, while the absolute coordinates of the embedding cones shift arbitrarily between models, the angular arrangement of the texts with respect to the images remains consistent. Mathematically, this means that the \emph{multimodal kernels} are approximately preserved across models: $\langle f, g \rangle \approx \langle \tilde{f}, \tilde{g} \rangle$. This observation can be viewed as a multimodal analogue of the Platonic Representation Hypothesis~\citep{huh2024prh} that posits that models converge to similar \emph{unimodal} kernels ($\langle f, f \rangle \approx \langle \tilde{f}, \tilde{f} \rangle$). In the next section, we prove that this preservation of multimodal kernels is a sufficient condition to constrain the functional form of $\psi$, forcing it to be an isometry. \looseness=-1

\section{Theoretical Perspectives}\label{sec:theory}

In this section, we formalize the above intuition by characterizing population-level optima and showing that contrastive models, even when trained on different distributions, can recover the same multimodal kernels up to a constant~(\Cref{sec:infonce_critic}, as shown in~\Cref{fig:multimodal_kernels}). We then show that this agreement on a minimal anchor set constrains the alignment map to be linear (\Cref{sec:main_linear_alignment}), which, under the hypersphere constraint, collapses to an isometry (\Cref{sec:main_rotation_alignment}). Next, we relax these conditions, proving that the recovered isometry preserves zero-shot retrieval even when pointwise alignment is imperfect (\Cref{sec:main_theory_class_retrieval}). Finally, we extend these guarantees to the approximate setting, showing stability under approximate kernel matching (\Cref{sec:approx_regime}). For a schematic overview of our theoretical analysis, refer to~\Cref{fig:theory_schematic} in the appendix.\looseness=-1

\subsection{Contrastive Representation Learning}
\label{sec:contrastive}
Building on the dual-encoder framework defined in~\Cref{sec:problem_formulation}, we assume access to data characterized by co-occurring pairs $(x,y)\in\mathcal X\times\mathcal Y$ (e.g., images and text) drawn from a joint distribution $p_{XY}$ with marginals $p_X$ and $p_Y$. Here, the contrastive objective aims to learn functions, $f: \mathcal{X} \to \Sph$ and $g: \mathcal{Y} \to \Sph$, which map inputs to a shared unit hypersphere $\Sph \subset \R^d$ where similarity of a pair $(x, y)$ is measured by the score function
\begin{equation}
\label{eq:score}
s(x,y) \;:=\; \frac{1}{\tau}\,\langle f(x),g(y)\rangle + b,
\end{equation}
where $\tau > 0$ is a temperature parameter and $b \in \R$ is a scalar bias. This alignment is achieved by optimizing the symmetric InfoNCE objective~\citep{oord2018cpc}. Specifically, for a batch of $N$ pairs $\{(x_i, y_i)\}_{i=1}^N$ drawn i.i.d. from $P_{XY}$, the loss $\mathcal{L}_N^{\mathrm{}}$ minimizes the cross-entropy of identifying the correct positive pair relative to $N-1$ negatives in both directions:
\begin{equation}
\label{eq:clip_loss}
\begin{aligned}
&\mathcal L_N(s)
:= \mathcal L_N^{(X)}(s) + \mathcal L_N^{(Y)}(s)
= \E\!\left[
-\sum_{i=1}^N \log \frac{e^{s(x_i,y_i)}}{\sum_{j=1}^N e^{s(x_i,y_j)}}
-\sum_{i=1}^N \log \frac{e^{s(x_i,y_i)}}{\sum_{j=1}^N e^{s(x_j,y_i)}}
\right].
\end{aligned}
\end{equation}

\subsection{Optimal Critic for Family of Contrastive Learners}\label{sec:infonce_critic}
We first characterize the score function $s(x,y)$ that minimizes the population InfoNCE loss. Let $r(x,y) \coloneqq \frac{p_{XY}(x,y)}{p_X(x)p_Y(y)}$ denote the pointwise density ratio, and $\log r(x,y)$ be the Pointwise Mutual Information (PMI). For a fixed training distribution
$P_{XY}$, any global minimizer of the symmetric objective $\mathcal L_N$ induces an optimal score
$s^\ast(x,y)=\log r(x,y)+C$, i.e.\ the pointwise mutual information (PMI) up to a global constant. We note that this has been partially studied in prior work~\citep{huh2024prh,oord2018cpc}, with a detailed derivation provided in~\Cref{sec:unimodal_critic} and~\Cref{sec:multimodal_critic}. Consequently, any two globally optimal models trained on the \emph{same} distribution---or on distributions related to $P_{XY}$ by a bijective reparameterization---induce kernels that differ only by a constant (\Cref{cor:unique_kernels}). 
Here, we go beyond the fixed-distribution setting and prove that this invariance persists even for models trained on \emph{distinct} internet-scale corpora.\looseness=-1

\noindent \textbf{Platonic distribution and Dataset Curation.}
Let $P^\star_{XY}$ denote an underlying ``reality'' distribution with density $p^\star_{XY}$ and marginals $p^\star_X,p^\star_Y$ (satisfying positivity on the domain of interest).
We model each training corpus as a \emph{curation} of $P^\star_{XY}$ i.e. for dataset $a\in\{1,2\}$, there exist
measurable weights $u_a:\mathcal X\to(0,\infty)$ and $v_a:\mathcal Y\to(0,\infty)$ such that
\begin{equation}
\label{eq:curation-main}
  p^{(a)}_{XY}(x,y)\;\propto\;u_a(x)\,v_a(y)\,p^\star_{XY}(x,y).
\end{equation}
Here $u_a$ and $v_a$ capture modality-specific curation (e.g., \ image quality and text language/safety etc.) that can substantially change the marginals while preserving the underlying semantic information. We further assume that curation acts \emph{independently across modalities} in the sense that the expected text acceptance
does not depend on the image it is paired with, and vice versa i.e.
\begin{equation}
\label{eq:no-cross-modal-bias-main}
\begin{aligned}
  \E^\star[v_a(Y)\mid X=x] &= \E^\star[v_a(Y)] \quad \text{a.e.\ in } x, \\
  \E^\star[u_a(X)\mid Y=y] &= \E^\star[u_a(X)] \quad \text{a.e.\ in } y.
\end{aligned}
\end{equation}
where $\E^\star$ corresponds to expectation under $p^\star$ and a.e. stands for almost everywhere, causing these expectation terms to appear in the PMI only as additive constants.\looseness=-1
\begin{theorem}
\label{thm:main-pmi-constant-across-datasets}
Under~\Cref{eq:no-cross-modal-bias-main}, if $s^\ast_1$ and $s^\ast_2$ are Bayes-optimal scores for the contrastive models trained on two distinct distributions defined in~\Cref{eq:curation-main}, then there exists a constant $\Delta'$ such that \looseness=-1
\begin{equation}
\label{eq:scores-constant-main}
  s^\ast_1(x,y)=s^\ast_2(x,y)+\Delta'
  \;\; \text{for $P^\star_X\otimes P^\star_Y$-a.e.\ $(x,y)$}.
\end{equation}
\end{theorem}

\noindent \Cref{thm:main-pmi-constant-across-datasets} establishes that even when trained on different distributions, independently trained contrastive learners can converge to optimal similarity scores that agree up to an additive constant. We note that our result holds for several widely used contrastive objectives such as softmax InfoNCE~\citep{oord2018cpc} and pairwise sigmoid objectives in SigLIP~\citep{zhai2023siglip}. Since contrastive models approximate this target using dot products, any two models converging to the same score must implicitly align their kernels: $\inner {f(x)}{g(y)} \approx \inner{\tilde f(x)}{\tilde g(y)}.$\footnote{Exact equality implies the target PMI respects the contrastive parameterization constraints. Since CLIP scores are bounded and low-rank, they act as a low-rank approximation of the generally unbounded, high-rank population target.} For detailed proofs and discussion, refer to~\Cref{sec:theory_pmi_different_data}.

\subsection{Linear Alignment of Contrastive Models}\label{sec:main_linear_alignment}

\par In the previous section, we showed that independently trained contrastive models can induce multimodal kernels that agree up to an additive constant, despite differences in data and modeling choices. We now assume this kernel agreement holds on domains of interest $\Omega_X\subseteq\mathcal X$ and $\Omega_Y\subseteq\mathcal Y$ (e.g., a downstream image-text dataset) and analyze its geometric consequences. Our first main result shows that matching kernels on a small set of \emph{anchor} points suffices to determine a linear map relating the two embedding spaces.\looseness=-1

\noindent  Let the contrastive model pairs $(f,g)$ and $(\tilde f, \tilde g)$ map inputs to $\mathbb S^{d-1} \subset \R^d$ and $ \mathbb S^{\tilde d-1} \subset \R^{\tilde d}$ respectively, where $d \le \tilde d$, without loss of generality. We fix a set of \emph{image anchors} $\{\bar x_j\}_{j=1}^d \subset \Omega_X$ and \emph{text anchors} $\{\bar y_i\}_{i=1}^{\tilde d} \subset \Omega_Y$ and collect their embeddings into the following matrices: $G \coloneqq [g(\bar y_1)\ \cdots\ g(\bar y_{\tilde d})]\in\R^{d\times \tilde d}, \;
  \tilde G \coloneqq [\tilde g(\bar y_1)\ \cdots\ \tilde g(\bar y_{\tilde d})]\in\R^{\tilde d\times \tilde d}, \;
  F \coloneqq [f(\bar x_1)\ \cdots\ f(\bar x_d)]\in\R^{d\times d}.$

\begin{restatable}{assumption}{AnchorKernelEquality}
\label{ass:kernel_anchors}
The multimodal kernels coincide on the set of anchors:
\begin{align*}\inner{f(x)}{g(\bar y_i)} &= \inner{\tilde f(x)}{\tilde g(\bar y_i)}\;\; \forall x\in\Omega_X,\ \forall i\in\{1,\dots,\tilde d\},\\ 
\inner{f(\bar x_j)}{g(y)} &= \inner{\tilde f(\bar x_j)}{\tilde g(y)} \;\;  \forall y\in\Omega_Y,\ \forall j\in\{1,\dots,d\}.
\end{align*}
\end{restatable}

\begin{restatable}{theorem}{linearthm}
\label{thm:lin}
\textit{(Linear Identifiability, proof in~\Cref{sec:exact_linear}).} Under~\Cref{ass:kernel_anchors}, suppose $\tilde G$ and $F$ are invertible. Then there exists a linear map $A$ such that $ \tilde f(x) = A f(x)\; \forall x\in\Omega_X.$ Further, if $A$ has full column rank, then for every $y\in\Omega_Y$, $\mathrm{Proj}_{\operatorname{Im}(A)}\,\tilde g(y)
=
A(A^\top A)^{-1} g(y).$ If $\tilde d=d$, then $\tilde g(y)=A^{-\top}g(y)$.\looseness=-1
\end{restatable}

\subsection{Isometric Alignment of Contrastive Models}\label{sec:main_rotation_alignment}

\Cref{thm:lin} shows that kernel matching identifies the representation up to a linear map $A$. But, contrastive encoders normalize embeddings to the unit hypersphere $\mathbb{S}^{d-1}$, forcing $\|\tilde f(x)\|_2 = \|A f(x)\|_2 = 1$ everywhere. This forces $A$ to be an isometry ($A^\top A = I_d$) only if the data is sufficiently diverse to probe the matrix in all directions. We formalize this diversity via the following condition.

\begin{restatable}{definition}{SymSpanning} 
\label{def:sym-spanning}
\textit{($\Sym(d)$-spanning)}
A set of vectors $S \subset \R^d$ is \emph{$\Sym(d)$-spanning} if the rank-one matrices $\{xx^\top : x \in S\}$ span the space of symmetric matrices $\Sym(d)$.
Equivalently: if $M \in \Sym(d)$ and $x^\top M x = 0$ for all $x \in S$, then $M=0$. This equivalence follows from the identity $x^\top M x = \inner{xx^\top}{M}$.
\end{restatable}

\begin{restatable}{theorem}{orthogonalthm}
\label{thm:orth}
\textit{(Orthogonal Identifiability, proof in~\Cref{sec:rotation_alignment}).}
Assume the conditions of~\Cref{thm:lin} hold. If the set of image embeddings $\{f(x) : x \in \Omega_X\}$ contains a $\Sym(d)$-spanning subset, then the linear map $A$ has orthonormal columns ($A^\top A = I_d$). Consequently, $\tilde f(x) = Q f(x) \quad \forall x \in \Omega_X,$ where $Q \coloneqq A$ satisfies $Q^\top Q = I_d$.
Furthermore, for the other modality:$$  \mathrm{Proj}_{\operatorname{Im}(Q)}\,\tilde g(y) = Q g(y) \quad \forall y \in \Omega_Y.$$
If $\tilde d=d$, $Q$ is orthogonal i.e. $Q \in O(d)$ and $\tilde g(y) = Q g(y)$.
\end{restatable}
\noindent In~\Cref{sec:exact_low_dim}, we extend~\Cref{thm:orth} to the case where $\mathrm{span}\{f(x): x\in\Omega_X\}$ lies in low-dimensions.\looseness=-1

\subsection{Isometric Alignment Up To Classification Boundaries of Independent Contrastive Models}\label{sec:main_theory_class_retrieval}

While the preceding theory establishes conditions for exact geometric alignment, in practice, one often cares about alignment up to concepts or classification. We now analyze this regime, where we seek to distinguish a finite family of prompts $\mathcal Y_{\mathrm{cls}}=\{y_c\}_{c=1}^K\subseteq \Omega_Y$. For any class prompt $y \in \mathcal Y_{\mathrm{cls}}$, decompose the embedding into an identifiable signal $u(y)$ and an unidentifiable residual $w(y)$:
$$  g(y) = u(y) + w(y), \;\; u(y) \coloneqq \mathrm{Proj}_{U_X} g(y),\ \ w(y) \in U_X^\perp.$$

\noindent Assuming the images are isometrically aligned ($\tilde f(x) = Q f(x)$), we define analogously for the target model: $
 \tilde g(y) = Q u(y) + \tilde w(y),\; \tilde w(y) \in (Q U_X)^\perp.$

\begin{definition}\label{def:margin-noise}Define the signal margin $\gamma$ as the class separability within the shared image subspace:$$  \gamma
\coloneqq
\min_{c}
\Big(
\|u(y_c)\|_2^2
-
\max_{k\neq c}\inner{u(y_c)}{u(y_{k})}
\Big).$$
Define the cross-model noise $\eta$ as the worst-case interaction of the unidentifiable residuals:$$  \eta
\coloneqq
\max_{c,k}
\big|\inner{Q w(y_c)}{\tilde w(y_{k})}\big|.$$\end{definition}

\begin{proposition}[Orthogonal Identifiability Up To Class Retrieval, proof in~\Cref{sec:exact_classification}]\label{prop:class-retrieval}
If the signal dominates the noise ($\gamma > 2\eta$), then the aligned prompt $Qg(y_c)$ correctly retrieves its counterpart $\tilde g(y_c)$ in the target model i.e. $\arg\max_{k} \inner{Q g(y_c)}{\tilde g(y_k)} = c
\; \forall c \in \{1,\dots,K\}.$
\end{proposition}

\noindent Proposition~\ref{prop:class-retrieval} explains our empirical results in~\Cref{sec:all_exp_results}, showing that if the semantic signal ($\gamma$) is robust enough to withstand interference from unidentifiable components ($\eta$), even with imperfect pointwise alignment, we can have perfect retrieval. Throughout the preceding analysis, we assumed exact kernel equality. In \Cref{sec:approx_regime}, we further relax~\Cref{ass:kernel_anchors} to an $\epsilon$-approximate bound: $|\inner{f(x)}{g(y)} - \inner{\tilde f(x)}{\tilde g(y)}| \le \epsilon$ and prove that the map $Q$ becomes an approximate isometry.\looseness=-1

\section{The Procrustes Algorithm}\label{sec:procrustes}
Guided by the theoretical guarantees in~\Cref{sec:theory}, we translate the alignment problem into an optimization procedure. We align the source and target manifolds of one modality (say images) using a set of $N$ unlabelled anchor images $\{x_i\}_{i=1}^N$. Let $X \in \mathbb{R}^{d \times N}$ and $\tilde{X} \in \mathbb{R}^{\tilde{d} \times N}$ be the data matrices containing the centered, normalized embeddings $f(x_i)$ and $\tilde{f}(x_i)$ as columns. We solve for the optimal isometry $\hat{Q}$ by minimizing the transport cost subject to an orthogonality constraint: $\hat{Q} = \text{argmin}_{Q^\top Q = I} \| \tilde{X} - Q X \|_F^2$. This is the classic \emph{Orthogonal Procrustes Problem}, which has a closed-form solution via the Singular Value Decomposition (SVD) of the cross-covariance matrix $M = \tilde{X} X^\top$. Let $U \Sigma V^\top = \text{SVD}(M)$, then $\hat{Q} = U I_{\tilde{d} \times d} V^\top.$ Here, $I_{\tilde{d} \times d}$ is a rectangular identity matrix. This formulation naturally handles $d < \tilde{d}$, where $\hat{Q}$ becomes a semi-orthogonal i.e. $Q^\top Q = I_d$. For additional details and pseudocode, refer to~\Cref{sec:app_learning}\looseness=-1

\section{Experimental Results}\label{sec:all_exp_results}
In this section, we empirically evaluate our central claim across various benchmarks and configurations, leading
to three main takeaways: (i) a single orthogonal map $\mathcal Q$ accurately captures the relationship between independently trained contrastive models and, crucially, the same $\mathcal Q$ applies to both image and text representations (\Cref{sec:mainpaper_main_results_section}); (ii) $\mathcal Q$ is data-efficient and only a few examples suffice to estimate it; it generalizes to unseen classes and even to new downstream datasets without re-fitting (\Cref{sec:mainpaper_seen_unseen,sec:mainpaper_cross_dataset}); and (iii) although more expressive linear or non-linear maps can increase pointwise similarity on the fitted domain, they fail to transfer to the second modality and distort task-relevant geometry, degrading downstream retrieval; enforcing orthogonality in contrast yields the most reliable transfer across models and modalities  (\Cref{sec:mainpaperalternative_maps}). Finally, we show that the learned orthogonal map approximately commutes with cross-modal retrieval across models i.e. direct image alignment and text-mediated alignment recover the same semantic neighborhoods across models (\Cref{sec:mainpaper_knn}).\looseness=-1

\subsection{Training Protocol}\label{sec:mainpaper_train_protocol}
We report each metric both before alignment and after applying the learned orthogonal map $\mathcal Q$, and average results over three random seeds.
We describe the evaluated model pairs, datasets, and metrics below and in detail in~\Cref{sec:app_training}.\looseness=-1

\noindent \textbf{Models.}
We evaluate three independently trained vision-language pairs:
(i) \textsc{CLIP} ViT-B/32 (\texttt{OpenAI}) and \textsc{CLIP} ViT-B/32 trained on \texttt{LAION-400M};
(ii) \textsc{CLIP} ViT-L/14 (\texttt{OpenAI}) and \textsc{SigLIP};
and (iii) \textsc{CLIP} ViT-L/14 (\texttt{OpenAI}) and \textsc{FLAVA} \citep{radford2021clip,schuhmann2021laion400m,zhai2023siglip,singh2022flava}.
We $\ell_2$-normalize all embeddings such that dot products equal cosine similarity.

\noindent \textbf{Datasets.}
We report results on Oxford-IIIT Pets~\citep{parkhi2012cats}, CIFAR-100~\citep{krizhevsky2009learning}, Caltech-101~\citep{fei2004learning}, STL10~\citep{coates2011analysis} and DTD~\citep{cimpoi2014describing}.
For more information about creating text prompts, refer to~\Cref{sec:appendix_benchmarks}. We report results only for Oxford Pets in the main paper and defer results on the remaining datasets to~\Cref{sec:app_additional_experiments}.

\noindent \textbf{Training.} We learn the alignment across models using the standard orthogonal Procrustes solution described in~\Cref{sec:procrustes}. In practice, the two models can differ by a constant offset in embedding space due to finite-sample effects and dataset mismatch. We therefore fit and apply $\mathcal Q$ on centered embeddings, and then re-add the target mean i.e. $z \;\mapsto\; Q\big(z - \mu^{(\cdot)}\big) + \tilde{\mu}^{(\cdot)},$ where $\mu^{(\cdot)}$ and $\tilde{\mu}^{(\cdot)}$ are modality-specific training means of the source and target embeddings i.e. $(\mu^{\text{img}}, \tilde{\mu}^{\text{img}})$ when aligning image embeddings and $(\mu^{\text{text}}, \tilde{\mu}^{\text{text}})$ when aligning text embeddings. Centering isolates the rotational relationship by removing this offset while preserving the orthogonal correspondence in the centered space\looseness=-1\footnote{We find that mean-centering has a negligible effect on class-level retrieval and decision geometry; it primarily changes pointwise cosine agreement (see~\Cref{sec:app_procrustes_variants}). Thus, a pure orthogonal map on raw embeddings suffices for semantic alignment and preserves decision geometry.}. Theoretically, this mean offset vanishes when the two models are exactly related by an orthogonal map, as discussed in~\Cref{remark:meancentering}.

\noindent \textbf{Performance Metrics.}
We report five evaluation metrics: (1) Paired-instance cosine similarity, measured between aligned and target embeddings for either images or texts; (2) top-1 retrieval across models, evaluated for both image–image and text–text retrieval by nearest-neighbor matching at the \emph{class} level. Zero-shot classification for (3) aligned images against target text (aligned image–text), (4) target images against aligned text (image–aligned text), and (5) both images and text aligned (aligned image–aligned text). All metrics are computed using cosine similarity; full metric definitions are deferred to~\Cref{sec:app_metrics}.\looseness=-1

\subsection{Independently Trained Contrastive Models Differ by an Orthogonal Map Common To Both Modalities\looseness=-1}\label{sec:mainpaper_main_results_section}

\begin{figure}[!htb]
    \centering
    \includegraphics[width=1.\linewidth]{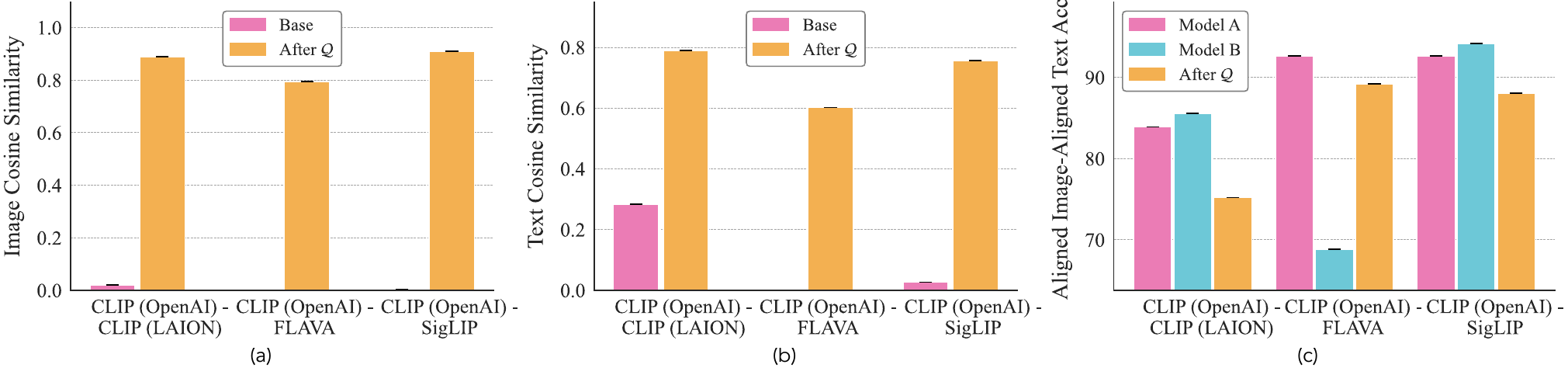}
    \caption{\textit{Inter-model alignment on Oxford-Pets before and after fitting a single orthogonal map $\mathcal Q$ from paired images.} (a) Image-image cosine similarity; (b) Text-text cosine similarity; (c) Aligned Image and Aligned text class retrieval accuracy. Here, Model A and Model B denote each model’s within-model image-to-text baseline. $\mathcal Q$ aligns images across models with a single orthogonal map, and the same $\mathcal Q$ learned only from image embeddings transfers to text, boosting text-text cosine from near-chance to near-oracle, all while preserving strong image classification accuracy.\looseness=-1}
    \label{fig:mainpaper_mainresults_oxford}
\end{figure}

\noindent \Cref{fig:mainpaper_mainresults_oxford} summarizes our findings on Oxford-Pets across three independently trained pairs. 

\noindent \textbf{An Orthogonal Map Aligns Different Models.} First, from~\Cref{fig:mainpaper_mainresults_oxford}(a), we find that a \emph{single orthogonal map} almost perfectly aligns \emph{image} embeddings across distinct multimodal contrastive models, improving the image-image cosine similarity from near zero to $\sim 0.8 - 0.9$. We observe analogous findings for \emph{text} embeddings (see \Cref{fig:app_text_trained_oxford})(c), indicating that independently trained contrastive models are related by an approximately orthogonal map. 

\noindent\textbf{This Map Transfers Across Modalities.} Second, and more importantly, this map is \emph{modality-invariant}: \Cref{fig:mainpaper_mainresults_oxford}(b) shows that the same orthogonal map $\mathcal Q$ fit using paired images sharply improves \emph{text} alignment, significantly improving text-text cosine similarity across model pairs.
Finally, as shown in \Cref{fig:mainpaper_mainresults_oxford}(c), aligned-image-to-aligned-text retrieval remains high, showing that $\mathcal Q$ preserves task-relevant geometry while eliminating any need to compute the second model’s text embeddings. Additionally, in some cases, $\mathcal Q$ effectively transfers model A’s stronger decision geometry into model B’s space, matching or even exceeding model B’s native performance.
Results across additional datasets and metrics appear in~\Cref{sec:app_all_metrics}.\looseness=-1

\noindent  We extend these findings to mismatched embedding dimensions in~\Cref{sec:app_different_dims} and also show the reverse direction i.e. fitting $\mathcal Q$ on text to align images in~\Cref{sec:app_text_trained}. Finally, we ablate mean-centering and find that it has negligible effect on class-level retrieval and decision geometry, and mainly affects pointwise cosine agreement. Thus, a pure orthogonal map on raw embeddings suffices for semantic alignment and preserves decision geometry (see \Cref{sec:app_procrustes_variants}).\looseness=-1

\subsection{Only a Few Data Points Are Needed to Learn the Orthogonal Map}\label{sec:mainpaper_seen_unseen}

\begin{figure*}[!htb]
\centering
\begin{minipage}{0.42\textwidth}
\centering
\includegraphics[width=1.\textwidth]{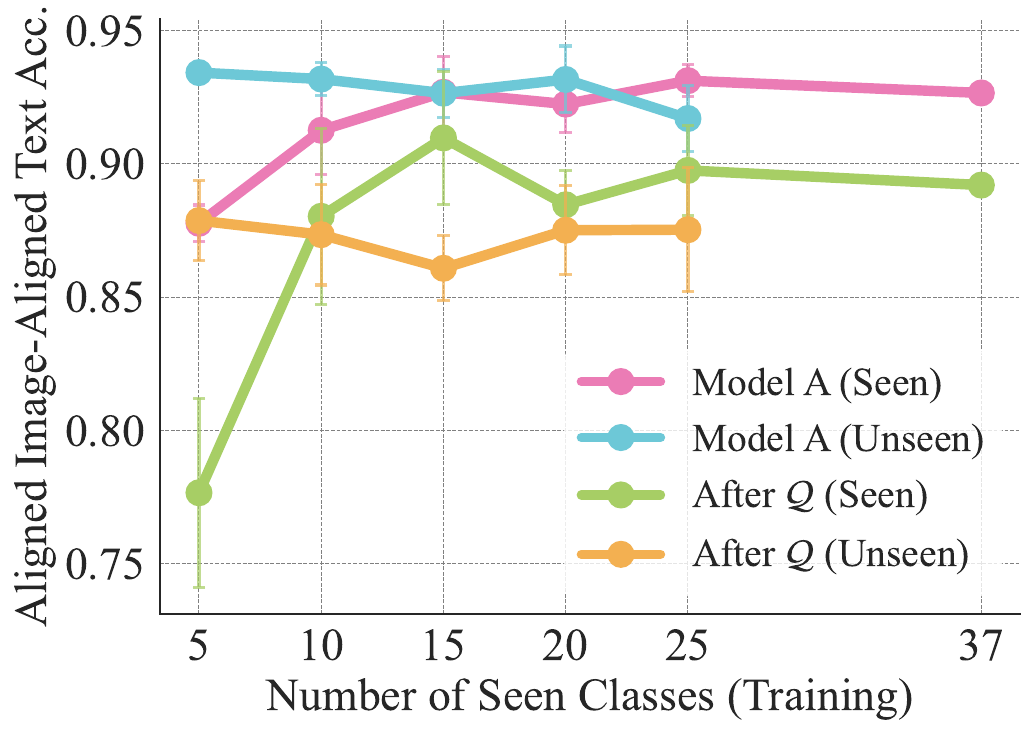}
\caption{Generalization of orthogonal alignment under limited supervision (Oxford Pets; \textsc{CLIP} ({OpenAI}) and FLAVA) as measured in aligned-image-to-aligned text accuracy. $\mathcal Q$ learned from a few classes transfers and generalizes to unseen classes.\looseness=-1}
\label{fig:mainpaper_split_oxford}
\end{minipage}
\hfill
\begin{minipage}{0.54\textwidth}
\centering
\begin{subfigure}{0.49\linewidth}
    \centering
    \includegraphics[width=\textwidth]{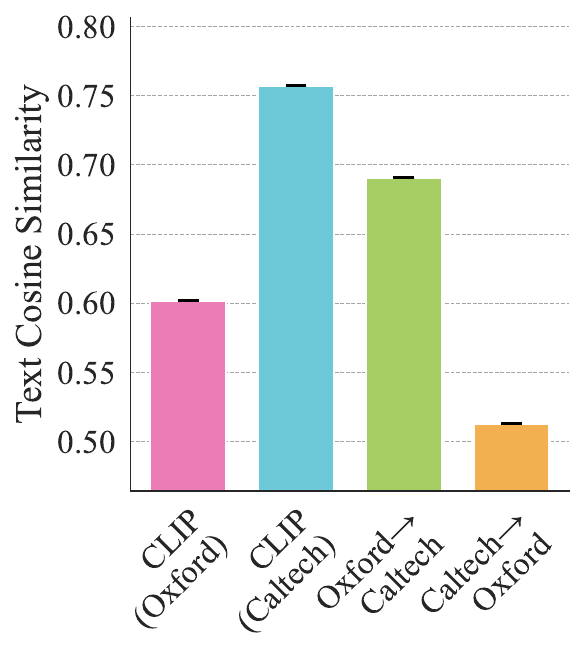}
    \end{subfigure}
    \begin{subfigure}{0.49\linewidth}
    \centering
    \includegraphics[width=\textwidth]{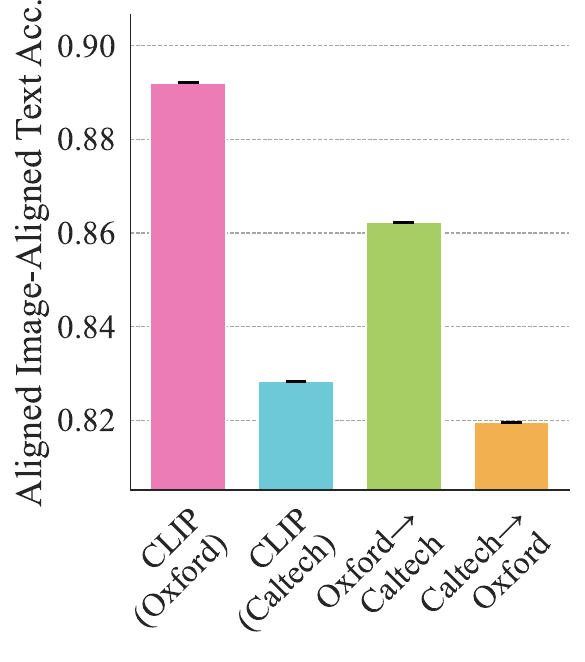}
    \end{subfigure}
    \caption{Generalization of the orthogonal map across data distributions across \textsc{CLIP} (OpenAI) and FLAVA. (left) text-text cosine similarity; (right) image classification accuracy post-transformation using aligned images and aligned texts. A map $\mathcal Q$ fit on Oxford-Pets transfers to Caltech-101 (and vice versa).\looseness=-1}
    \label{fig:mainpaper_cross_model_oxford_caltech}
\end{minipage}
\end{figure*}

\noindent \Cref{thm:orth} states that 
if the multimodal kernels induced by two contrastive models agree on a sufficiently rich but \emph{small finite} set of anchors, a single global orthogonal map aligns their representations across both modalities. As shown in~\Cref{fig:mainpaper_split_oxford}, we empirically validate this on Oxford-Pets, fitting $\mathcal Q$ using paired images from only $N$ classes and evaluating transfer on the remaining unseen classes. Here, Model A and Model B denote each model’s within-model image-to-text baseline. Performance on both seen and unseen classes improves quickly with just a few anchor classes and essentially saturates once $N$ reaches a modest value (around 10-15 classes), after which additional anchors provide little benefit. Thus, practitioners can recover near-full cross-model transfer by fitting $\mathcal Q$ on a lightweight image-only calibration set, rather than curating large-scale cross-model supervision. For additional metrics and results across model pairs and datasets, refer to~\Cref{sec:app_seen_unseen}.\looseness=-1

\subsection{The Orthogonal Map Generalizes Broadly}\label{sec:mainpaper_cross_dataset}
The previous experiment shows that $\mathcal Q$ is identifiable from a few anchors and generalizes to unseen classes within the same dataset. We next ask a stronger version of this question: does the \emph{same} $\mathcal Q$ transfer to a completely new downstream distribution \emph{without} re-fitting? From~\Cref{fig:mainpaper_cross_model_oxford_caltech} (left), a map learned on Oxford-Pets substantially increases text-text cosine similarity on Caltech-101 (and vice versa). Correspondingly, on the right, aligned-image-to-aligned-text classification remains strong under transfer---often closely matching or even exceeding an in-domain fit---indicating that $\mathcal Q$  generalizes beyond the calibration dataset.\looseness=-1

\subsection{Evaluating Alternative Alignment Maps Than The Orthogonal Mapping}\label{sec:mainpaperalternative_maps}

\begin{figure}[!htb]
    \centering
    \includegraphics[width=1.\linewidth]{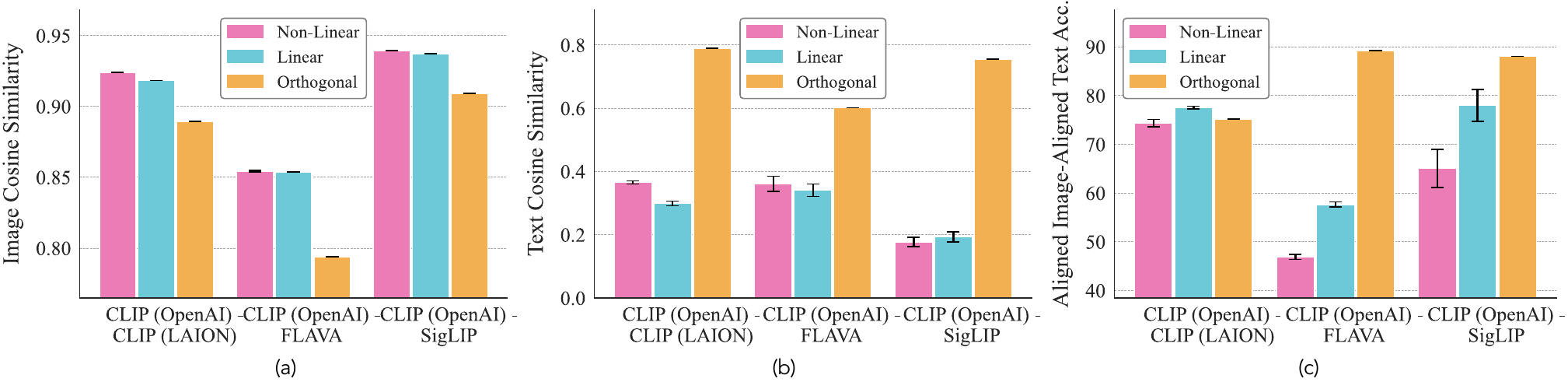}
    \caption{Comparison of alignment strategies on Oxford Pets across model pairs in terms of (a) image cosine similarity; (b) text cosine similarity, and (c) aligned-image-to-aligned-text accuracy. Unlike orthogonal maps, linear and non-linear maps distort task-relevant geometry and do not transfer to text, leading to poor downstream performance.\looseness=-1}
    \label{fig:mainpaper_compare_aligners}
\end{figure}

\noindent Here, we ablate the alignment design by comparing three maps of increasing expressiveness: (i) an orthogonal map $Q$, (ii) a linear map, and (iii) a non-linear MLP. As shown in~\Cref{fig:mainpaper_compare_aligners}, more expressive maps improve \emph{pointwise} image-image cosine similarity. However, these maps transfer poorly to the text modality and distort the image-text geometry. In contrast, the orthogonal map consistently performs best on both pointwise text cosine similarity and geometry-sensitive downstream metrics. Extended results for additional datasets are provided in \Cref{sec:app_compare_strategies}.\looseness=-1

\noindent \textbf{Additional Ablations.} In~\Cref{sec:app_compositionality}, we show that $\mathcal Q$ remains consistent under \emph{cycle} and \emph{composition}. When we hold the training data fixed and vary only the design choices, transfer is even stronger than under dataset shift (as shown in~\Cref{sec:app_openi_openai}). Finally, in~\Cref{sec:app_semantic} we show that $\mathcal Q$ preserves fine-grained attributes (pose, etc.), beyond coarse class-level semantics.\looseness=-1

\subsection{Qualitative Evidence of Commutativity of Image-Text Alignment Paths}\label{sec:mainpaper_knn}
\begin{figure}[!htb]
    \centering
    \includegraphics[width=.87\textwidth]{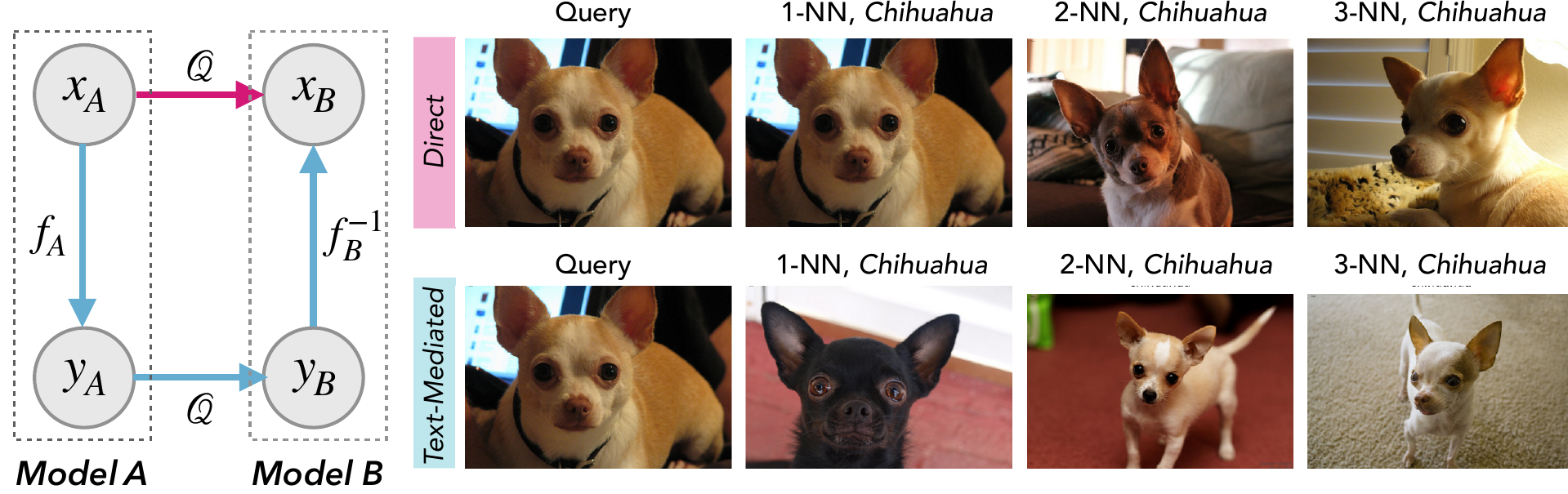} 
    \caption{\textit{Qualitative k-NN retrieval under two alignment paths for Oxford Pets.} We compare two routes from a source image $x$ into the target model’s image space after alignment by $\mathcal Q$: a direct image-alignment path (top row) and a text-mediated path (bottom row). Both routes yield highly consistent neighborhoods, showing that the two alignment paths approximately close as a commuting diagram.\looseness=-1}  
    \label{fig:mainpaper_knn_alignment_results_1}
\end{figure}
\noindent In this section, we test whether $\mathcal Q$ induces a consistent, modality-invariant geometry by comparing two retrieval routes from a source image $x$. \emph{Direct} (\Cref{fig:mainpaper_knn_alignment_results_1} (pink)): map the image embedding with $\mathcal Q$ and retrieve its top-$5$ nearest target images. \emph{Text-mediated} (\Cref{fig:mainpaper_knn_alignment_results_1} (blue)): retrieve the nearest source text for $x$, map it with $\mathcal Q$, retrieve the top-$1$ nearest target text, then retrieve the top-$5$ target images associated with that text.

\noindent \Cref{fig:mainpaper_knn_alignment_results_1} (and Appendix~\Cref{fig:knn_alignment_results_2}) shows that both routes recover essentially the same semantic neighborhood. In terms of the commuting diagram, transporting $x$ by $\mathcal Q$ and then applying target-space retrieval agrees with first retrieving through the source image-to-text operator, transporting via $\mathcal Q$, and then retrieving back to images. This indicates that $\mathcal Q$ approximately commutes with the cross-modal nearest-neighbor operators on this domain.

\section{Discussion and Conclusion}\label{sec:mainpaper_discussion}
\textbf{Conclusion.} In this work, we show a rigid form of geometric convergence in multimodal contrastive models: across independently trained systems (with different data and modeling choices), a single orthogonal map learned in one modality can approximately align both image and text representations, inducing a shared coordinate system. Moreover, estimating this map requires only a small anchor set from a single modality (image or text). Theoretically, we characterize conditions under which agreement of the multimodal similarity kernel forces such a shared isometry and establish guarantees even under approximate agreement.

\noindent \textbf{Discussion and Implications.} Our results have several practical and scientific implications. In large embedding systems, switching models typically triggers full re-embedding, often infeasible at modern scale (billions of vectors)~\citep{jayaram2019diskann, johnson2019billion} and costly in both time and compute~\citep{openai_pricing}. We show that a small anchor set can recover the orthogonal map that restores compatibility across models. Since it preserves inner products, it supports model upgrades without re-encoding while keeping the embedding geometry intact. Finally, models often specialize differently; one might have a stronger vision tower, while another has a stronger or multilingual text tower. Our approach lets practitioners swap and combine towers while preserving image-text geometry. For an extended discussion, refer to~\Cref{sec: app_discussion}. \looseness=-1 

\noindent \textbf{Limitations and Future Work.} Our evaluation focuses on classification-style semantics; we do not establish gains for fine-grained retrieval or dense ranking. Although an orthogonal map preserves angles, we do not test whether fine-grained attributes remain easily decodable after alignment; a natural next step is to train lightweight decoders on the aligned space. Finally, we study image–text contrastive encoders; extending to other modalities (e.g., audio, video) is an important direction.

\section{Acknowledgements}
S.G. acknowledges the support of the MathWorks Engineering Fellowship. P.I. acknowledges support from a Packard Fellowship, the MIT-IBM Watson AI Lab, and the ONR MURI grant N00014-22-1-2740. S.J. acknowledges the support of the NSF AI Institute TILOS (NSF CCF2112665) and the Alexander von Humboldt Foundation. V.G. acknowledges the support from Saab-WASP (grant 411025), Academy of Finland (grant 342077), and the Jane and Aatos Erkko Foundation (grant 7001703). We thank Kiril Bangachev for thorough and insightful discussions.

\newpage

\bibliographystyle{abbrvnat}
\bibliography{bib}

@string{icml = "International Conference on Machine Learning (ICML)"}

@string{neurips = "Advances in Neural Information Processing Systems (NeurIPS)"}

@string{iccv = "International Conference on Computer Vision (ICCV)"}

@string{cvpr = "IEEE Conference on Computer Vision and Pattern Recognition (CVPR)"}

@string{iclr = "International Conference on Learning Representations (ICLR)"}

@article{huh2024prh,
  title   = {The Platonic Representation Hypothesis},
  author  = {Huh, Minyoung and Cheung, Brian and Wang, Tongzhou and Isola, Phillip},
  journal = {arXiv preprint arXiv:2405.07987},
  year    = {2024},
  url     = {https://arxiv.org/abs/2405.07987}
}

@inproceedings{radford2021clip,
  title     = {Learning Transferable Visual Models From Natural Language Supervision},
  author    = {Radford, Alec and Kim, Jong Wook and Hallacy, Chris and Ramesh, Aditya and Goh, Gabriel and Agarwal, Sandhini and Sastry, Girish and Askell, Amanda and Mishkin, Pamela and Clark, Jack and Krueger, Gretchen and Sutskever, Ilya},
  booktitle = {Proceedings of the 38th International Conference on Machine Learning},
  series    = {Proceedings of Machine Learning Research},
  volume    = {139},
  pages     = {8748--8763},
  year      = {2021},
  publisher = {PMLR},
  url       = {https://proceedings.mlr.press/v139/radford21a.html}
}

@inproceedings{kornblith2019cka,
  title     = {Similarity of Neural Network Representations Revisited},
  author    = {Kornblith, Simon and Norouzi, Mohammad and Lee, Honglak and Hinton, Geoffrey},
  booktitle = {Proceedings of the 36th International Conference on Machine Learning},
  series    = {Proceedings of Machine Learning Research},
  volume    = {97},
  pages     = {3519--3529},
  year      = {2019},
  publisher = {PMLR},
  url       = {https://proceedings.mlr.press/v97/kornblith19a.html}
}

@article{merullo2022linearly,
  title={Linearly mapping from image to text space},
  author={Merullo, Jack and Castricato, Louis and Eickhoff, Carsten and Pavlick, Ellie},
  journal={arXiv preprint arXiv:2209.15162},
  year={2022}
}

@article{gupta2025better,
  title={Better together: Leveraging unpaired multimodal data for stronger unimodal models},
  author={Gupta, Sharut and Sundaram, Shobhita and Wang, Chenyu and Jegelka, Stefanie and Isola, Phillip},
  journal=iclr,
  year={2026}
}

@article{bangachev2025global,
  title={Global Minimizers of Sigmoid Contrastive Loss},
  author={Bangachev, Kiril and Bresler, Guy and Noman, Iliyas and Polyanskiy, Yury},
  journal={arXiv preprint arXiv:2509.18552},
  year={2025}
}

@inproceedings{shi2023towards,
  title={Towards understanding the modality gap in clip},
  author={Shi, Peiyang and Welle, Michael C and Bj{\"o}rkman, M{\aa}rten and Kragic, Danica},
  booktitle={ICLR 2023 workshop on multimodal representation learning: perks and pitfalls},
  year={2023}
}

@article{udandarao2022understanding,
  title={Understanding and fixing the modality gap in vision-language models},
  author={Udandarao, Vishaal},
  journal={Master's thesis, University of Cambridge},
  volume={32},
  year={2022}
}

@article{schrodi2024two,
  title={Two effects, one trigger: On the modality gap, object bias, and information imbalance in contrastive vision-language representation learning},
  author={Schrodi, Simon and Hoffmann, David T and Argus, Max and Fischer, Volker and Brox, Thomas},
  journal={arXiv preprint arXiv:2404.07983},
  year={2024}
}

@InProceedings{roeder21a,
  title = 	 {On Linear Identifiability of Learned Representations},
  author =       {Roeder, Geoffrey and Metz, Luke and Kingma, Durk},
  booktitle = 	 {Proceedings of the 38th International Conference on Machine Learning},
  pages = 	 {9030--9039},
  year = 	 {2021},
  editor = 	 {Meila, Marina and Zhang, Tong},
  volume = 	 {139},
  series = 	 {Proceedings of Machine Learning Research},
  month = 	 {18--24 Jul},
  publisher =    {PMLR},
  url = 	 {https://proceedings.mlr.press/v139/roeder21a.html}
}

@article{dev2021closed,
  title={Closed form word embedding alignment},
  author={Dev, Sunipa and Hassan, Safia and Phillips, Jeff M},
  journal={Knowledge and Information Systems},
  volume={63},
  number={3},
  pages={565--588},
  year={2021},
  publisher={Springer}
}

@article{gurnee2024universal,
  title={Universal neurons in gpt2 language models},
  author={Gurnee, Wes and Horsley, Theo and Guo, Zifan Carl and Kheirkhah, Tara Rezaei and Sun, Qinyi and Hathaway, Will and Nanda, Neel and Bertsimas, Dimitris},
  journal={arXiv preprint arXiv:2401.12181},
  year={2024}
}

@article{ziyin2025proof,
  title={Proof of a perfect platonic representation hypothesis},
  author={Ziyin, Liu and Chuang, Isaac},
  journal={arXiv preprint arXiv:2507.01098},
  year={2025}
}

@inproceedings{chughtai2023toy,
  title={A toy model of universality: Reverse engineering how networks learn group operations},
  author={Chughtai, Bilal and Chan, Lawrence and Nanda, Neel},
  booktitle={International Conference on Machine Learning},
  pages={6243--6267},
  year={2023},
  organization={PMLR}
}

@inproceedings{raghu2017svcca,
  title     = {SVCCA: Singular Vector Canonical Correlation Analysis for Deep Learning Dynamics and Interpretability},
  author    = {Raghu, Maithra and Gilmer, Justin and Yosinski, Jason and Sohl-Dickstein, Jascha},
  booktitle = {Advances in Neural Information Processing Systems},
  year      = {2017},
  url       = {https://arxiv.org/abs/1706.05806}
}

@inproceedings{bansal2021stitching,
  title     = {Revisiting Model Stitching to Compare Neural Representations},
  author    = {Bansal, Yamini and Nakkiran, Preetum and Barak, Boaz},
  booktitle = {Advances in Neural Information Processing Systems},
  year      = {2021},
  url       = {https://arxiv.org/abs/2106.07682}
}

@article{schuhmann2021laion400m,
  title   = {LAION-400M: Open Dataset of CLIP-Filtered 400 Million Image-Text Pairs},
  author  = {Schuhmann, Christoph and Vencu, Richard and Beaumont, Romain and Kaczmarczyk, Robert and Mullis, Clayton and Katta, Aarush and Coombes, Theo and Jitsev, Jenia and Komatsuzaki, Aran},
  journal = {arXiv preprint arXiv:2111.02114},
  year    = {2021},
  url     = {https://arxiv.org/abs/2111.02114}
}

@article{kriegeskorte2008rsa,
  title   = {Representational similarity analysis---connecting the branches of systems neuroscience},
  author  = {Kriegeskorte, Nikolaus and Mur, Marieke and Bandettini, Peter A.},
  journal = {Frontiers in Systems Neuroscience},
  volume  = {2},
  pages   = {4},
  year    = {2008},
  doi     = {10.3389/neuro.06.004.2008},
  url     = {https://www.frontiersin.org/articles/10.3389/neuro.06.004.2008/full}
}

@inproceedings{morcos2018pwcca,
  title     = {Insights on representational similarity in neural networks with canonical correlation},
  author    = {Morcos, Ari S. and Raghu, Maithra and Bengio, Samy},
  booktitle = {Advances in Neural Information Processing Systems},
  year      = {2018},
  url       = {https://arxiv.org/abs/1806.05759}
}

@inproceedings{lenc2015equivariance,
  title     = {Understanding Image Representations by Measuring Their Equivariance and Equivalence},
  author    = {Lenc, Karel and Vedaldi, Andrea},
  booktitle = {Proceedings of the IEEE Conference on Computer Vision and Pattern Recognition (CVPR)},
  year      = {2015},
  url       = {https://openaccess.thecvf.com/content_cvpr_2015/html/Lenc_Understanding_Image_Representations_2015_CVPR_paper.html}
}

@article{mikolov2013exploiting,
  title         = {Exploiting Similarities among Languages for Machine Translation},
  author        = {Mikolov, Tomas and Le, Quoc V. and Sutskever, Ilya},
  journal       = {arXiv preprint arXiv:1309.4168},
  year          = {2013},
  url           = {https://arxiv.org/abs/1309.4168}
}

@InProceedings{moayeri2023text2concept,
  title     = {Text-To-Concept (and Back) via Cross-Model Alignment},
  author    = {Moayeri, Mazda and Rezaei, Keivan and Sanjabi, Maziar and Feizi, Soheil},
  booktitle = {Proceedings of the 40th International Conference on Machine Learning},
  pages     = {25037--25060},
  year      = {2023},
  editor    = {Krause, Andreas and Brunskill, Emma and Cho, Kyunghyun and Engelhardt, Barbara and Sabato, Sivan and Scarlett, Jonathan},
  volume    = {202},
  series    = {Proceedings of Machine Learning Research},
  month     = {23--29 Jul},
  publisher = {PMLR},
  url       = {https://proceedings.mlr.press/v202/moayeri23a.html}
}

@article{maystre2025embeddingmodelsmeet,
  title         = {When Embedding Models Meet: Procrustes Bounds and First-Order Equivalences for Embedding Interoperability},
  author        = {Maystre, Lucas and others},
  journal       = {arXiv preprint arXiv:2510.13406},
  year          = {2025},
  url           = {https://arxiv.org/abs/2510.13406}
}

@InProceedings{jia2021align,
  title     = {Scaling Up Visual and Vision-Language Representation Learning With Noisy Text Supervision},
  author    = {Jia, Chao and Yang, Yinfei and Xia, Ye and Chen, Yi-Ting and Parekh, Zarana and Pham, Hieu and Le, Quoc and Sung, Yun-Hsuan and Li, Zhen and Duerig, Tom},
  booktitle = {Proceedings of the 38th International Conference on Machine Learning},
  pages     = {4904--4916},
  year      = {2021},
  editor    = {Meila, Marina and Zhang, Tong},
  volume    = {139},
  series    = {Proceedings of Machine Learning Research},
  month     = {18--24 Jul},
  publisher = {PMLR},
  url       = {https://proceedings.mlr.press/v139/jia21b.html}
}

@article{haochen2021provable,
  title={Provable guarantees for self-supervised deep learning with spectral contrastive loss},
  author={HaoChen, Jeff Z and Wei, Colin and Gaidon, Adrien and Ma, Tengyu},
  journal={Advances in neural information processing systems},
  volume={34},
  pages={5000--5011},
  year={2021}
}

@article{zhai2023siglip,
  title         = {Sigmoid Loss for Language Image Pre-training},
  author        = {Zhai, Xiaohua and others},
  journal       = {arXiv preprint arXiv:2303.15343},
  year          = {2023},
  url           = {https://arxiv.org/abs/2303.15343}
}

@article{singh2022flava,
  title         = {{FLAVA}: A Foundational Language and Vision Alignment Model},
  author        = {Singh, Amanpreet and others},
  journal       = {arXiv preprint arXiv:2112.04482},
  year          = {2022},
  url           = {https://arxiv.org/abs/2112.04482}
}

@inproceedings{liang2022modalitygap,
  title     = {Mind the Gap: Understanding the Modality Gap in Multi-modal Contrastive Representation Learning},
  author    = {Liang, Weixin and Zhang, Yuhui and Kwon, Yongchan and Yeung, Serena and Zou, James},
  booktitle = {Advances in Neural Information Processing Systems},
  year      = {2022},
}

@inproceedings{shen2020towards,
  title={Towards backward-compatible representation learning},
  author={Shen, Yantao and Xiong, Yuanjun and Xia, Wei and Soatto, Stefano},
  booktitle={Proceedings of the IEEE/CVF Conference on Computer Vision and Pattern Recognition},
  pages={6368--6377},
  year={2020}
}

@inproceedings{hu2022bcemb,
  title     = {Learning Backward Compatible Embeddings},
  author    = {Hu, (First name unknown) and others},
  booktitle = {Proceedings of the ACM SIGKDD International Conference on Knowledge Discovery and Data Mining},
  year      = {2022}
}

@article{gupta2023structuring,
  title={Structuring Representation Geometry with Rotationally Equivariant Contrastive Learning},
  author={Gupta, Sharut and Robinson, Joshua and Lim, Derek and Villar, Soledad and Jegelka, Stefanie},
  journal=iclr,
  year={2023}
}

@inproceedings{jang2025xbt,
  title     = {Towards Cross-modal Backward-compatible Representation Learning for Vision-Language Models},
  author    = {Jang, Young Kyun and Lim, Ser-nam},
  booktitle = {Proceedings of the IEEE/CVF International Conference on Computer Vision (ICCV)},
  year      = {2025},
  url       = {https://openaccess.thecvf.com/content/ICCV2025/papers/Jang_Towards_Cross-modal_Backward-compatible_Representation_Learning_for_Vision-Language_Models_ICCV_2025_paper.pdf}
}

@misc{openai_pricing,
  author       = {{OpenAI}},
  title        = {API Pricing},
  year         = {2024},
  howpublished = {\url{https://openai.com/api/pricing/}},
  note         = {Accessed: January 2026}
}

@article{johnson2019billion,
  title={Billion-scale similarity search with GPUs},
  author={Johnson, Jeff and Douze, Matthijs and J{\'e}gou, Herv{\'e}},
  journal={IEEE Transactions on Big Data},
  volume={7},
  number={3},
  pages={535--547},
  year={2019},
  publisher={IEEE}
}

@article{jayaram2019diskann,
  title={Diskann: Fast accurate billion-point nearest neighbor search on a single node},
  author={Jayaram Subramanya, Suhas and Devvrit, Fnu and Simhadri, Harsha Vardhan and Krishnawamy, Ravishankar and Kadekodi, Rohan},
  journal={Advances in neural information processing Systems},
  volume={32},
  year={2019}
}

@inproceedings{meng2021learning,
  title={Learning compatible embeddings},
  author={Meng, Qiang and Zhang, Chixiang and Xu, Xiaoqiang and Zhou, Feng},
  booktitle={Proceedings of the IEEE/CVF International Conference on Computer Vision},
  pages={9939--9948},
  year={2021}
}

@article{oord2018cpc,
  title         = {Representation Learning with Contrastive Predictive Coding},
  author        = {Oord, Aaron van den and Li, Yazhe and Vinyals, Oriol},
  journal       = {arXiv preprint arXiv:1807.03748},
  year          = {2018},
  url           = {https://arxiv.org/abs/1807.03748}
}

@article{arora2019contrastive,
  title         = {A Theoretical Analysis of Contrastive Unsupervised Representation Learning},
  author        = {Arora, Sanjeev and Khandeparkar, Hrishikesh and Khodak, Mikhail and Plevrakis, Orestis and Saunshi, Nikunj},
  journal       = {arXiv preprint arXiv:1902.09229},
  year          = {2019},
  url           = {https://arxiv.org/abs/1902.09229}
}

@InProceedings{wang2020alignmentuniformity,
  title     = {Understanding Contrastive Representation Learning through Alignment and Uniformity on the Hypersphere},
  author    = {Wang, Tongzhou and Isola, Phillip},
  booktitle = {Proceedings of the 37th International Conference on Machine Learning},
  pages     = {9929--9939},
  year      = {2020},
  editor    = {III, Hal Daum{\'e} and Singh, Aarti},
  volume    = {119},
  series    = {Proceedings of Machine Learning Research},
  month     = {13--18 Jul},
  publisher = {PMLR},
  url       = {https://proceedings.mlr.press/v119/wang20k.html}
}

@InProceedings{zimmermann2021invert,
  title     = {Contrastive Learning Inverts the Data Generating Process},
  author    = {Zimmermann, Roland S. and Sharma, Yash and Schneider, Steffen and Bethge, Matthias and Brendel, Wieland},
  booktitle = {Proceedings of the 38th International Conference on Machine Learning},
  pages     = {12979--12990},
  year      = {2021},
  editor    = {Meila, Marina and Zhang, Tong},
  volume    = {139},
  series    = {Proceedings of Machine Learning Research},
  month     = {18--24 Jul},
  publisher = {PMLR},
  url       = {https://proceedings.mlr.press/v139/zimmermann21a.html}
}

@inproceedings{cristianini2002kernelalignment,
  title     = {On Kernel-Target Alignment},
  author    = {Cristianini, Nello and Elisseeff, Andre and Shawe-Taylor, John and Kandola, Jaz},
  booktitle = {Advances in Neural Information Processing Systems},
  year      = {2002},
  url       = {https://papers.neurips.cc/paper/1946-on-kernel-target-alignment.pdf}
}

@article{cortes2012centeredalignment,
  title   = {Algorithms for Learning Kernels Based on Centered Alignment},
  author  = {Cortes, Corinna and Mohri, Mehryar and Rostamizadeh, Afshin},
  journal = {Journal of Machine Learning Research},
  volume  = {13},
  number  = {28},
  pages   = {795--828},
  year    = {2012},
  url     = {https://jmlr.org/papers/v13/cortes12a.html}
}

@article{garipov2018modeconnectivity,
  title={Loss surfaces, mode connectivity, and fast ensembling of dnns},
  author={Garipov, Timur and Izmailov, Pavel and Podoprikhin, Dmitrii and Vetrov, Dmitry P and Wilson, Andrew G},
  journal={Advances in neural information processing systems},
  volume={31},
  year={2018}
}

@inproceedings{draxler2018essentially,
  title={Essentially no barriers in neural network energy landscape},
  author={Draxler, Felix and Veschgini, Kambis and Salmhofer, Manfred and Hamprecht, Fred},
  booktitle={International conference on machine learning},
  pages={1309--1318},
  year={2018},
  organization={PMLR}
}

@article{entezari2021role,
  title={The role of permutation invariance in linear mode connectivity of neural networks},
  author={Entezari, Rahim and Sedghi, Hanie and Saukh, Olga and Neyshabur, Behnam},
  journal={arXiv preprint arXiv:2110.06296},
  year={2021}
}

@article{ainsworth2022git,
  title={Git re-basin: Merging models modulo permutation symmetries},
  author={Ainsworth, Samuel K and Hayase, Jonathan and Srinivasa, Siddhartha},
  journal={arXiv preprint arXiv:2209.04836},
  year={2022}
}

@inproceedings{parkhi2012cats,
  title={Cats and dogs},
  author={Parkhi, Omkar M and Vedaldi, Andrea and Zisserman, Andrew and Jawahar, CV},
  booktitle={2012 IEEE conference on computer vision and pattern recognition},
  pages={3498--3505},
  year={2012},
  organization={IEEE}
}

@inproceedings{fei2004learning,
  title={Learning generative visual models from few training examples: An incremental bayesian approach tested on 101 object categories},
  author={Fei-Fei, Li and Fergus, Rob and Perona, Pietro},
  booktitle={2004 conference on computer vision and pattern recognition workshop},
  pages={178--178},
  year={2004},
  organization={IEEE}
}

@misc{krizhevsky2009learning,
  title={Learning multiple layers of features from tiny images.(2009)},
  author={Krizhevsky, Alex and Hinton, Geoffrey and others},
  year={2009}
}

@inproceedings{cimpoi2014describing,
  title={Describing textures in the wild},
  author={Cimpoi, Mircea and Maji, Subhransu and Kokkinos, Iasonas and Mohamed, Sammy and Vedaldi, Andrea},
  booktitle={Proceedings of the IEEE conference on computer vision and pattern recognition},
  pages={3606--3613},
  year={2014}
}

@inproceedings{coates2011analysis,
  title={An analysis of single-layer networks in unsupervised feature learning},
  author={Coates, Adam and Ng, Andrew and Lee, Honglak},
  booktitle={Proceedings of the fourteenth international conference on artificial intelligence and statistics},
  pages={215--223},
  year={2011},
  organization={JMLR Workshop and Conference Proceedings}
}

@InProceedings{saunshi22contrastive,
  author = 	 {N. Saunshi and J. Ash and S. Goel and D. Misra and C. Zhang and S. Arora and S. Kakade and A. Krishnamurthy},
  title = 	 {Understanding Contrastive Learning Requires Incorporating Inductive Biases},
  booktitle = icml,
  year = 	 2022}

@InProceedings{robinson21short,
  author = 	 {J. Robinson and L. Sun and K. Yu and K. Batmanghelich and S. Jegelka and S. Sra},
  title = 	 {Can contrastive learning avoid shortcut solutions?},
  booktitle = {Neural Information Processing Systems (NeurIPS)},
  year = 	 2021}

\newpage
\appendix
\crefname{appendix}{Appendix}{Appendices}
\Crefname{appendix}{Appendix}{Appendices}
\crefalias{section}{appendix}
\crefalias{subsection}{appendix}
\crefalias{subsubsection}{appendix}
\startcontents[appendices]
\printcontents[appendices]{}{1}{\section*{Appendix}}
\newpage

\section{Additional Related Works}\label{sec:related_work_appendix}

\textbf{Representational Convergence.} A long-standing theme in representation learning is whether independently trained networks converge to comparable internal representations. This question is frequently studied through representational similarity analyses that compare activations up to broad equivalence classes, including RSA \citep{kriegeskorte2008rsa}, CCA-based methods such as SVCCA \citep{raghu2017svcca} and its refinements \citep{morcos2018pwcca}, and kernel-based comparisons such as CKA \citep{kornblith2019cka}. Earlier work also probed convergence by explicitly matching units or subspaces across independently trained networks
(e.g., via neuron matching and sparse prediction), highlighting that models may learn similar \emph{subspaces} even when individual coordinates do not align.
These tools have been influential for documenting empirical convergence trends and motivating hypotheses such as the Platonic Representation Hypothesis (PRH) \citep{huh2024prh}.
PRH formalizes convergence at the level of induced similarity structure: two representations are considered aligned when their (co-occurrence) kernels agree on corresponding inputs.
However, by construction, many similarity measures are invariant to broad classes of transformations (e.g., invertible linear maps) or compare induced kernels rather than producing an explicit coordinate-level correspondence.
Our work targets this stronger notion of convergence: whether independently trained \emph{multimodal embedding spaces} agree in (or can be brought into) a shared coordinate system via a simple global map.

\noindent \textbf{Model Stitching and Functional Interoperability.} 
Beyond similarity indices, several works test whether independently trained representations become \emph{interchangeable} under simple transformations such as linear or orthogonal maps. Early studies introduced \emph{stitching} layers to test equivalence between networks by swapping intermediate features \citep{lenc2015equivariance}, and subsequent work systematized stitching as a methodology for comparing learned representations and their compositionality across training regimes \citep{bansal2021stitching}.In parallel, \citep{mikolov2013exploiting} showed that independently trained word embedding spaces can often be aligned by a single linear map learned from limited supervision, suggesting that semantic structure is preserved up to a global change of basis. More recently, linear aligners have been used to translate representations between modern pretrained models, such as mapping vision features into CLIP space \citep{moayeri2023text2concept}, and to study when embeddings from different models are mutually transferable via low-complexity maps \citep{maystre2025embeddingmodelsmeet}. These approaches provide evidence that two networks encode similar information, even when coordinates do not match directly. 

\noindent \textbf{Symmetries, Global Minima, and Landscape Connectivity.}
The existence of many apparently distinct solutions is also consistent with known symmetries of neural networks and the geometry of the loss landscape.
Empirically, SGD solutions are often connected by low-loss paths (mode connectivity) \citep{garipov2018modeconnectivity,draxler2018essentially},
and accounting for permutation symmetries of hidden units can further reduce apparent barriers between independently trained models \citep{entezari2021role}.
Recent work makes this operational by explicitly \emph{rebasing} one model to another via permutation alignment to enable weight-space merging \citep{ainsworth2022git}.
In this context, our results can be viewed as a representation-space analogue: while weights admit large symmetry groups, we find that multimodal embedding spaces often differ primarily by an (approximately) orthogonal transform that is shared across modalities.

\noindent \textbf{Vision-Language Contrastive Pretraining and the Modality Gap.}
Large-scale vision-language representation learning is dominated by dual-encoder contrastive objectives, as exemplified by CLIP \citep{radford2021clip} and ALIGN \citep{jia2021align}, with many variants exploring alternative losses, scaling strategies, and training recipes (e.g., \citealp{zhai2023siglip,singh2022flava}). A recurring geometric observation in such models is the \emph{modality gap}: image and text embeddings can occupy systematically shifted regions even after normalization, with consequences for optimization and downstream behavior \citep{liang2022modalitygap,udandarao2022understanding,shi2023towards}. Our findings complement this line of work by showing that, despite the within-model modality gap, cross-model relationships exhibit a surprising rigidity: different models’ image and text spaces are often related by the \emph{same} global orthogonal map, so that aligning one modality effectively determines the other.

\noindent \textbf{Theory of Contrastive Learning and Kernel Alignment.} Theoretical analyses of contrastive objectives have clarified what is identifiable from paired data and how representation geometry emerges. Key work includes contrastive predictive coding and InfoNCE-style objectives \citep{oord2018cpc}, general theoretical frameworks and guarantees for contrastive representation learning \citep{arora2019contrastive,haochen2021provable,bangachev2025global}, and geometric decompositions into alignment and uniformity \citep{wang2020alignmentuniformity}. ~\citet{gupta2023structuring} show that minimizing a loss which preserves unimodal kernels across augmentations induces an orthogonally equivariant structure in the contrastive embedding space.  Other analyses connect contrastive learning to inversion of latent generative structure under suitable assumptions \citep{zimmermann2021invert}.
\par 
\noindent Separately, \emph{kernel alignment} formalizes agreement between similarity functions and targets \citep{cristianini2002kernelalignment,cortes2012centeredalignment},
and underlies modern representational comparisons such as CKA \citep{kornblith2019cka}.
PRH also emphasizes kernel-level convergence, and recent theoretical work proves ``perfect" PRH in simplified (deep linear) settings where representations align up to an orthogonal transformation~\citep{ziyin2025proof}. Our theoretical results operate at this interface:
we provide minimal conditions under which agreement of cross-modal similarity structure on a small anchor set forces the existence of a shared orthogonal transform coupling modalities across models,
and we establish stability guarantees translating approximate kernel agreement into reliable retrieval transfer.

\section{Additional Discussion and Implications}\label{sec: app_discussion}

Our results have several practical and scientific implications. First, they point to a simple mechanism for backward-compatible upgrades in large embedding systems. Upgrading embedding models typically forces full re-embedding and ANN index rebuilds, which are prohibitively costly at modern scale (hundreds of millions to billions of vectors)~\citep{jayaram2019diskann, johnson2019billion} and often require multi-hour rebuilds and six-figure re-embedding budgets~\citep{openai_pricing}. This has motivated prior work on backward-compatible and interoperable embeddings, where new models are explicitly trained to preserve compatibility with deployed representations \citep{jang2025xbt, meng2021learning, hu2022bcemb,shen2020towards}. Our findings show that a small anchor set can often restore cross-model comparability for multimodal contrastive systems. Because orthogonal maps preserve inner products, this can enable upgrades without re-encoding stored corpora and often without rebuilding indexes. 

\noindent  Second, a shared coordinate system enables mix-and-match multimodal pipelines. Different models often excel in different components, such as stronger vision encoders or stronger and more multilingual text encoders. When representations are aligned by a single orthogonal map, image and text towers from different models can be combined into a common space, enabling retrieval across heterogeneous encoders. This perspective complements recent work on representation compatibility and model stitching, which studies when independently trained networks can be connected via lightweight alignment layers \citep{bansal2021stitching}.

\noindent  Finally, the ability to align text representations without accessing text has implications for data governance and security. In many deployments, raw text may be unavailable due to privacy, licensing, or retention constraints, even though embeddings are stored. Our results show that, in multimodal systems, text-space alignment can be recovered without using text, given a small anchor set from another modality. At the same time, easy cross-model transferability raises security considerations: if embeddings across models and modalities are easily transformable, then stored embeddings may encode more transferable semantic information than anticipated, reinforcing the need to treat embeddings as sensitive artifacts rather than model-specific byproducts.

\section{Proof of Theoretical Results}\label{sec:app_proof}

\textit{For a schematic overview of our theoretical analysis, refer to~\Cref{fig:theory_schematic}.\\
}
\begin{figure}[!htb]
    \centering
    \includegraphics[width=0.7\linewidth]{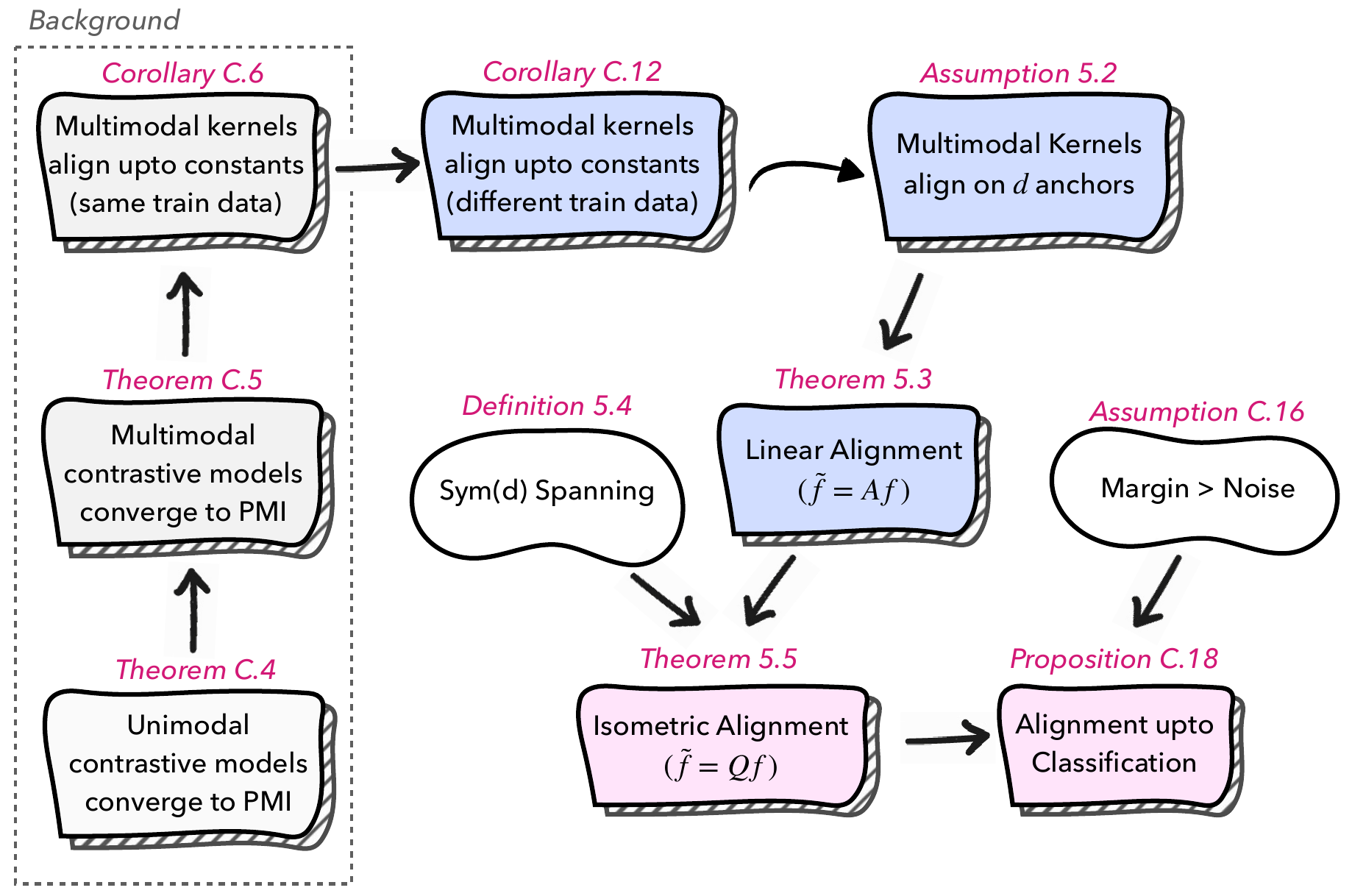}
    \caption{High-level overview of our theoretical results. Background results establish multimodal kernel alignment; additional assumptions progressively strengthen this to linear, then isometric, and finally task-level (classification) alignment.}
    \label{fig:theory_schematic}
\end{figure}

\subsection{Setup and Notations}\label{sec:theory_notations}
Let $(X,Y)$ be random variables on measurable spaces $\mathcal X,\mathcal Y$ with joint density $p_{XY}$ and marginals $p_X,p_Y$.

\begin{assumption}[Positivity on the domain]
\label{ass:pos}
On the domain of interest, assume $p_{XY}(x,y)>0$ and $p_X(x)p_Y(y)>0$.
\end{assumption}

\begin{remark}
The above assumptions can be expressed without densities by replacing ratios of densities with Radon-Nikodym derivatives. Specifically, on the domain of interest, assume $P_{XY}\ll P_X\otimes P_Y$ (equivalently, $P_{Y\mid X=x}\ll P_Y$ for $P_X$-a.e.\ $x$),
and let $r(x,y):=\frac{dP_{XY}}{d(P_X\otimes P_Y)}(x,y)$ denote the density ratio.
Optionally, assume $r(x,y)>0$ a.e.\ if $\log r$ is required to be finite.
For clarity, we'll use densities.
\end{remark}

\begin{definition}[Density ratio and Pointwise Mutual Information]
Define the mutual density ratio $r:\mathcal X\times\mathcal Y\to(0,\infty)$ by
\[
  r(x,y):=\frac{p_{XY}(x,y)}{p_X(x)p_Y(y)},\qquad
  K^\ast(x,y):=\log r(x,y).
\]
\end{definition}

\noindent A contrastive model is a pair of measurable maps $f:\mathcal X\to\Sph, g:\mathcal Y\to\Sph,$ where $\Sph := \{u\in\R^d : \norm{u}_2 = 1\}$, and a temperature $\tau>0$, defining a score $$ s(x,y):=\frac{1}{\tau}\,\inner{f(x)}{g(y)}+b,\; b\in\R. $$

\noindent In practice, contrastive learning constructs a finite candidate set by pairing each query with one positive and several independent negatives; the InfoNCE loss is exactly the cross-entropy for identifying the positive within this list. Below, we formalize the standard finite-sample sampling procedure underlying InfoNCE~\citep{oord2018cpc} principle.

\noindent \textbf{InfoNCE Sampling Model.} Fix $N\ge 2$. Sample $X\sim P_X$, draw one \emph{positive} $Y^+\sim P_{Y\mid X}(\cdot\mid X)$, and draw $N-1$ \emph{negatives}
$Y^{(1)},\dots,Y^{(N-1)}\stackrel{\mathrm{iid}}{\sim}P_Y$ independently of $(X,Y^+)$.
Sample $J\sim\Unif\{0,\dots,N-1\}$ and place $Y^+$ in slot $J$ to form the candidate list $(Y_0,\dots,Y_{N-1})$.
The learner observes $(X,Y_{0:N-1})$ but not $J$.

\noindent Given a measurable score $s:\mathcal X\times\mathcal Y\to\R$, define
\[
  q_s(j\mid x,y_{0:N-1})
  :=
  \frac{\exp(s(x,y_j))}{\sum_{k=0}^{N-1}\exp(s(x,y_k))},
  \qquad
  \mathcal L_N^{(X)}(s):=\E\big[-\log q_s(J\mid X,Y_{0:N-1})\big].
\]
Define $\mathcal L_N^{(Y)}$ analogously by swapping the roles of $X$ and $Y$, and set
\[
  \mathcal L_N(s):=\mathcal L_N^{(X)}(s)+\mathcal L_N^{(Y)}(s).
\]
This gives us the objective of a multimodal contrastive learner. We now characterize the minimizers of $\mathcal L_N$. 

\subsection{Bayes Posterior and Optimal Critic for Unimodal Contrastive Learners}\label{sec:unimodal_critic}

\begin{theorem}[Bayes posterior and characterization of minimizers]
\label{thm:infonce}
Under~\Cref{ass:pos},
\begin{enumerate}
\item For almost every $(x,y_{0:N-1})$ under the InfoNCE sampling model,

\[
  q^\ast(j\mid x,y_{0:N-1})
  :=
  \Pr(J=j\mid X=x,Y_{0:N-1}=y_{0:N-1})
  =
  \frac{r(x,y_j)}{\sum_{k=0}^{N-1} r(x,y_k)}.
\]
\item Any measurable $s^\ast$ achieving the infimum of $\mathcal L_N^{(X)}$ must satisfy
\[
  s^\ast(x,y)=\log r(x,y)+c(x)
\]
for some measurable $c:\mathcal X\to\R$, holding for almost every $(x,y)$ with respect to $P_X\otimes P_Y$.
\end{enumerate}
\end{theorem}
\begin{proof}
(i) Conditioning on $(J=j,X=x)$ gives
\[
  p(y_{0:N-1}\mid J=j,X=x)=p_{Y\mid X}(y_j\mid x)\prod_{i\neq j} p_Y(y_i).
\]
Since $J$ is uniform, Bayes' rule yields
\[
  \Pr(J=j\mid x,y_{0:N-1})
  =
  \frac{p_{Y\mid X}(y_j\mid x)\prod_{i\neq j} p_Y(y_i)}{\sum_{k=0}^{N-1} p_{Y\mid X}(y_k\mid x)\prod_{i\neq k} p_Y(y_i)}.
\]
Dividing numerator and denominator by $\prod_{i=0}^{N-1} p_Y(y_i)$ gives (i), since
$\frac{p_{Y\mid X}(y\mid x)}{p_Y(y)}=r(x,y)$.

\medskip
\noindent (ii) For each realization $(x,y_{0:N-1})$, the conditional risk is
\[
  \E\big[-\log q_s(J\mid x,y_{0:N-1})\mid x,y_{0:N-1}\big]
  = H\!\left(q^\ast(\cdot\mid x,y_{0:N-1}),\,q_s(\cdot\mid x,y_{0:N-1})\right),
\]
so $\mathcal L_N^{(X)}(s)$ is minimized iff $q_s(\cdot\mid x,y_{0:N-1})=q^\ast(\cdot\mid x,y_{0:N-1})$ a.e.
On this full-measure set, for any $j\neq k$,
\[
  \exp\!\big(s^\ast(x,y_j)-s^\ast(x,y_k)\big)
  =
  \frac{q_{s^\ast}(j\mid x,y_{0:N-1})}{q_{s^\ast}(k\mid x,y_{0:N-1})}
  =
  \frac{q^\ast(j\mid x,y_{0:N-1})}{q^\ast(k\mid x,y_{0:N-1})}
  =
  \frac{r(x,y_j)}{r(x,y_k)}.
\]
Integrating out $Y_2,\dots,Y_{N-1}$ (Fubini--Tonelli), gives
$\exp(s^\ast(x,y)-s^\ast(x,y'))=r(x,y)/r(x,y')$ for $(P_X\otimes P_Y\otimes P_Y)$-a.e. $(x,y,y')$.
Thus $s^\ast(x,y)-\log r(x,y)$ is (for $P_X$-a.e.\ $x$) independent of $y$ on a $P_Y$-full-measure set, i.e.
$s^\ast(x,y)=\log r(x,y)+c(x)$ for some measurable $c$, holding $(P_X\otimes P_Y)$-a.e.
\end{proof}

\subsection{Optimal Critic for Multimodal Contrastive Learners}\label{sec:multimodal_critic}

\begin{theorem}
\label{thm:sym}
Under the~\Cref{ass:pos}, any measurable $s^\ast$ achieving the infimum of $\mathcal L_N$
must satisfy
\[
  s^\ast(x,y)=\log r(x,y)+C
\]
for some constant $C\in\R$, holding for almost every $(x,y)$ with respect to $P_X\otimes P_Y$.
\end{theorem}
\begin{proof}
Let $\ell_X:=\inf_s \mathcal L_N^{(X)}(s)$ and $\ell_Y:=\inf_s \mathcal L_N^{(Y)}(s)$.
By~\Cref{thm:infonce}(ii), the score $s_0(x,y):=\log r(x,y)$ simultaneously achieves $\ell_X$ and $\ell_Y$ (up to an additive constant), hence
$\inf_s \mathcal L_N(s)=\ell_X+\ell_Y$.

\noindent If $s^\ast$ achieves the infimum of $\mathcal L_N$, then
\[
  \mathcal L_N^{(X)}(s^\ast)+\mathcal L_N^{(Y)}(s^\ast)=\ell_X+\ell_Y.
\]
Since each term is bounded below by its own infimum, we must have
$\mathcal L_N^{(X)}(s^\ast)=\ell_X$ and $\mathcal L_N^{(Y)}(s^\ast)=\ell_Y$.
Applying~\Cref{thm:infonce}(ii) in each direction yields
\[
  s^\ast(x,y)=\log r(x,y)+c(x)=\log r(x,y)+d(y)
\]
$(P_X\otimes P_Y)$-a.e. for measurable $c$ and $d$.
Thus $c(x)=d(y)$ on a full $(P_X\otimes P_Y)$-measure set.
Let $E\subseteq\mathcal X\times\mathcal Y$ be a full $(P_X\otimes P_Y)$-measure set where $c(x)=d(y)$.
By Fubini-Tonelli's theorem, there exists $y_0$ such that the section $E_{y_0}:=\{x:(x,y_0)\in E\}$ has $P_X(E_{y_0})=1$.
Then for $x\in E_{y_0}$, $c(x)=d(y_0)$, so $c$ is $P_X$-a.e.\ constant. Call this constant $C$.
Plugging back gives $s^\ast(x,y)=\log r(x,y)+C$ for $(P_X\otimes P_Y)$-a.e.\ $(x,y)$.
\end{proof}

\begin{corollary}
\label{cor:unique_kernels}
Let $(f,g)$ and $(\tilde f,\tilde g)$ be two contrastive models trained on the same distribution that achieve the global infimum of the symmetric objective $\mathcal L_N$.
Then there exists a constant $\Delta \in \R$ such that
\[
  \inner{f(x)}{g(y)} = \inner{\tilde f(x)}{\tilde g(y)} + \Delta
\]
for almost every $(x,y)$ with respect to $P_X \otimes P_Y$. If the temperatures differ, the relation becomes affine in $\langle \tilde f(x),\tilde g(y)\rangle$.
\end{corollary}
\begin{remark}
On a finite domain $\{x_1,\dots,x_n\}\times\{y_1,\dots,y_m\}$, the CLIP score matrix
$S\in\R^{n\times m}$ with entries $S_{ij}=\tau^{-1}\langle f(x_i),g(y_j)\rangle+b$ can be written as
$S=\tau^{-1}F^\top G + b\,\mathbf 1_n\mathbf 1_m^\top$, where $F=[f(x_1)\ \cdots\ f(x_n)]$ and $G=[g(y_1)\ \cdots\ g(y_m)]$.
Hence $\mathrm{rank}(S)\le d+1$. Therefore, exact equality $S_{ij}=K^\ast(x_i,y_j)+C$ would require a strong low-rank structure of the
PMI matrix on that domain. This is \emph{not} assumed in~\Cref{thm:infonce} and \Cref{thm:sym} as it becomes relevant only when one
tries to realize the Bayes-optimal critic within a dot-product parameterization.
\end{remark}

\subsection{Bayes-Optimal Scores Are Identifiable Up to a Constant}\label{sec:theory_pmi_different_data}
So far, \Cref{thm:sym} and \Cref{cor:unique_kernels} show that \emph{for a fixed training distribution}
$P_{XY}$, any two global minimizers of the symmetric InfoNCE objective induce the same kernel
(up to a global additive constant). We now give a simple, formal setting in which this conclusion
can \emph{also} hold for two \emph{different} training distributions.

\noindent \textbf{Dataset Curation.} Let $P^\star_{XY}$ be a ground truth distribution on $\mathcal X\times\mathcal Y$ with density $p^\star_{XY}$ and
marginals $p^\star_X,p^\star_Y$, satisfying positivity on the domain of interest (cf.\ \Cref{ass:pos}). For each dataset $a\in\{1,2\}$, let $P^{(a)}_{XY}$ be a training distribution with density $p^{(a)}_{XY}$
and marginals $p^{(a)}_X,p^{(a)}_Y$.
Let $r^\star,K^\star$ and $r^{(a)},K^{(a)}$ denote the density ratios and PMIs defined earlier,
applied to $p^\star_{XY}$ and $p^{(a)}_{XY}$, respectively.

\begin{assumption}
\label{ass:separable-curation}
For each dataset $a\in\{1,2\}$, there exist measurable weights
$u_a:\mathcal X\to(0,\infty)$ and $v_a:\mathcal Y\to(0,\infty)$ such that
\[
  Z_a \coloneqq \E^\star[u_a(X)v_a(Y)] < \infty,
  \qquad
  p^{(a)}_{XY}(x,y) = \frac{u_a(x)v_a(y)}{Z_a}\,p^\star_{XY}(x,y).
\]
\end{assumption}

\noindent In~\Cref{ass:separable-curation}, we view large-scale training corpora as distinct \emph{curations} of a common underlying distribution, where each modality (image/text) is filtered by its own criteria (e.g., quality, safety, language) independently of the other.

\begin{lemma}
\label{lem:ratio-curation}
Under \Cref{ass:separable-curation}, for each $a\in\{1,2\}$ and for
$P^\star_X\otimes P^\star_Y$-a.e.\ $(x,y)$,
\[
  r^{(a)}(x,y)
  =
  r^\star(x,y)\cdot
  \frac{Z_a}{\E^\star[v_a(Y)\mid X=x]\;\E^\star[u_a(X)\mid Y=y]}.
\]
Equivalently,
\[
  K^{(a)}(x,y)
  =
  K^\star(x,y)
  + \log Z_a
  - \log \E^\star[v_a(Y)\mid X=x]
  - \log \E^\star[u_a(X)\mid Y=y].
\]
\end{lemma}

\begin{proof}
By \Cref{ass:separable-curation},
$p^{(a)}_{XY}(x,y)=\frac{u_a(x)v_a(y)}{Z_a}p^\star_{XY}(x,y)$.
Thus
\[
  p^{(a)}_X(x)
  =
  \int p^{(a)}_{XY}(x,y)\,dy
  =
  \frac{u_a(x)p^\star_X(x)}{Z_a}\,\E^\star[v_a(Y)\mid X=x],
\]
and similarly
\[
  p^{(a)}_Y(y)
  =
  \frac{v_a(y)p^\star_Y(y)}{Z_a}\,\E^\star[u_a(X)\mid Y=y].
\]
Plugging into $r^{(a)}(x,y)=\frac{p^{(a)}_{XY}(x,y)}{p^{(a)}_X(x)p^{(a)}_Y(y)}$
gives the stated identity, and taking logs yields the PMI form.
\end{proof}

\noindent \Cref{lem:ratio-curation} shows that curation perturbs the density ratio by a multiplicative factor governed by the conditional expectations $\E^\star[v(Y)\mid X=x]$ and $\E^\star[u(X)\mid Y=y]$. Below, we impose a mild condition on these terms, requiring that dataset curation acts independently across modalities i.e the expected acceptance rate of texts does not depend on the image they are paired with, and vice versa. 

\begin{assumption}
\label{ass:no-cross-modal-bias}
For each dataset $a\in\{1,2\}$,
\[
  \E^\star[v_a(Y)\mid X=x]=\E^\star[v_a(Y)]
  \quad \text{for $P^\star_X$-a.e.\ $x$},
  \qquad
  \E^\star[u_a(X)\mid Y=y]=\E^\star[u_a(X)]
  \quad \text{for $P^\star_Y$-a.e.\ $y$}.
\]
\end{assumption}

\begin{theorem}
\label{thm:pmi-constant-across-datasets}
Under~\Cref{ass:separable-curation} and~\Cref{ass:no-cross-modal-bias}, for each $a\in\{1,2\}$ there exists
a constant $C_a\in\R$ such that $K^{(a)}(x,y)=K^\star(x,y)+C_a$ for $P^\star_X\otimes P^\star_Y$-a.e.\ $(x,y)$.
Consequently, there exists a constant $\Delta\in\R$ such that
\[
  K^{(1)}(x,y)=K^{(2)}(x,y)+\Delta
\]
for $P^\star_X\otimes P^\star_Y$-a.e.\ $(x,y)$.
\end{theorem}

\begin{proof}
Under \Cref{ass:no-cross-modal-bias}, the conditional expectations in \Cref{lem:ratio-curation} are constants:
$\E^\star[v_a(Y)\mid X=x]=\E^\star[v_a(Y)]$ and $\E^\star[u_a(X)\mid Y=y]=\E^\star[u_a(X)]$ a.e.
Hence $r^{(a)}(x,y)=\alpha_a\, r^\star(x,y)$ a.e., where
\[
  \alpha_a \coloneqq \frac{Z_a}{\E^\star[u_a(X)]\,\E^\star[v_a(Y)]},
\]
so $K^{(a)}(x,y)=K^\star(x,y)+\log\alpha_a$ a.e.
The second claim follows by subtraction with $\Delta=\log\alpha_1-\log\alpha_2$.
\end{proof}

\begin{corollary}
\label{cor:critics-across-datasets}
Let $s^\ast_1$ and $s^\ast_2$ be measurable global minimizers of the symmetric objective $\mathcal L_N$ when the
InfoNCE sampling model is defined under $P^{(1)}_{XY}$ and $P^{(2)}_{XY}$, respectively (cf.\ \Cref{sec:theory_notations}).
Under \Cref{thm:pmi-constant-across-datasets}, there exists $\Delta'\in\R$ such that
\[
  s^\ast_1(x,y)=s^\ast_2(x,y)+\Delta'
\]
for $P^\star_X\otimes P^\star_Y$-a.e.\ $(x,y)$.
\end{corollary}

\begin{proof}
By \Cref{thm:sym} applied to each dataset,
$s^\ast_a(x,y)=K^{(a)}(x,y)+C'_a$ for constants $C'_a$ (a.e.).
By \Cref{thm:pmi-constant-across-datasets}, $K^{(1)}(x,y)=K^{(2)}(x,y)+\Delta$ (a.e.),
hence $s^\ast_1-s^\ast_2$ is constant.
\end{proof}

\begin{remark}[Implication for contrastive dot-product kernels]
If each critic is realized through $(f,g)$ and $(\tilde f,\tilde g)$,
then \Cref{cor:critics-across-datasets} implies
\[
  \langle f(x),g(y)\rangle = \langle \tilde f(x),\tilde g(y)\rangle + \widetilde\Delta
\]
for $P^\star_X\otimes P^\star_Y$-a.e.\ $(x,y)$, for some constant $\widetilde\Delta\in\R$.
If the temperatures differ, the relation becomes affine in $\langle \tilde f(x),\tilde g(y)\rangle$.
\end{remark}

\subsection{Guarantees Under Exact Alignment}\label{sec:exact_regime}
\subsubsection{Linear Alignment of Independent Contrastive
Models}\label{sec:exact_linear}
In this section, we prove our first main result showing that matching multimodal kernels on a small set of “anchor” points is sufficient to lock the two embedding spaces together up to a linear transformation. 

\noindent  Let $\Omega_X\subseteq\mathcal X$ and $\Omega_Y\subseteq\mathcal Y$ be subsets (the domain of interest, e.g.\ a downstream dataset of images and prompts). Let the contrastive model pairs $(f,g)$ and $(\tilde f, \tilde g)$ map inputs to $\mathbb S^{d-1} \subset \R^d$ and $ \mathbb S^{\tilde d-1} \subset \R^{\tilde d}$ respectively, where $d \le \tilde d$, without loss of generality. We fix a set of \emph{image anchors} $\{\bar x_j\}_{j=1}^d \subset \Omega_X$ and \emph{text anchors} $\{\bar y_i\}_{i=1}^{\tilde d} \subset \Omega_Y$ and collect their embeddings into the following matrices:
\begin{gather*}
  G \coloneqq [g(\bar y_1)\ \cdots\ g(\bar y_{\tilde d})]\in\R^{d\times \tilde d}, \\
  \tilde G \coloneqq [\tilde g(\bar y_1)\ \cdots\ \tilde g(\bar y_{\tilde d})]\in\R^{\tilde d\times \tilde d}, \\
  F \coloneqq [f(\bar x_1)\ \cdots\ f(\bar x_d)]\in\R^{d\times d}.
\end{gather*}

\AnchorKernelEquality*

\linearthm*

\begin{proof}
    Fix $x\in\Omega_X$ and define
\[
  k(x):=\big(\inner{f(x)}{g(\bar y_i)}\big)_{i=1}^{\tilde d}=G^\top f(x),\qquad
  \tilde k(x):=\big(\inner{\tilde f(x)}{\tilde g(\bar y_i)}\big)_{i=1}^{\tilde d}=\tilde G^\top \tilde f(x).
\]
By~\Cref{ass:kernel_anchors}, $k(x)=\tilde k(x)$, hence $G^\top f(x)=\tilde G^\top \tilde f(x)$.
Since $\tilde G$ is invertible,
\[
  \tilde f(x)=\tilde G^{-\top}G^\top f(x)=A f(x),
\]
for all $x\in\Omega_X$. Now fix $y\in\Omega_Y$. For each anchor $\bar x_j$,
\[
  \inner{f(\bar x_j)}{g(y)}
  =
  \inner{\tilde f(\bar x_j)}{\tilde g(y)}
  =
  \inner{A f(\bar x_j)}{\tilde g(y)}
  =
  \inner{f(\bar x_j)}{A^\top\tilde g(y)}.
\]
Thus $\inner{f(\bar x_j)}{g(y)-A^\top\tilde g(y)}=0$ for all $j\in\{1,\dots,d\}$.
Since $F=[f(\bar x_1)\ \cdots\ f(\bar x_d)]$ is invertible, $\{f(\bar x_j)\}_{j=1}^d$ spans $\R^d$, hence
$g(y)=A^\top\tilde g(y)$ for all $y\in\Omega_Y$.

\noindent Finally, if $A$ has full column rank then $P:=A(A^\top A)^{-1}A^\top$ is the orthogonal projector onto $\operatorname{Im}(A)$, and
\[
  \mathrm{Proj}_{\operatorname{Im}(A)}\tilde g(y)
  =P\tilde g(y)
  =A(A^\top A)^{-1}A^\top\tilde g(y)
  =A(A^\top A)^{-1}g(y).
\]
If $\tilde d=d$ and $A$ is invertible, then $A(A^\top A)^{-1}=A^{-\top}$, giving $\tilde g(y)=A^{-\top}g(y)$.
\end{proof}

\subsubsection{Isometric Alignment of Independent Contrastive Models}\label{sec:rotation_alignment}

\Cref{thm:lin} establishes that the representation is fixed up to a linear transformation $A$. However, standard contrastive encoders normalize embeddings to the unit hypersphere $\mathbb{S}^{d-1}$, forcing $\|\tilde f(x)\|_2 = \|A f(x)\|_2 = 1$ everywhere. This forces $A$ to be an isometry ($A^\top A = I_d$) only if the data is sufficiently diverse to probe the matrix in all directions. We formalize this diversity via the following condition.

\SymSpanning*

\begin{lemma}
\label{lem:orth}
Let $A\in\R^{\tilde d\times d}$ and let $U\subseteq \Sph\subset\R^d$.
Assume $\|Au\|_2=1$ for all $u\in U$.
If $U$ contains a $\Sym(d)$-spanning subset, then $A^\top A=I_d$.
\end{lemma}
\begin{proof}
Let $M:=A^\top A-I_d\in\Sym(d)$. For any $u\in\Sph$,
\[
  \|Au\|_2^2=u^\top A^\top A u = u^\top(I_d+M)u = 1+u^\top M u.
\]
Thus $\|Au\|_2=1$ implies $u^\top M u=0$ for all $u\in U$.
Choose a $\Sym(d)$-spanning subset $S=\{u_i\}\subseteq U$.
Then $u_i^\top M u_i=0$ for all $i$ implies $M=0$ by definition of $\Sym(d)$-spanning, hence $A^\top A=I_d$.
\end{proof}

\orthogonalthm*
\begin{proof}
By~\Cref{thm:lin}, $\tilde f(x)=Af(x)$ for all $x\in\Omega_X$.
Since $\|f(x)\|_2=\|\tilde f(x)\|_2=1$, we have $\|A f(x)\|_2=1$ for all $x\in\Omega_X$. Apply Lemma~\ref{lem:orth} to conclude $A^\top A=I_d$. Set $Q:=A$.

\noindent By Theorem~\ref{thm:lin}, $g(y)=Q^\top\tilde g(y)$ for all $y\in\Omega_Y$. Multiplying by $Q$ gives
\[
  Qg(y)=QQ^\top \tilde g(y)=\mathrm{Proj}_{\operatorname{Im}(Q)}\,\tilde g(y),
\]
since $QQ^\top$ is the orthogonal projector onto $\operatorname{Im}(Q)$ when $Q^\top Q=I_d$.
If $\tilde d=d$, then $\operatorname{Im}(Q)=\R^d$ and $\mathrm{Proj}_{\operatorname{Im}(Q)}$ is the identity.
\end{proof}
\begin{remark}\label{remark:meancentering}
Suppose the exact regime of \Cref{thm:orth} holds: $\tilde f(x)=Qf(x)$ for all $x\in\Omega_X$ (and similarly for $g$),
and suppose the means are taken w.r.t.\ the same distribution on $\Omega_X$:
$\mu_f:=\E[f(X)]$, $\mu_{\tilde f}:=\E[\tilde f(X)]$.
Then $\mu_{\tilde f}=Q\mu_f$, hence
\[
  \tilde f(x)-\mu_{\tilde f} = Q\big(f(x)-\mu_f\big).
\]
Therefore the rigid map $z\mapsto Q(z-\mu_f)+\mu_{\tilde f}$ reduces exactly to the pure rotation $z\mapsto Qz$.
In this sense, mean-centering is a finite-sample / distribution-mismatch correction that leaves the exact-identifiability
statement unchanged.
\end{remark}

\subsubsection{Isometric Alignment of Independent Contrastive Models When One Modality Lies in Low Dimension}\label{sec:exact_low_dim}
The above analysis assumes that $\mathrm{span}\{f(x):x\in\Omega_X\} = \mathbb R^d$ since we assumed the existence of $d$ anchor vectors from $\Omega_X$ that make $F$ invertible. In other works, we assumed the existence of $d$ linearly independent vectors $f(\bar x_1),... f(\bar x_d)$, that span $\mathbb R^d$. We extend this analysis below in~\Cref{thm:orth_subspace} to cases where these embeddings instead lie in a lower-dimensional subspace. 

\noindent  Let
\[
  U_X := \mathrm{span}\{f(x):x\in\Omega_X\}\subseteq \R^d,
  \qquad
  \tilde U_X := \mathrm{span}\{\tilde f(x):x\in\Omega_X\}\subseteq \R^{\tilde d}.
\]
where $U_X$ is a subspace of $\R^d$ with rank $r:=\dim(U_X)$. Choose anchors $\bar x_1,\dots,\bar x_r\in\Omega_X$ such that
$\mathrm{span}\{f(\bar x_1),\dots,f(\bar x_r)\}=U_X$.

\begin{theorem}[\textit{Orthogonal Identifiability}]\label{thm:orth_subspace}
Assume multimodal kernel alignment as in~\Cref{ass:kernel_anchors} but on $r$ anchors $\bar x_1,\dots,\bar x_r\in\Omega_X$ instead of $d$. Assume there exists a matrix $Q\in\R^{\tilde d\times d}$ with $Q^\top Q=I_d$ such that $\tilde f(x)=Q f(x)$ for all $x\in\Omega_X$. Then $\tilde U_X = Q U_X$ and for every $y\in\Omega_Y$,
\[
  \mathrm{Proj}_{\tilde U_X}\,\tilde g(y)
  =
  Q\,\mathrm{Proj}_{U_X}\,g(y).
\]
In particular, if $U_X=\R^d$, then $\mathrm{Proj}_{\operatorname{Im}(Q)}\tilde g(y)=Qg(y)$ for all $y\in \Omega_Y$,
and if additionally $\tilde d=d$ then $\tilde g(y)=Qg(y)$.

\end{theorem}
\begin{proof}
    Since $\tilde f(x)=Q f(x)$ and $Q^\top Q=I_d$, for each anchor $\bar x_j$ with $j\in\{1,\dots,r\}$ and $y\in \Omega_Y$,
\[
  \inner{f(\bar x_j)}{g(y)}
  =
  \inner{\tilde f(\bar x_j)}{\tilde g(y)}
  =
  \inner{Q f(\bar x_j)}{\tilde g(y)}
  =
  \inner{f(\bar x_j)}{Q^\top \tilde g(y)}.
\]
Hence $\inner{f(\bar x_j)}{g(y)-Q^\top \tilde g(y)}=0$ for all anchors $\bar x_j$.
By linearity, $\inner{u}{g(y)-Q^\top \tilde g(y)}=0$ for all $u\in U_X$, so
$g(y)-Q^\top \tilde g(y)\in U_X^\perp$.
Equivalently,
$\mathrm{Proj}_{U_X} g(y)=\mathrm{Proj}_{U_X}(Q^\top\tilde g(y))$.
Applying $Q$ and using $\tilde U_X=Q U_X$ and $\mathrm{Proj}_{\tilde U_X}=Q\,\mathrm{Proj}_{U_X}\,Q^\top$ yields
\[
  \mathrm{Proj}_{\tilde U_X}\tilde g(y)
  =
  Q\,\mathrm{Proj}_{U_X}\,Q^\top \tilde g(y)
  =
  Q\,\mathrm{Proj}_{U_X}\,g(y).
\]
If $U_X=\R^d$, then $\tilde U_X=Q\R^d=\operatorname{Im}(Q)$, giving $\mathrm{Proj}_{\operatorname{Im}(Q)}\tilde g(y)=Qg(y)$.
If moreover $\tilde d=d$, then $\operatorname{Im}(Q)=\R^d$ and the projection is the identity.
\end{proof}

\subsubsection{Alignment Upto Classification Boundaries}\label{sec:exact_classification}
While the preceding theory establishes conditions for exact geometric alignment. In practice however, one often cares about alignment up to concepts or classification that are relevant for the downstream task. We now analyze this regime, where we seek to distinguish a finite family of prompts $\mathcal Y_{\mathrm{cls}}=\{y_c\}_{c=1}^K\subseteq \Omega_Y$ even when pointwise alignment is imprecise. Specifically, we care about the within-modality cross-model top-1 retrieval accuracy i.e. 
\[
  \hat c(y_c)\ \in\ \arg\max_{c'\in\{1,\dots,K\}}\ \inner{Q\,g(y_c)}{\tilde g(y_{c'})}.
\]

\noindent \Cref{thm:orth_subspace} is stated for all $y\in\Omega_Y$ and its proof is pointwise in $y$ and uses the multimodal
kernel equalities only for the pairs $(\bar x_j,y)$ with $j\in\{1,\dots,r\}$.
Therefore, if one only needs the conclusion for a subset $\mathcal Y_0\subseteq\Omega_Y$ (e.g.\ $\mathcal Y_0=\mathcal Y_{\mathrm{cls}}$),
it suffices to assume the anchor kernel equalities only on $\{\bar x_1,\dots,\bar x_r\}\times \mathcal Y_0$. Formally,
\begin{assumption}[Anchor kernel equalities only for class prompts]
\label{ass:kernel-class-only}
For the image anchors $\bar x_1,\dots,\bar x_r$ spanning $U_X$, assume that for all class prompts $y\in\mathcal Y_{\mathrm{cls}}$,
\[
  \inner{f(\bar x_j)}{g(y)}=\inner{\tilde f(\bar x_j)}{\tilde g(y)}
  \qquad \forall j\in\{1,\dots,r\}.
\]
\end{assumption}

\noindent Now, for any class prompt $y \in \mathcal Y_{\mathrm{cls}}$, decompose the embedding into an identifiable signal $u(y)$ and an unidentifiable residual $w(y)$:
$$  g(y) = u(y) + w(y), \;\; u(y) \coloneqq \mathrm{Proj}_{U_X} g(y),\ \ w(y) \in U_X^\perp.$$
Then, from~\Cref{thm:orth_subspace},
define analogously
\[
  \tilde g(y) = Q u(y) + \tilde w(y),
  \qquad
  \tilde w(y)\in (Q U_X)^\perp.
\]

\begin{definition}
\label{def:margin-noise}
Let $\mathcal Y_{\mathrm{cls}}=\{y_c\}_{c=1}^K$
Define
\[
  \gamma
  :=
  \min_{c}
  \Big(
    \|u(y_c)\|_2^2
    -
    \max_{c'\neq c}\inner{u(y_c)}{u(y_{c'})}
  \Big),
\]
that measures class separability \emph{within the image span} $U_X$. Define,
\[
  \eta
  :=
  \max_{c,c'}
  \big|\inner{Q w(y_c)}{\tilde w(y_{c'})}\big|.
\]
that measures the worst-case cross-model interaction of the unidentifiable residuals across models.
\end{definition}

\begin{proposition}
\label{prop:class-retrieval}
Under~\Cref{ass:kernel-class-only}, if $\gamma > 2\eta,$ then nearest-neighbor class retrieval is correct for every class prompt:
\[
  \hat c(y_c) = c
  \qquad \forall c\in\{1,\dots,K\}.
\]
\end{proposition}

\begin{proof}
The conclusion of~\Cref{thm:orth_subspace} under~\Cref{ass:kernel-class-only} implies 
$$\tilde U_X = Q U_X \quad \text{and} \quad \mathrm{Proj}_{\tilde U_X}\,\tilde g(y) = Q\,\mathrm{Proj}_{U_X}\,g(y) \;\; \forall y \in \mathcal Y_{\mathrm{cls}}$$
Fix a query class $c$ and a candidate $c'$. Using $g(y)=u(y)+w(y)$ and $\tilde g(y)=Q u(y)+\tilde w(y)$,
\[
  \inner{Q g(y_c)}{\tilde g(y_{c'})}
  =
  \inner{Q(u(y_c)+w(y_c))}{Q u(y_{c'})+\tilde w(y_{c'})}.
\]
$Q u(y_c)$ is orthogonal to $\tilde w(y_{c'})$ and similarly $\inner{Q w(y_c)}{Q u(y_{c'})}=0$ since $w(y_c)\in U_X^\perp$. Also, $Q^\top Q=I_d$ implies $\inner{Q u}{Q u'}=\inner{u}{u'}$. Thus,
$$ \inner{Q g(y_c)}{\tilde g(y_{c'})}
  =
  \inner{u(y_c)}{u(y_{c'})} + \inner{Q w(y_c)}{\tilde w(y_{c'})}.
$$
By the definition of $\gamma$, for any $c'\neq c$ we have
$\inner{u(y_c)}{u(y_{c'})}\le \|u(y_c)\|_2^2-\gamma$.
Therefore, for all $c'\neq c$,
\[
  \inner{Q g(y_c)}{\tilde g(y_c)} - \inner{Q g(y_c)}{\tilde g(y_{c'})}
  \ge (\|u(y_c)\|_2^2-\eta) - (\|u(y_c)\|_2^2-\gamma+\eta)
  = \gamma-2\eta.
\]
If $\gamma>2\eta$, the right-hand side is strictly positive, so $c$ is the unique maximizer and $\hat c(y_c)=c$.
\end{proof}

\noindent \Cref{prop:class-retrieval} shows that class retrieval depends on a \emph{margin} condition inside $U_X$, and is stable
as long as the residual terms in $U_X^\perp$ do not overwhelm that margin.

\subsection{Guarantees Under Approximate Alignment}\label{sec:approx_regime}

We now relax the anchor equalities in~\Cref{sec:exact_regime} and allow small discrepancies on the anchor pairs.
Like before, we assume the contrastive model pairs $(f,g)$ and $(\tilde f, \tilde g)$ to map inputs to $\mathbb S^{d-1} \subset \R^d$ and $ \mathbb S^{\tilde d-1} \subset \R^{\tilde d}$ respectively, where $d \le \tilde d$, without loss of generality.

\subsubsection{Approximate Linear Alignment of Independent Contrastive Models}\label{sec:approx_linear}

Similar to before, fix $\tilde d$ text anchors $\bar y_1,\dots,\bar y_{\tilde d}\in\Omega_Y$ and define
\[
  G := [g(\bar y_1)\ \cdots\ g(\bar y_{\tilde d})]\in\R^{d\times \tilde d},\qquad
  \tilde G := [\tilde g(\bar y_1)\ \cdots\ \tilde g(\bar y_{\tilde d})]\in\R^{\tilde d\times \tilde d}.
\]

\begin{assumption}[Approximate multimodal kernel equalities]
\label{ass:discX}
There exists $\varepsilon\ge 0$ such that for all $x\in\Omega_X$ and all $i\in\{1,\dots,\tilde d\}$,
\[
  \big|\inner{f(x)}{g(\bar y_i)}-\inner{\tilde f(x)}{\tilde g(\bar y_i)}\big|\le \varepsilon.
\]
\end{assumption}

\begin{theorem}
\label{thm:pertX}
Assume $\tilde G$ is invertible with smallest singular value $\sigma_{\min}(\tilde G)>0$. Then under~\Cref{ass:discX}, there exists a linear map $A$ such that for every $x\in\Omega_X$,
\[
  \norm{\tilde f(x)- A f(x)}_2 \le \frac{\sqrt{\tilde d}}{\sigma_{\min}(\tilde G)}\,\varepsilon.
\]
\end{theorem}

\begin{proof}
Define $k(x):=G^\top f(x)\in\R^{\tilde d}$ and $\tilde k(x):=\tilde G^\top \tilde f(x)\in\R^{\tilde d}$.
Set $A := \tilde G^{-\top} G^\top \in \R^{\tilde d\times d}$, then
\[
  \tilde f(x)-A f(x)
  = \tilde f(x)-\tilde G^{-\top}G^\top f(x)
  = \tilde G^{-\top}\big(\tilde G^\top \tilde f(x)-G^\top f(x)\big)
  = \tilde G^{-\top}\big(\tilde k(x)-k(x)\big).
\]
Therefore
\[
  \norm{\tilde f(x)-A f(x)}_2
  \le \|\tilde G^{-\top}\|_2\,\|\tilde k(x)-k(x)\|_2
  = \frac{1}{\sigma_{\min}(\tilde G)}\,\|\tilde k(x)-k(x)\|_2.
\]
By~\Cref{ass:discX}, each coordinate of $\tilde k(x)-k(x)$ has magnitude at most $\varepsilon$,
so $\|\tilde k(x)-k(x)\|_2\le \sqrt{\tilde d}\,\varepsilon$.
\end{proof}

\subsubsection{Approximate Isometric Alignment of Independent Contrastive Models}\label{sec:approx_isometric}

\begin{assumption}[\emph{$\Sym(d)$-spanning}]
\label{ass:qrichX}
There exist $\bar x_1,\dots,\bar x_m\in\Omega_X$ such that
$S =\{f(\bar x_1),\dots,f(\bar x_m)\}\subset\mathbb S^{d-1}$ is $\Sym(d)$-spanning
\end{assumption}

\noindent Fix $A\in\R^{\tilde d\times d}$ and define the uniform deviation
\[
  \Delta \;:=\; \sup_{x\in\Omega_X}\|\tilde f(x)-A f(x)\|_2.
\]
Also define the constant
\[
  \kappa_S \;:=\; \sup_{M\in\Sym(d)\setminus\{0\}}
  \frac{\|M\|_2}{\max_{u\in S}|u^\top M u|} \;<\;\infty.
\]

\begin{theorem}[Approximate Orthogonal Identifiability]
\label{lem:qrich-makes-A-orth}
Assume $A$ has full column rank. Then under~\Cref{ass:qrichX}, there exists $Q^\top Q = I_d$, such that for all $x\in\Omega_X$
$$ \|\tilde f(x)-Q f(x)\|_2 \le \Delta + \kappa_S\,(2+\Delta)\Delta $$
\end{theorem}
\begin{proof}
Let $M:=A^\top A-I_d\in\Sym(d)$. For any $u=f(\bar x_i)\in S$,
$\|\tilde f(\bar x_i)\|_2=1$ and $\|Au-\tilde f(\bar x_i)\|_2\le \Delta$ imply $|\|Au\|_2-1|\le \Delta$ and hence
$|\|Au\|_2^2-1|\le (2+\Delta)\Delta$.
But $\|Au\|_2^2-1=u^\top M u$, so $\max_{u\in S}|u^\top M u|\le (2+\Delta)\Delta$.
By definition of $\kappa_S$, $\|M\|_2\le \kappa_S\max_{u\in S}|u^\top M u|$, showing that 
\[
  \|A^\top A-I_d\|_2 \;\le\; \kappa_S\,(2+\Delta)\Delta.
\]
Now, let $Q\in\R^{\tilde d\times d}$ denote the orthogonal factor in the polar decomposition of $A$ i.e. $A = Q (A^\top A)^{1/2}, \;\; Q^\top Q = I_d.$
Then $\|A-Q\|_2 = \|(A^\top A)^{1/2}-I_d\|_2 \leq \|A^\top A-I_d\|_2$. Finally,
$\|\tilde f(x)-Qf(x)\|_2\le \|\tilde f(x)-Af(x)\|_2+\|(A-Q)f(x)\|_2
\le \Delta+\|A-Q\|_2$, proving the claim.
\end{proof}

\noindent Now, fix $d$ image anchors $\bar x_1,\dots,\bar x_d\in\Omega_X$ and define
\[
  F := [f(\bar x_1)\ \cdots\ f(\bar x_d)]\in\R^{d\times d},
  \qquad
  \sigma_{\min}(F)>0.
\]

\begin{assumption}[Approximate multimodal kernel equalities]
\label{ass:discY}
There exists $\varepsilon'\ge 0$ such that for all $y\in\Omega_Y$ and all $j\in\{1,\dots,d\}$,
\[
  \big|\inner{f(\bar x_j)}{g(y)}-\inner{\tilde f(\bar x_j)}{\tilde g(y)}\big|\le \varepsilon'.
\]
\end{assumption}

\begin{theorem}
\label{thm:pertY}
Assume~\Cref{ass:discY} and let $Q\in\R^{\tilde d\times d}$ satisfy $Q^\top Q=I_d$ and
\[
  \max_{j\in\{1,\dots,d\}} \norm{\tilde f(\bar x_j)-Q f(\bar x_j)}_2 \le \delta_f.
\]
Then for every $y\in\Omega_Y$,
\[
  \big\|\mathrm{Proj}_{\operatorname{Im}(Q)}\tilde g(y)-Q g(y)\big\|_2
  =
  \|Q^\top \tilde g(y)-g(y)\|_2
  \le \frac{\sqrt d}{\sigma_{\min}(F)}\,(\varepsilon'+\delta_f).
\]
In particular, if $\tilde d=d$ then $\operatorname{Im}(Q)=\R^d$ and the projection is redundant, giving a bound on $\|\tilde g(y)-Qg(y)\|_2$.
\end{theorem}

\begin{proof}
Fix $y\in\Omega_Y$ and set $v:=g(y)-Q^\top \tilde g(y)\in\R^d$.
For each anchor $\bar x_j$,
\[
  \inner{f(\bar x_j)}{v}
  =
  \inner{f(\bar x_j)}{g(y)}-\inner{Q f(\bar x_j)}{\tilde g(y)}.
\]
Thus
\[
  |\inner{f(\bar x_j)}{v}|
  \le
  |\inner{f(\bar x_j)}{g(y)}-\inner{\tilde f(\bar x_j)}{\tilde g(y)}|
  +
  |\inner{\tilde f(\bar x_j)-Q f(\bar x_j)}{\tilde g(y)}|
  \le \varepsilon'+\delta_f,
\]
since $\|\tilde g(y)\|_2=1$.
Hence $\|F^\top v\|_2\le \sqrt d(\varepsilon'+\delta_f)$ and, since $F$ is invertible,
\[
  \|v\|_2\le \|F^{-\top}\|_2\|F^\top v\|_2
  =\frac{1}{\sigma_{\min}(F)}\|F^\top v\|_2
  \le \frac{\sqrt d}{\sigma_{\min}(F)}(\varepsilon'+\delta_f).
\]
Finally, using $Q^\top Q=I_d$,
\[
  \|\mathrm{Proj}_{\operatorname{Im}(Q)}\tilde g(y)-Q g(y)\|_2
  = \|Q(Q^\top \tilde g(y)-g(y))\|_2
  = \|v\|_2.
\]
\end{proof}

\begin{remark}
~\Cref{ass:discX} and~\Cref{ass:discY} are $\varepsilon$-relaxations of~\Cref{ass:kernel_anchors}.~\Cref{thm:pertX} and~\Cref{thm:pertY} are the perturbation analogues of~\Cref{thm:lin} and~\Cref{thm:orth} respectively.
\end{remark}

\subsubsection{Approximate Isometric Alignment of Independent Contrastive Models When One Modality Lies in Low Dimension}\label{sec:approx_low_dim}
Similar to~\Cref{thm:orth_subspace}, we extend our analysis below to cases where the image embeddings lie in a lower-dimensional subspace. Let $U_X := \mathrm{span}\{f(x):x\in\Omega_X\}\subseteq \R^d$
where $U_X$ is a subspace of $\R^d$ of  $\dim(U_X)=r$. Choose anchors $\bar x_1,\dots,\bar x_r\in\Omega_X$ such that
$\mathrm{span}\{f(\bar x_1),\dots,f(\bar x_r)\}=U_X$. Accordingly, denote \[
  F := [f(\bar x_1)\ \cdots\ f(\bar x_r)]\in\R^{d\times r},
  \qquad
  \sigma_{\min}(F)>0.
\]

\begin{assumption}
\label{ass:approx-class-kernel}
Assume there exists a matrix $Q\in\R^{\tilde d\times d}$ with $Q^\top Q=I_d$ and constants $\delta_f,\varepsilon\ge 0$ such that
\[
  \sup_{x\in\Omega_X}\|\tilde f(x)-Q f(x)\|_2 \le \delta_f,
  \qquad
  \sup_{x\in\Omega_X,\ y\in \Omega_Y}
  \big|\inner{f(x)}{g(y)}-\inner{\tilde f(x)}{\tilde g(y)}\big|
  \le \varepsilon.
\]
\end{assumption}

\begin{theorem}
\label{thm:approx-proj-class}
Under~\Cref{ass:approx-class-kernel}, for every $y\in\Omega_Y$,
\[
  \Big\|\mathrm{Proj}_{U_X}\big(g(y)-Q^\top \tilde g(y)\big)\Big\|_2
  \le
  \rho,
  \qquad
  \rho := \frac{\sqrt r}{\sigma_{\min}(F)}\,(\varepsilon+\delta_f).
\]
Equivalently,
\[
  \Big\|\mathrm{Proj}_{Q U_X}\tilde g(y) - Q\,\mathrm{Proj}_{U_X} g(y)\Big\|_2 \le \rho.
\]
\end{theorem}
\begin{proof}
Fix $y\in\Omega_Y$ and define $h:=g(y)-Q^\top \tilde g(y)$.
For any $x\in\Omega_X$,
\begin{align*}
  \big|\inner{f(x)}{h}\big|
  &=
  \big|\inner{f(x)}{g(y)}-\inner{Q f(x)}{\tilde g(y)}\big| \\
  &\le
  \big|\inner{f(x)}{g(y)}-\inner{\tilde f(x)}{\tilde g(y)}\big|
  +
  \big|\inner{\tilde f(x)-Q f(x)}{\tilde g(y)}\big|
  \le \varepsilon+\delta_f,
\end{align*}
since $\|\tilde g(y)\|_2=1$.
Let $b\in\R^r$ have coordinates $b_j:=\inner{f(\bar x_j)}{h}$.
Then $\|b\|_2\le \sqrt r(\varepsilon+\delta_f)$.
Moreover, $F^\top \mathrm{Proj}_{U_X}h = F^\top h = b$, and since $F$ has full column rank,
\[
  \|\mathrm{Proj}_{U_X}h\|_2 \le \|(F^\top)^{\dagger}\|_2\,\|b\|_2
  = \frac{1}{\sigma_{\min}(F)}\,\|b\|_2
  \le \frac{\sqrt r}{\sigma_{\min}(F)}(\varepsilon+\delta_f).
\]
The equivalent bound in $\tilde U_X$ follows by left-multiplying by $Q$ and using
$\mathrm{Proj}_{Q U_X}=Q\,\mathrm{Proj}_{U_X}\,Q^\top$ (which holds whenever $Q^\top Q=I_d$).

\end{proof}

\section{Supplementary Experimental Details and Assets Disclosure}\label{sec:exp_setup}

\subsection{Assets}\label{sec:appendix_assets}
We do not introduce new data in the course of this work. Instead, we use publicly available, widely used image datasets for the purposes of benchmarking and comparison.

\subsection{Hardware and setup}
\label{sec:compute}
Each experiment was conducted on 1 NVIDIA Tesla V100 GPUs, each with 32GB of accelerator RAM. The CPUs used were Intel Xeon E5-2698 v4 processors with 20 cores and 384GB of RAM. All experiments were implemented using the PyTorch deep learning framework. We provide our experimental code for reproducibility in the supplementary material.

\subsection{Datasets}\label{sec:appendix_datasets}

\subsubsection{Image Classification Benchmarks}
\label{sec:appendix_benchmarks}

We evaluate cross-model alignment on three standard image classification benchmarks:
Oxford Pets~\citep{parkhi2012cats},
CIFAR-100~\citep{krizhevsky2009learning}, Caltech-101~\citep{fei2004learning}, STL10~\citep{coates2011analysis} and DTD~\citep{cimpoi2014describing}.
More details about dataset statistics and splits are provided in~\Cref{tab:dataset-stats}.

\begin{table}[!htb]
  \centering
  \caption{Detailed statistics of the 5 datasets for image classification.}
  \label{tab:dataset-stats}
  \begin{tabular}{l c r r r}
    \toprule
    Dataset                 & Classes &  Train  & Val   &  Test   \\
    \midrule
    Caltech-101~\citep{fei2004learning} &     100 &   4,128 & 1,649 &   2,465 \\
    Oxford Pets~\citep{parkhi2012cats} &      37 &   2,944 &   736 &   3,669 \\
    CIFAR100~\citep{krizhevsky2009learning} & 100 & 50,000 & - & 10,000\\
    STL10~\citep{coates2011analysis} & 10 & 50,000 & -& 8,000 \\
    DTD~\citep{cimpoi2014describing} & 47 & 2,820 & 1,128 & 1,692\\
    \bottomrule
  \end{tabular}
\end{table}

\subsubsection{Constructing Text Templates}
\label{sec:appendix_templates}

We construct class-level text prototypes by instantiating natural-language templates with each class name and embedding the resulting strings using each model’s text encoder. 
All templates are taken directly from the official CLIP repository and prior CLIP evaluation code, ensuring that our protocol matches standard zero-shot classification practice.~\Cref{tab:text-templates} summarizes the exact prompts used for each dataset. We embed all prompts for a class, aggregate them into a single prototype by averaging, and then $\ell_2$-normalizing. Concretely, for templates $\{\tau_k(c)\}_{k=1}^K$ for a class $c$,
\begin{equation}
    p_c = \text{norm}\Big(\frac{1}{K} \sum_k g(\tau_k(c))\Big)
\end{equation}
When the orthogonal map $Q$ is fit on images and evaluated on text, we use these averaged class prototypes $p_c$ for text-side evaluation. In contrast, when fitting $Q$ on text, averaging would collapse each class to a single vector and drastically reduce supervision; therefore, we fit $Q$ on individual prompt embeddings ${g(\tau_k(c))}_{k=1}^K$, treating each template instantiation as a separate training sample.

\begin{table}[!htb]
\centering
\small
\resizebox{0.7\textwidth}{!}{%
\begin{tabular}{p{3cm}p{10cm}}
\toprule
\textbf{Dataset} & \textbf{Text Templates (from official CLIP repo)} \\
\midrule

Oxford Pets &
\texttt{a photo of a \{class\}, a type of pet.} \\[10pt]

CIFAR-100 &
\begin{minipage}[t]{10cm}
\texttt{a photo of a \{class\}.}\\
\texttt{a blurry photo of a \{class\}.}\\
\texttt{a black and white photo of a \{class\}.}\\
\texttt{a low contrast photo of a \{class\}.}\\
\texttt{a high contrast photo of a \{class\}.}\\
\texttt{a bad photo of a \{class\}.}\\
\texttt{a good photo of a \{class\}.}\\
\texttt{a photo of a small \{class\}.}\\
\texttt{a photo of a big \{class\}.}\\
\texttt{a photo of the \{class\}.}\\
\texttt{a blurry photo of the \{class\}.}\\
\texttt{a black and white photo of the \{class\}.}\\
\texttt{a low contrast photo of the \{class\}.}\\
\texttt{a high contrast photo of the \{class\}.}\\
\texttt{a bad photo of the \{class\}.}\\
\texttt{a good photo of the \{class\}.}\\
\texttt{a photo of the small \{class\}.}\\
\texttt{a photo of the big \{class\}.}\\[10pt]
\end{minipage}
\\[10pt]

Caltech-101 &
\begin{minipage}[t]{10cm}
\texttt{a photo of a \{class\}.}\\
\texttt{a painting of a \{class\}.}\\
\texttt{a plastic \{class\}.}\\
\texttt{a sculpture of a \{class\}.}\\
\texttt{a sketch of a \{class\}.}\\
\texttt{a tattoo of a \{class\}.}\\
\texttt{a toy \{class\}.}\\
\texttt{a rendition of a \{class\}.}\\
\texttt{a embroidered \{class\}.}\\
\texttt{a cartoon \{class\}.}\\
\texttt{a \{class\} in a video game.}\\
\texttt{a plushie \{class\}.}\\
\texttt{a origami \{class\}.}\\
\texttt{art of a \{class\}.}\\
\texttt{graffiti of a \{class\}.}\\
\texttt{a drawing of a \{class\}.}\\
\texttt{a doodle of a \{class\}.}\\
\texttt{a photo of the \{class\}.}\\
\texttt{a painting of the \{class\}.}\\
\texttt{the plastic \{class\}.}\\
\texttt{a sculpture of the \{class\}.}\\
\texttt{a sketch of the \{class\}.}\\
\texttt{a tattoo of the \{class\}.}\\
\texttt{the toy \{class\}.}\\
\texttt{a rendition of the \{class\}.}\\
\texttt{the embroidered \{class\}.}\\
\texttt{the cartoon \{class\}.}\\
\texttt{the \{class\} in a video game.}\\
\texttt{the plushie \{class\}.}\\
\texttt{the origami \{class\}.}\\
\texttt{art of the \{class\}.}\\
\texttt{graffiti of the \{class\}.}\\
\texttt{a drawing of the \{class\}.}\\
\texttt{a doodle of the \{class\}.}\\[10pt]
\end{minipage}
\\[10pt]

STL-10 &
\begin{minipage}[t]{10cm}
\texttt{a photo of a \{class\}.}\\
\texttt{a photo of the \{class\}.}\\[10pt]
\end{minipage}
\\[10pt]

DTD &
\begin{minipage}[t]{10cm}
\texttt{a photo of a \{class\} texture.}\\
\texttt{a photo of a \{class\} pattern.}\\
\texttt{a photo of a \{class\} thing.}\\
\texttt{a photo of a \{class\} object.}\\
\texttt{a photo of the \{class\} texture.}\\
\texttt{a photo of the \{class\} pattern.}\\
\texttt{a photo of the \{class\} thing.}\\
\texttt{a photo of the \{class\} object.}
\end{minipage}
\\

\bottomrule
\end{tabular}
}
\caption{Text templates used to construct class-level text prototypes for each dataset.
All templates are taken from the official CLIP repository.
For all reported results, we use a single template per class for evaluation}
\label{tab:text-templates}
\end{table}

\FloatBarrier

\subsection{Training and Evaluation Protocol}\label{sec:app_training}
\subsubsection{Learning An Orthogonal Transformation}\label{sec:app_learning}
We consider two independently trained multimodal contrastive models with image encoders $f,\tilde f$ and text encoders $g,\tilde g$. Our goal is to learn a single orthogonal map $\mathcal R \in O(d)$ that aligns representations across models. All experiments use fixed train/validation/test splits with a fixed random seed. We report the mean and standard deviation averaged over three random seeds. We fit $\mathcal R$ using paired embeddings from a single modality (images or text). Let $\{x_i\}_{i=1}^n$ denote paired inputs in the chosen modality, and define
\[
X = \begin{bmatrix} f(x_1)^\top \\ \vdots \\ f(x_n)^\top \end{bmatrix},
\qquad
\tilde X = \begin{bmatrix} \tilde f(x_1)^\top \\ \vdots \\ \tilde f(x_n)^\top \end{bmatrix},
\]
with $f,\tilde f$ replaced by $g,\tilde g$ when fitting on text. For numerical stability, we center the embeddings only during estimation i.e. $ X \rightarrow X - \mu_X, \;\; \tilde X \rightarrow \tilde X - \mu_{\tilde X},$
and compute the closed form Orthogonal Procrustes solution
\[
\mathcal Q \;=\; \arg\min_{Q \in O(d)} \| X Q - {\tilde X}\|_F
\;=\; UV^\top,
\]
where $U\Sigma V^\top = \mathrm{SVD}( X^\top {\tilde X})$.
At evaluation, we deploy $\mathcal Q$ as a pure orthogonal transformation on the centered point clouds,
\[
z \;\mapsto\; Q\big(z - \mu^{(\cdot)}\big) + \tilde{\mu}^{(\cdot)},
\]
where $\mu^{(\cdot)}$ and $\tilde{\mu}^{(\cdot)}$ denote the empirical means of the training embeddings from the source and target models, respectively, for the modality on which the map is applied. In particular, when deploying on image embeddings we use the training image means $(\mu^{\text{img}}, \tilde{\mu}^{\text{img}})$, and when deploying on text embeddings we use the training text means $(\mu^{\text{text}}, \tilde{\mu}^{\text{text}})$. Centering isolates the rotational relationship by removing this offset while preserving the orthogonal correspondence in the centered space. Note that, we also ablate using pure orthogonal transformation on the raw embedding point clouds and find that mean centering mainly improves \emph{pointwise} cosine agreement, while class-level retrieval and decision geometry remain essentially unchanged even without mean-centering (see~\Cref{sec:app_procrustes_variants}).

\noindent For more details, refer to the pseudocode in~\Cref{alg:procrustes}. 
\begin{algorithm}[!htb]
\caption{Pseudocode for fitting an orthogonal cross-model map $\mathcal R$ (Procrustes)}
\label{alg:procrustes}
\begin{minipage}{\linewidth}
\begin{lstlisting}[language=python,basicstyle=\ttfamily\small]
# f, g: source image/text encoders with outputs of unit-norm
# f_tilde, g_tilde: target image/text encoders with outputs of unit-norm
# pairs: paired samples from ONE modality (images x_i or texts y_i)
# fit_modality: "image" or "text"
# norm(v): l2-normalize v

def fit_procrustes(pairs, fit_modality = 'image'):
    X, X_tilde = [], []
    for s in pairs:
        if fit_modality == "image":
            X.append(f(s))
            X_tilde.append(f_tilde(s))
        else: 
            X.append(g(s))
            X_tilde.append(g_tilde(s))

    X = stack_rows(X)              # (n, d)
    X_tilde = stack_rows(X_tilde)  # (n, d)

    # center for numerical stability during estimation
    Xc = X - mean_row(X) 
    Xtc = X_tilde - mean_row(X_tilde)

    # orthogonal Procrustes
    M = transpose(Xc) @ Xtc
    U, S, Vt = svd(M)
    R = U @ Vt

# deployment (used for BOTH modalities at evaluation time)
def deploy(z, R, mu, mu_tilde):
    return R @ (z-mu) + mu_tilde  

R = fit_procrustes(pairs, fit_modality)
\end{lstlisting}
\end{minipage}
\end{algorithm}

\FloatBarrier
\subsubsection{Performance Metrics at Evaluation}\label{sec:app_metrics}
We evaluate alignment at two granularities: the instance level and the class level. For the former, we report paired-image cosine $\langle \mathcal R f(x),\,\tilde f(x)\rangle$ and paired-text cosine $\langle \mathcal R g(y),\,\tilde g(y)\rangle$ and the correspondingly the Euclidean distance between image embeddings ($R f(x)$ and $\tilde f(x)$) and text embeddings ($R g(y)$ and $\tilde g(y)$), measuring how well a single orthogonal map $\mathcal R$ matches corresponding embeddings across models at the instance level. For the class level alignment, we report:
\begin{itemize}
    \item \textbf{Intra-modal alignment across models.} We measure class-level image–image and text–text top-1 retrieval across models. For image–image, each query is an aligned image embedding $\mathcal R f(x)$, and we predict its class by nearest-neighbor search over the target image embedding set $\{\tilde f(x')\}_{x'}$, counting a hit if the retrieved neighbor shares the same class label. For text-text, each query is an aligned class text embedding $\mathcal R g(y_c)$, and we retrieve within the target class text embedding set $\{\tilde g(y_{c'})\}_{c'}$, counting a hit if the retrieved class equals $c$.
    \item \textbf{Cross-modal transfer across models (task accuracy).}
    We evaluate zero-shot classification under three retrieval settings, each probing a different notion of transfer: (a) aligned images with target text $\hat c=\arg\max_c \langle \mathcal R f(x),\,\tilde g(y_c)\rangle$ measures whether image–text semantics are preserved when only the image space is aligned; (b) target images with aligned text $\hat c=\arg\max_c \langle \tilde f(x),\,\mathcal R g(y_c)\rangle$ measures whether text semantics transfer under alignment without modifying the image space; (c) aligned images with aligned text $\hat c=\arg\max_c \langle \mathcal R f(x),\,\mathcal R g(y_c)\rangle$ measures whether a single map $\mathcal R$ induces a coherent shared space across both modalities.
\end{itemize}

In all, we report five metrics: (1) paired-instance cosine (or euclidean distance), either image $\langle \mathcal R f(x),,\tilde f(x)\rangle$ or text $\langle \mathcal R g(y),,\tilde g(y)\rangle$ depending on the transfer direction; (2) class-level image-image top-1 retrieval or text-text top-1 retrieval; (3) zero-shot accuracy with aligned images and target text, $\hat c=\arg\max_c \langle \mathcal R f(x),,\tilde g(y_c)\rangle$; (4) zero-shot accuracy with target images and aligned text, $\hat c=\arg\max_c \langle \tilde f(x),,\mathcal R g(y_c)\rangle$; and (5) zero-shot accuracy with aligned images and aligned text, $\hat c=\arg\max_c \langle \mathcal R f(x),,\mathcal R g(y_c)\rangle$. Throughout the evaluation, every inner product $\langle \cdot,\cdot\rangle$ is measured through \emph{cosine similarity}, i.e., a dot product between $\ell_2$-normalized embeddings.

\FloatBarrier

\section{Additional Experiments}\label{sec:app_additional_experiments}

\subsection{Modality Gap in Contrastive Models}

\begin{figure*}[!htb]
    \centering
\includegraphics[width=1.\linewidth]{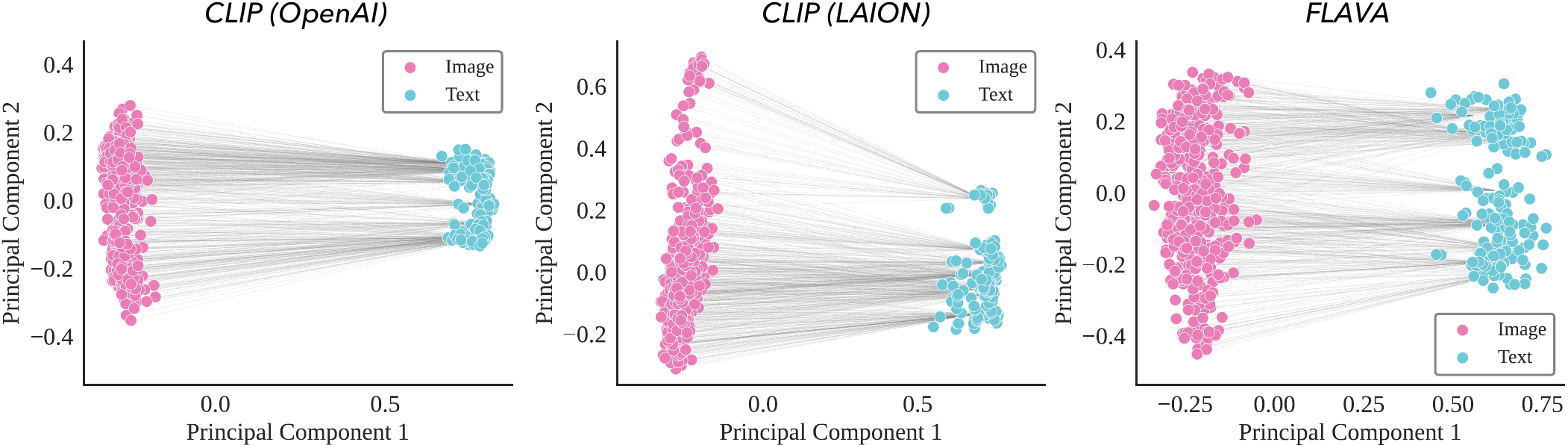}
    \caption{PCA visualization in 2D of generated embeddings of CIFAR-100 from pre-trained models showing pervasive modality gap in multi-modal contrastive representation learning whereby image and text embeddings are located in two completely separate regions of the embedding space.}
    \label{fig:appendix_modality_gap}
\end{figure*}

\Cref{fig:appendix_modality_gap} visualizes the modality gap in contrastive models via a 2D PCA projection of image and text embeddings. Across CLIP (OpenAI), CLIP (LAION), and FLAVA, image and text representations form two well-separated clusters, occupying distinct regions of the embedding space rather than a unified manifold. Gray lines connect matched image-text pairs, highlighting that although the pairs are semantically aligned, their embeddings remain geometrically distant. 

\FloatBarrier

\subsection{Visualization of Multimodal Kernels}
In~\Cref{fig:appendix_kernels}, we visualize the PCA projection of the image/text geometry across models. We observe that while the absolute coordinates of the embedding cones shift arbitrarily between models, the angular arrangement of the texts with respect to the images remains consistent. Mathematically, this means that the \emph{multimodal kernels} are approximately preserved across models: $\langle f, g \rangle \approx \langle \tilde{f}, \tilde{g} \rangle$. This observation can be viewed as a multimodal analogue of the Platonic Representation Hypothesis~\citep{huh2024prh} that posits that models converge to similar \emph{unimodal} kernels ($\langle f, f \rangle \approx \langle \tilde{f}, \tilde{f} \rangle$).

\begin{figure*}[!htb]
    \centering
\includegraphics[width=1.\linewidth]{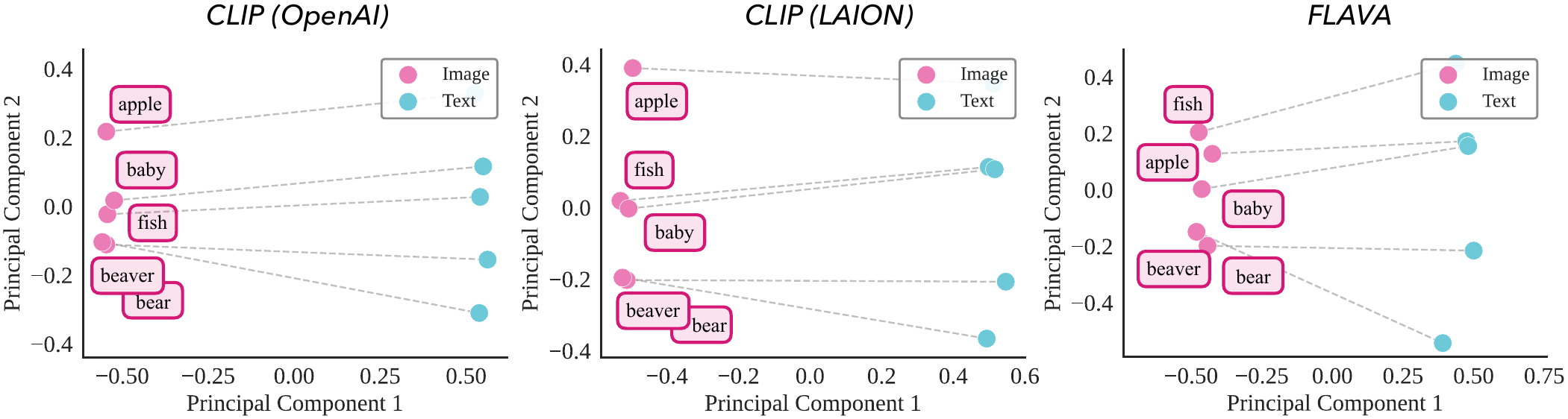}
    \caption{Across CLIP variants, the multimodal kernel $\langle f(x), g(y)\rangle$ (relative angles between image and text embeddings) is strongly preserved (dashed lines), unlike the unimodal kernel $\langle f(x), \tilde f(x')\rangle$;\looseness=-1}
    \label{fig:appendix_kernels}
\end{figure*}

\FloatBarrier

\subsection{Independently Trained Contrastive Models Differ by an Orthogonal Map That Is Shared Across Modalities}\label{sec:app_all_metrics}
In this section, we report complete results across five benchmarks—Oxford Pets (\Cref{fig:mainresults_oxford}), Caltech-101 (\Cref{fig:mainresults_caltech101}), CIFAR-100 (\Cref{fig:mainresults_cifar100}), STL10 (\Cref{fig:mainresults_stl10}), and DTD (\Cref{fig:mainresults_dtd})—covering all five evaluation metrics used throughout the paper and explained in~\Cref{sec:app_metrics}. Two consistent conclusions emerge across datasets and model pairs. First, a single orthogonal map $\mathcal Q$ accurately aligns image embeddings between independently trained cpmtrastive models, yielding near-oracle image-image agreement. Second, the \emph{same} map transfers across modalities: applying $\mathcal Q$ learned from images substantially improves text-text pointwise alignment, cross-model retrieval, and zero-shot classification, without degrading image performance. Together, these results confirm that independently trained contrastive models are related by a shared, modality-invariant orthogonal reparameterization.

\begin{figure*}[!htb]
    \centering
    \begin{subfigure}{0.32\textwidth}
    \centering
    \includegraphics[width=\textwidth]{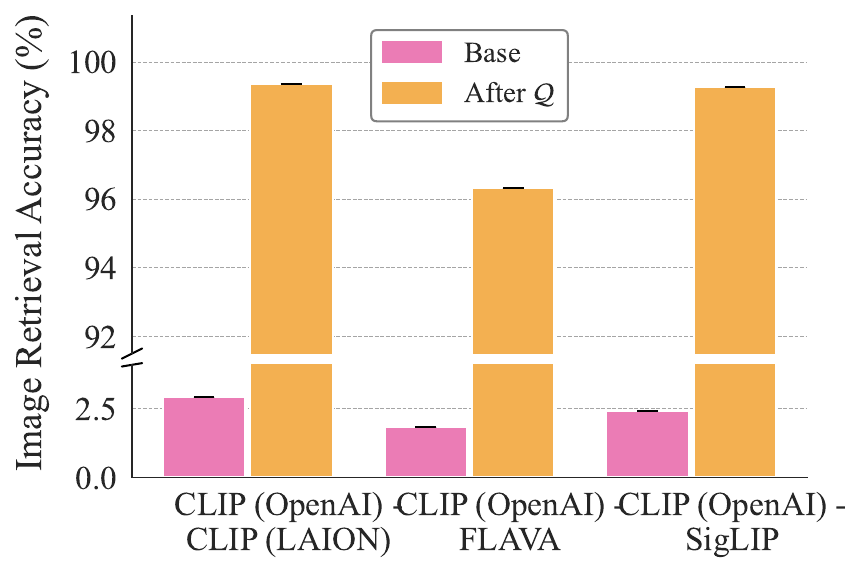}
    \caption{}
    \end{subfigure}
    \begin{subfigure}{0.32\textwidth}
    \centering
    \includegraphics[width=\textwidth]{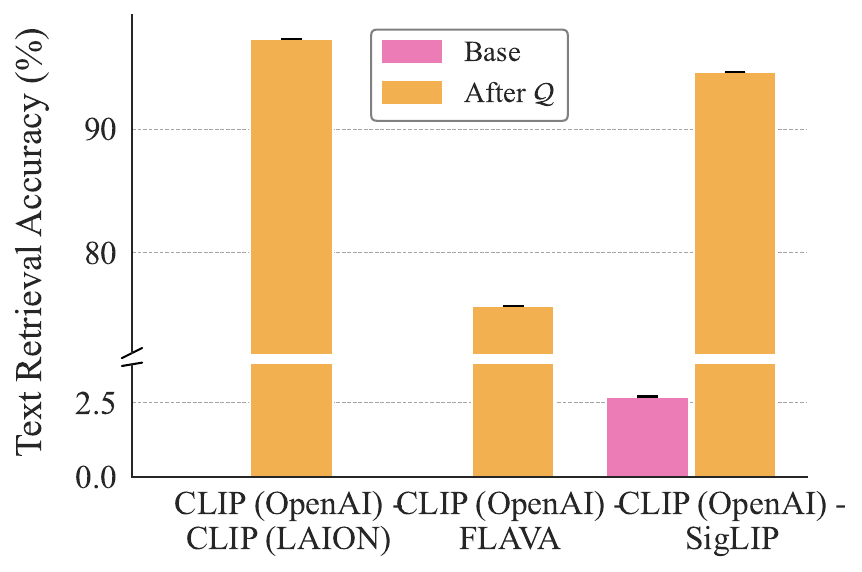}
    \caption{}
    \end{subfigure}
    \begin{subfigure}{0.32\textwidth}
    \centering
    \includegraphics[width=\textwidth]{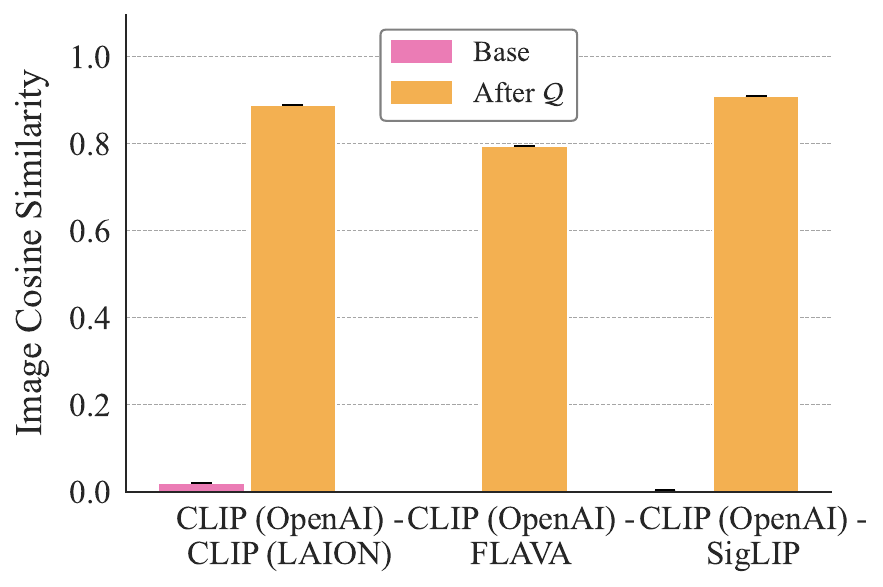}
    \caption{}
    \end{subfigure}
    \begin{subfigure}{0.32\textwidth}
    \centering
    \includegraphics[width=\textwidth]{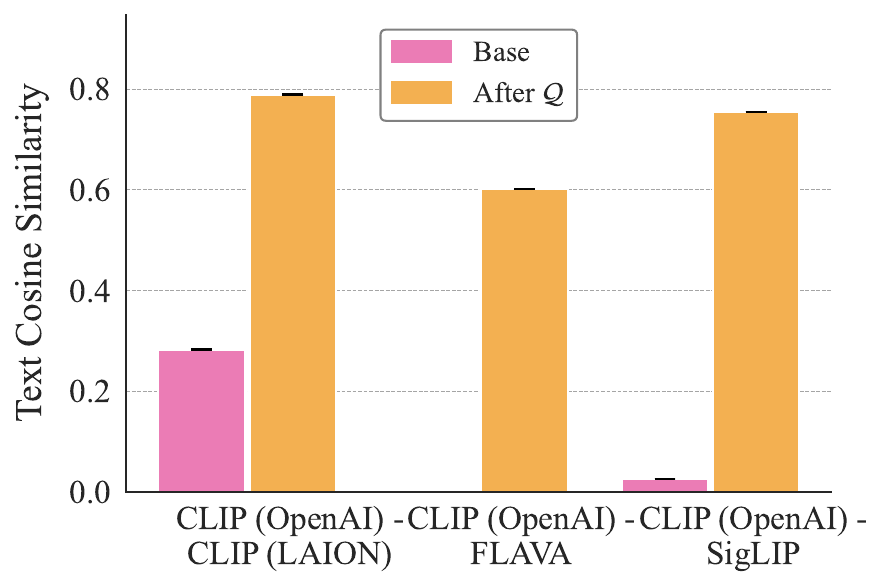}
    \caption{}
    \end{subfigure}
    \begin{subfigure}{0.32\textwidth}
    \centering
    \includegraphics[width=\textwidth]{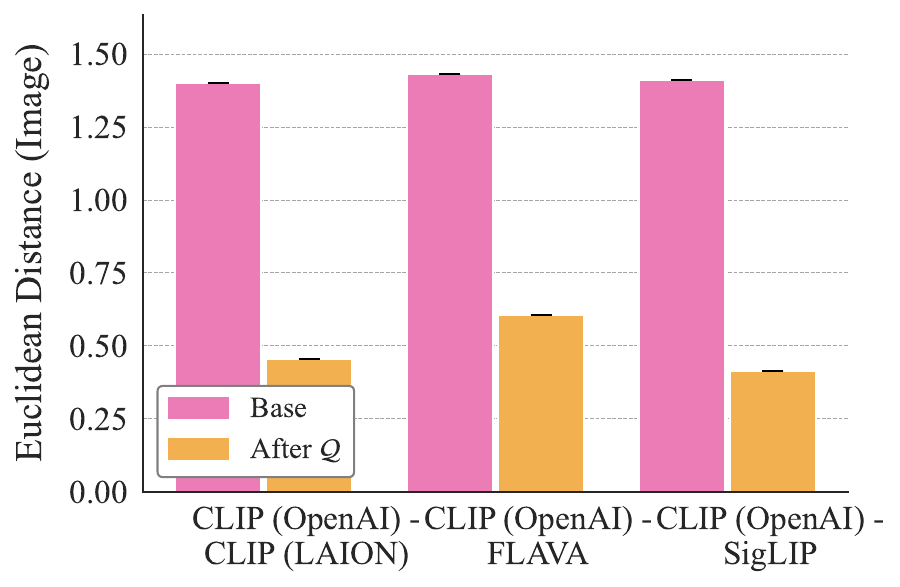}
    \caption{}
    \end{subfigure}
    \begin{subfigure}{0.32\textwidth}
    \centering
    \includegraphics[width=\textwidth]{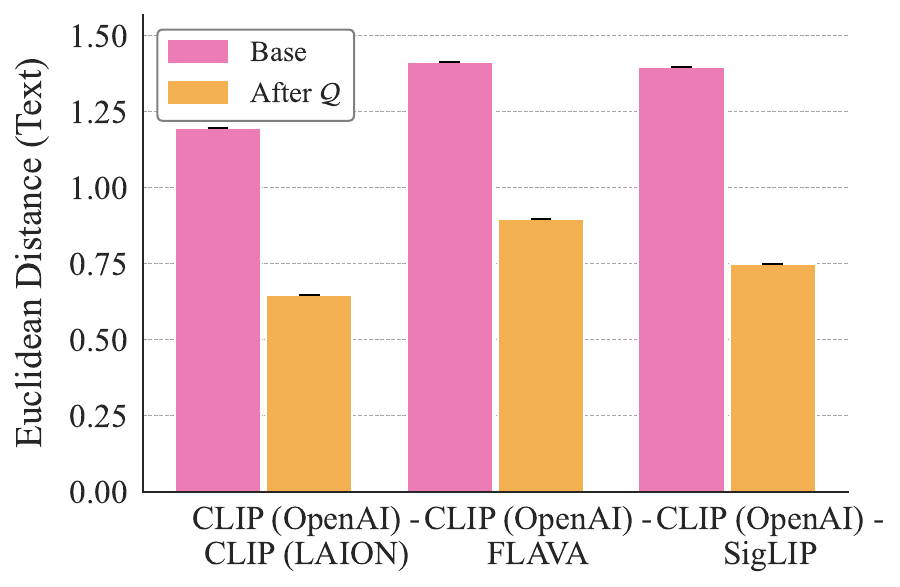}
    \caption{}
    \end{subfigure}
    \begin{subfigure}{0.32\textwidth}
    \centering
    \includegraphics[width=\textwidth]{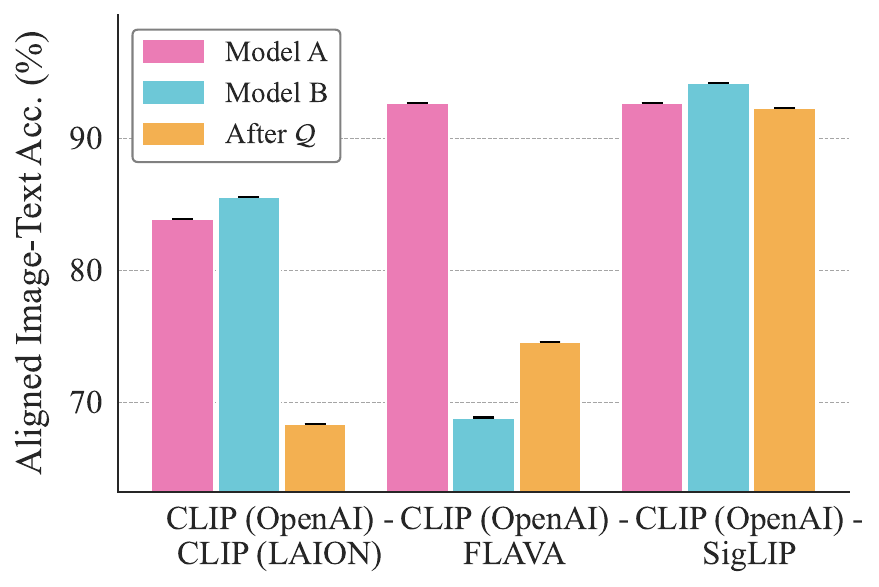}
    \caption{}
    \end{subfigure}
    \begin{subfigure}{0.32\textwidth}
    \centering
    \includegraphics[width=\textwidth]{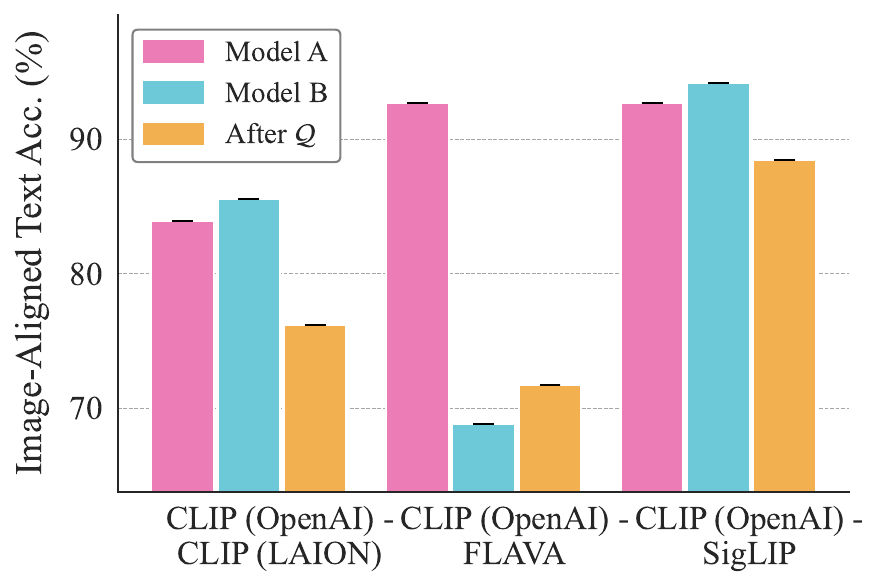}
    \caption{}
    \end{subfigure}
    \begin{subfigure}{0.32\textwidth}
    \centering
    \includegraphics[width=\textwidth]{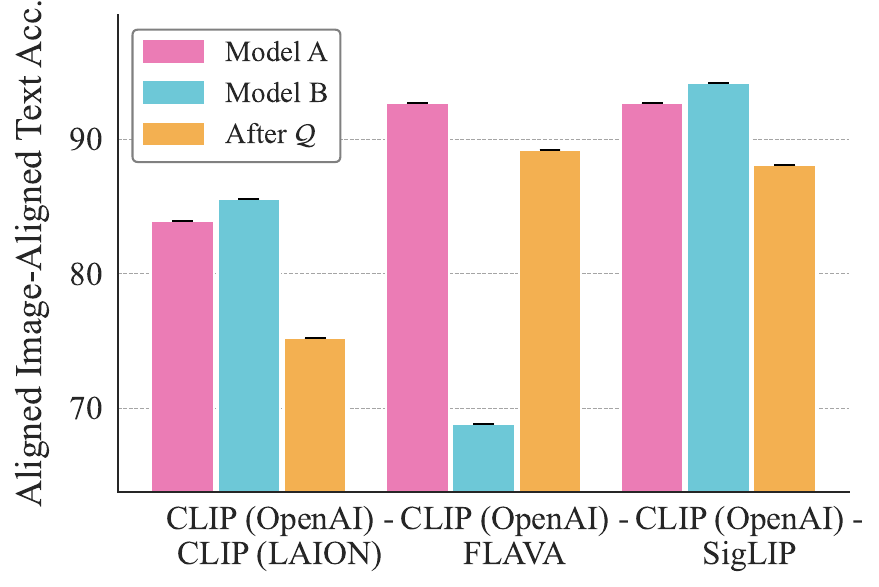}
    \caption{}
    \end{subfigure}
    \caption{\textit{Cross-model alignment on Oxford Pets across independently trained contrastive models before and after fitting a single orthogonal map $\mathcal R$.} (a) Image–image class retrieval and (b) text–text class retrieval (c) Mean image–image cosine similarity. (d) Mean text–text cosine similarity. (e) Euclidean distance (without normalization) between image embeddings. (f) Euclidean distance (without normalization) between text embeddings. (g) Image–text retrieval using aligned images from model A and text from model B. (h) Image–text retrieval using images from model B and aligned text from model A. (i) Image–text retrieval using aligned images and aligned text from model A. $\mathcal Q$ aligns images across models with a single orthogonal map, and the same $\mathcal Q$ learned only from image embeddings transfers to text, boosting text-text retrieval from near-chance to near-oracle, all while preserving strong image classification accuracy.}
    
    \label{fig:mainresults_oxford}
\end{figure*}

\begin{figure*}[!htb]
    \centering
    \begin{subfigure}{0.32\textwidth}
    \centering
    \includegraphics[width=\textwidth]{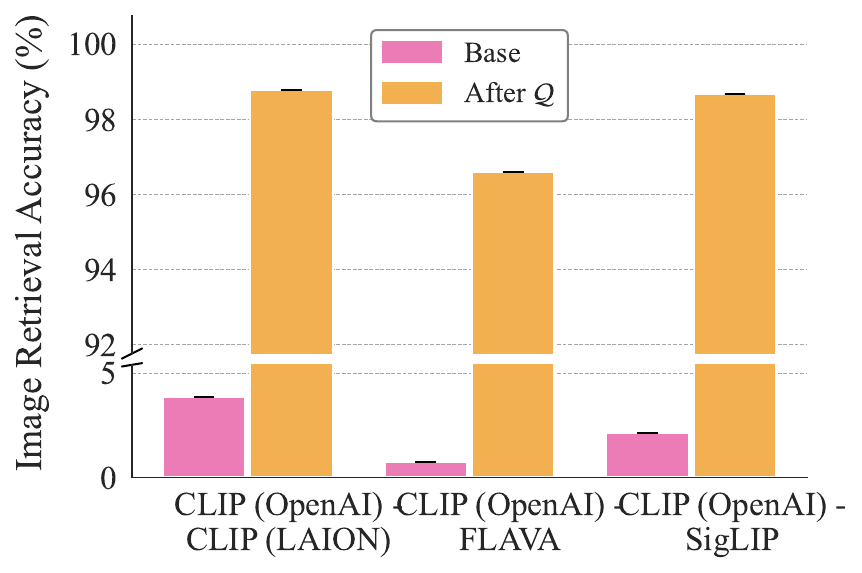}
    \caption{}
    \end{subfigure}
    \begin{subfigure}{0.32\textwidth}
    \centering
    \includegraphics[width=\textwidth]{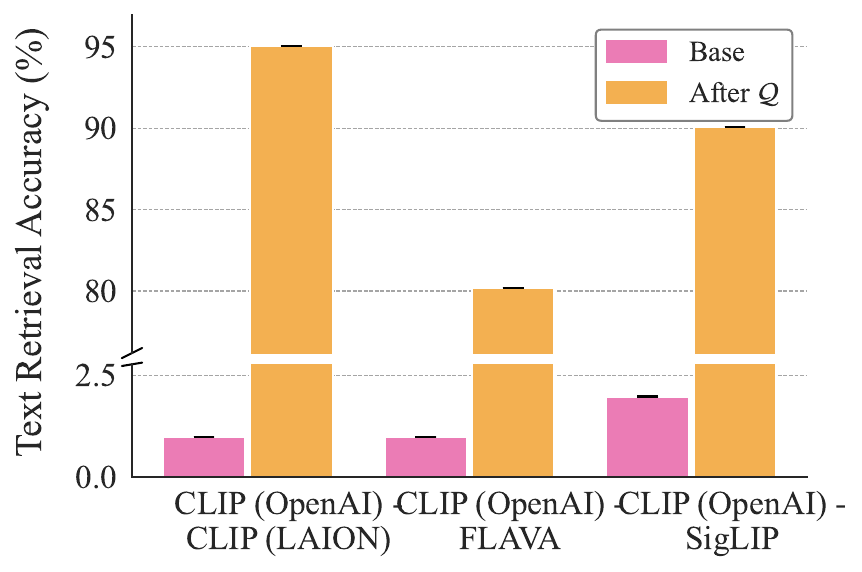}
    \caption{}
    \end{subfigure}
    \begin{subfigure}{0.32\textwidth}
    \centering
    \includegraphics[width=\textwidth]{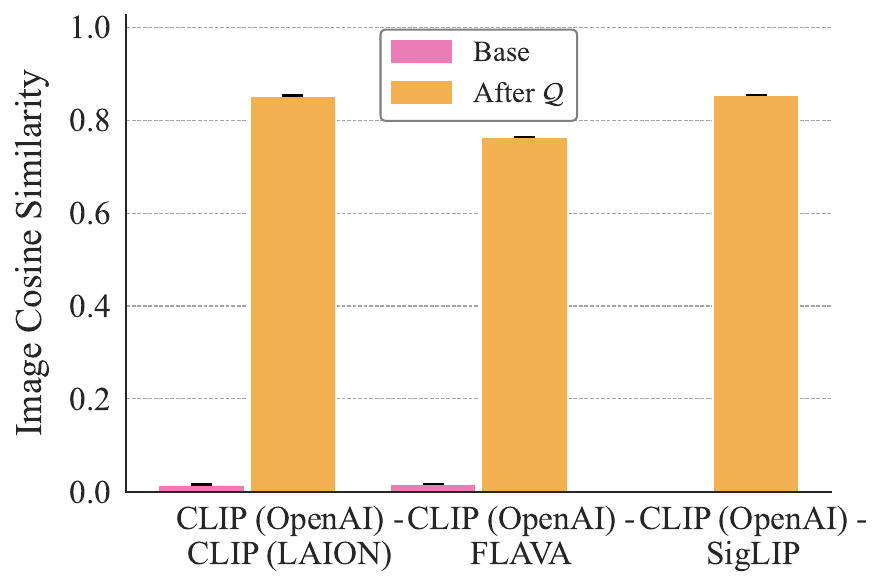}
    \caption{}
    \end{subfigure}
    \begin{subfigure}{0.32\textwidth}
    \centering
    \includegraphics[width=\textwidth]{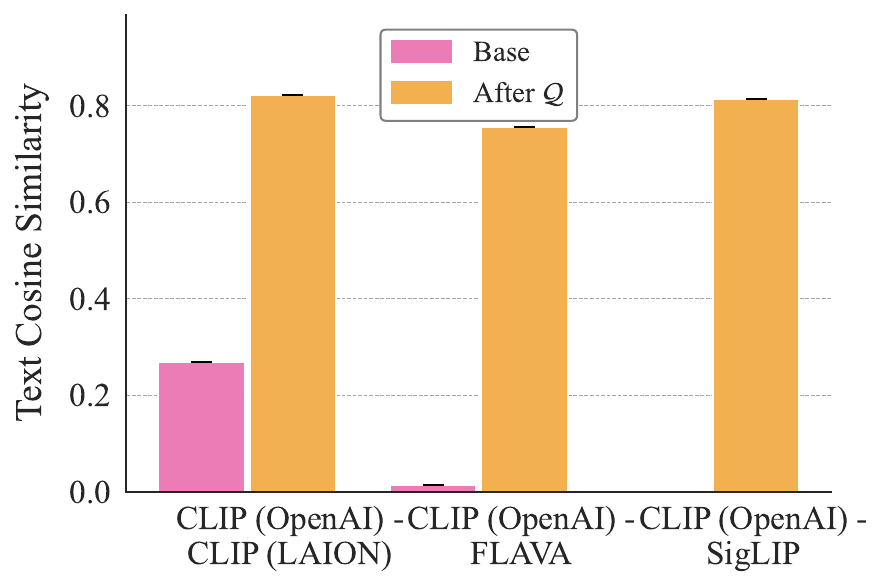}
    \caption{}
    \end{subfigure}
    \begin{subfigure}{0.32\textwidth}
    \centering
    \includegraphics[width=\textwidth]{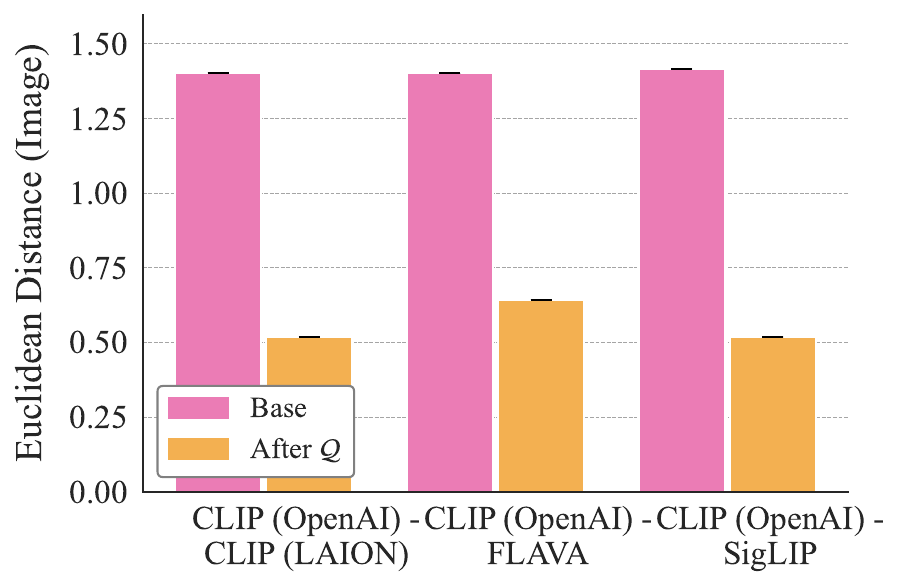}
    \caption{}
    \end{subfigure}
    \begin{subfigure}{0.32\textwidth}
    \centering
    \includegraphics[width=\textwidth]{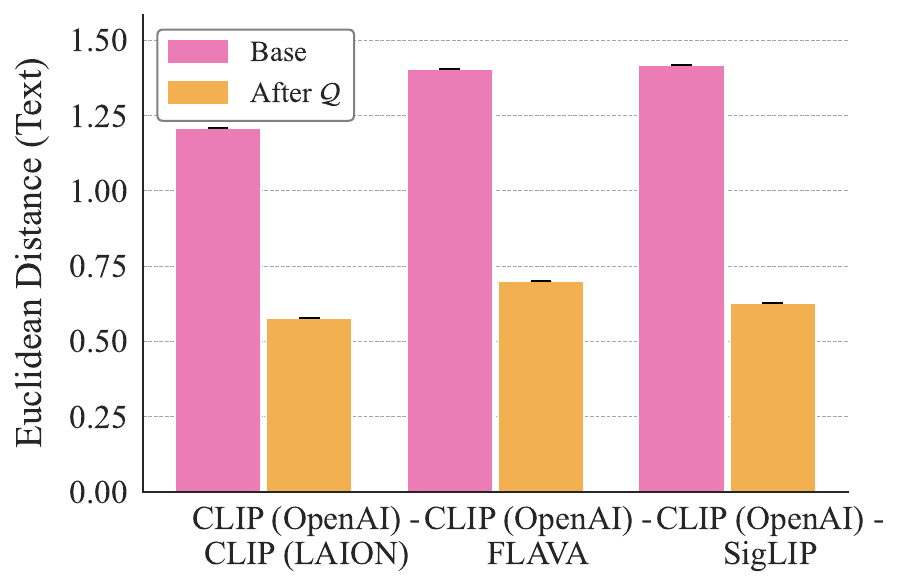}
    \caption{}
    \end{subfigure}
    \begin{subfigure}{0.32\textwidth}
    \centering
    \includegraphics[width=\textwidth]{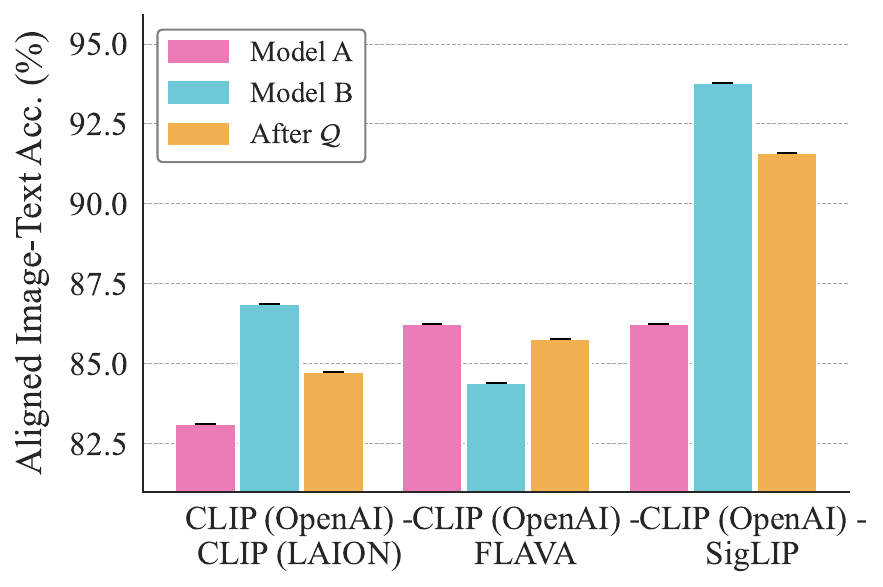}
    \caption{}
    \end{subfigure}
    \begin{subfigure}{0.32\textwidth}
    \centering
    \includegraphics[width=\textwidth]{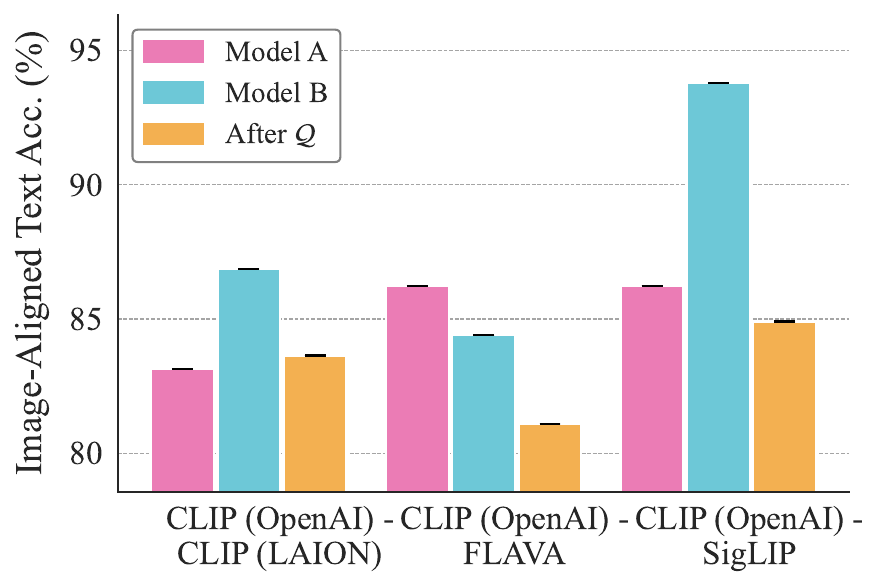}
    \caption{}
    \end{subfigure}
    \begin{subfigure}{0.32\textwidth}
    \centering
    \includegraphics[width=\textwidth]{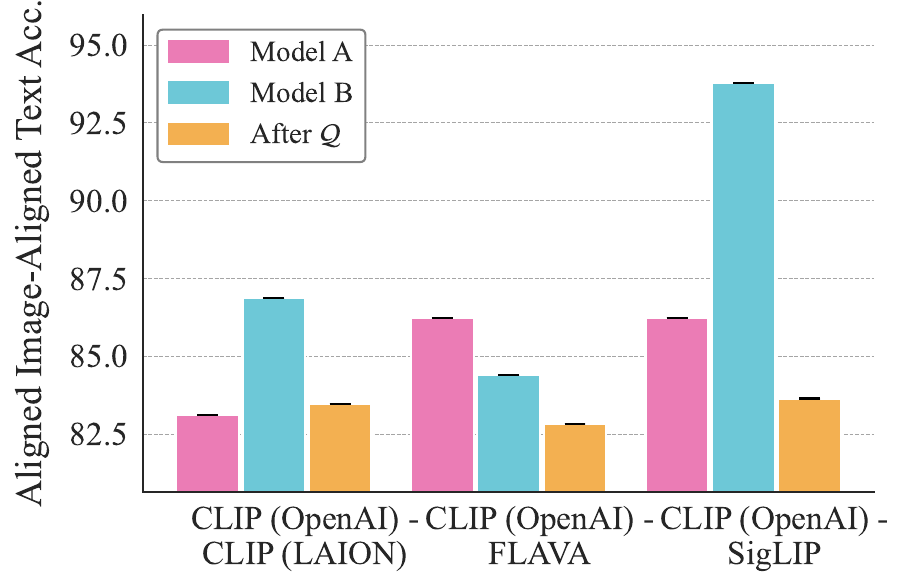}
    \caption{}
    \end{subfigure}
    \caption{\textit{Cross-model alignment on Caltech-101 across independently trained contrastive models before and after fitting a single orthogonal map $\mathcal R$.} (a) Image–image class retrieval and (b) text–text class retrieval (c) Mean image–image cosine similarity. (d) Mean text–text cosine similarity. (e) Euclidean distance (without normalization) between image embeddings. (f) Euclidean distance (without normalization) between text embeddings. (g) Image–text retrieval using aligned images from model A and text from model B. (h) Image–text retrieval using images from model B and aligned text from model A. (i) Image–text retrieval using aligned images and aligned text from model A. $\mathcal Q$ aligns images across models with a single orthogonal map, and the same $\mathcal Q$ learned only from image embeddings transfers to text, boosting text-text retrieval from near-chance to near-oracle, all while preserving strong image classification accuracy.}
    \label{fig:mainresults_caltech101}
\end{figure*}

\begin{figure*}[!htb]
    \centering
    \begin{subfigure}{0.32\textwidth}
    \centering
    \includegraphics[width=\textwidth]{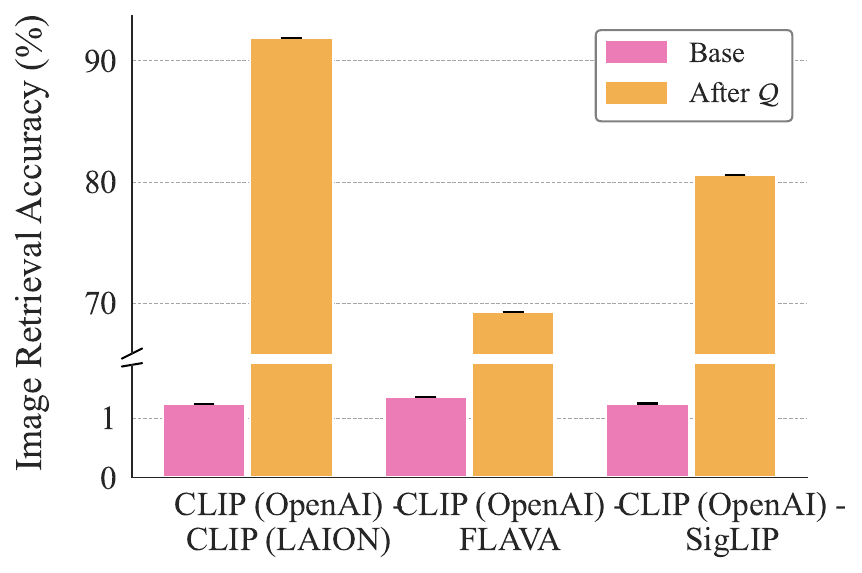}
    \caption{}
    \end{subfigure}
    \begin{subfigure}{0.32\textwidth}
    \centering
    \includegraphics[width=\textwidth]{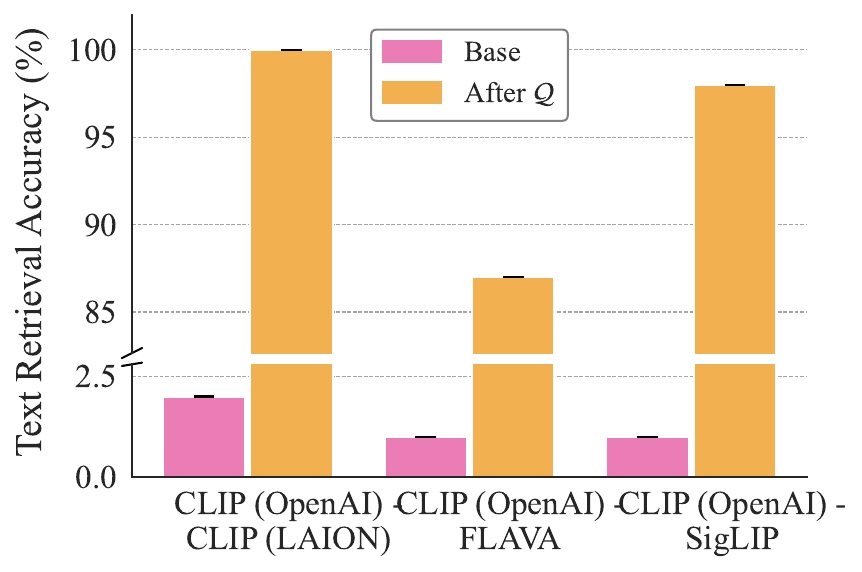}
    \caption{}
    \end{subfigure}
    \begin{subfigure}{0.32\textwidth}
    \centering
    \includegraphics[width=\textwidth]{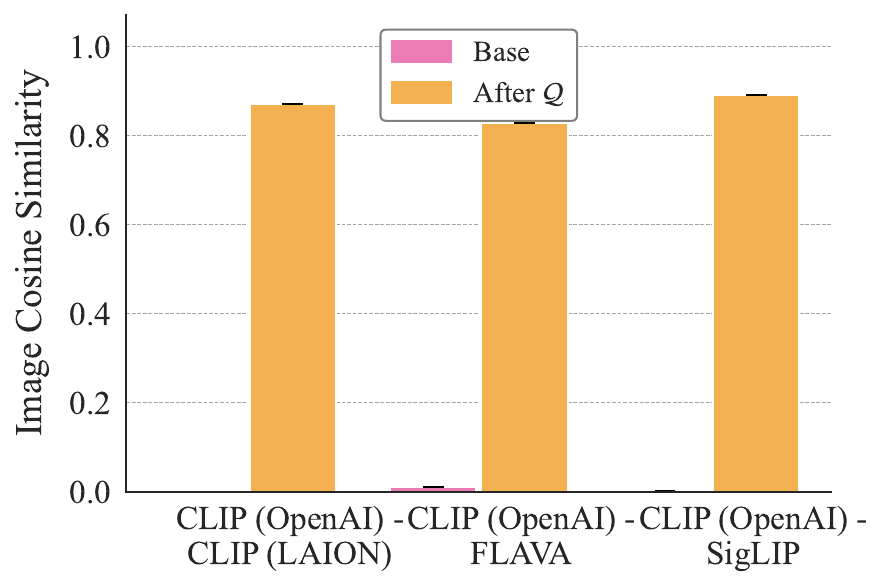}
    \caption{}
    \end{subfigure}
    \begin{subfigure}{0.32\textwidth}
    \centering
    \includegraphics[width=\textwidth]{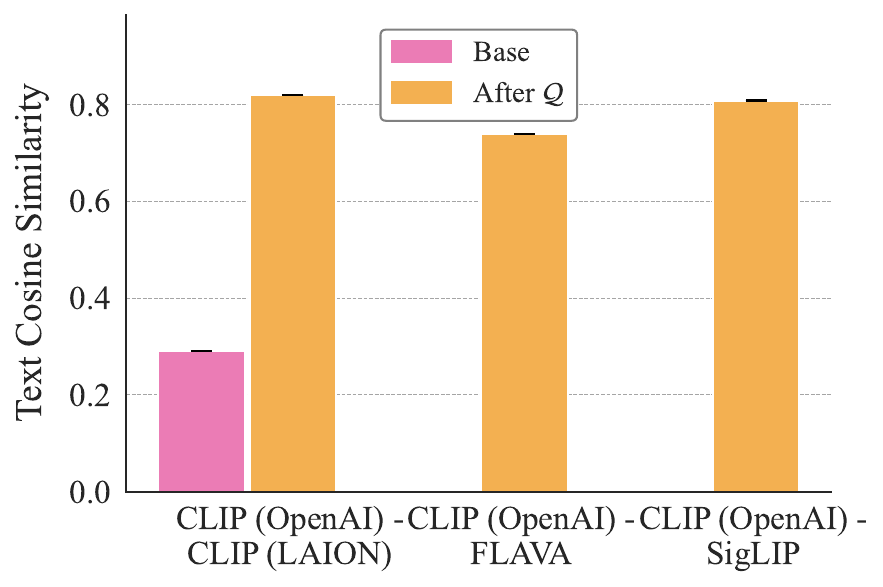}
    \caption{}
    \end{subfigure}
    \begin{subfigure}{0.32\textwidth}
    \centering
    \includegraphics[width=\textwidth]{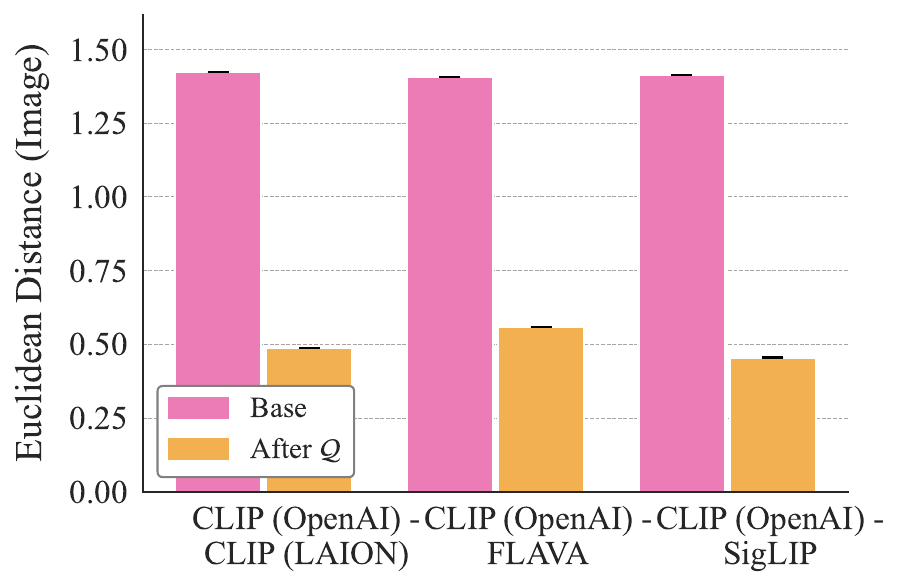}
    \caption{}
    \end{subfigure}
    \begin{subfigure}{0.32\textwidth}
    \centering
    \includegraphics[width=\textwidth]{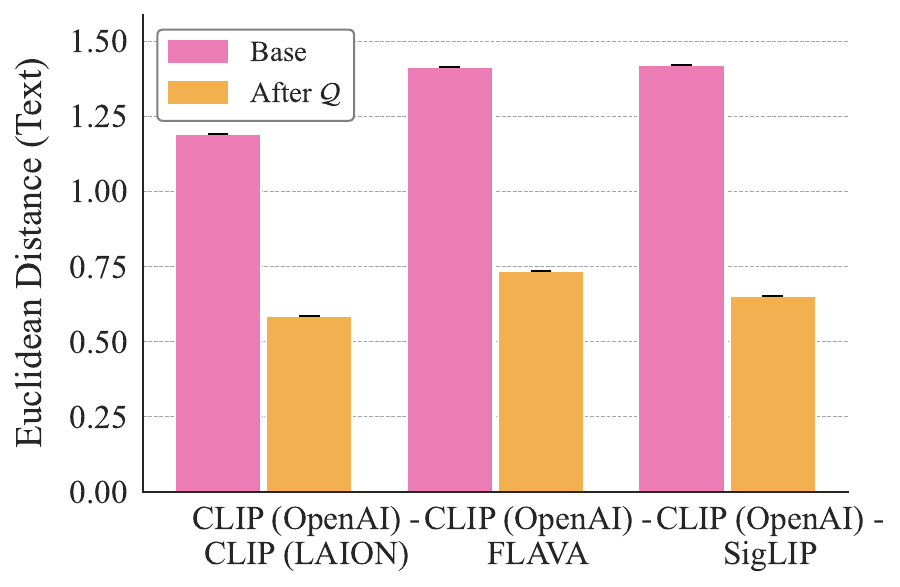}
    \caption{}
    \end{subfigure}
    \begin{subfigure}{0.32\textwidth}
    \centering
    \includegraphics[width=\textwidth]{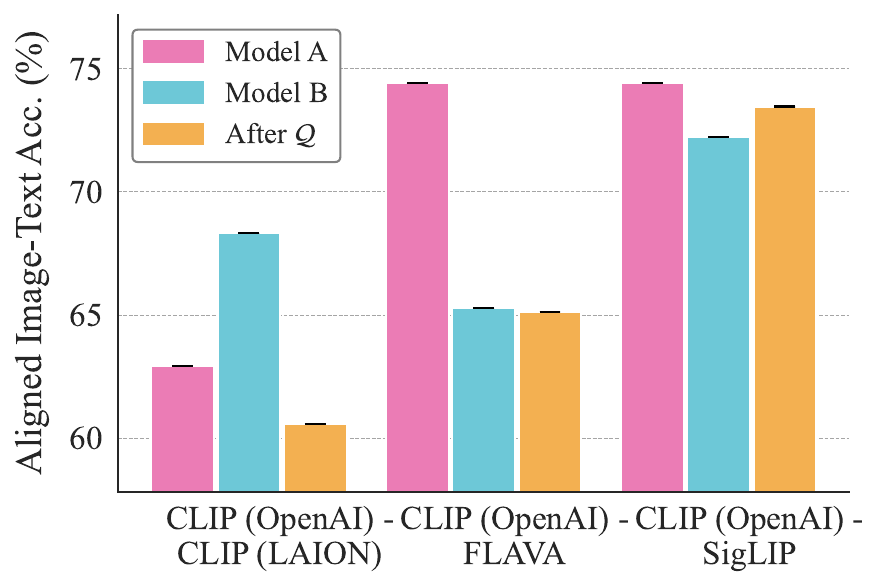}
    \caption{}
    \end{subfigure}
    \begin{subfigure}{0.32\textwidth}
    \centering
    \includegraphics[width=\textwidth]{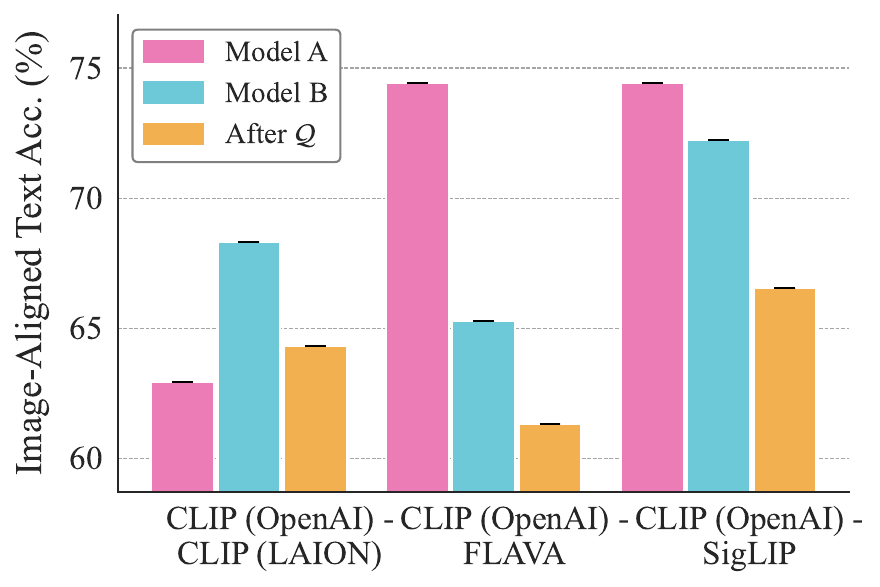}
    \caption{}
    \end{subfigure}
    \begin{subfigure}{0.32\textwidth}
    \centering
    \includegraphics[width=\textwidth]{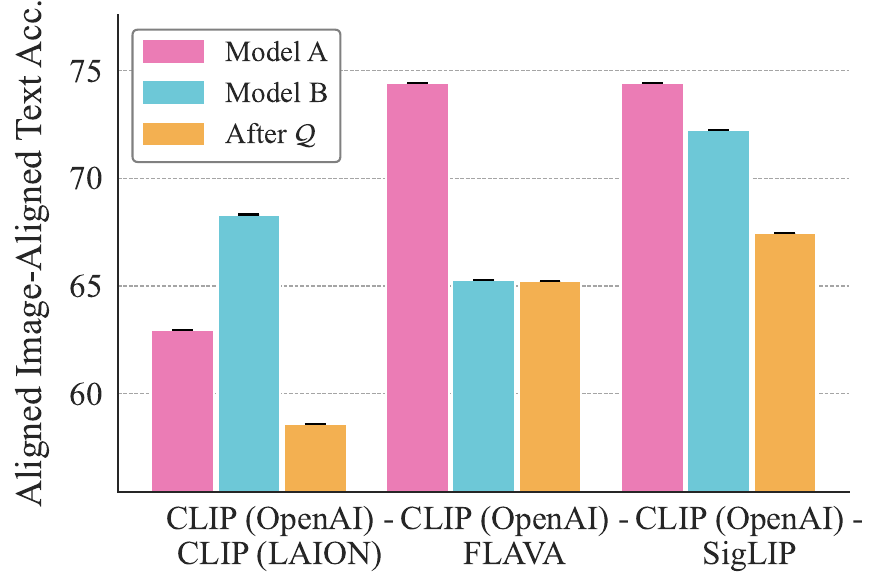}
    \caption{}
    \end{subfigure}
    \caption{\textit{Cross-model alignment on CIFAR-100 across independently trained contrastive models before and after fitting a single orthogonal map $\mathcal R$.} (a) Image–image class retrieval and (b) text–text class retrieval (c) Mean image–image cosine similarity. (d) Mean text–text cosine similarity. (e) Euclidean distance (without normalization) between image embeddings. (f) Euclidean distance (without normalization) between text embeddings. (g) Image–text retrieval using aligned images from model A and text from model B. (h) Image–text retrieval using images from model B and aligned text from model A. (i) Image–text retrieval using aligned images and aligned text from model A. $\mathcal Q$ aligns images across models with a single orthogonal map, and the same $\mathcal Q$ learned only from image embeddings transfers to text, boosting text-text retrieval from near-chance to near-oracle, all while preserving strong image classification accuracy.}
    \label{fig:mainresults_cifar100}
\end{figure*}

\begin{figure*}[!htb]
    \centering
    \begin{subfigure}{0.32\textwidth}
    \centering
    \includegraphics[width=\textwidth]{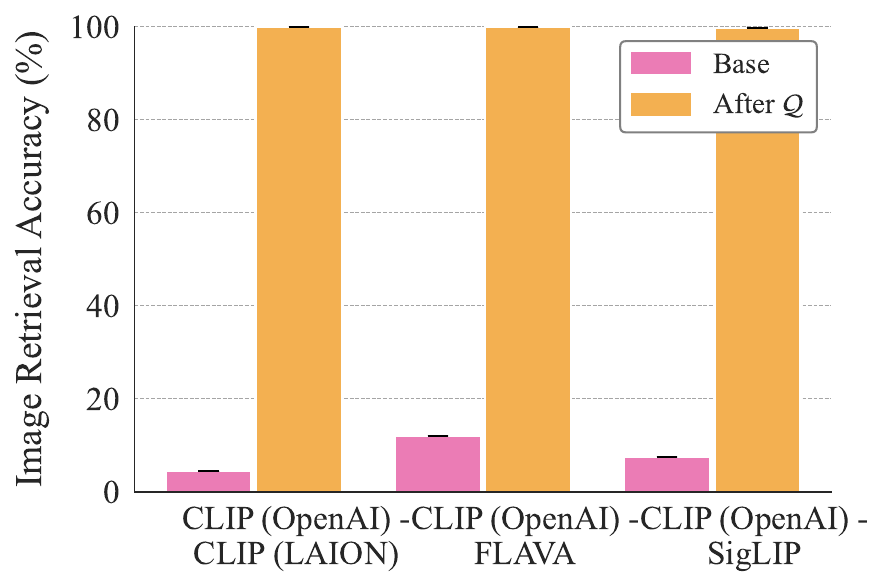}
    \caption{}
    \end{subfigure}
    \begin{subfigure}{0.32\textwidth}
    \centering
    \includegraphics[width=\textwidth]{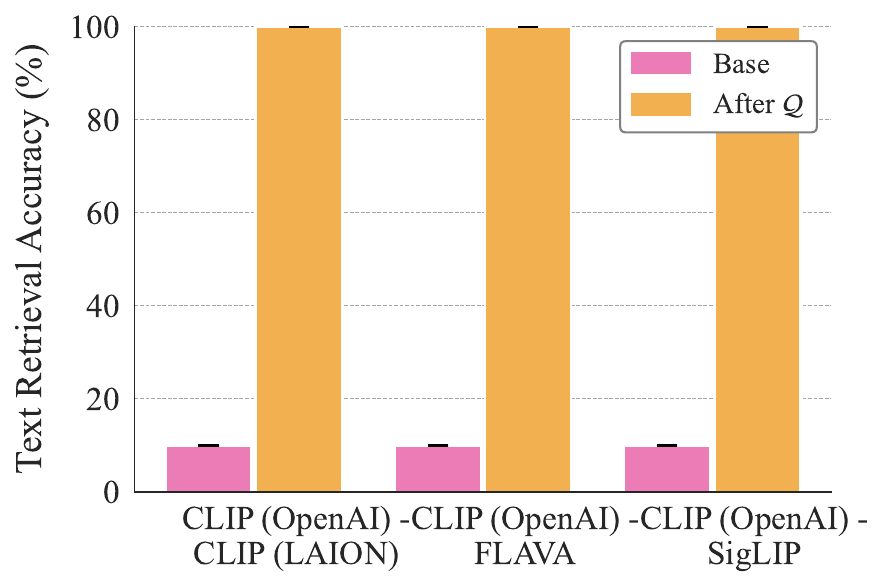}
    \caption{}
    \end{subfigure}
    \begin{subfigure}{0.32\textwidth}
    \centering
    \includegraphics[width=\textwidth]{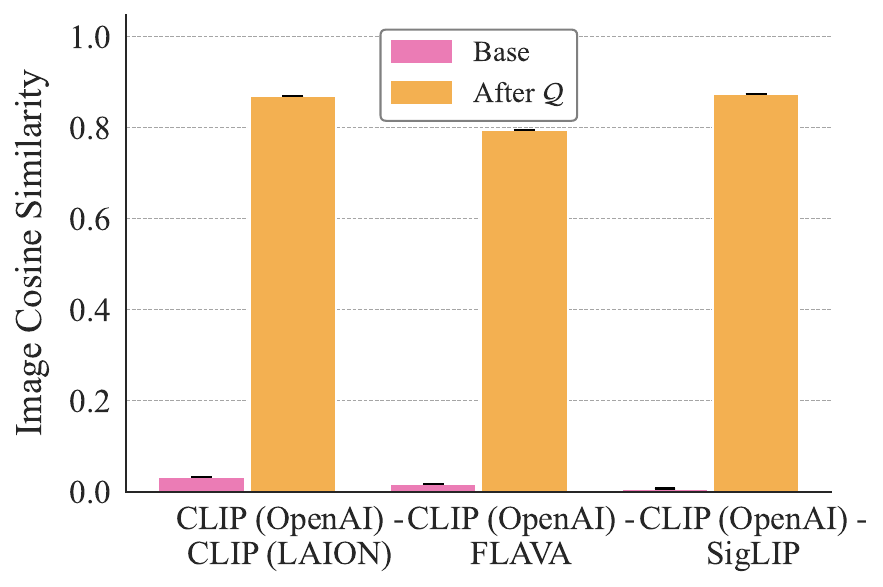}
    \caption{}
    \end{subfigure}
    \begin{subfigure}{0.32\textwidth}
    \centering
    \includegraphics[width=\textwidth]{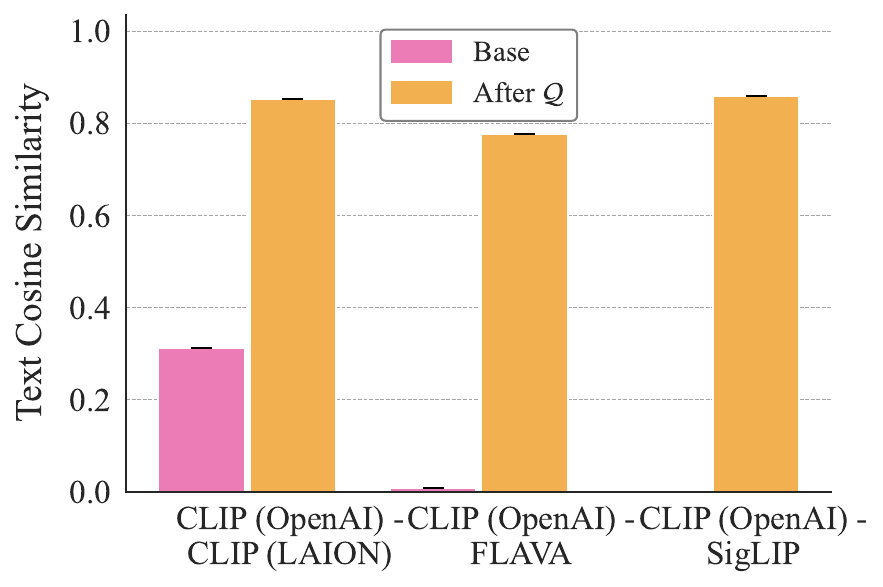}
    \caption{}
    \end{subfigure}
    \begin{subfigure}{0.32\textwidth}
    \centering
    \includegraphics[width=\textwidth]{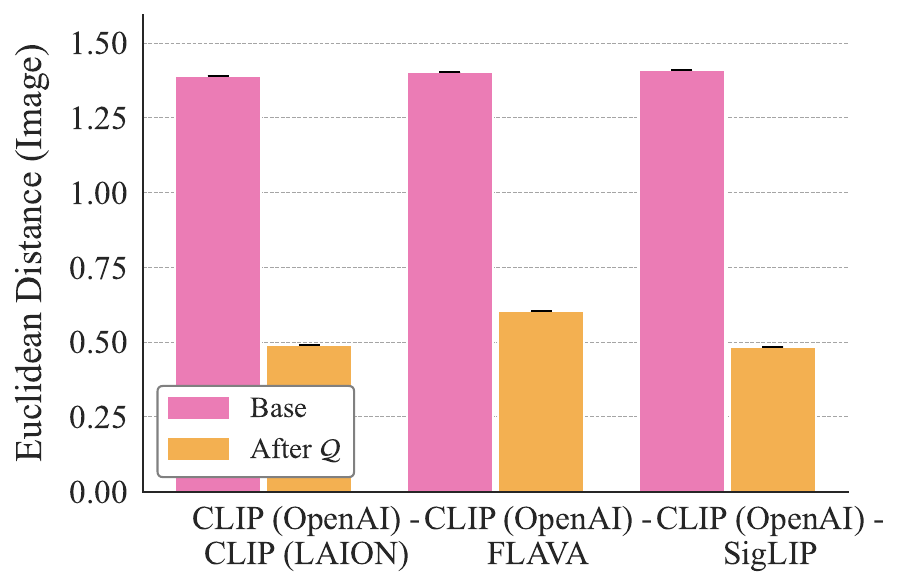}
    \caption{}
    \end{subfigure}
    \begin{subfigure}{0.32\textwidth}
    \centering
    \includegraphics[width=\textwidth]{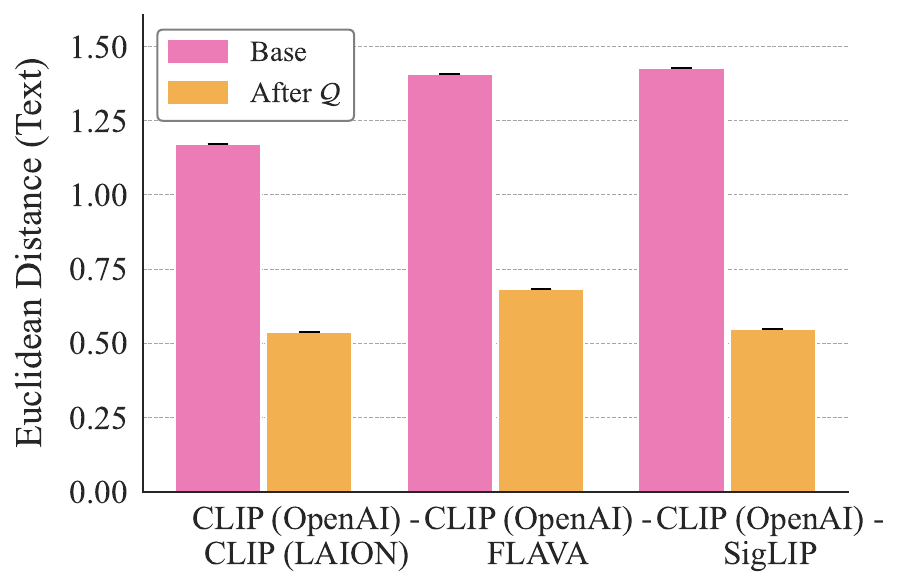}
    \caption{}
    \end{subfigure}
    \begin{subfigure}{0.32\textwidth}
    \centering
    \includegraphics[width=\textwidth]{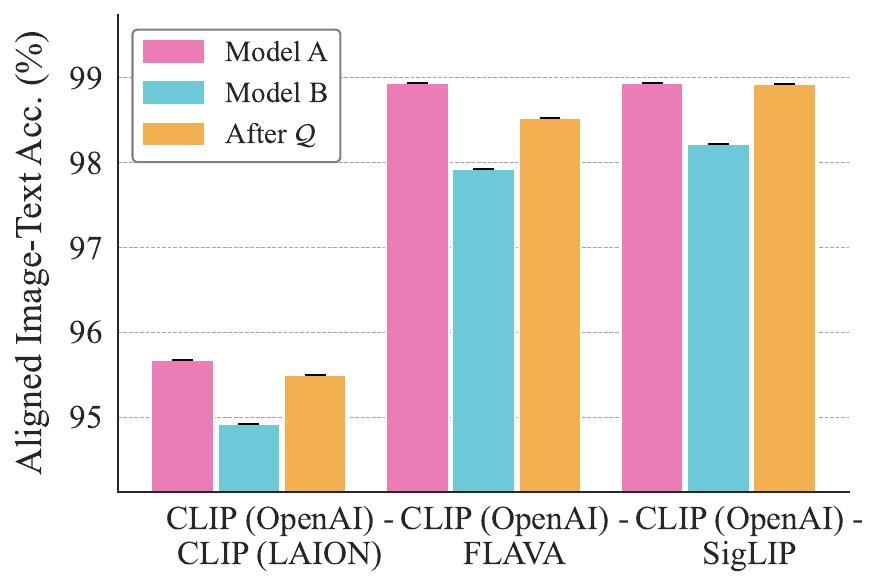}
    \caption{}
    \end{subfigure}
    \begin{subfigure}{0.32\textwidth}
    \centering
    \includegraphics[width=\textwidth]{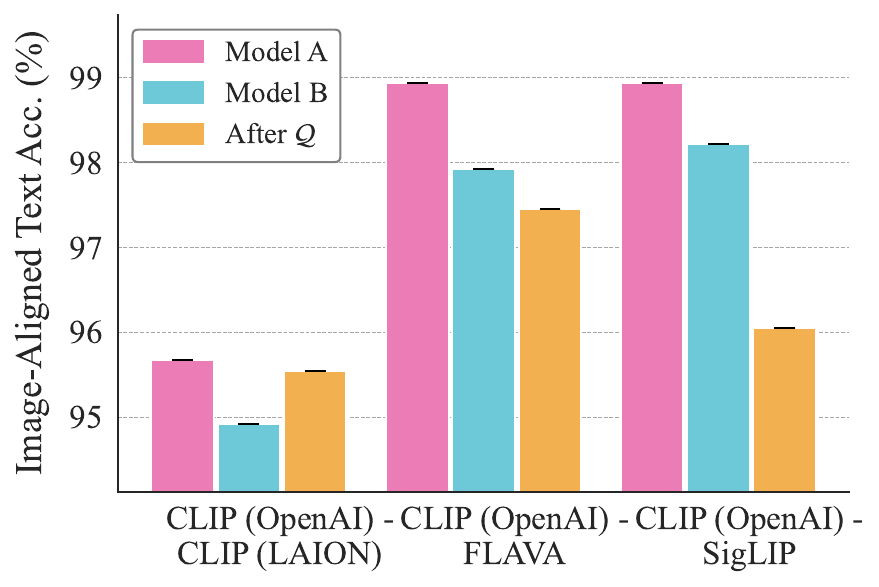}
    \caption{}
    \end{subfigure}
    \begin{subfigure}{0.32\textwidth}
    \centering
    \includegraphics[width=\textwidth]{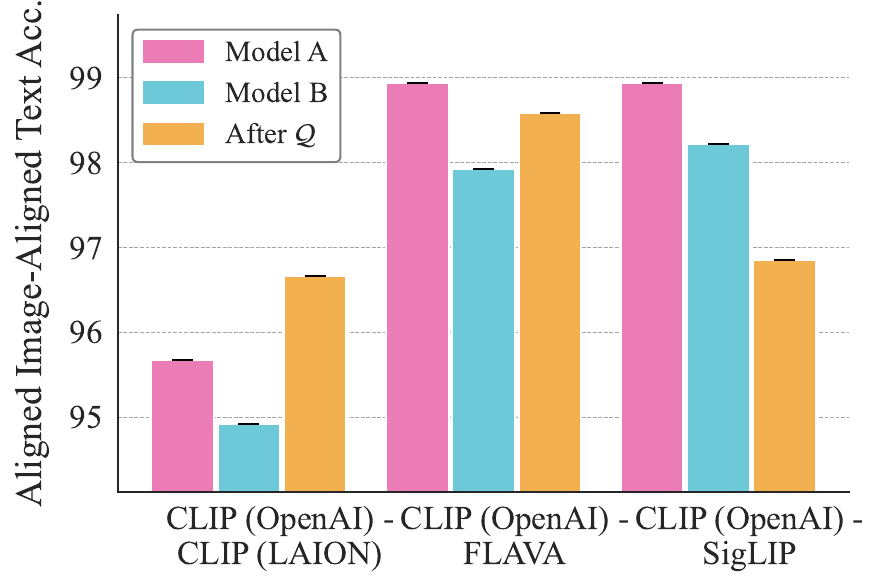}
    \caption{}
    \end{subfigure}
    \caption{\textit{Cross-model alignment on STL10 across independently trained contrastive models before and after fitting a single orthogonal map $\mathcal R$.} (a) Image–image class retrieval and (b) text–text class retrieval (c) Mean image–image cosine similarity. (d) Mean text–text cosine similarity. (e) Euclidean distance (without normalization) between image embeddings. (f) Euclidean distance (without normalization) between text embeddings. (g) Image–text retrieval using aligned images from model A and text from model B. (h) Image–text retrieval using images from model B and aligned text from model A. (i) Image–text retrieval using aligned images and aligned text from model A. $\mathcal Q$ aligns images across models with a single orthogonal map, and the same $\mathcal Q$ learned only from image embeddings transfers to text, boosting text-text retrieval from near-chance to near-oracle, all while preserving strong image classification accuracy.}
    \label{fig:mainresults_stl10}
\end{figure*}

\begin{figure*}[!htb]
    \centering
    \begin{subfigure}{0.32\textwidth}
    \centering
    \includegraphics[width=\textwidth]{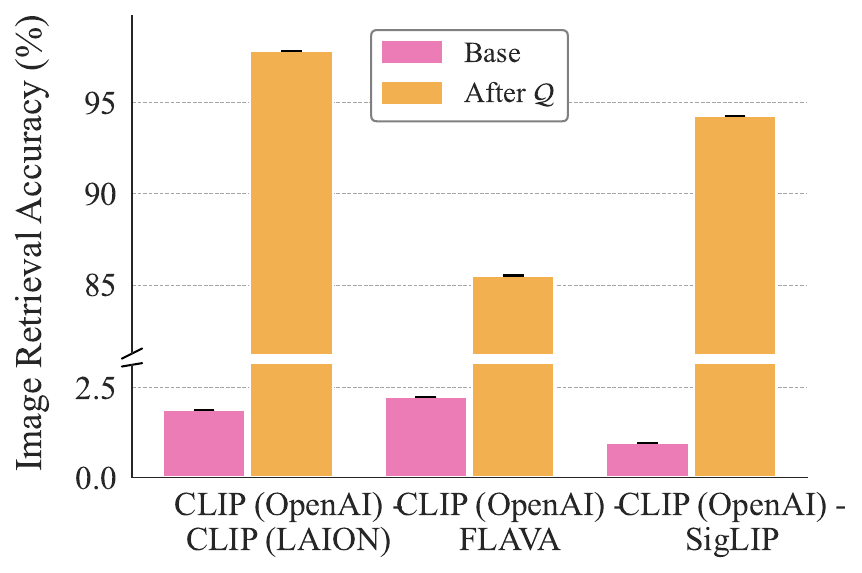}
    \caption{}
    \end{subfigure}
    \begin{subfigure}{0.32\textwidth}
    \centering
    \includegraphics[width=\textwidth]{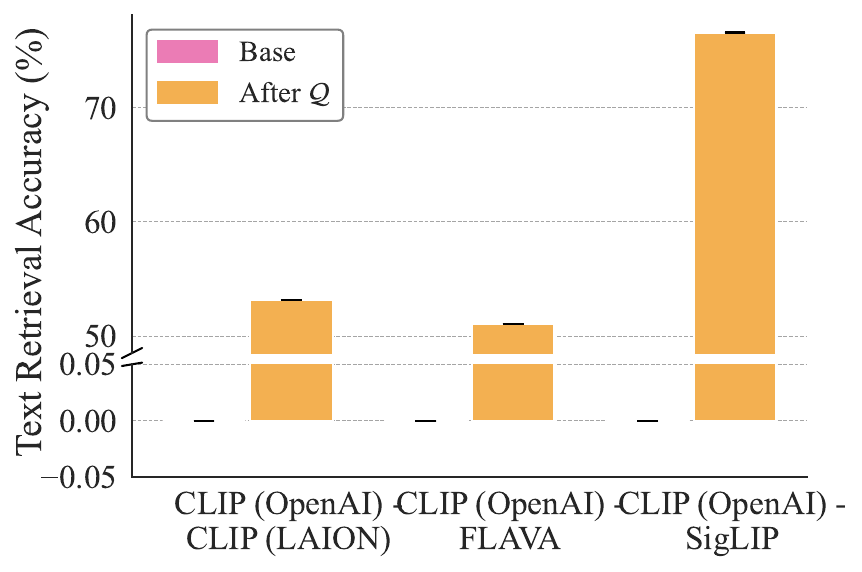}
    \caption{}
    \end{subfigure}
    \begin{subfigure}{0.32\textwidth}
    \centering
    \includegraphics[width=\textwidth]{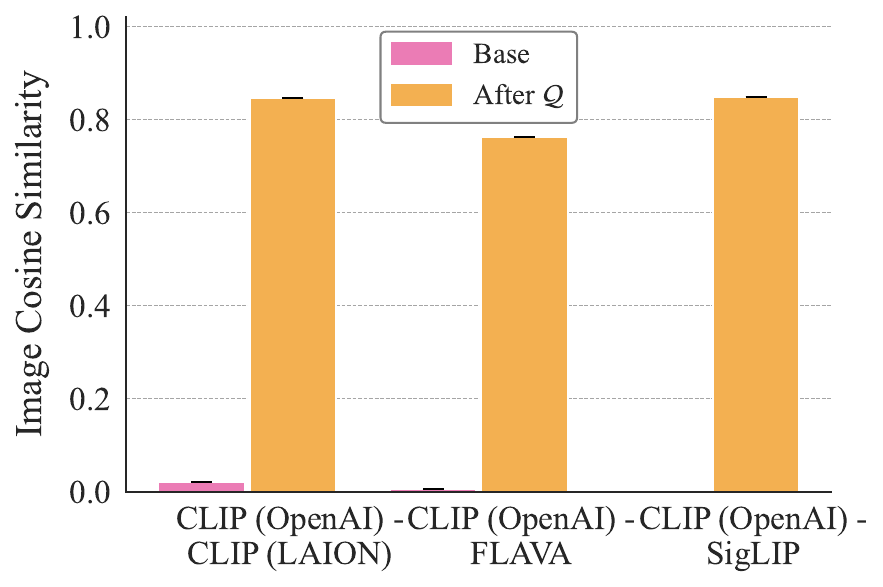}
    \caption{}
    \end{subfigure}
    \begin{subfigure}{0.32\textwidth}
    \centering
    \includegraphics[width=\textwidth]{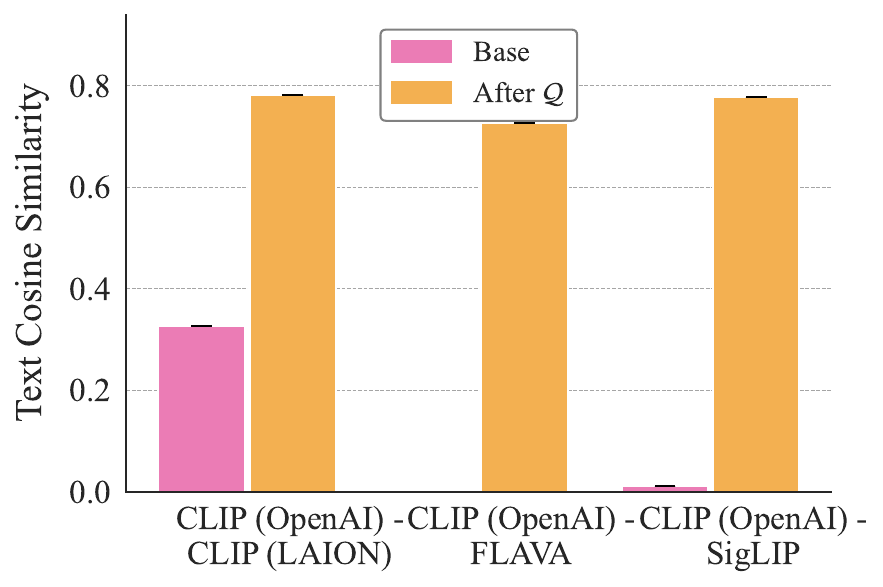}
    \caption{}
    \end{subfigure}
    \begin{subfigure}{0.32\textwidth}
    \centering
    \includegraphics[width=\textwidth]{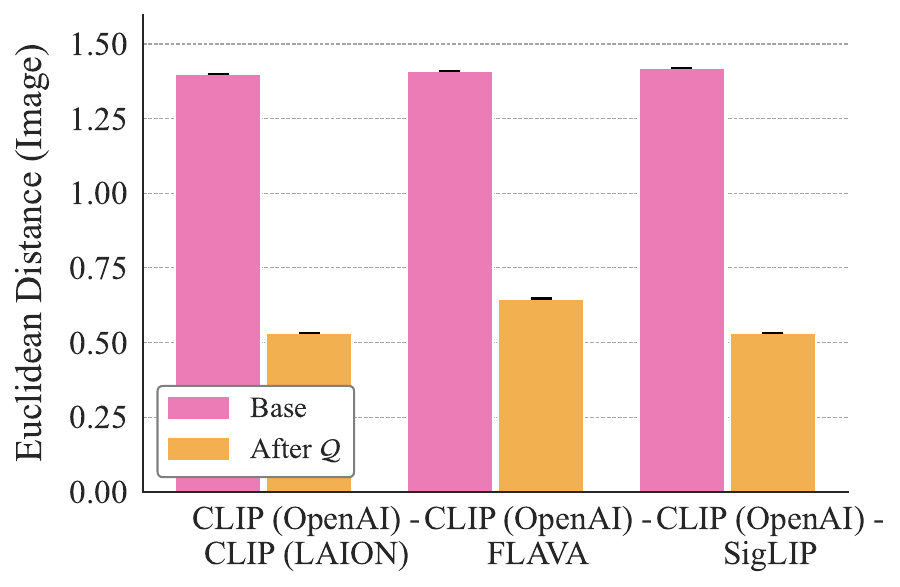}
    \caption{}
    \end{subfigure}
    \begin{subfigure}{0.32\textwidth}
    \centering
    \includegraphics[width=\textwidth]{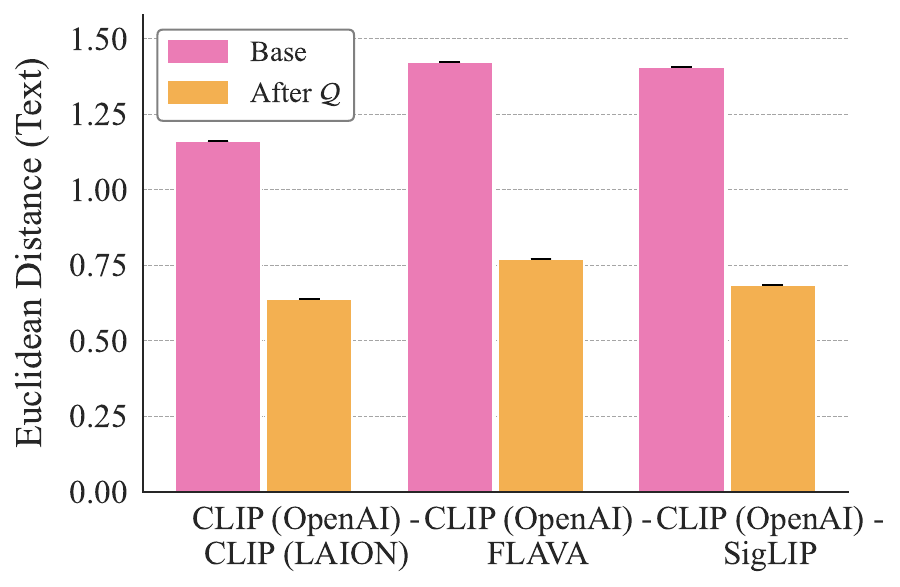}
    \caption{}
    \end{subfigure}
    \begin{subfigure}{0.32\textwidth}
    \centering
    \includegraphics[width=\textwidth]{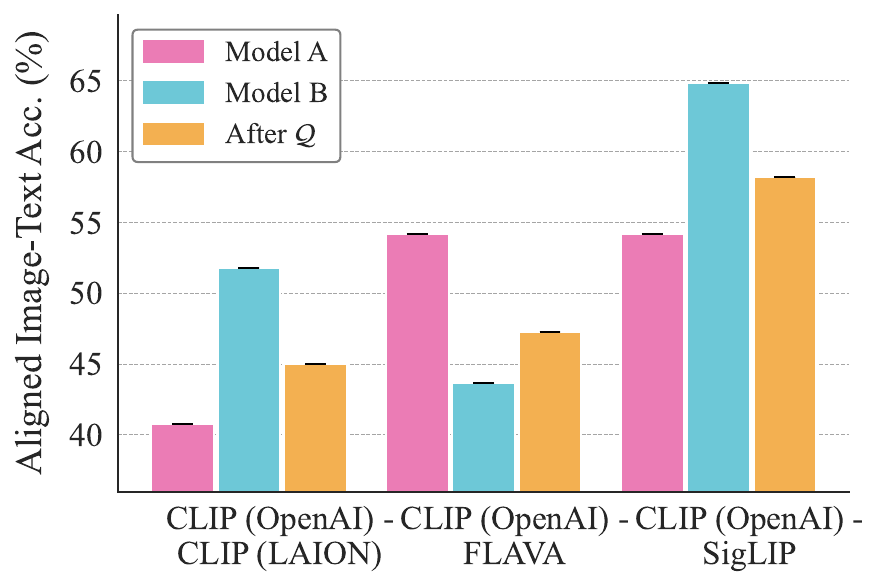}
    \caption{}
    \end{subfigure}
    \begin{subfigure}{0.32\textwidth}
    \centering
    \includegraphics[width=\textwidth]{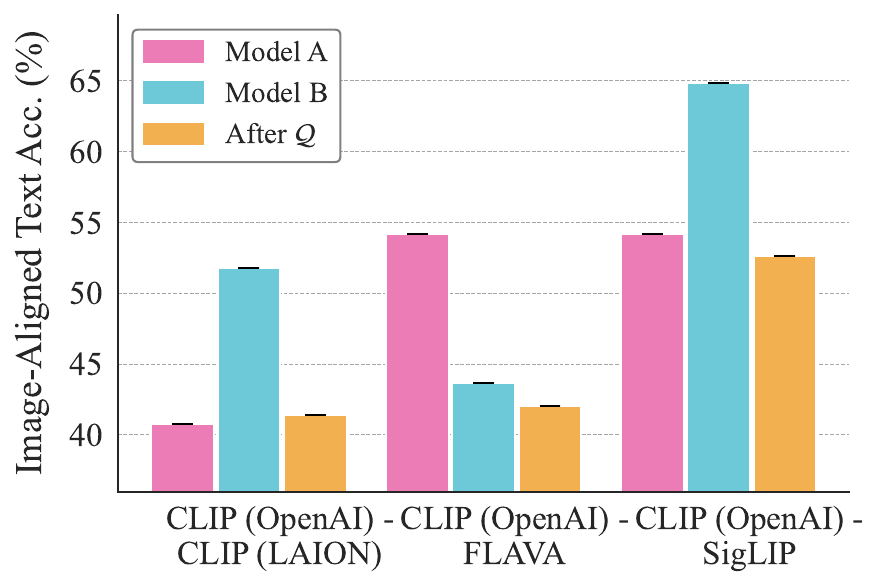}
    \caption{}
    \end{subfigure}
    \begin{subfigure}{0.32\textwidth}
    \centering
    \includegraphics[width=\textwidth]{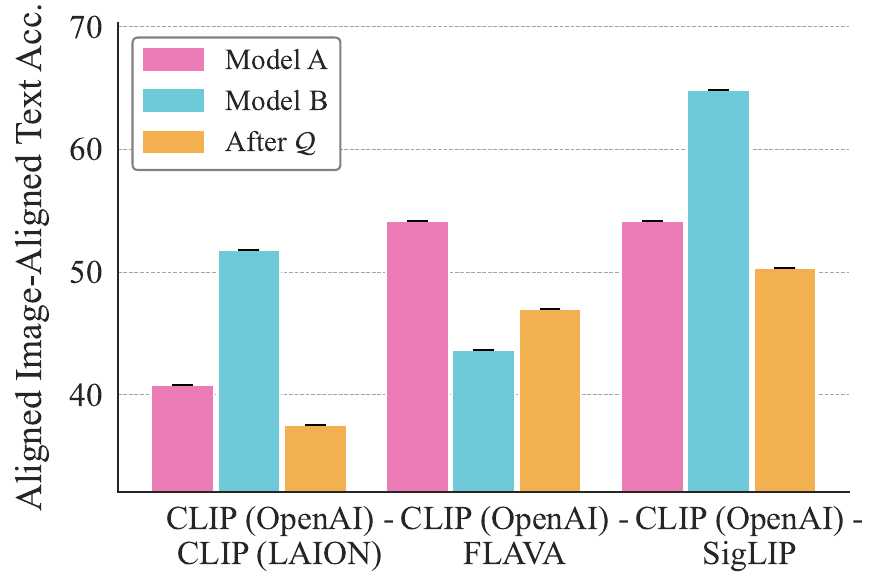}
    \caption{}
    \end{subfigure}
    \caption{\textit{Cross-model alignment on DTD across independently trained contrastive models before and after fitting a single orthogonal map $\mathcal R$.} (a) Image–image class retrieval and (b) text–text class retrieval (c) Mean image–image cosine similarity. (d) Mean text–text cosine similarity. (e) Euclidean distance (without normalization) between image embeddings. (f) Euclidean distance (without normalization) between text embeddings. (g) Image–text retrieval using aligned images from model A and text from model B. (h) Image–text retrieval using images from model B and aligned text from model A. (i) Image–text retrieval using aligned images and aligned text from model A. $\mathcal Q$ aligns images across models with a single orthogonal map, and the same $\mathcal Q$ learned only from image embeddings transfers to text, boosting text-text retrieval from near-chance to near-oracle, all while preserving strong image classification accuracy.}
    \label{fig:mainresults_dtd}
\end{figure*}

\FloatBarrier

\subsection{Only a Few Data Points Are Needed to Learn the
Orthogonal Map}
\label{sec:app_seen_unseen}
In~\Cref{sec:rotation_alignment}, we proved that if the multimodal kernels induced by two contrastive models agree on a sufficiently rich but \emph{small finite} anchor set, then a single global orthogonal map aligns their representations across both modalities. We empirically validate this prediction by fitting $\mathcal Q$ using paired images from only $N$ classes and evaluating transfer on the remaining unseen classes. We report this analysis across additional model pairs and datasets. The results for \textsc{CLIP} (OpenAI) aligned to \textsc{CLIP} (OpenAI), \textsc{CLIP} (LAION), \textsc{FLAVA}, and \textsc{SigLIP} on Caltech-101 are shown in~\Cref{fig:app_split_caltech_openai_to_openai,fig:app_split_caltech_openai_to_laion,fig:app_split_caltech_openai_to_flava,fig:app_split_caltech_openai_to_siglip}. For CIFAR100 and Oxford Pets, the analogous results are reports in~\Cref{fig:app_split_cifar_openai_to_openai,fig:app_split_cifar_openai_to_laion,fig:app_split_cifar_openai_to_flava,fig:app_split_cifar_openai_to_siglip} and~\Cref{fig:app_split_oxford_openai_to_openai,fig:app_split_oxford_openai_to_laion,fig:app_split_oxford_openai_to_flava,fig:app_split_oxford_openai_to_siglip}. 

Across all settings, the same trend holds: performance on both seen and unseen classes improves quickly with just a few anchor classes and essentially saturates once $N$ reaches a modest value, after which additional anchors provide little benefit. Thus, practitioners can recover near-full cross-model transfer by fitting $\mathcal Q$ on a lightweight image-only calibration set, rather than curating large-scale cross-model supervision.

\begin{figure*}[!htb]
    \centering
    \begin{subfigure}{0.32\textwidth}
    \centering
    \includegraphics[width=\textwidth]{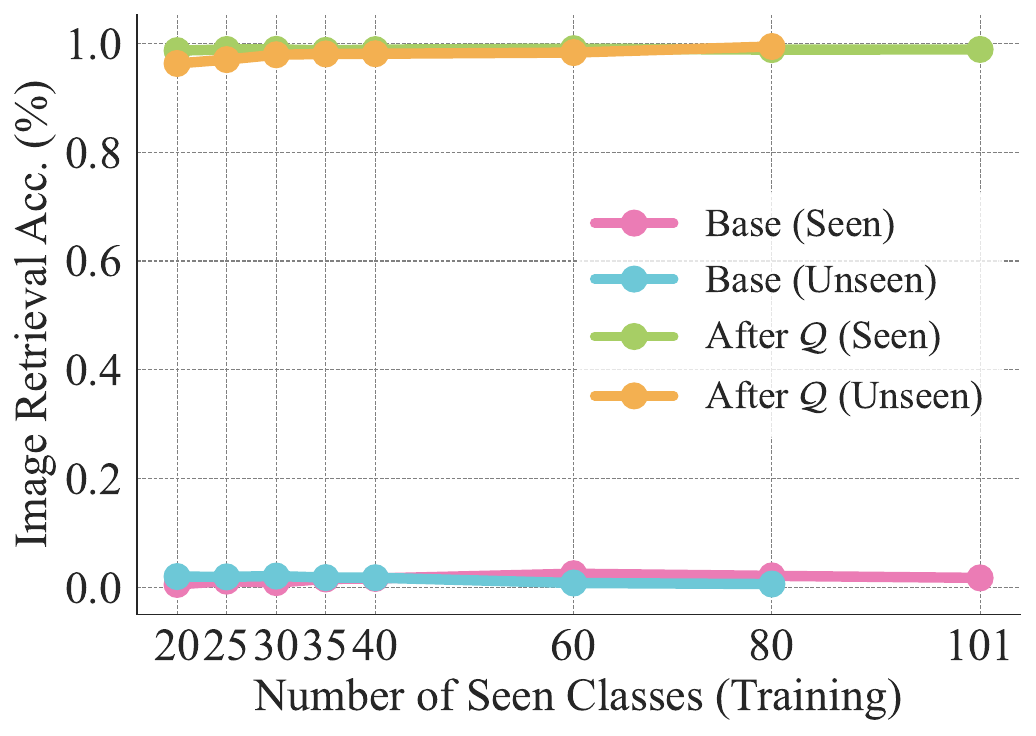}
    \caption{}
    \end{subfigure}
    \begin{subfigure}{0.32\textwidth}
    \centering
    \includegraphics[width=\textwidth]{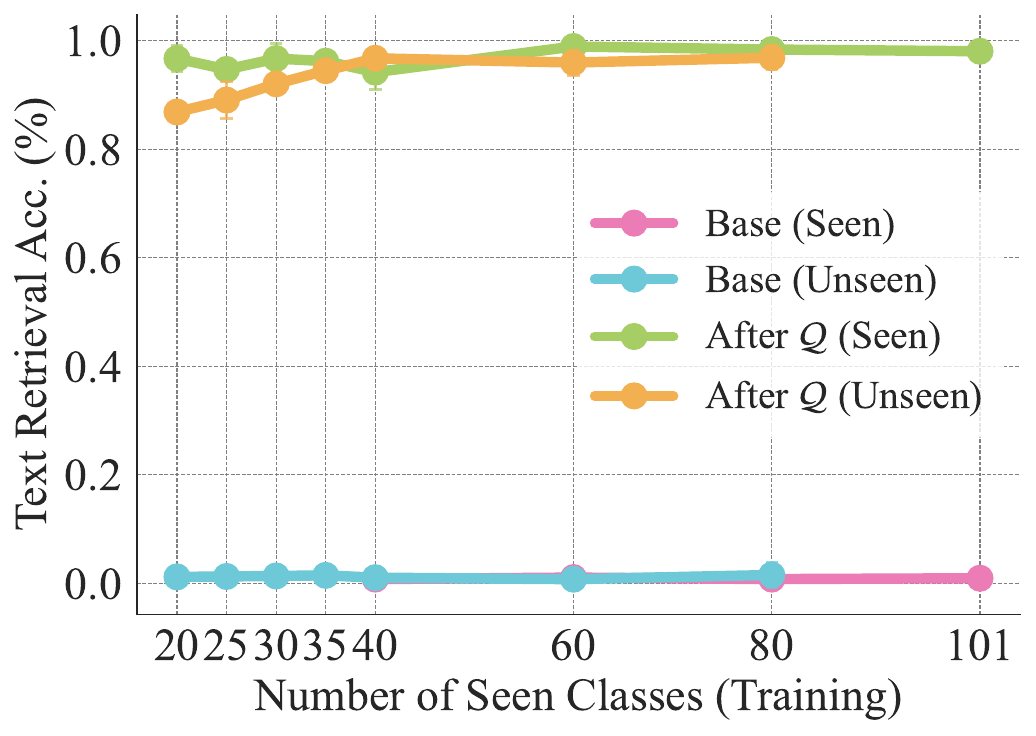}
    \caption{}
    \end{subfigure}
    \begin{subfigure}{0.32\textwidth}
    \centering
    \includegraphics[width=\textwidth]{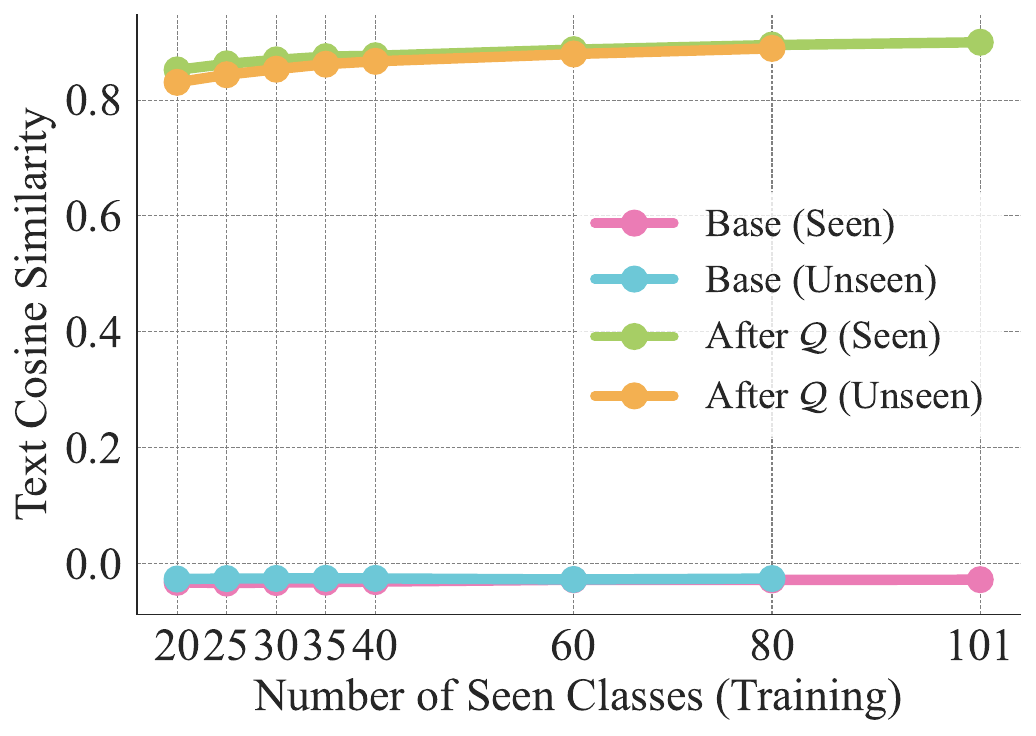}
    \caption{}
    \end{subfigure}
    \begin{subfigure}{0.32\textwidth}
    \centering
    \includegraphics[width=\textwidth]{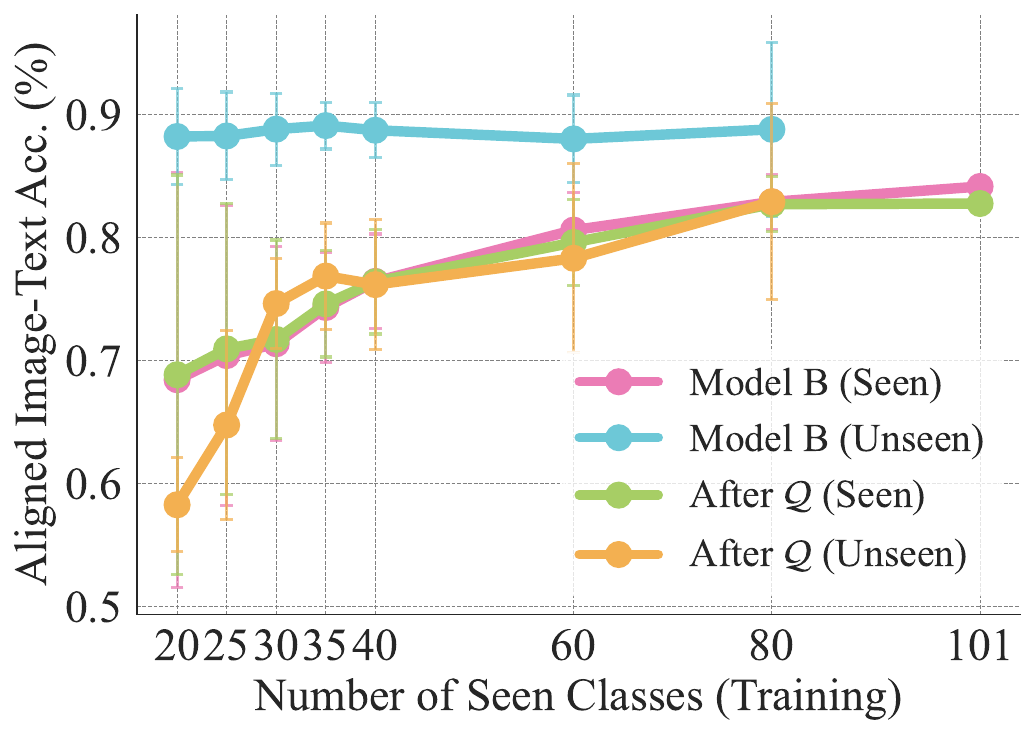}
    \caption{}
    \end{subfigure}
    \begin{subfigure}{0.32\textwidth}
    \centering
    \includegraphics[width=\textwidth]{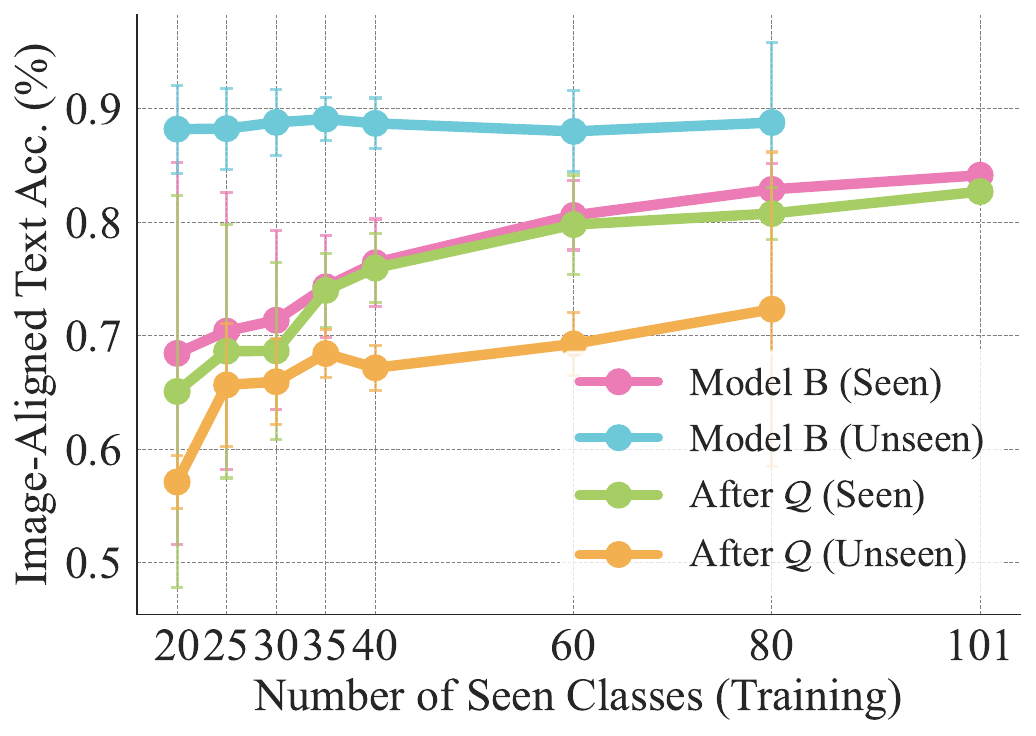}
    \caption{}
    \end{subfigure}
    \begin{subfigure}{0.32\textwidth}
    \centering
    \includegraphics[width=\textwidth]{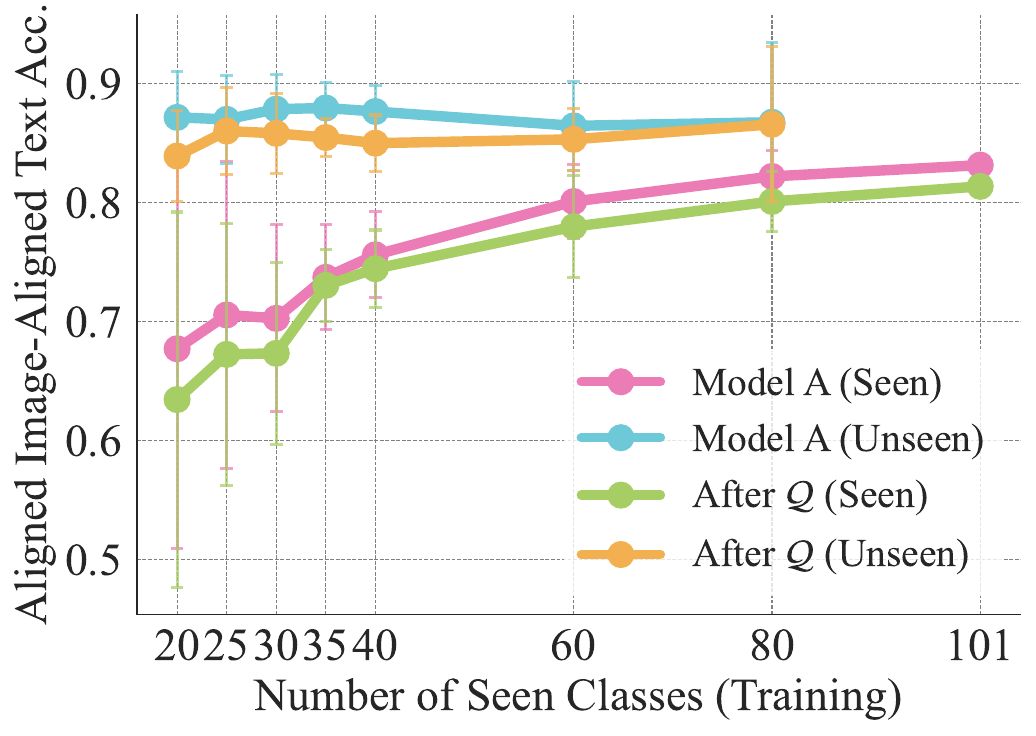}
    \caption{}
    \end{subfigure}
    \caption{\textit{Generalization of orthogonal alignment under limited supervision (Caltech-101; CLIP ViT-B/32 to CLIP ViT-B/16, OpenAI).} (a) Image-image class retrieval, (b) text-text class retrieval, and (c) mean text-text cosine similarity, each reported on seen and unseen classes. (d-f) Downstream transfer: (d) aligned images from model A with text from model B, (e) images from model B with aligned text from model A, and (f) aligned images with aligned text from model A. $\mathcal R$ learned from few classes transfers across modalities and generalizes to unseen classes, achieving near-oracle cross-model retrieval and classification.\looseness=-1}
\label{fig:app_split_caltech_openai_to_openai}
\end{figure*}

\begin{figure*}[!htb]
    \centering
    \begin{subfigure}{0.32\textwidth}
    \centering
    \includegraphics[width=\textwidth]{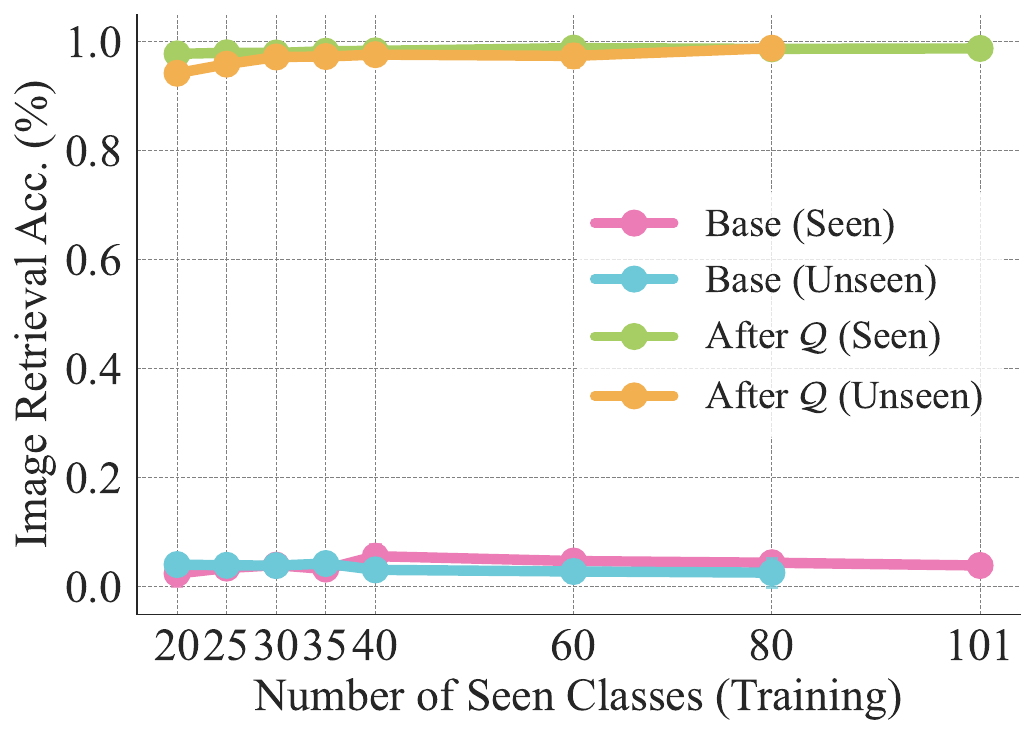}
    \caption{}
    \end{subfigure}
    \begin{subfigure}{0.32\textwidth}
    \centering
    \includegraphics[width=\textwidth]{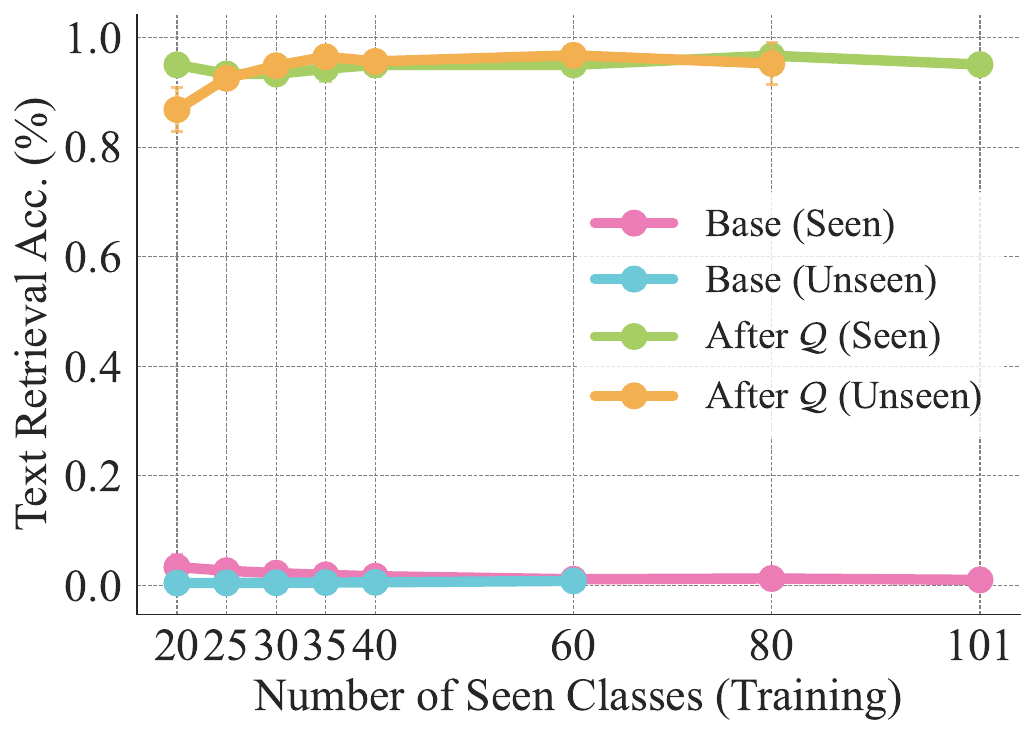}
    \caption{}
    \end{subfigure}
    \begin{subfigure}{0.32\textwidth}
    \centering
    \includegraphics[width=\textwidth]{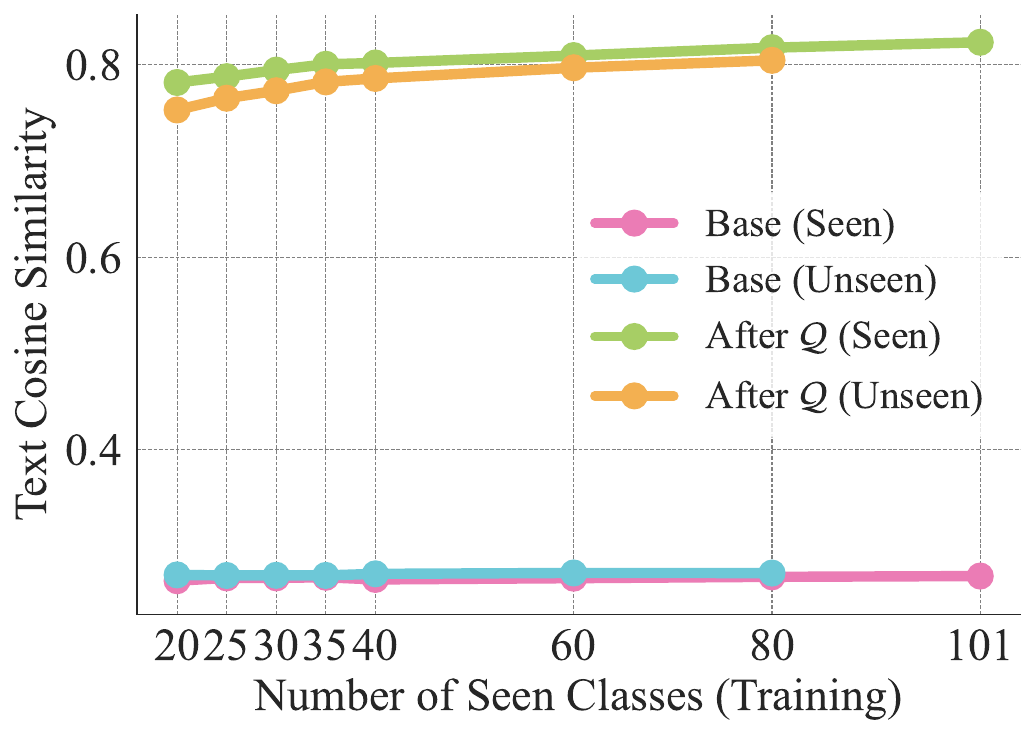}
    \caption{}
    \end{subfigure}
    \begin{subfigure}{0.32\textwidth}
    \centering
    \includegraphics[width=\textwidth]{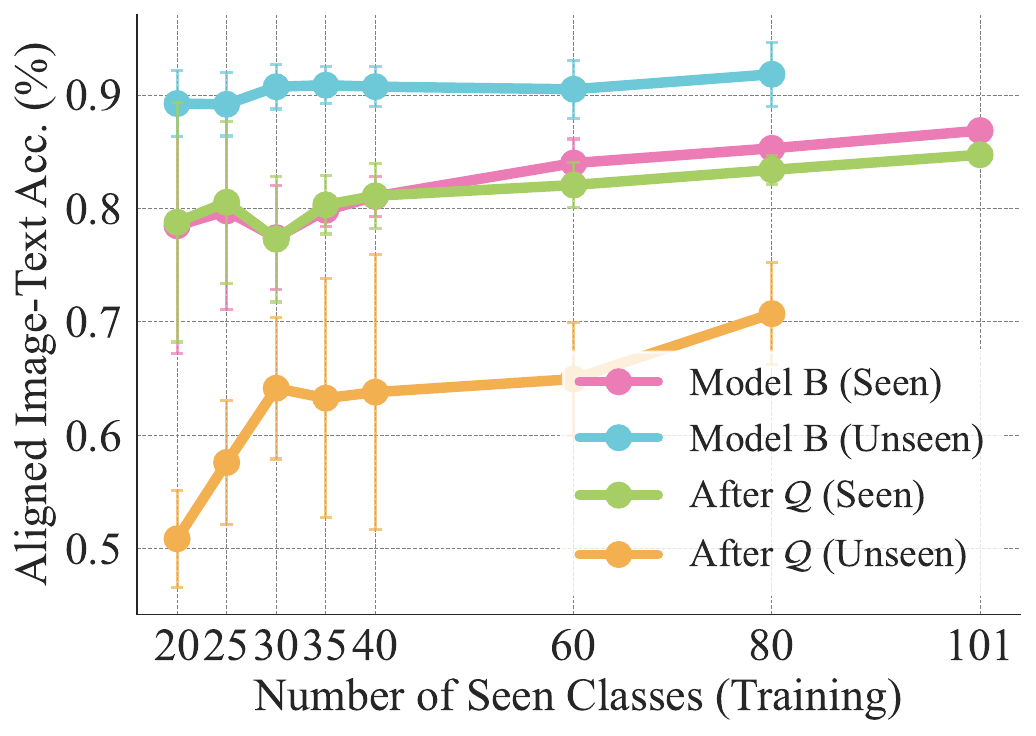}
    \caption{}
    \end{subfigure}
    \begin{subfigure}{0.32\textwidth}
    \centering
    \includegraphics[width=\textwidth]{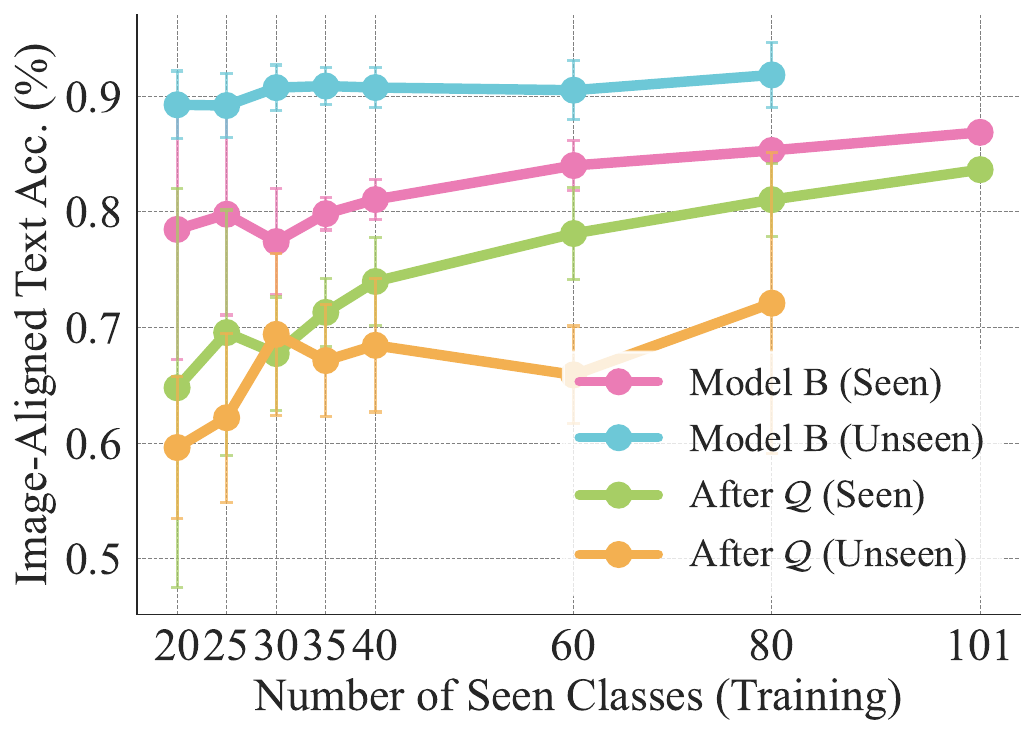}
    \caption{}
    \end{subfigure}
    \begin{subfigure}{0.32\textwidth}
    \centering
    \includegraphics[width=\textwidth]{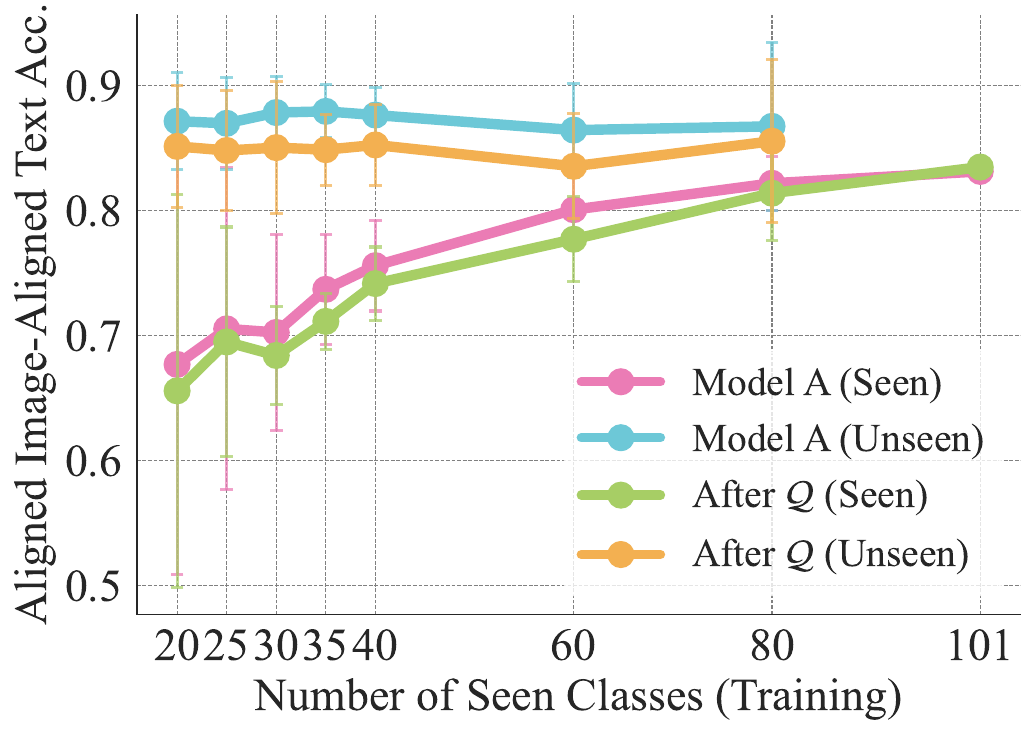}
    \caption{}
    \end{subfigure}
    \caption{\textit{Generalization of orthogonal alignment under limited supervision (Caltech-101; CLIP ViT-B/32 (OpenAI) to CLIP ViT-B/32 (LAION).} (a) Image-image class retrieval, (b) text-text class retrieval, and (c) mean text-text cosine similarity, each reported on seen and unseen classes. (d-f) Downstream transfer: (d) aligned images from model A with text from model B, (e) images from model B with aligned text from model A, and (f) aligned images with aligned text from model A. $\mathcal R$ learned from few classes transfers across modalities and generalizes to unseen classes, achieving near-oracle cross-model retrieval and classification.\looseness=-1}
\label{fig:app_split_caltech_openai_to_laion}
\end{figure*}

\begin{figure*}[!htb]
    \centering
    \begin{subfigure}{0.32\textwidth}
    \centering
    \includegraphics[width=\textwidth]{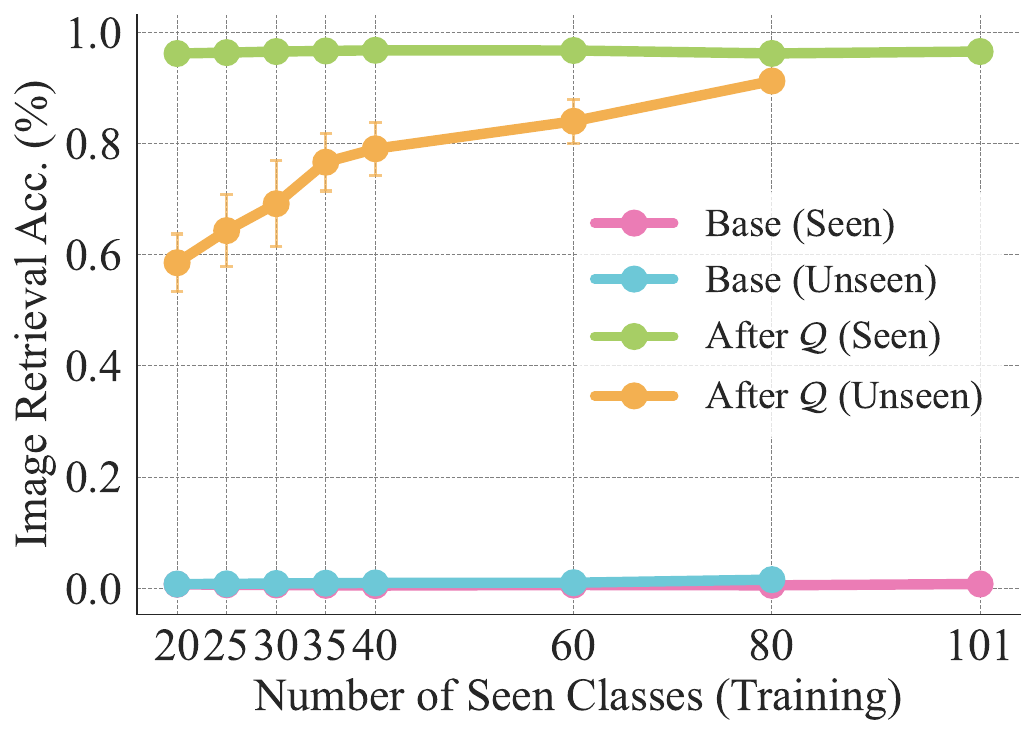}
    \caption{}
    \end{subfigure}
    \begin{subfigure}{0.32\textwidth}
    \centering
    \includegraphics[width=\textwidth]{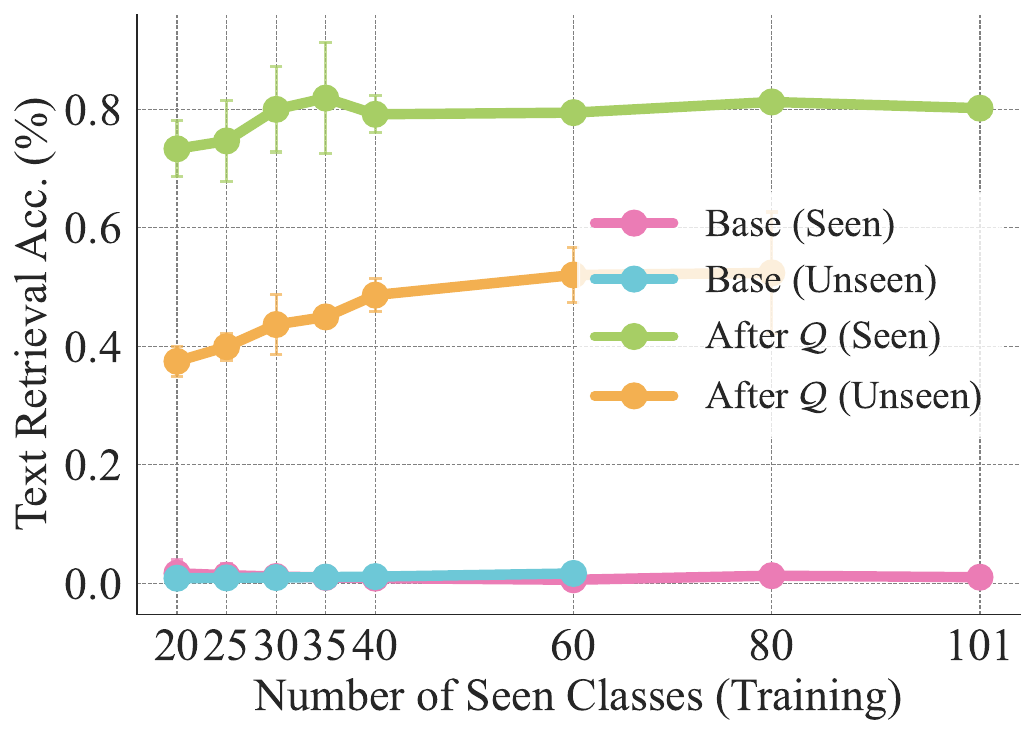}
    \caption{}
    \end{subfigure}
    \begin{subfigure}{0.32\textwidth}
    \centering
    \includegraphics[width=\textwidth]{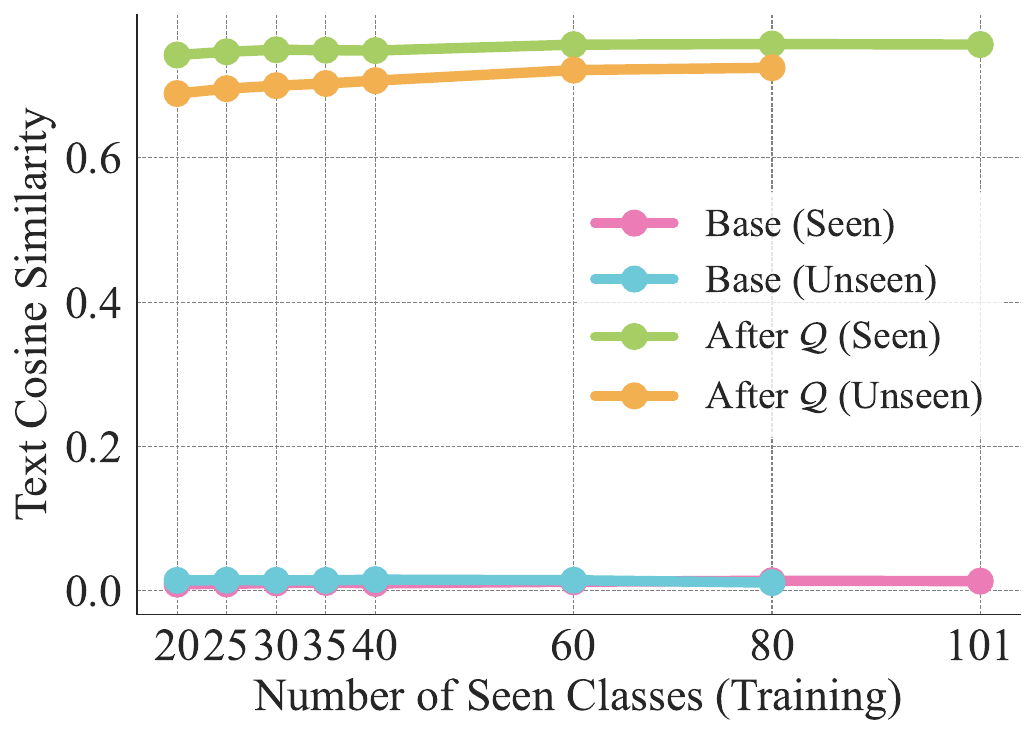}
    \caption{}
    \end{subfigure}
    \begin{subfigure}{0.32\textwidth}
    \centering
    \includegraphics[width=\textwidth]{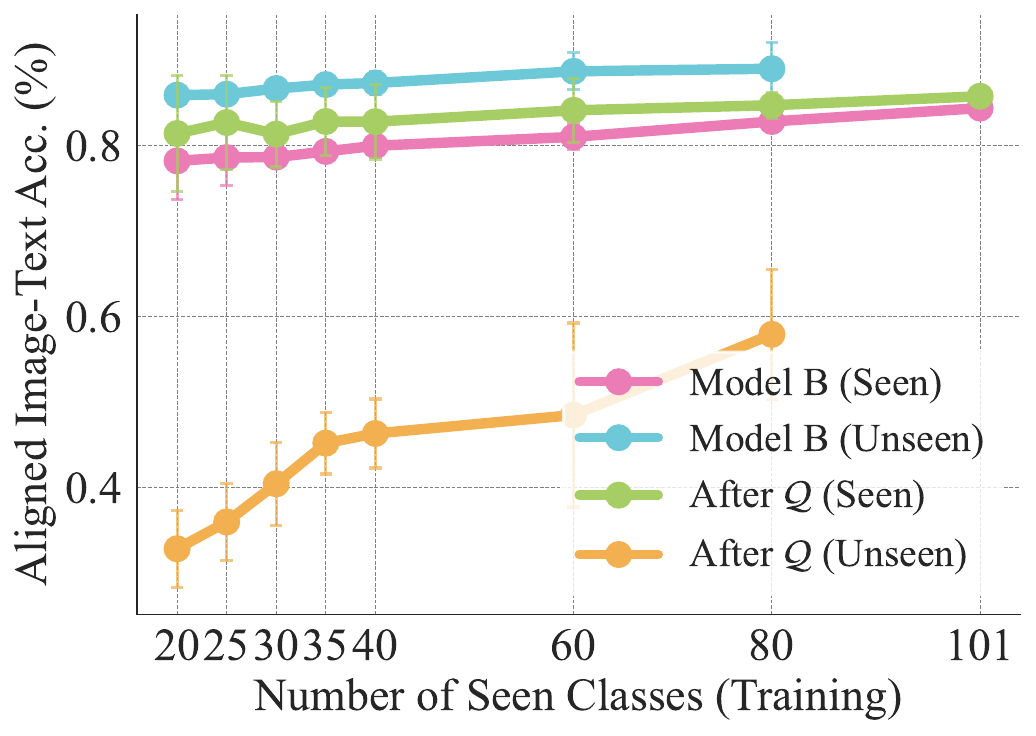}
    \caption{}
    \end{subfigure}
    \begin{subfigure}{0.32\textwidth}
    \centering
    \includegraphics[width=\textwidth]{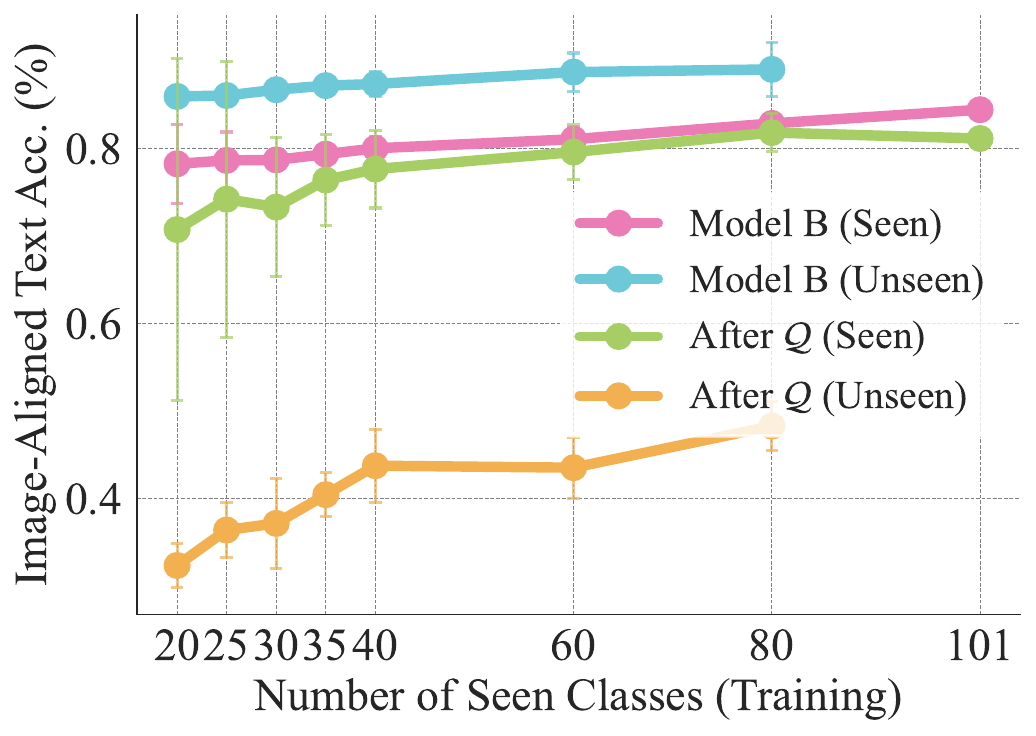}
    \caption{}
    \end{subfigure}
    \begin{subfigure}{0.32\textwidth}
    \centering
    \includegraphics[width=\textwidth]{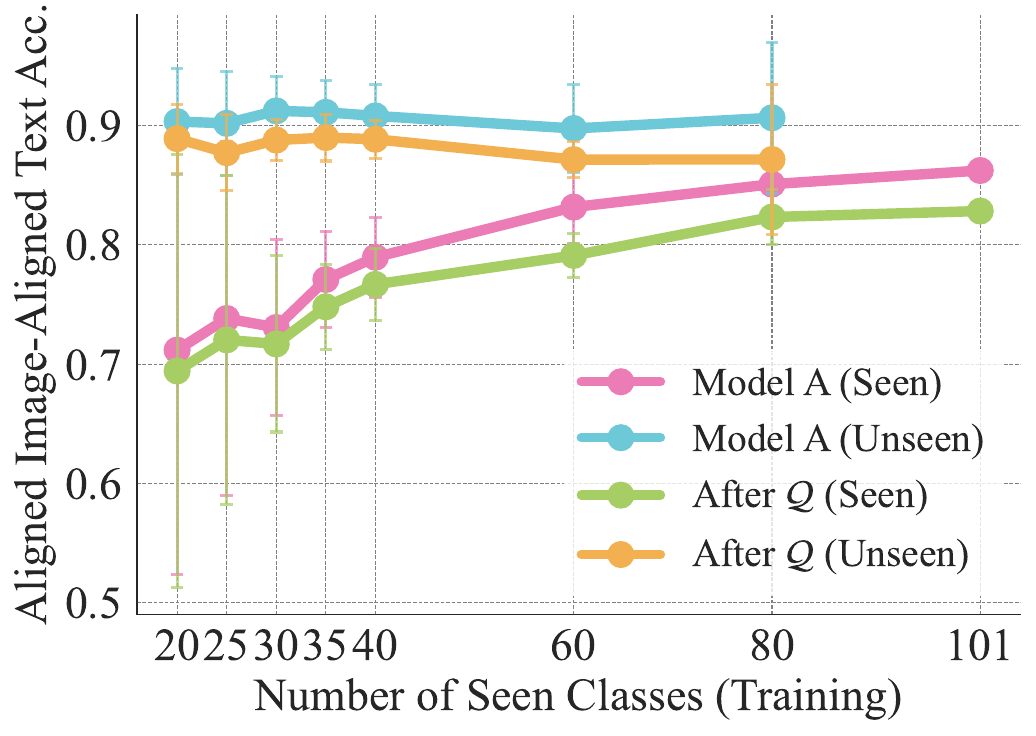}
    \caption{}
    \end{subfigure}
    \caption{\textit{Generalization of orthogonal alignment under limited supervision (Caltech-101; CLIP ViT-L/14 (OpenAI) to FLAVA} (a) Image-image class retrieval, (b) text-text class retrieval, and (c) mean text-text cosine similarity, each reported on seen and unseen classes. (d-f) Downstream transfer: (d) aligned images from model A with text from model B, (e) images from model B with aligned text from model A, and (f) aligned images with aligned text from model A. $\mathcal R$ learned from few classes transfers across modalities and generalizes to unseen classes, achieving near-oracle cross-model retrieval and classification.\looseness=-1}
\label{fig:app_split_caltech_openai_to_flava}
\end{figure*}

\begin{figure*}[!htb]
    \centering
    \begin{subfigure}{0.32\textwidth}
    \centering
    \includegraphics[width=\textwidth]{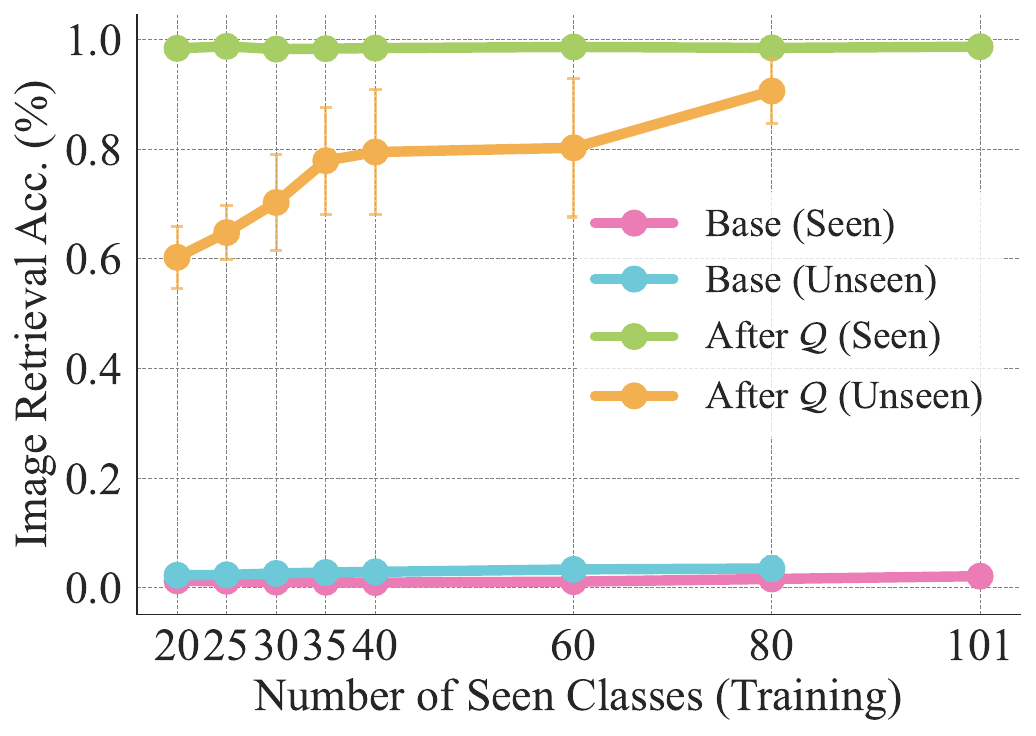}
    \caption{}
    \end{subfigure}
    \begin{subfigure}{0.32\textwidth}
    \centering
    \includegraphics[width=\textwidth]{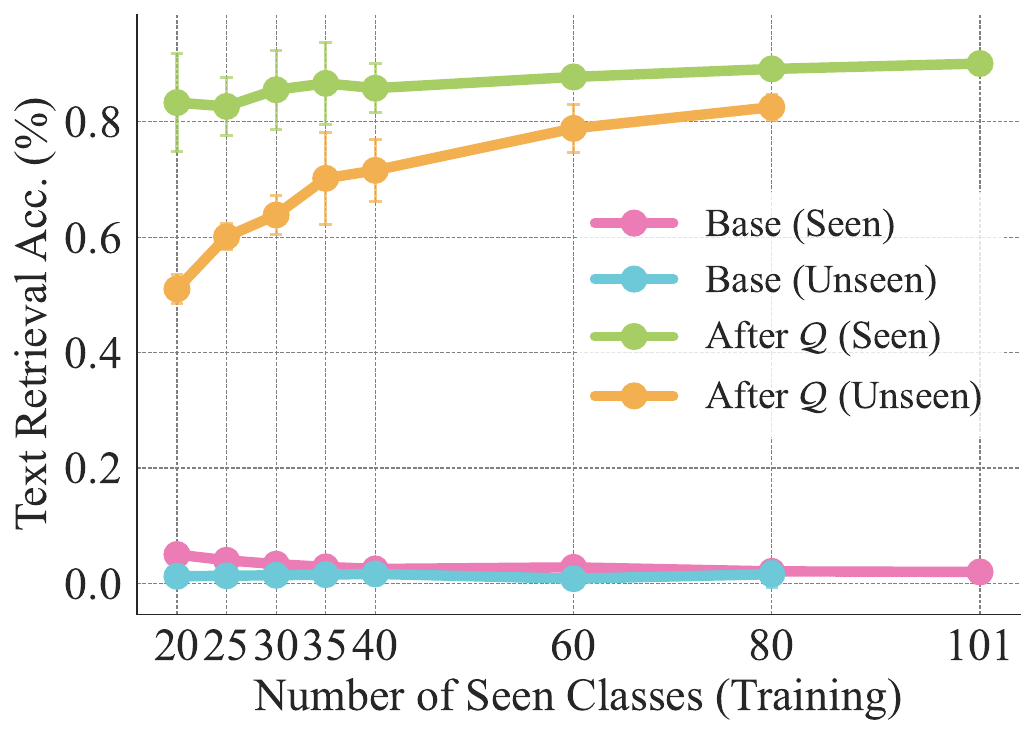}
    \caption{}
    \end{subfigure}
    \begin{subfigure}{0.32\textwidth}
    \centering
    \includegraphics[width=\textwidth]{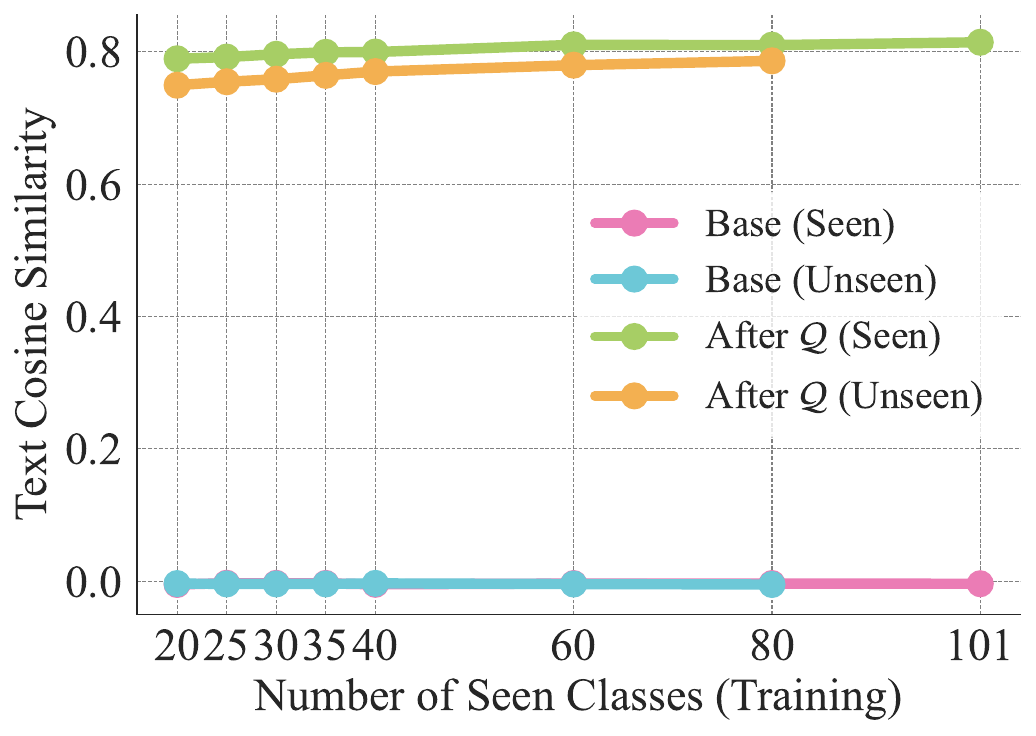}
    \caption{}
    \end{subfigure}
    \begin{subfigure}{0.32\textwidth}
    \centering
    \includegraphics[width=\textwidth]{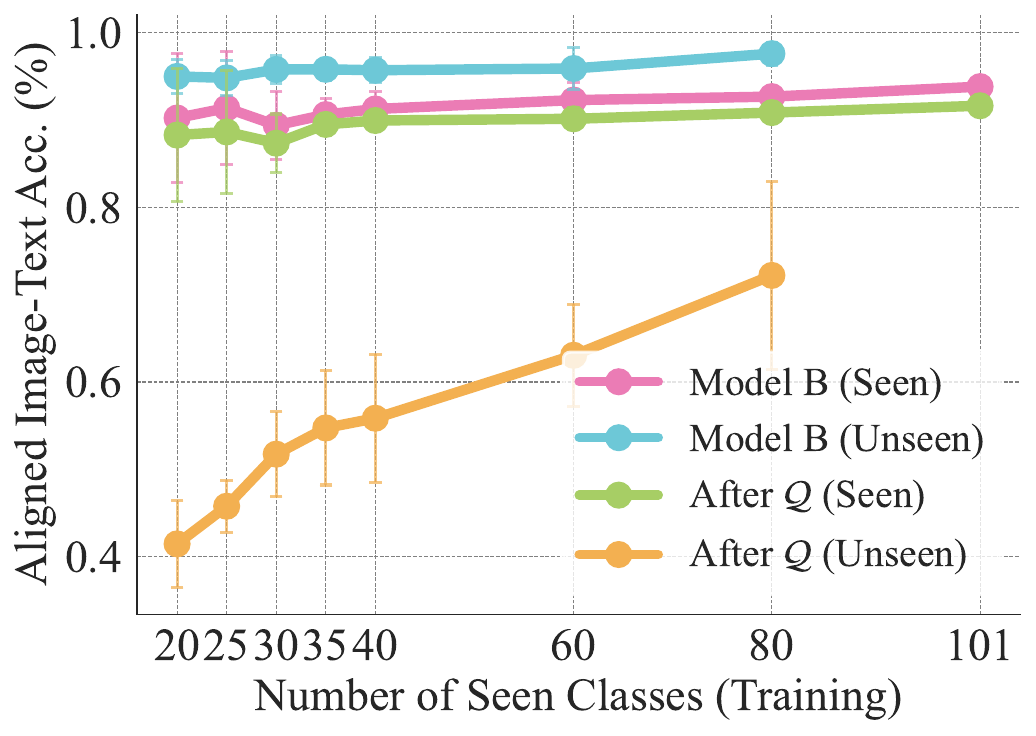}
    \caption{}
    \end{subfigure}
    \begin{subfigure}{0.32\textwidth}
    \centering
    \includegraphics[width=\textwidth]{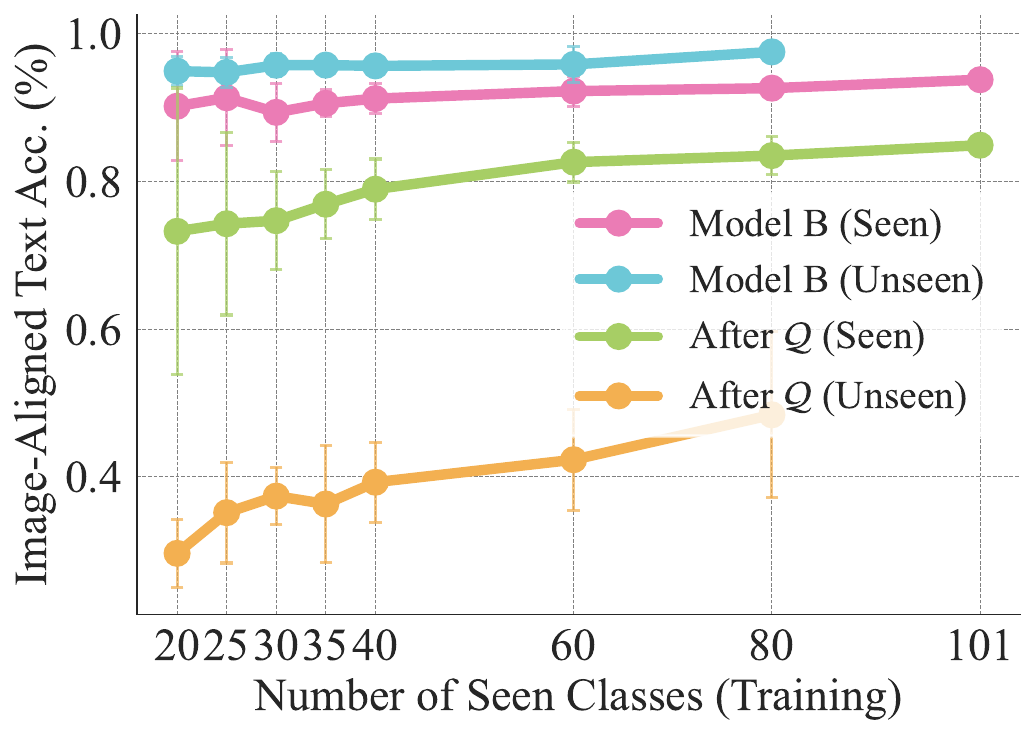}
    \caption{}
    \end{subfigure}
    \begin{subfigure}{0.32\textwidth}
    \centering
    \includegraphics[width=\textwidth]{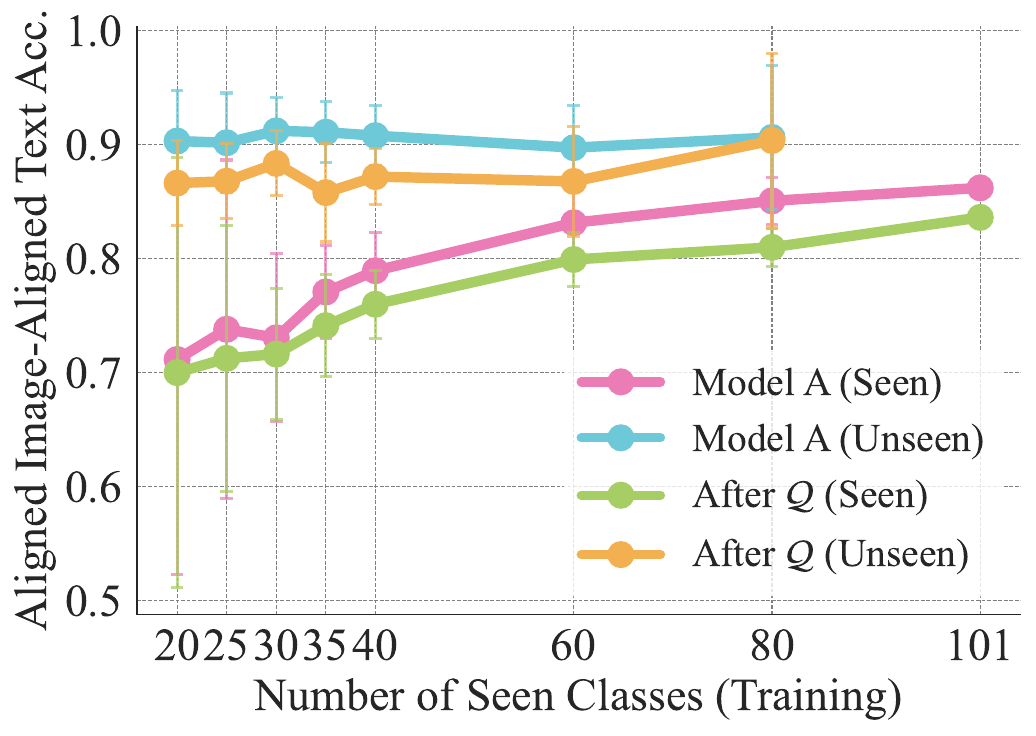}
    \caption{}
    \end{subfigure}
    \caption{\textit{Generalization of orthogonal alignment under limited supervision (Caltech-101; CLIP ViT-L/14 (OpenAI) to SigLIP.} (a) Image-image class retrieval, (b) text-text class retrieval, and (c) mean text-text cosine similarity, each reported on seen and unseen classes. (d-f) Downstream transfer: (d) aligned images from model A with text from model B, (e) images from model B with aligned text from model A, and (f) aligned images with aligned text from model A. $\mathcal R$ learned from few classes transfers across modalities and generalizes to unseen classes, achieving near-oracle cross-model retrieval and classification.\looseness=-1}
\label{fig:app_split_caltech_openai_to_siglip}
\end{figure*}

\begin{figure*}[!htb]
    \centering
    \begin{subfigure}{0.32\textwidth}
    \centering
    \includegraphics[width=\textwidth]{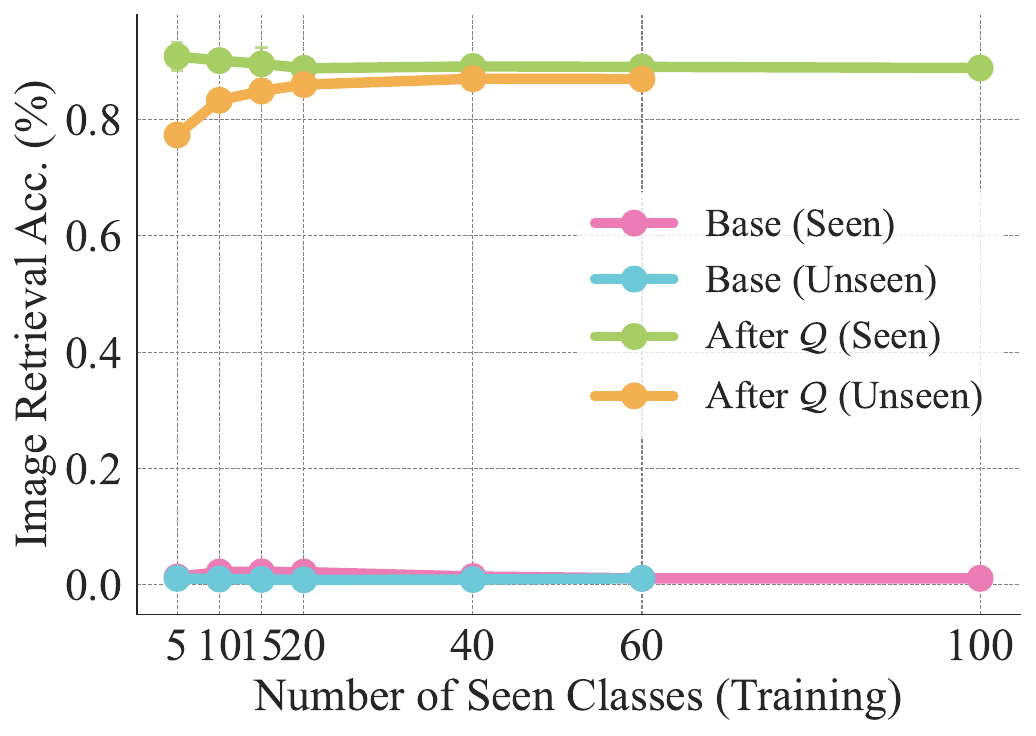}
    \caption{}
    \end{subfigure}
    \begin{subfigure}{0.32\textwidth}
    \centering
    \includegraphics[width=\textwidth]{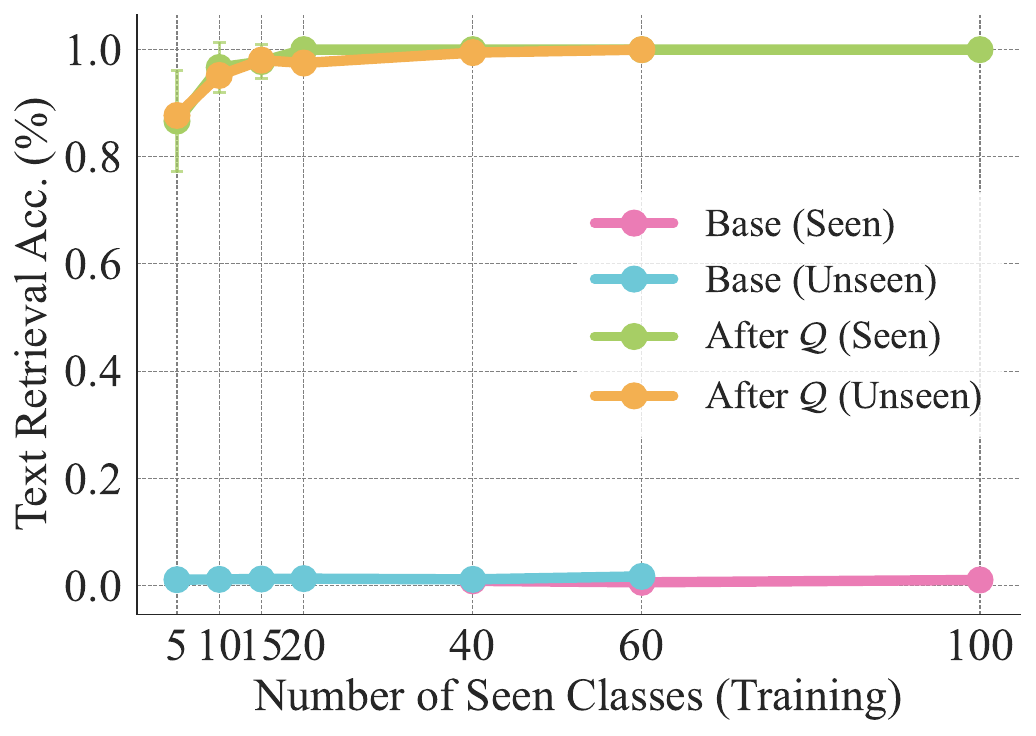}
    \caption{}
    \end{subfigure}
    \begin{subfigure}{0.32\textwidth}
    \centering
    \includegraphics[width=\textwidth]{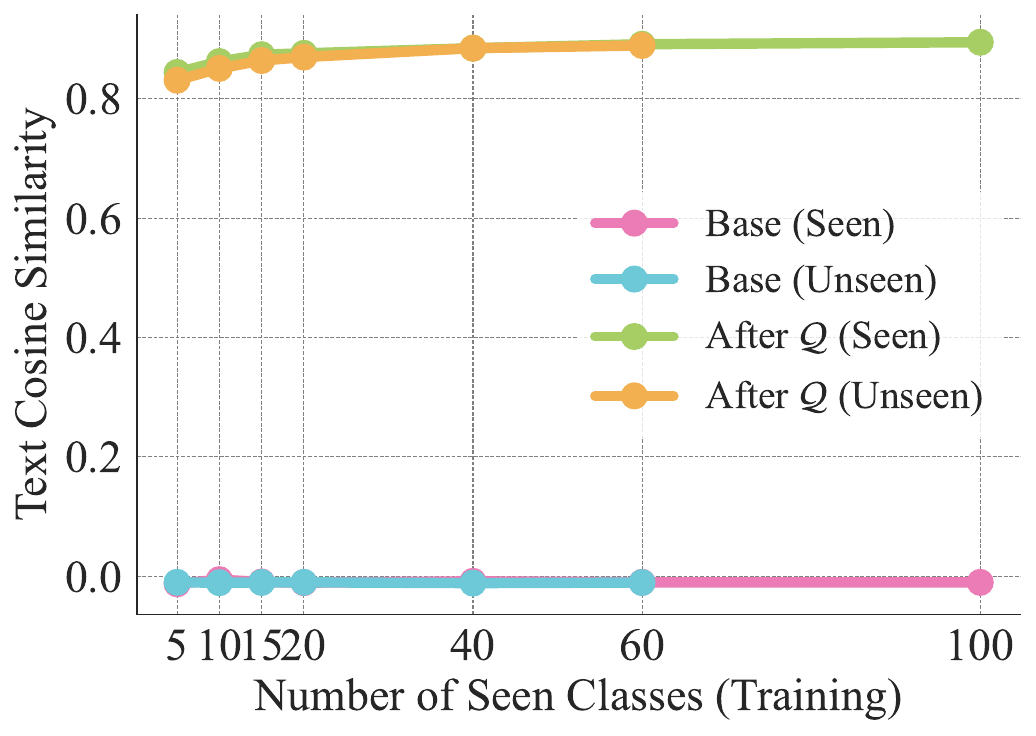}
    \caption{}
    \end{subfigure}
    \begin{subfigure}{0.32\textwidth}
    \centering
    \includegraphics[width=\textwidth]{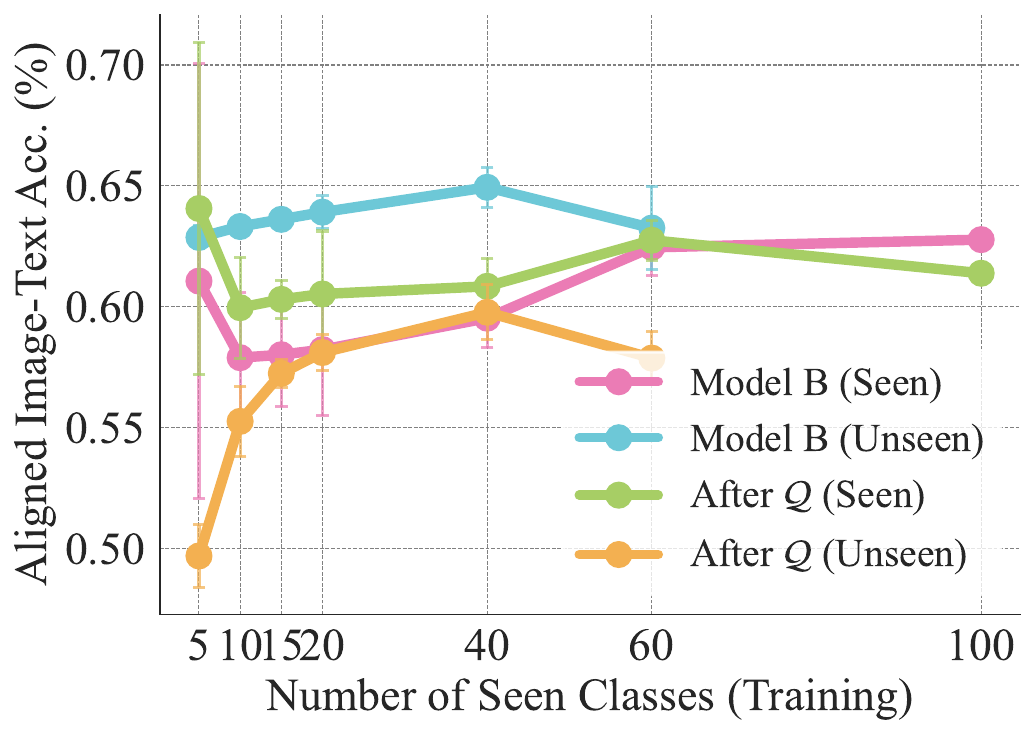}
    \caption{}
    \end{subfigure}
    \begin{subfigure}{0.32\textwidth}
    \centering
    \includegraphics[width=\textwidth]{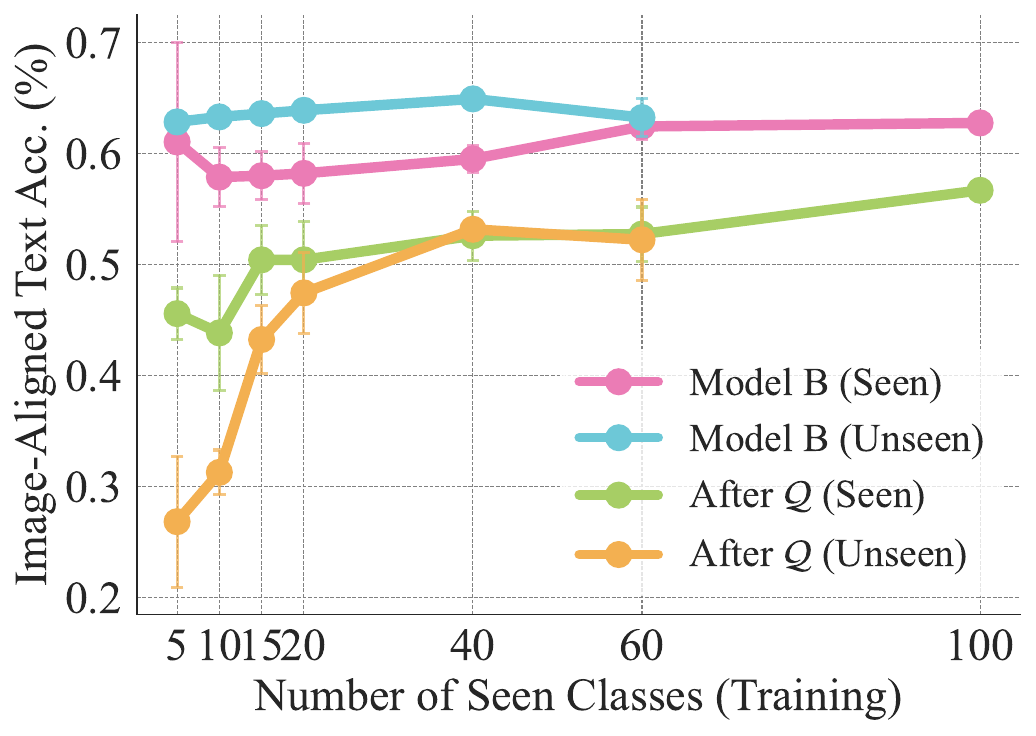}
    \caption{}
    \end{subfigure}
    \begin{subfigure}{0.32\textwidth}
    \centering
    \includegraphics[width=\textwidth]{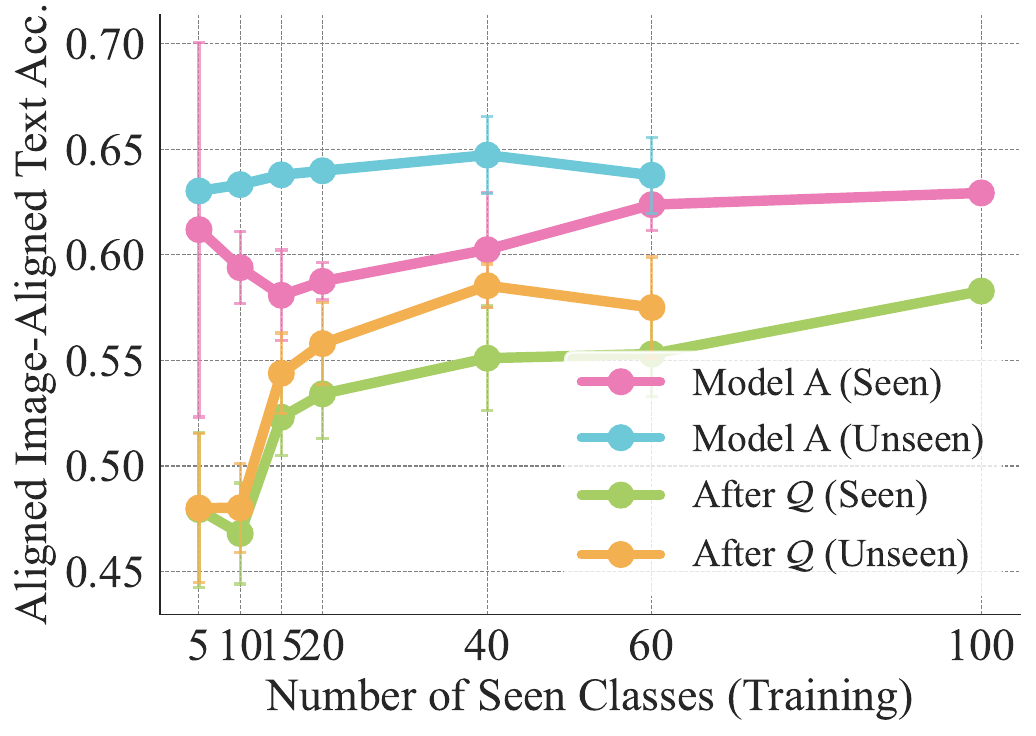}
    \caption{}
    \end{subfigure}
    \caption{\textit{Generalization of orthogonal alignment under limited supervision (CIFAR-100; CLIP ViT-B/32 to CLIP ViT-B/16, OpenAI).} (a) Image-image class retrieval, (b) text-text class retrieval, and (c) mean text-text cosine similarity, each reported on seen and unseen classes. (d-f) Downstream transfer: (d) aligned images from model A with text from model B, (e) images from model B with aligned text from model A, and (f) aligned images with aligned text from model A. $\mathcal R$ learned from few classes transfers across modalities and generalizes to unseen classes, achieving near-oracle cross-model retrieval and classification.\looseness=-1}
\label{fig:app_split_cifar_openai_to_openai}
\end{figure*}

\begin{figure*}[!htb]
    \centering
    \begin{subfigure}{0.32\textwidth}
    \centering
    \includegraphics[width=\textwidth]{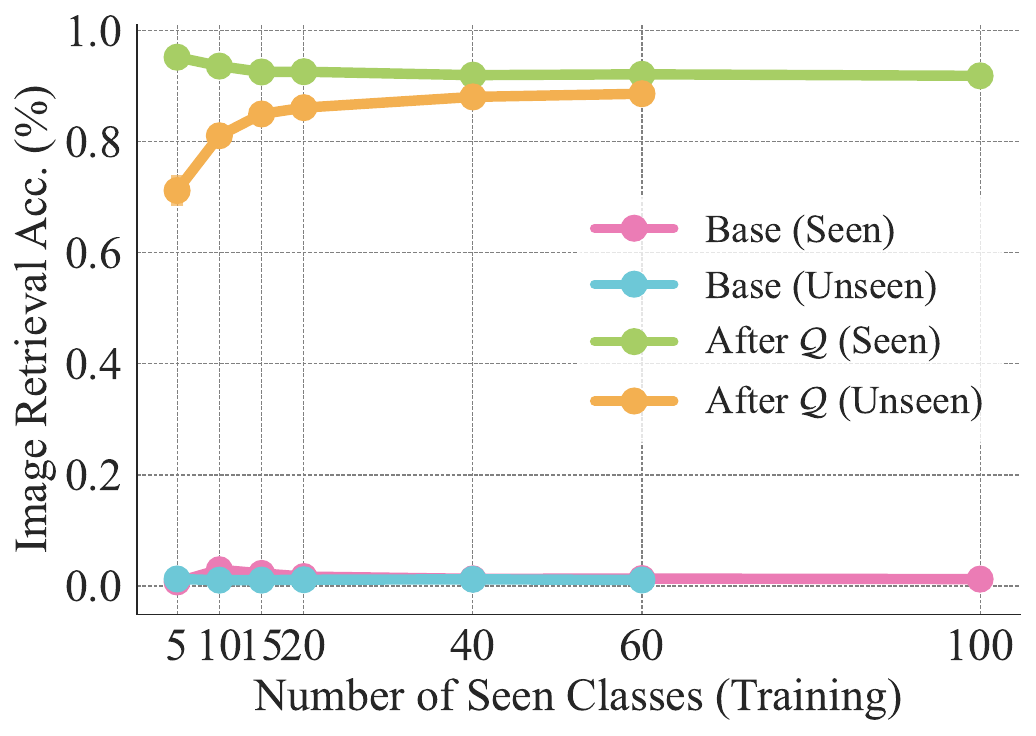}
    \caption{}
    \end{subfigure}
    \begin{subfigure}{0.32\textwidth}
    \centering
    \includegraphics[width=\textwidth]{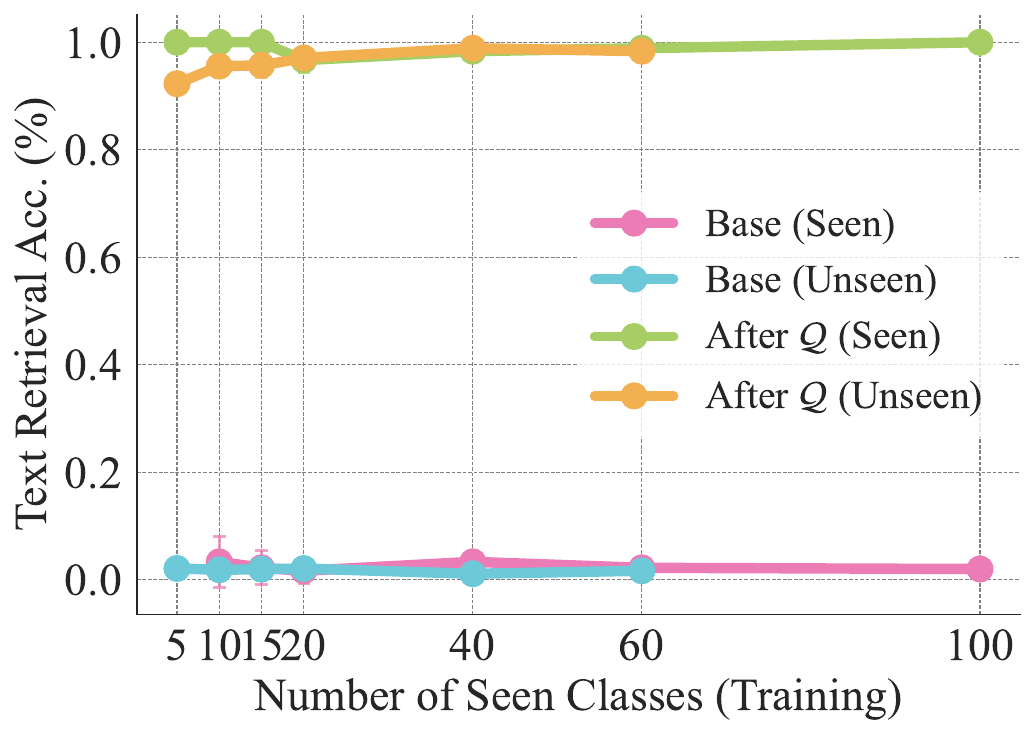}
    \caption{}
    \end{subfigure}
    \begin{subfigure}{0.32\textwidth}
    \centering
    \includegraphics[width=\textwidth]{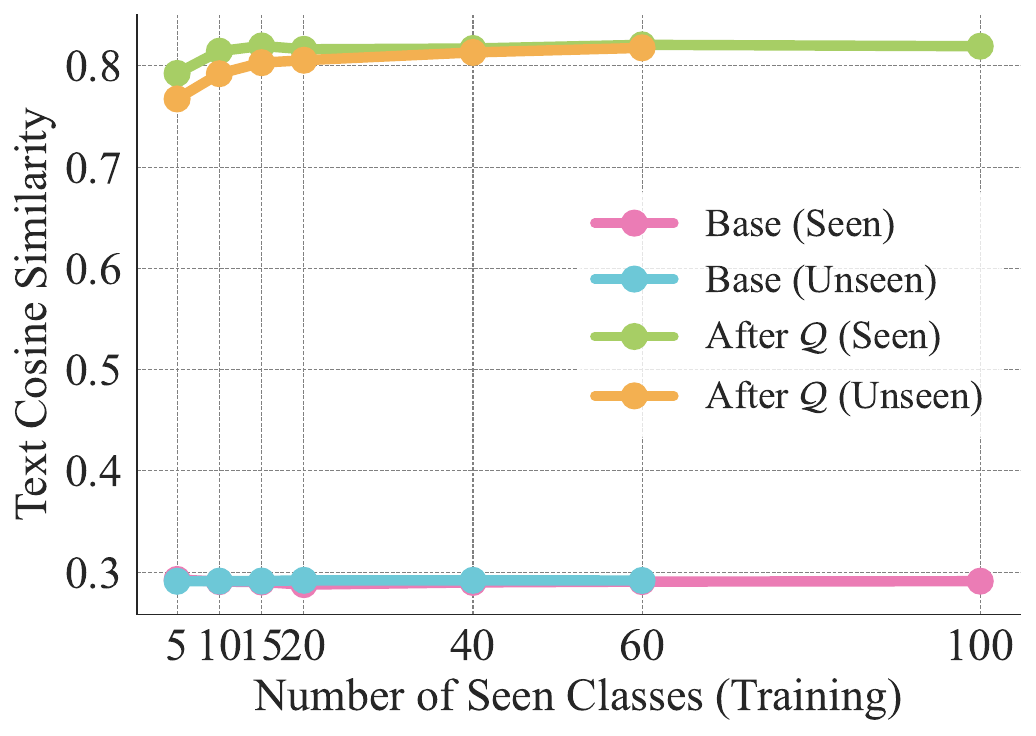}
    \caption{}
    \end{subfigure}
    \begin{subfigure}{0.32\textwidth}
    \centering
    \includegraphics[width=\textwidth]{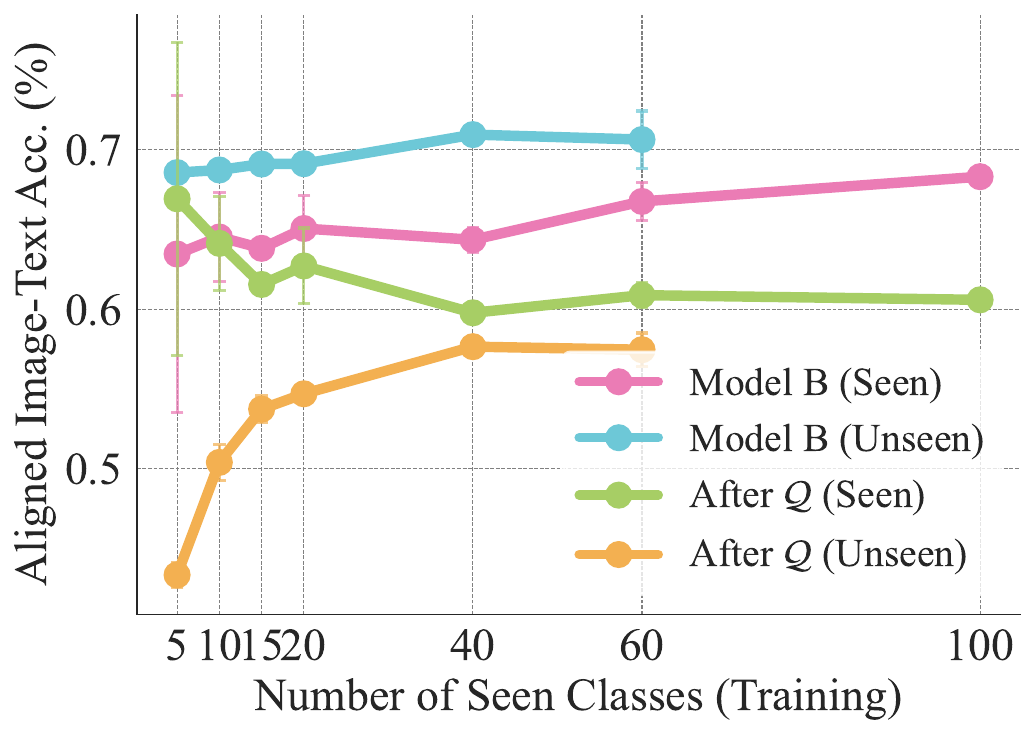}
    \caption{}
    \end{subfigure}
    \begin{subfigure}{0.32\textwidth}
    \centering
    \includegraphics[width=\textwidth]{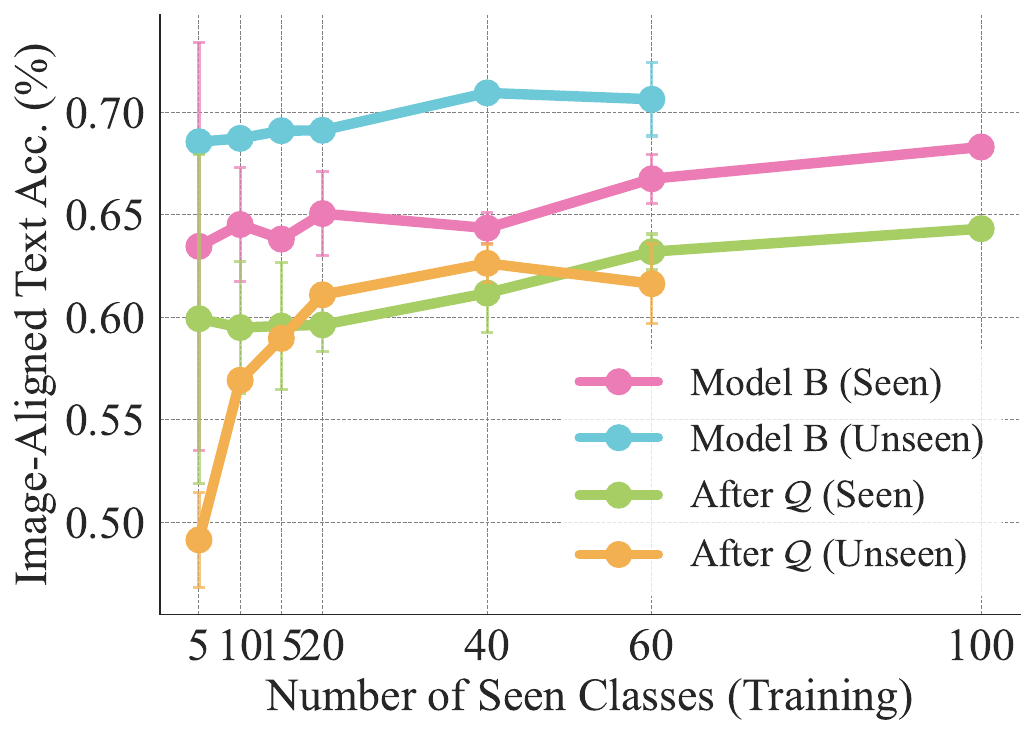}
    \caption{}
    \end{subfigure}
    \begin{subfigure}{0.32\textwidth}
    \centering
    \includegraphics[width=\textwidth]{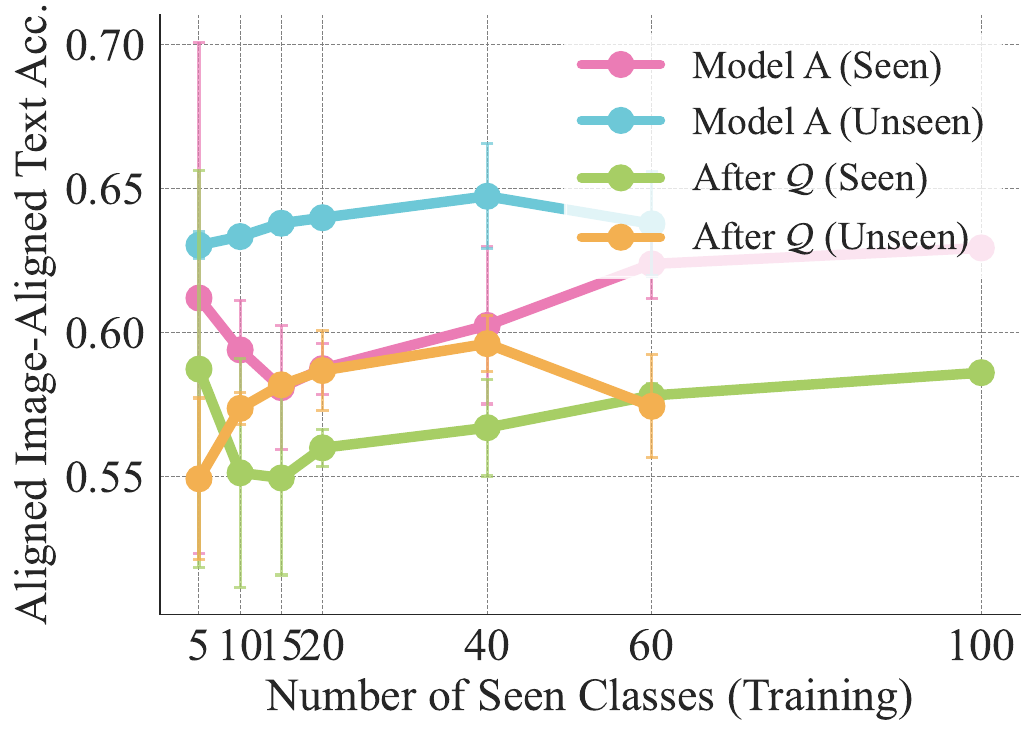}
    \caption{}
    \end{subfigure}
    \caption{\textit{Generalization of orthogonal alignment under limited supervision (CIFAR100; CLIP ViT-B/32 (OpenAI) to CLIP ViT-B/32 (LAION).} (a) Image-image class retrieval, (b) text-text class retrieval, and (c) mean text-text cosine similarity, each reported on seen and unseen classes. (d-f) Downstream transfer: (d) aligned images from model A with text from model B, (e) images from model B with aligned text from model A, and (f) aligned images with aligned text from model A. $\mathcal R$ learned from few classes transfers across modalities and generalizes to unseen classes, achieving near-oracle cross-model retrieval and classification.\looseness=-1}
\label{fig:app_split_cifar_openai_to_laion}
\end{figure*}

\begin{figure*}[!htb]
    \centering
    \begin{subfigure}{0.32\textwidth}
    \centering
    \includegraphics[width=\textwidth]{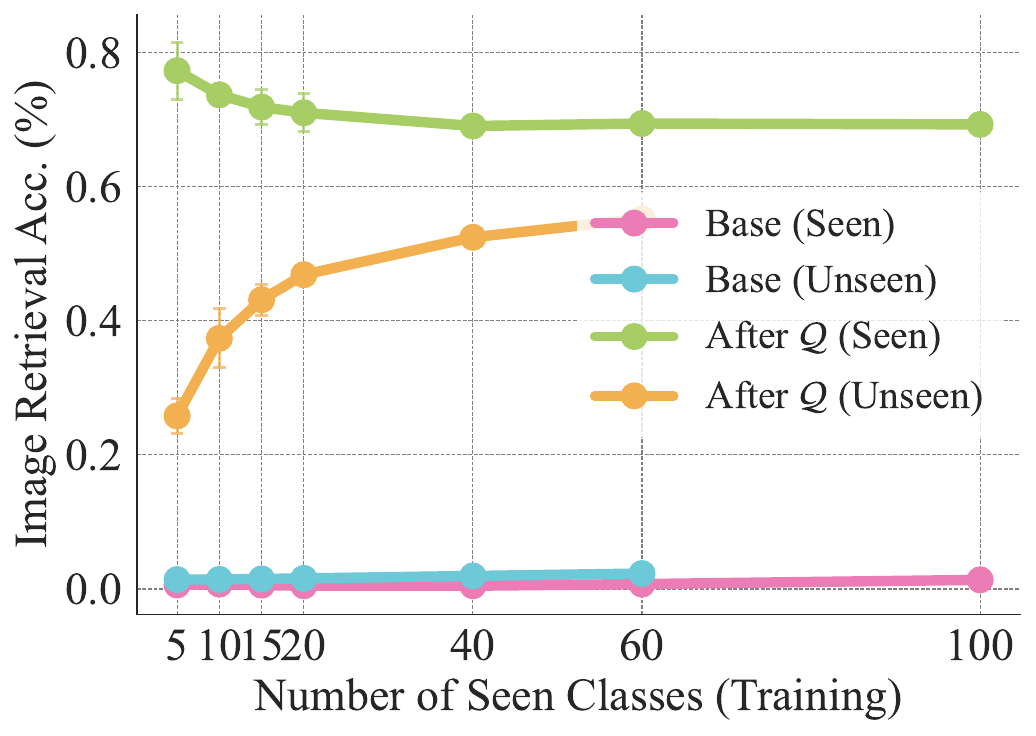}
    \caption{}
    \end{subfigure}
    \begin{subfigure}{0.32\textwidth}
    \centering
    \includegraphics[width=\textwidth]{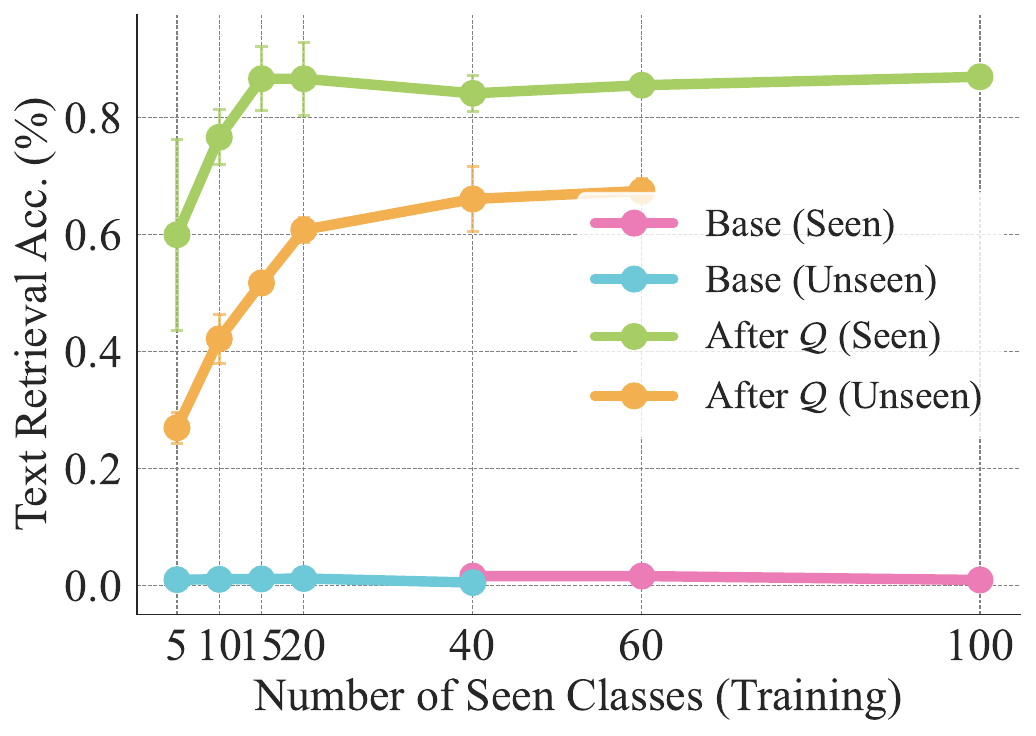}
    \caption{}
    \end{subfigure}
    \begin{subfigure}{0.32\textwidth}
    \centering
    \includegraphics[width=\textwidth]{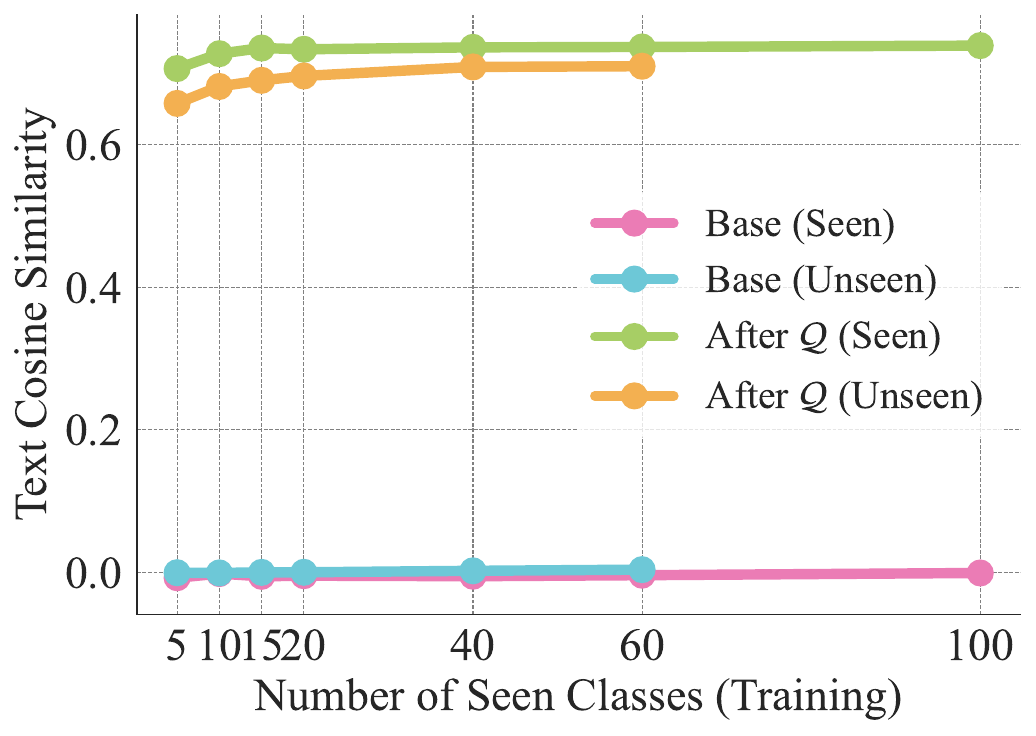}
    \caption{}
    \end{subfigure}
    \begin{subfigure}{0.32\textwidth}
    \centering
    \includegraphics[width=\textwidth]{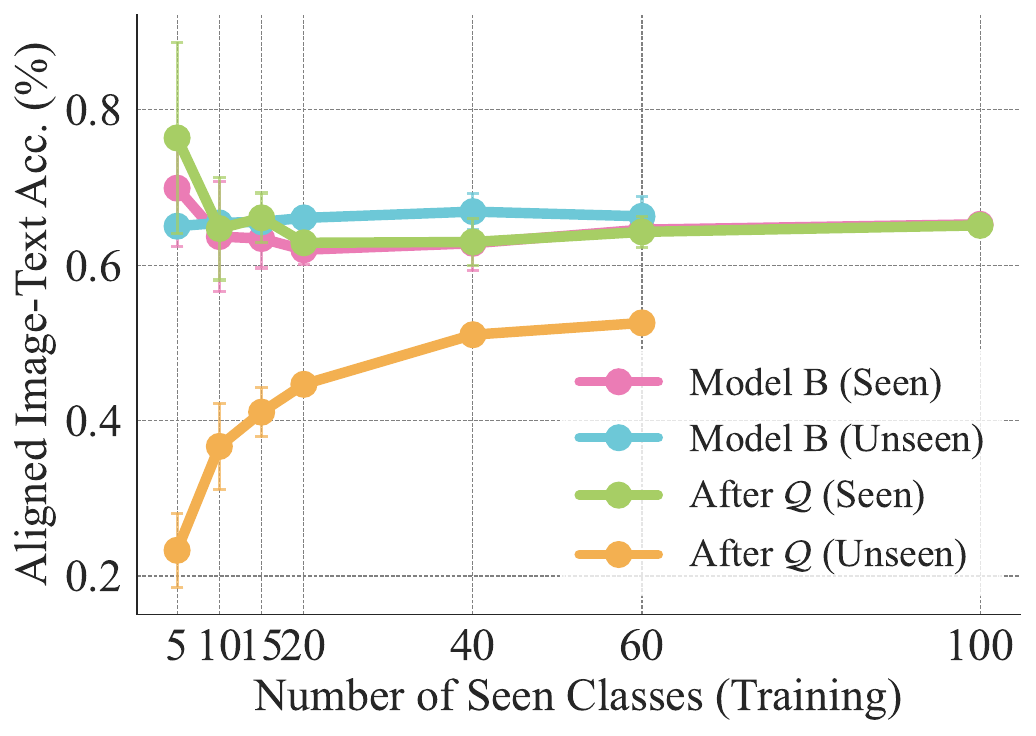}
    \caption{}
    \end{subfigure}
    \begin{subfigure}{0.32\textwidth}
    \centering
    \includegraphics[width=\textwidth]{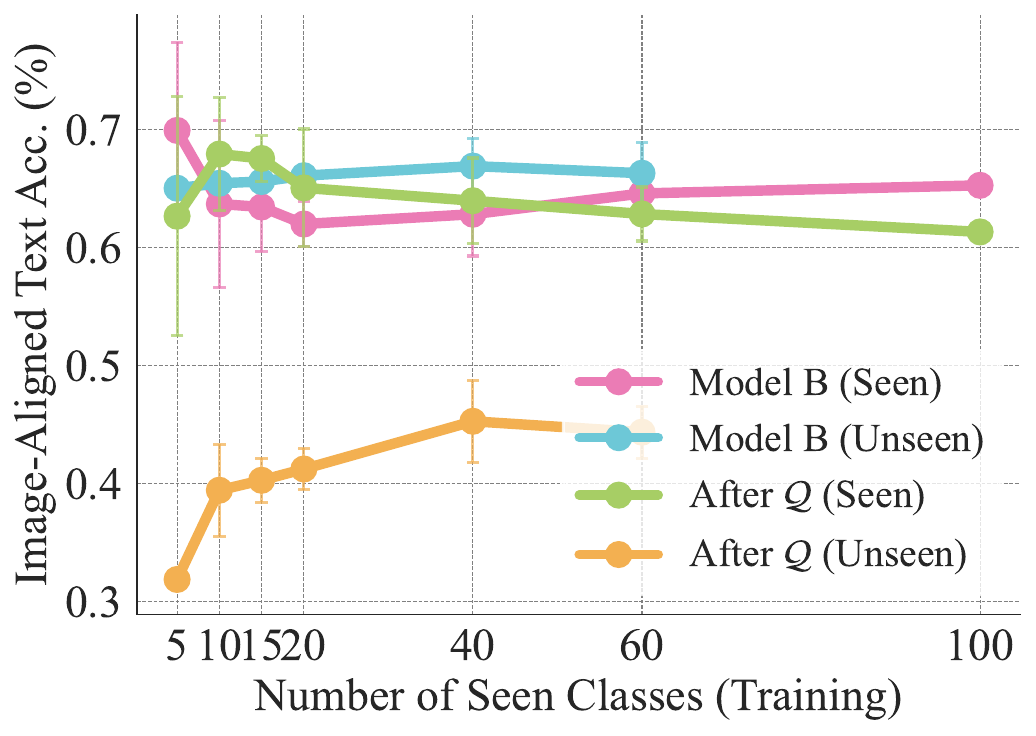}
    \caption{}
    \end{subfigure}
    \begin{subfigure}{0.32\textwidth}
    \centering
    \includegraphics[width=\textwidth]{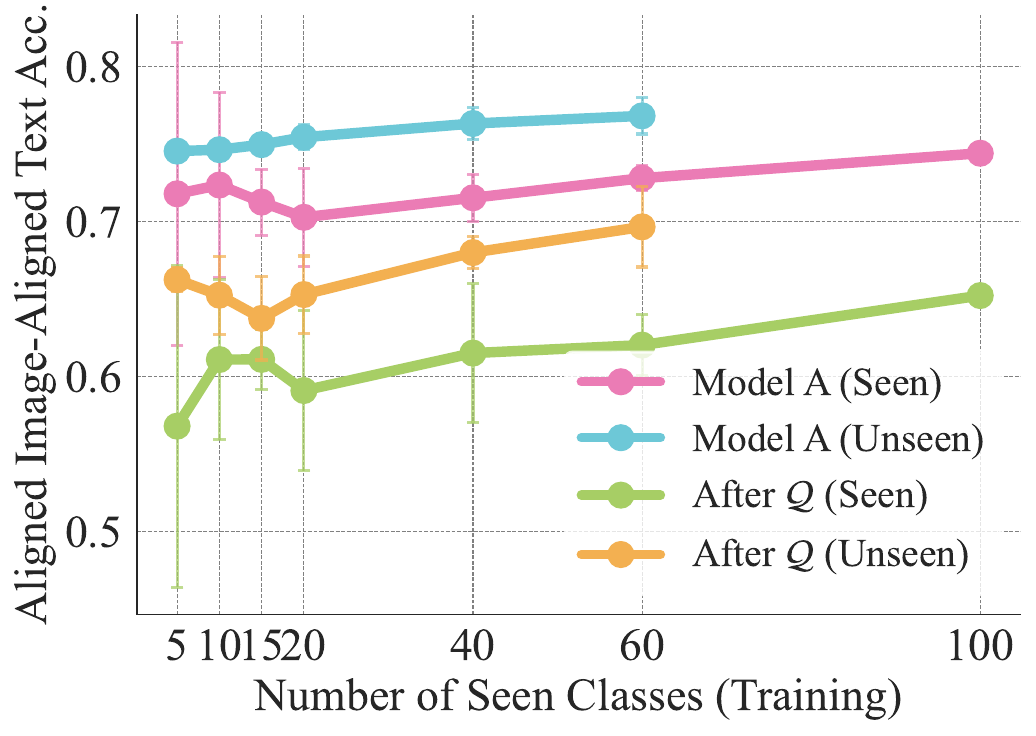}
    \caption{}
    \end{subfigure}
    \caption{\textit{Generalization of orthogonal alignment under limited supervision (CIFAR100; CLIP ViT-L/14 (OpenAI) to FLAVA} (a) Image-image class retrieval, (b) text-text class retrieval, and (c) mean text-text cosine similarity, each reported on seen and unseen classes. (d-f) Downstream transfer: (d) aligned images from model A with text from model B, (e) images from model B with aligned text from model A, and (f) aligned images with aligned text from model A. $\mathcal R$ learned from few classes transfers across modalities and generalizes to unseen classes, achieving near-oracle cross-model retrieval and classification.\looseness=-1}
\label{fig:app_split_cifar_openai_to_flava}
\end{figure*}

\begin{figure*}[!htb]
    \centering
    \begin{subfigure}{0.32\textwidth}
    \centering
    \includegraphics[width=\textwidth]{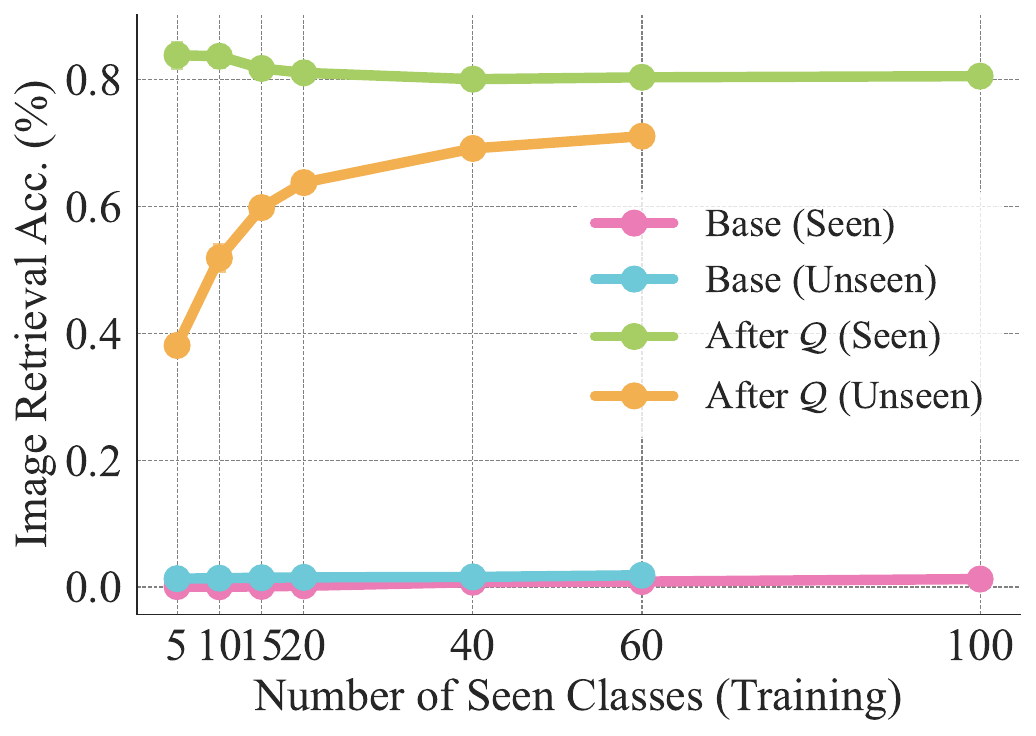}
    \caption{}
    \end{subfigure}
    \begin{subfigure}{0.32\textwidth}
    \centering
    \includegraphics[width=\textwidth]{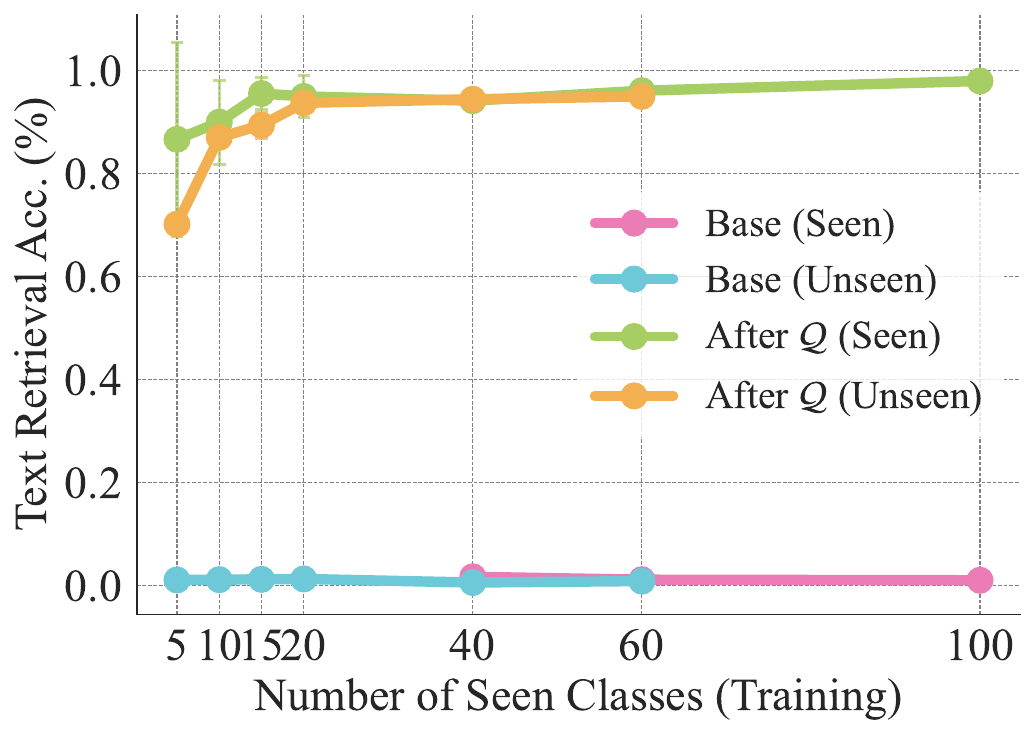}
    \caption{}
    \end{subfigure}
    \begin{subfigure}{0.32\textwidth}
    \centering
    \includegraphics[width=\textwidth]{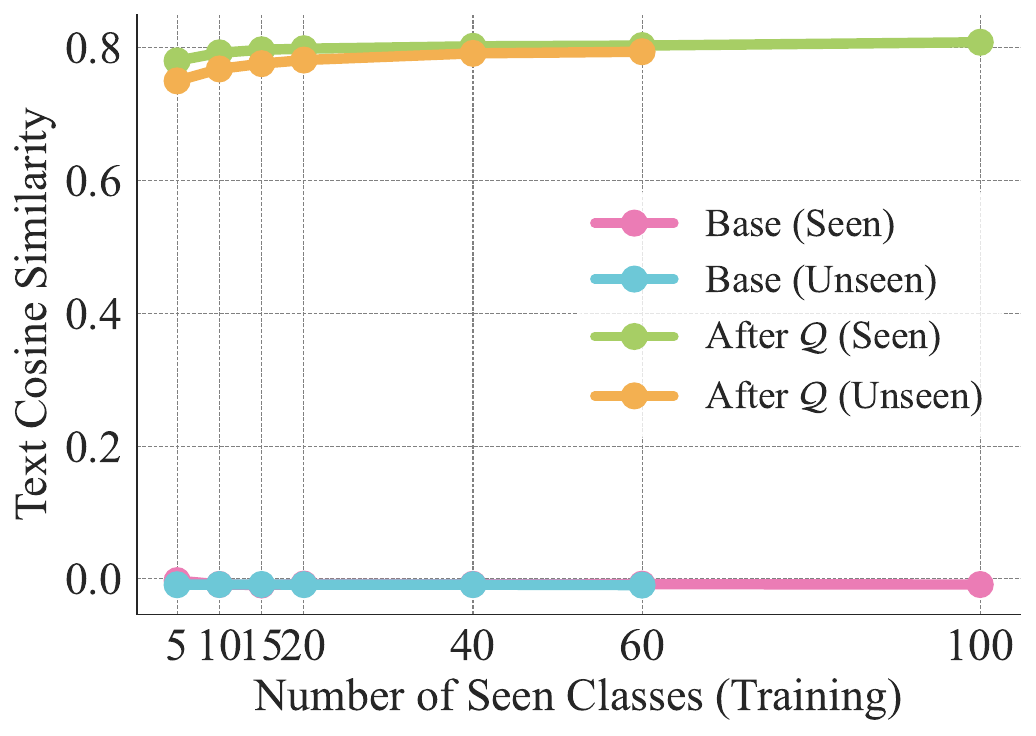}
    \caption{}
    \end{subfigure}
    \begin{subfigure}{0.32\textwidth}
    \centering
    \includegraphics[width=\textwidth]{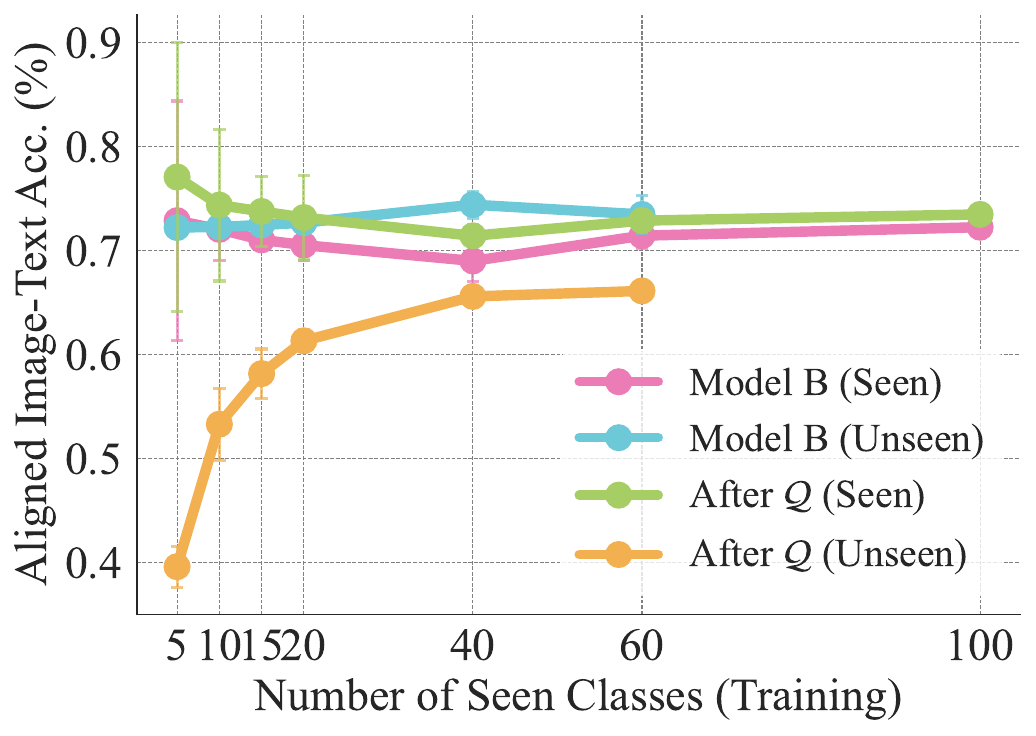}
    \caption{}
    \end{subfigure}
    \begin{subfigure}{0.32\textwidth}
    \centering
    \includegraphics[width=\textwidth]{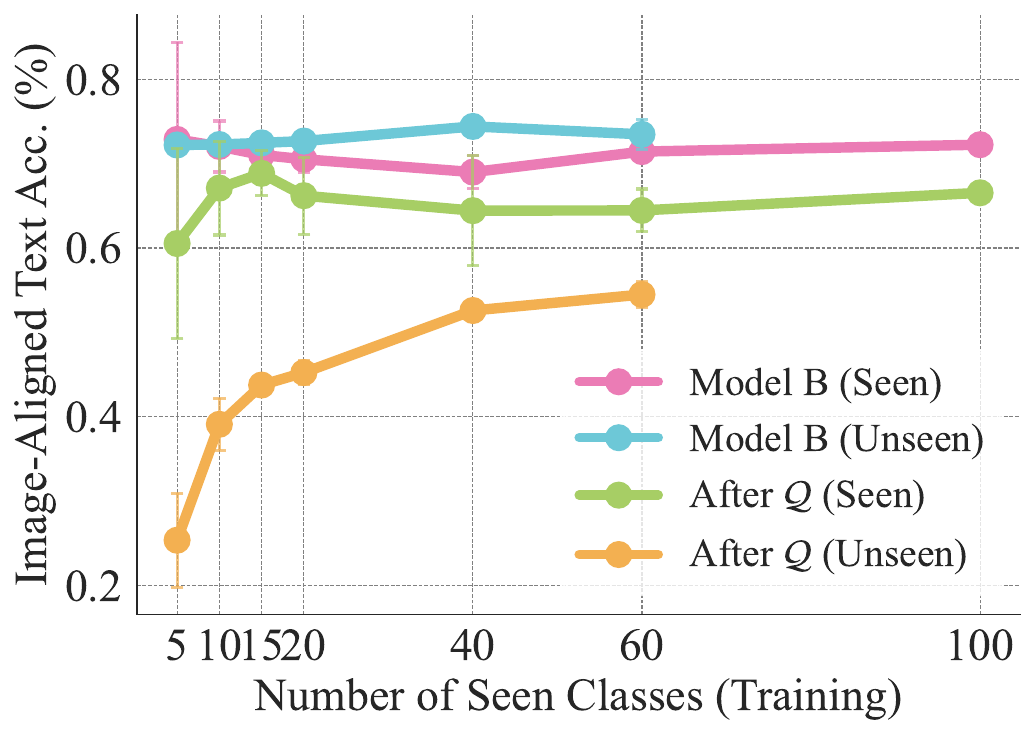}
    \caption{}
    \end{subfigure}
    \begin{subfigure}{0.32\textwidth}
    \centering
    \includegraphics[width=\textwidth]{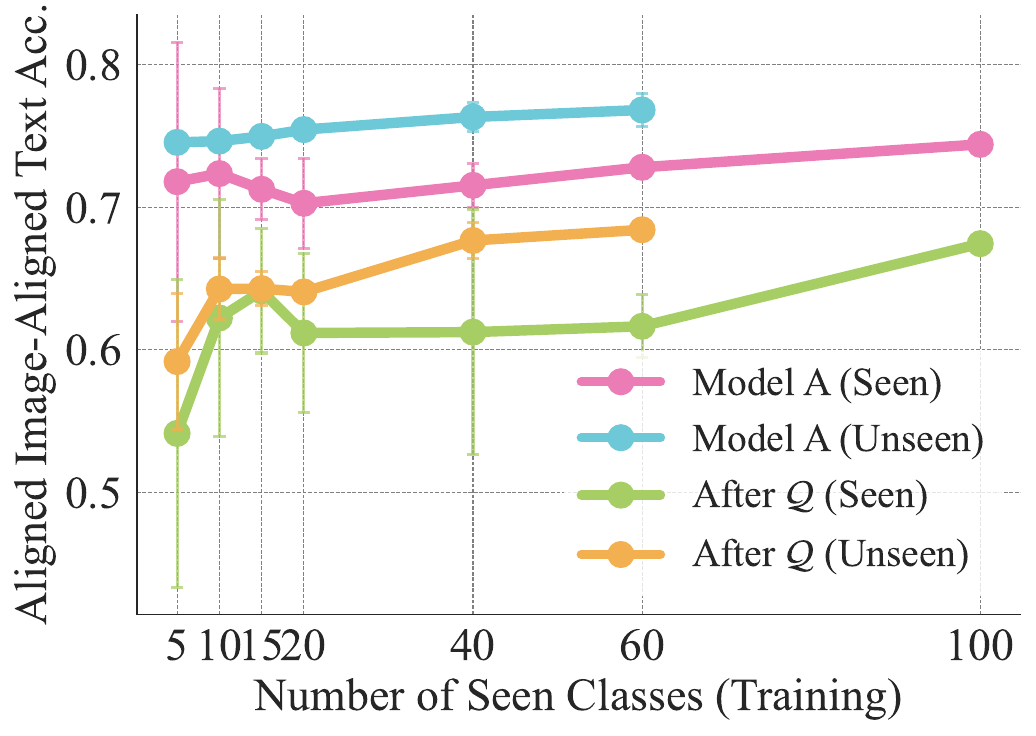}
    \caption{}
    \end{subfigure}
    \caption{\textit{Generalization of orthogonal alignment under limited supervision (CIFAR100; CLIP ViT-L/14 (OpenAI) to SigLIP.} (a) Image-image class retrieval, (b) text-text class retrieval, and (c) mean text-text cosine similarity, each reported on seen and unseen classes. (d-f) Downstream transfer: (d) aligned images from model A with text from model B, (e) images from model B with aligned text from model A, and (f) aligned images with aligned text from model A. $\mathcal R$ learned from few classes transfers across modalities and generalizes to unseen classes, achieving near-oracle cross-model retrieval and classification.\looseness=-1}
\label{fig:app_split_cifar_openai_to_siglip}
\end{figure*}

\begin{figure*}[!htb]
    \centering
    \begin{subfigure}{0.32\textwidth}
    \centering
    \includegraphics[width=\textwidth]{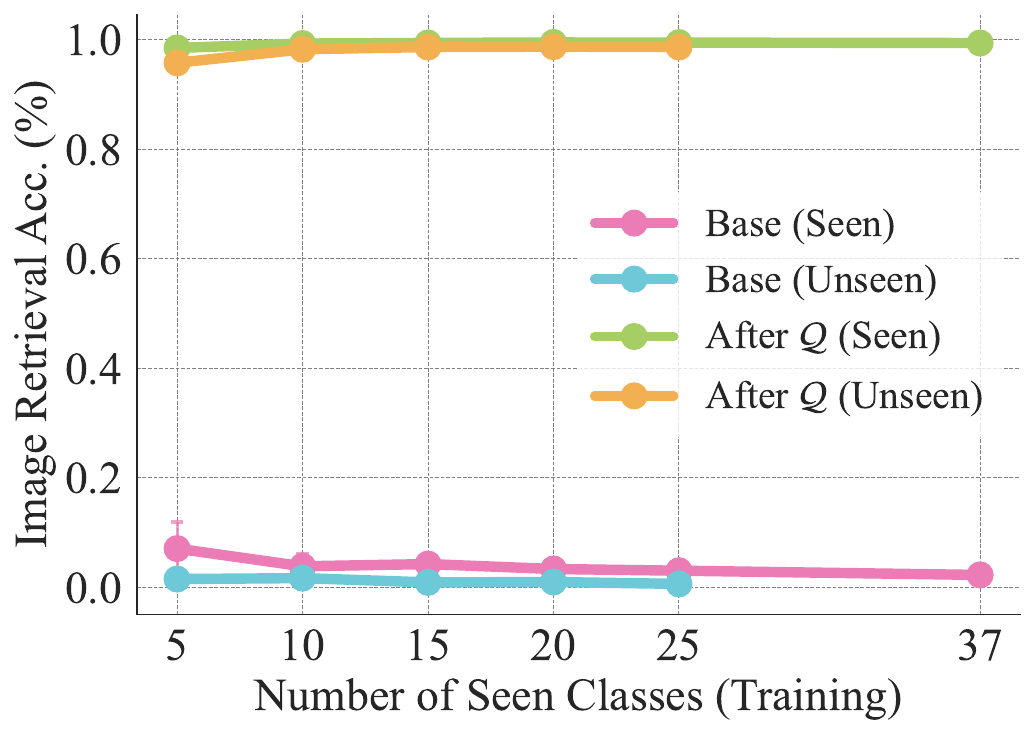}
    \caption{}
    \end{subfigure}
    \begin{subfigure}{0.32\textwidth}
    \centering
    \includegraphics[width=\textwidth]{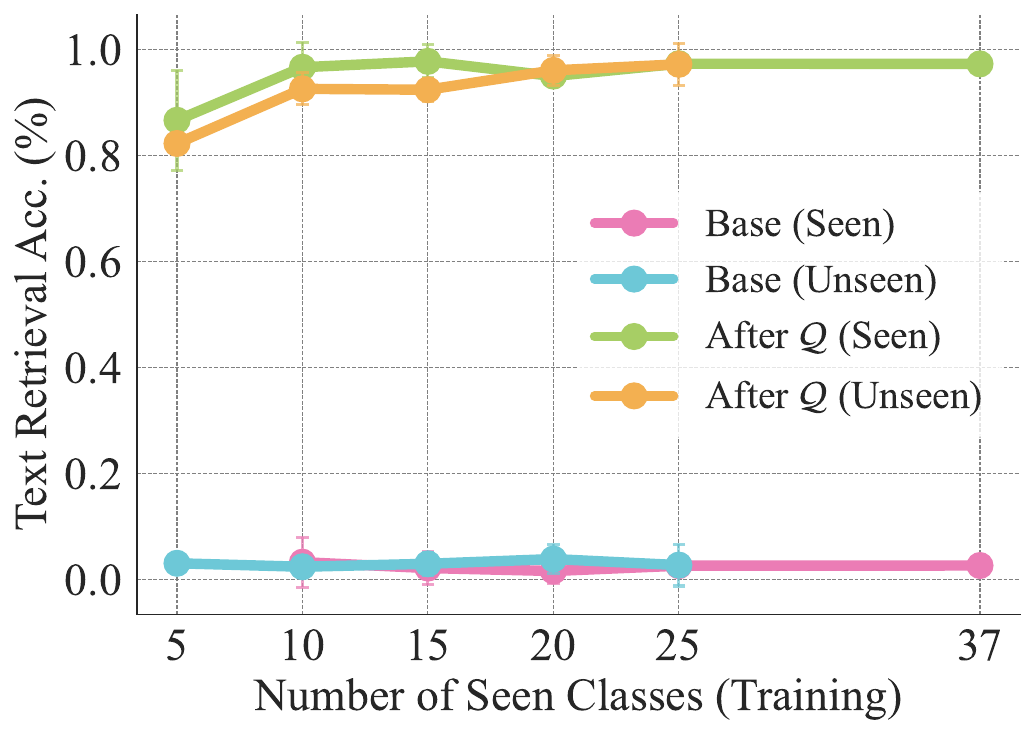}
    \caption{}
    \end{subfigure}
    \begin{subfigure}{0.32\textwidth}
    \centering
    \includegraphics[width=\textwidth]{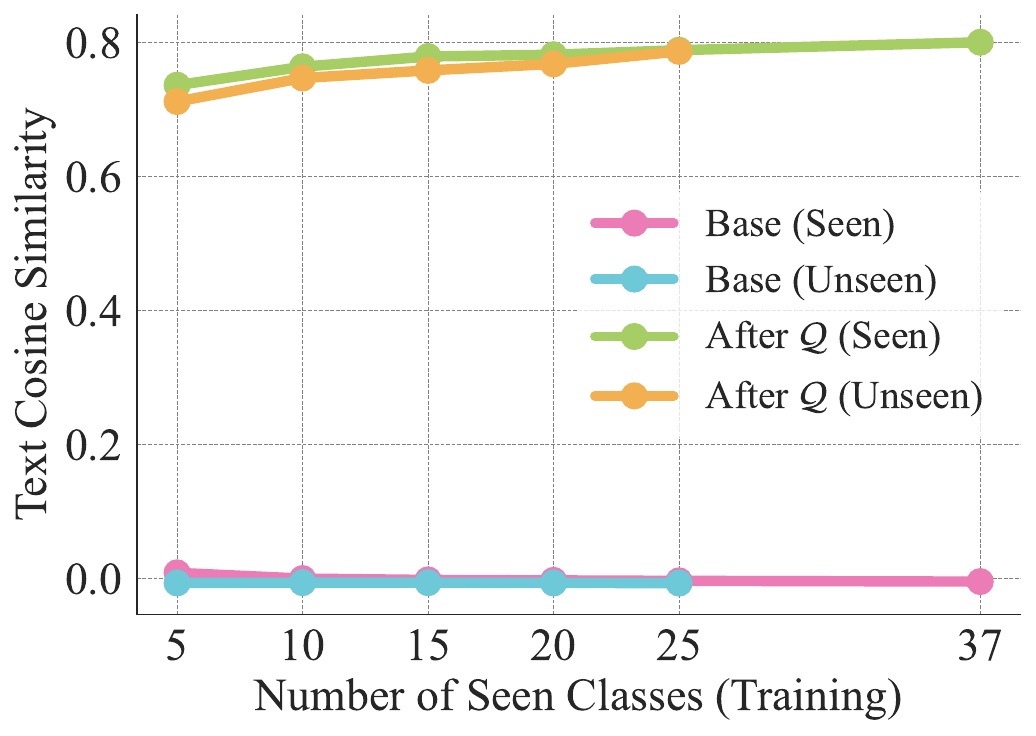}
    \caption{}
    \end{subfigure}
    \begin{subfigure}{0.32\textwidth}
    \centering
    \includegraphics[width=\textwidth]{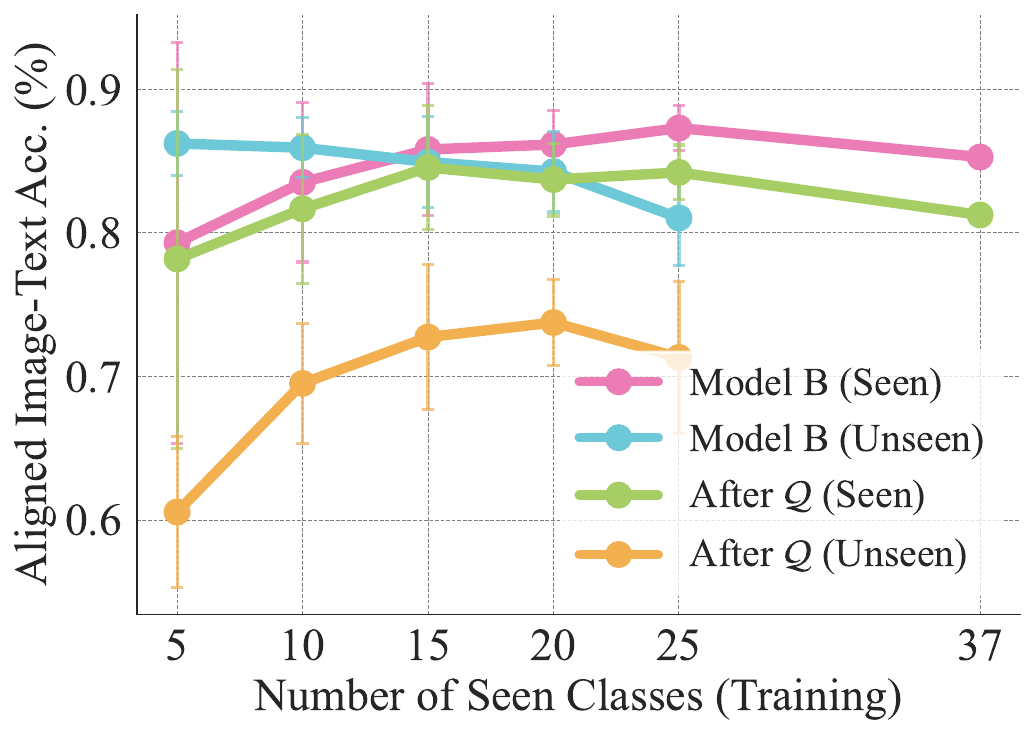}
    \caption{}
    \end{subfigure}
    \begin{subfigure}{0.32\textwidth}
    \centering
    \includegraphics[width=\textwidth]{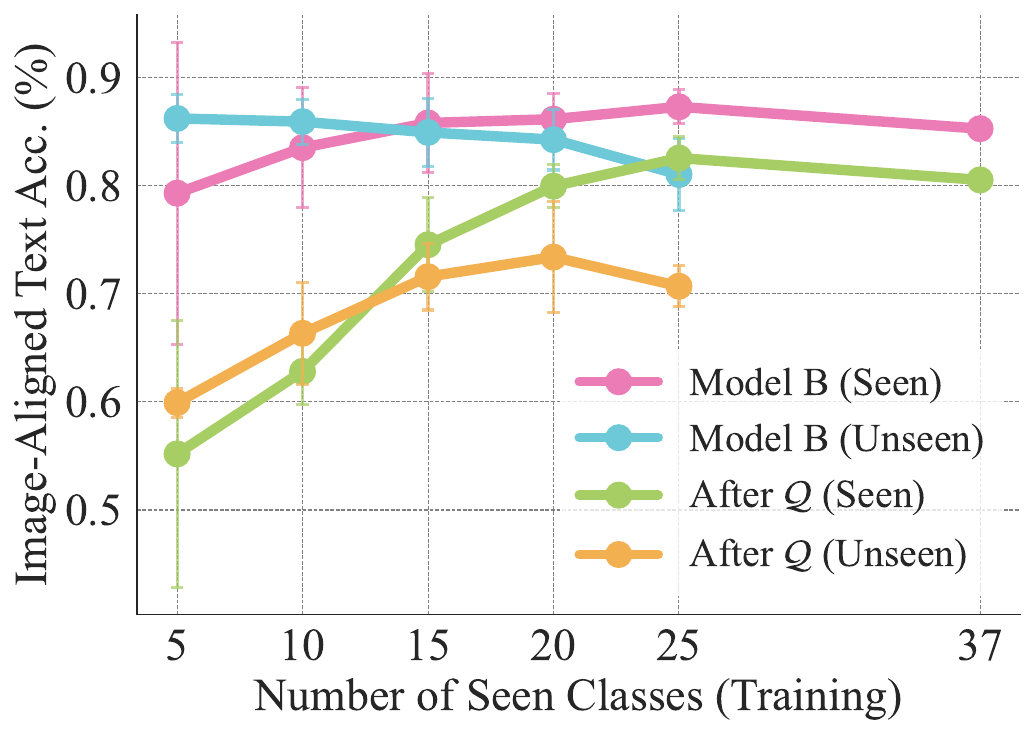}
    \caption{}
    \end{subfigure}
    \begin{subfigure}{0.32\textwidth}
    \centering
    \includegraphics[width=\textwidth]{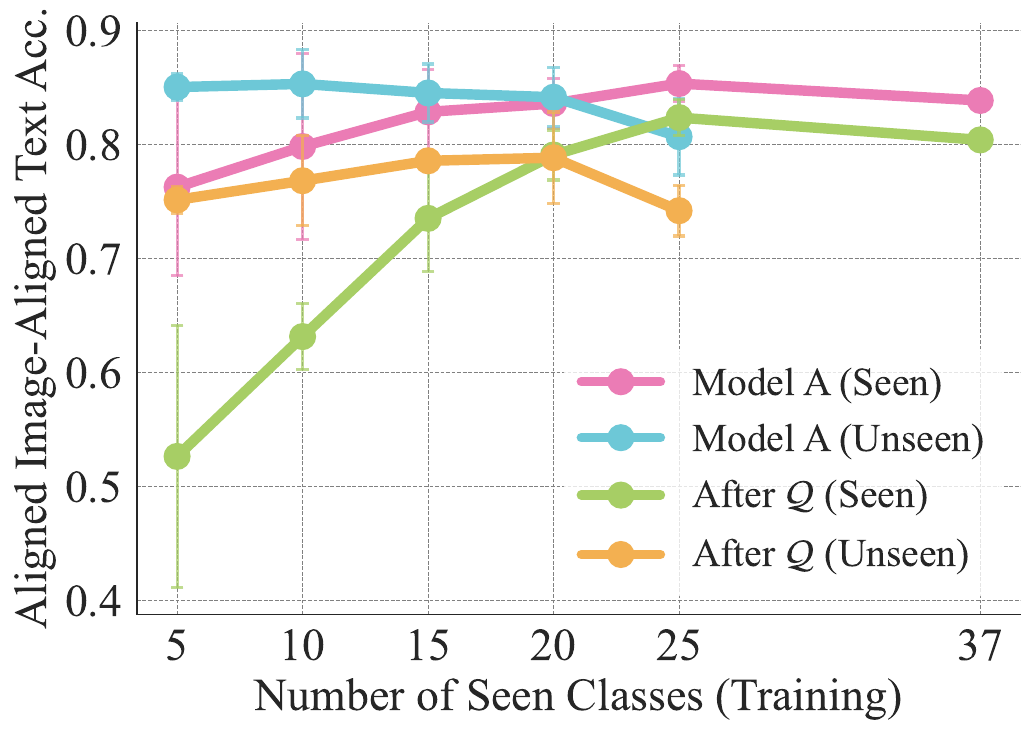}
    \caption{}
    \end{subfigure}
    \caption{\textit{Generalization of orthogonal alignment under limited supervision (Oxford Pets; CLIP ViT-B/32 to CLIP ViT-B/16, OpenAI).} (a) Image-image class retrieval, (b) text-text class retrieval, and (c) mean text-text cosine similarity, each reported on seen and unseen classes. (d-f) Downstream transfer: (d) aligned images from model A with text from model B, (e) images from model B with aligned text from model A, and (f) aligned images with aligned text from model A. $\mathcal R$ learned from few classes transfers across modalities and generalizes to unseen classes, achieving near-oracle cross-model retrieval and classification.\looseness=-1}
\label{fig:app_split_oxford_openai_to_openai}
\end{figure*}

\begin{figure*}[!htb]
    \centering
    \begin{subfigure}{0.32\textwidth}
    \centering
    \includegraphics[width=\textwidth]{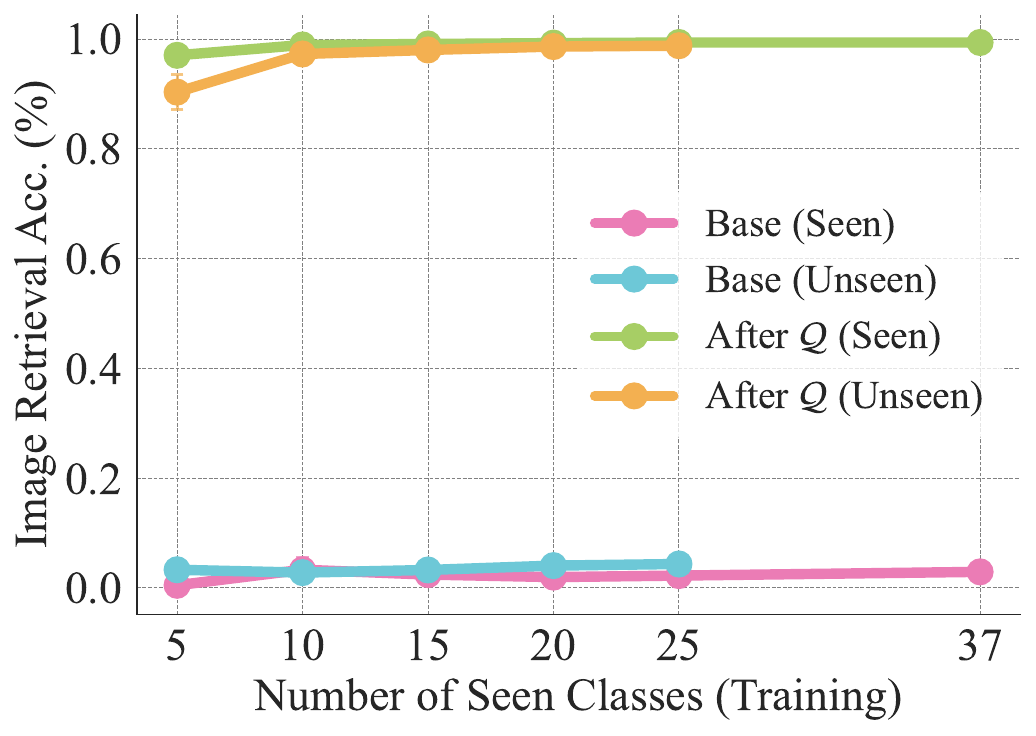}
    \caption{}
    \end{subfigure}
    \begin{subfigure}{0.32\textwidth}
    \centering
    \includegraphics[width=\textwidth]{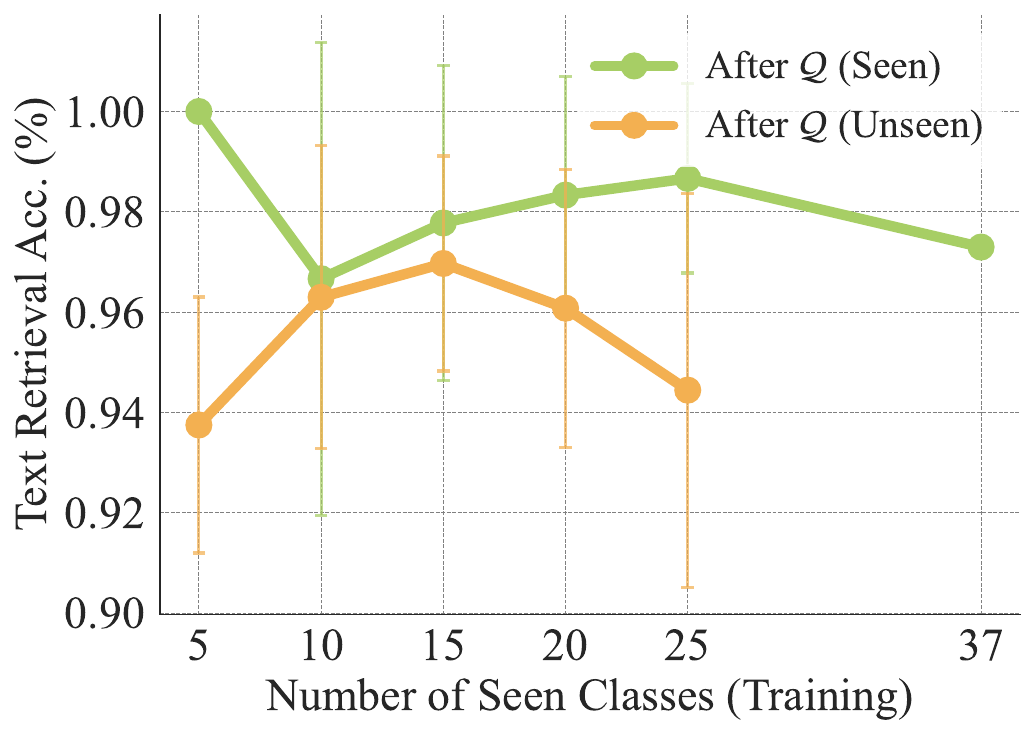}
    \caption{}
    \end{subfigure}
    \begin{subfigure}{0.32\textwidth}
    \centering
    \includegraphics[width=\textwidth]{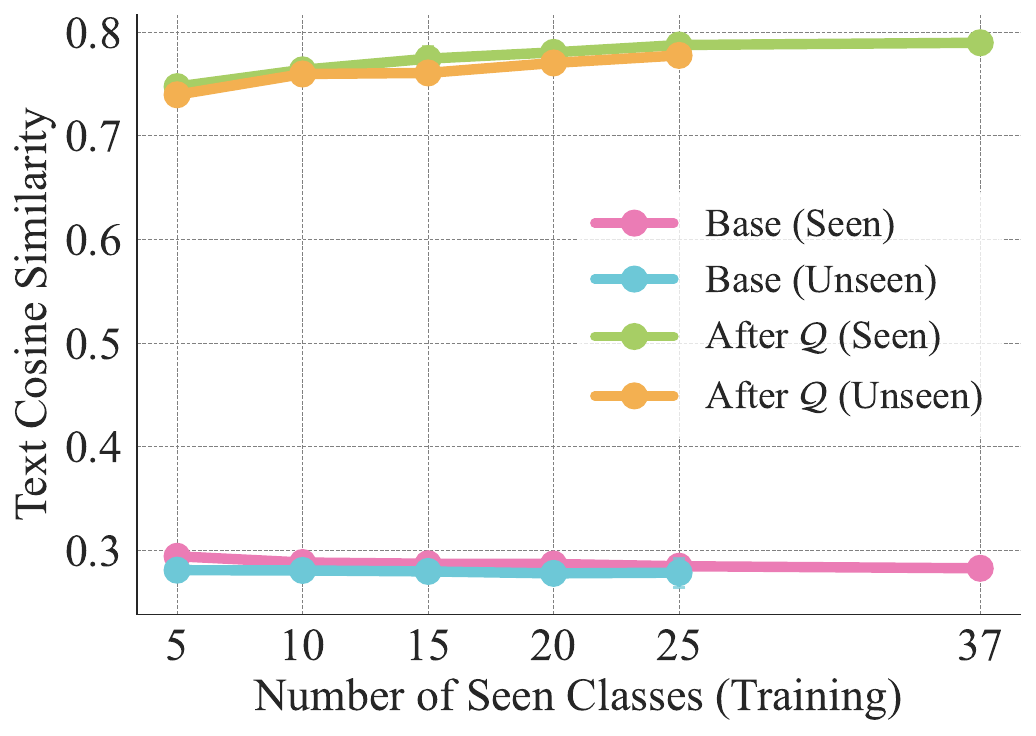}
    \caption{}
    \end{subfigure}
    \begin{subfigure}{0.32\textwidth}
    \centering
    \includegraphics[width=\textwidth]{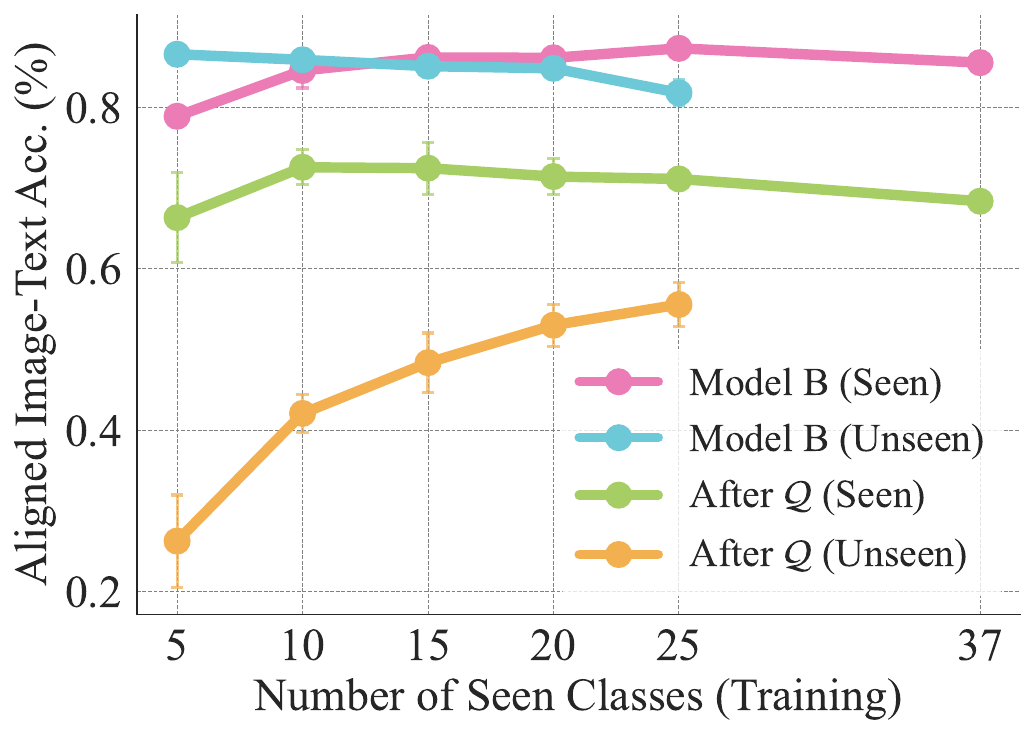}
    \caption{}
    \end{subfigure}
    \begin{subfigure}{0.32\textwidth}
    \centering
    \includegraphics[width=\textwidth]{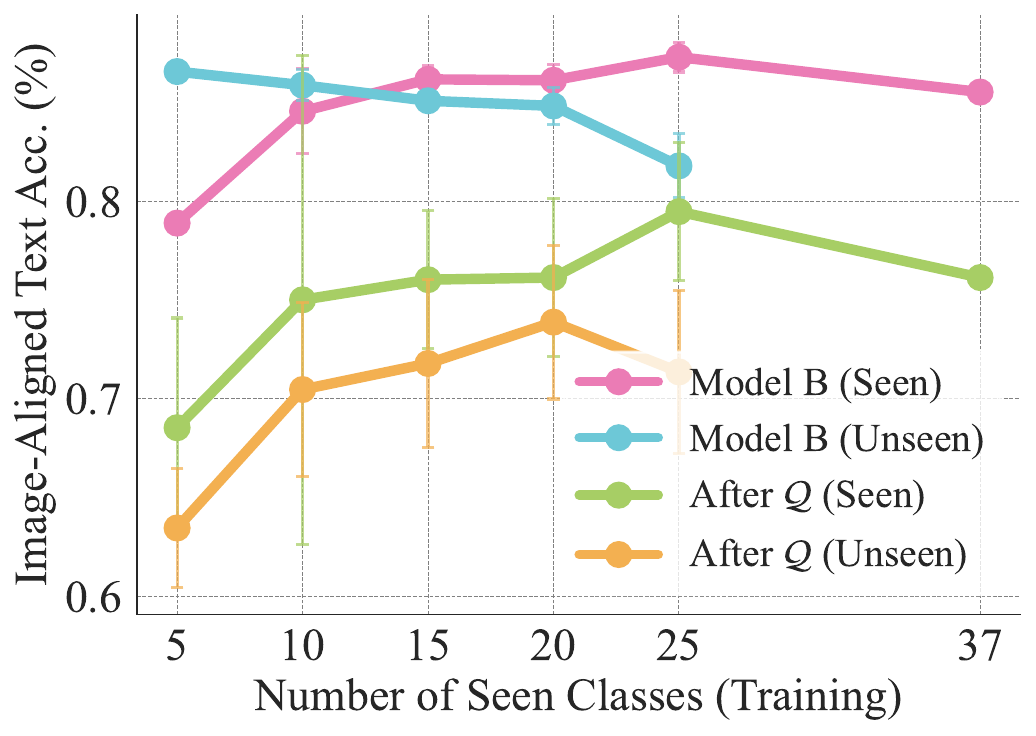}
    \caption{}
    \end{subfigure}
    \begin{subfigure}{0.32\textwidth}
    \centering
    \includegraphics[width=\textwidth]{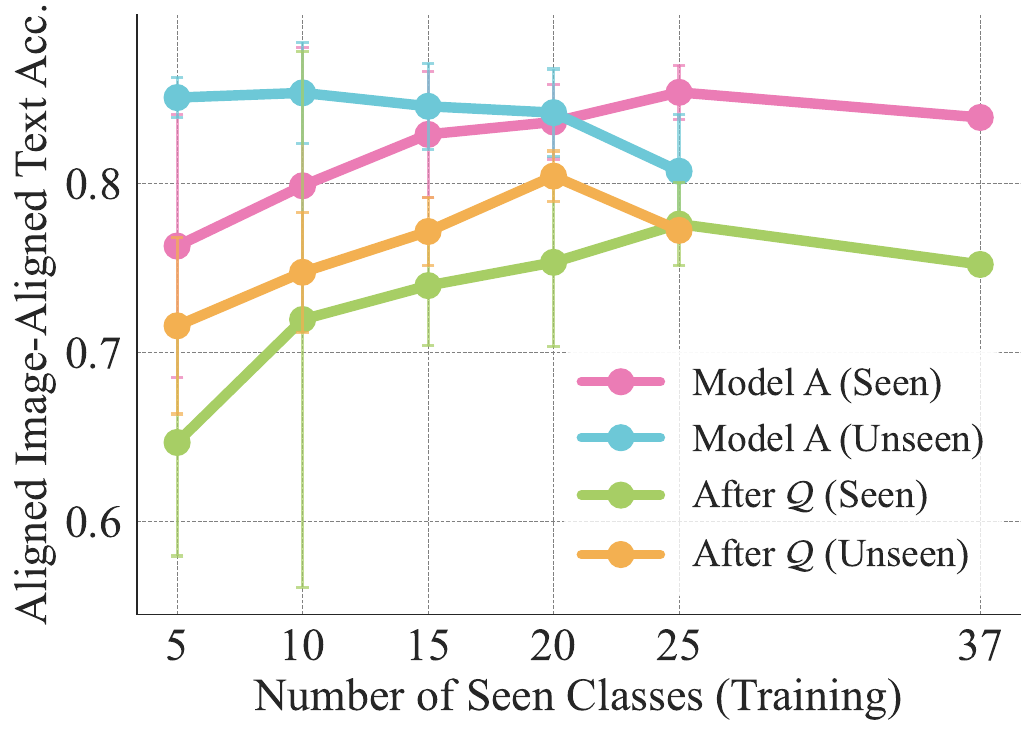}
    \caption{}
    \end{subfigure}
    \caption{\textit{Generalization of orthogonal alignment under limited supervision (Oxford Pets; CLIP ViT-B/32 (OpenAI) to CLIP ViT-B/32 (LAION).} (a) Image-image class retrieval, (b) text-text class retrieval, and (c) mean text-text cosine similarity, each reported on seen and unseen classes. (d-f) Downstream transfer: (d) aligned images from model A with text from model B, (e) images from model B with aligned text from model A, and (f) aligned images with aligned text from model A. $\mathcal R$ learned from few classes transfers across modalities and generalizes to unseen classes, achieving near-oracle cross-model retrieval and classification.\looseness=-1}
\label{fig:app_split_oxford_openai_to_laion}
\end{figure*}

\begin{figure*}[!htb]
    \centering
    \begin{subfigure}{0.32\textwidth}
    \centering
    \includegraphics[width=\textwidth]{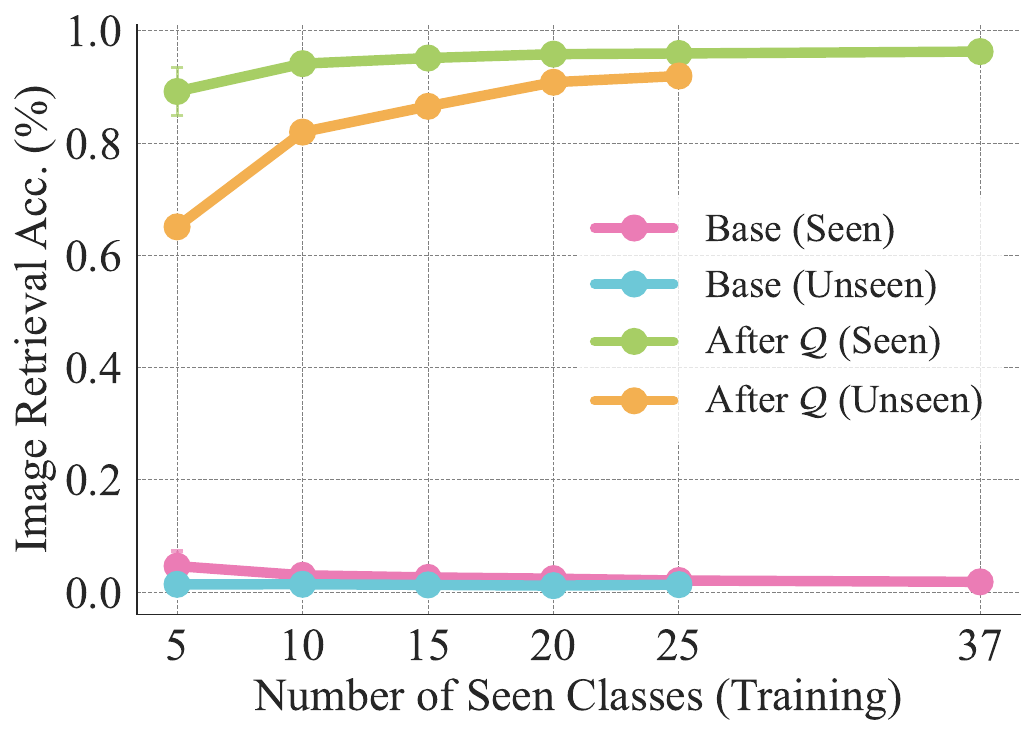}
    \caption{}
    \end{subfigure}
    \begin{subfigure}{0.32\textwidth}
    \centering
    \includegraphics[width=\textwidth]{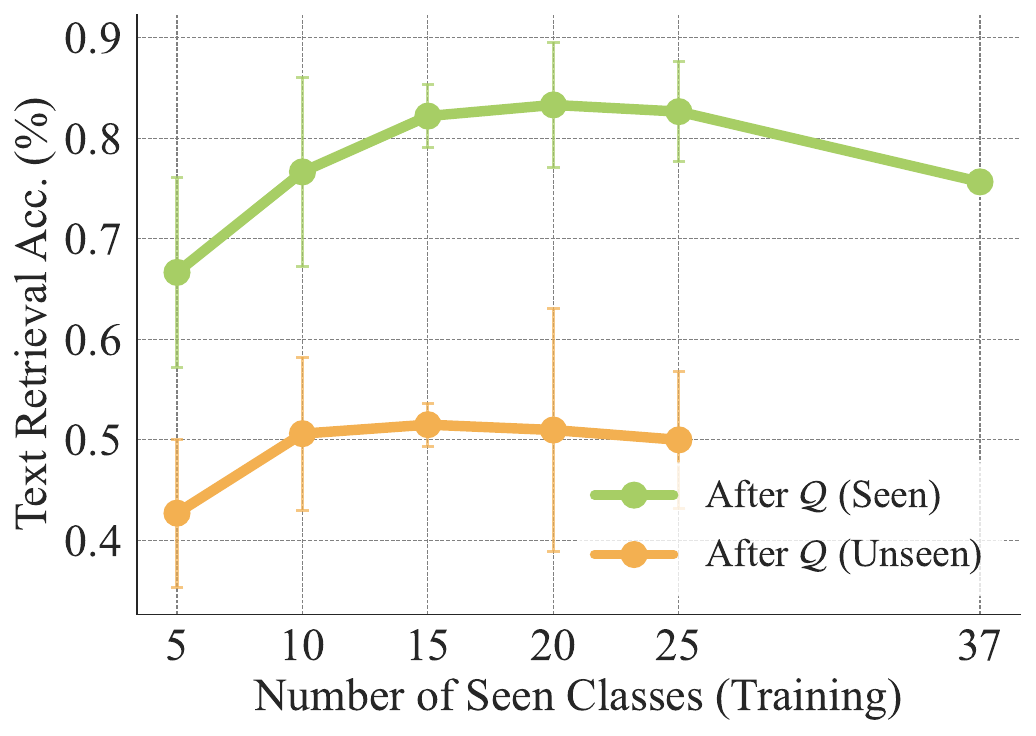}
    \caption{}
    \end{subfigure}
    \begin{subfigure}{0.32\textwidth}
    \centering
    \includegraphics[width=\textwidth]{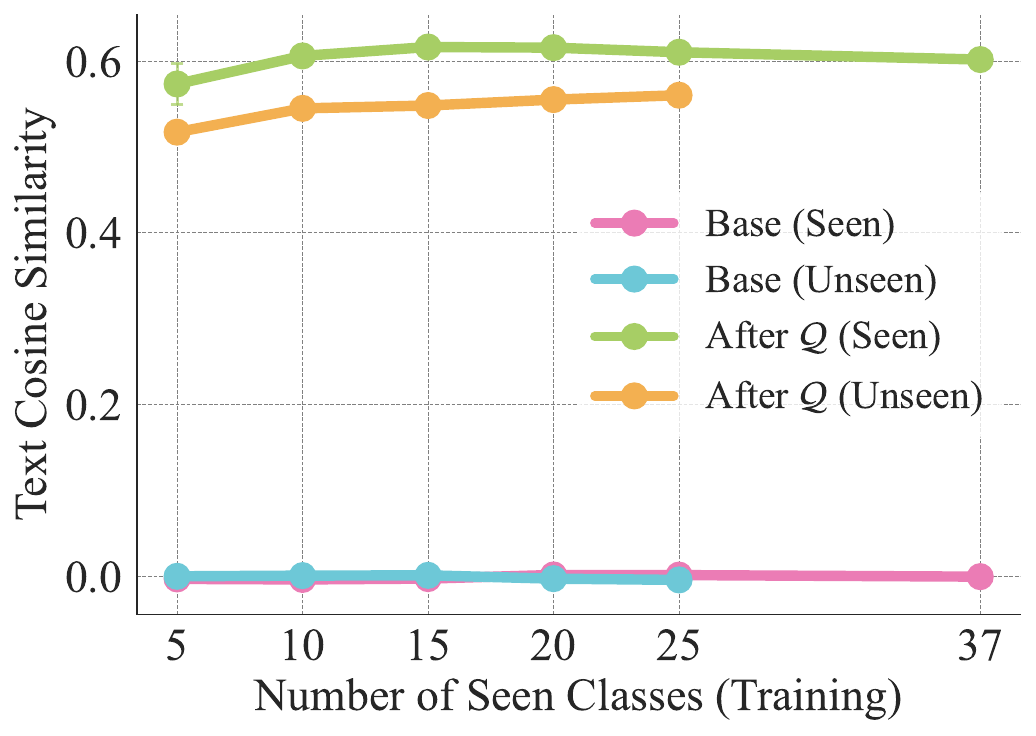}
    \caption{}
    \end{subfigure}
    \begin{subfigure}{0.32\textwidth}
    \centering
    \includegraphics[width=\textwidth]{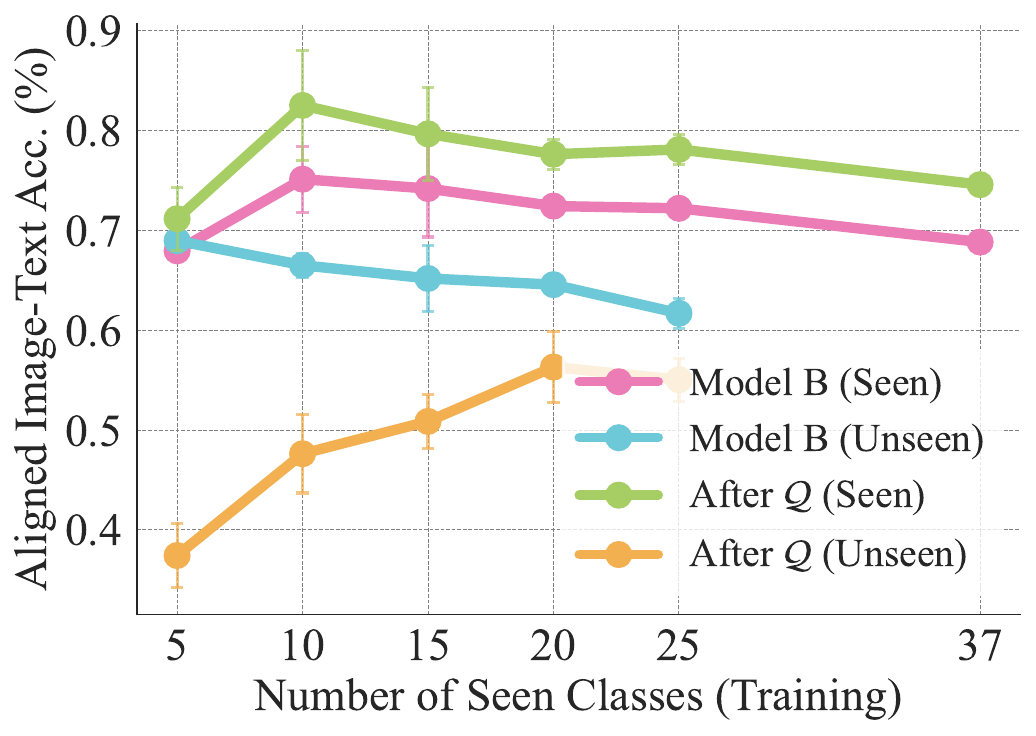}
    \caption{}
    \end{subfigure}
    \begin{subfigure}{0.32\textwidth}
    \centering
    \includegraphics[width=\textwidth]{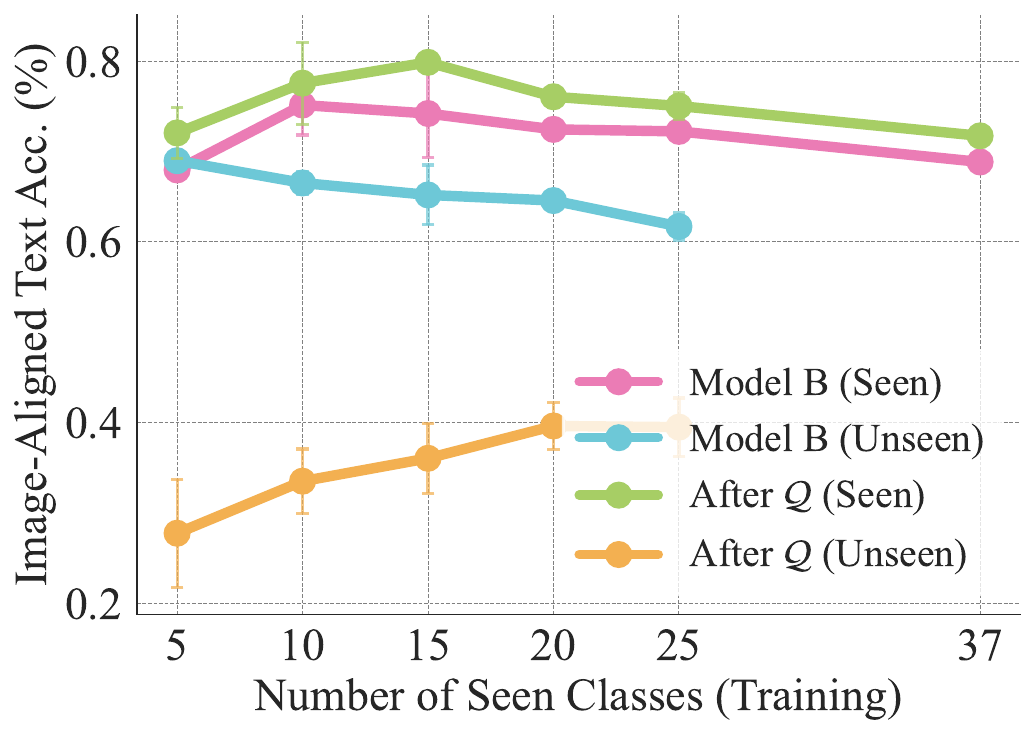}
    \caption{}
    \end{subfigure}
    \begin{subfigure}{0.32\textwidth}
    \centering
    \includegraphics[width=\textwidth]{arxiv_figs/appendix/oxford_A=ViT-L-14-openai__B=flava-facebook_flava-full_4_aligned_img_to_aligned_text.pdf}
    \caption{}
    \end{subfigure}
    \caption{\textit{Generalization of orthogonal alignment under limited supervision (Oxford Pets; CLIP ViT-L/14 (OpenAI) to FLAVA} (a) Image-image class retrieval, (b) text-text class retrieval, and (c) mean text-text cosine similarity, each reported on seen and unseen classes. (d-f) Downstream transfer: (d) aligned images from model A with text from model B, (e) images from model B with aligned text from model A, and (f) aligned images with aligned text from model A. $\mathcal R$ learned from few classes transfers across modalities and generalizes to unseen classes, achieving near-oracle cross-model retrieval and classification.\looseness=-1}
\label{fig:app_split_oxford_openai_to_flava}
\end{figure*}

\begin{figure*}[!htb]
    \centering
    \begin{subfigure}{0.32\textwidth}
    \centering
    \includegraphics[width=\textwidth]{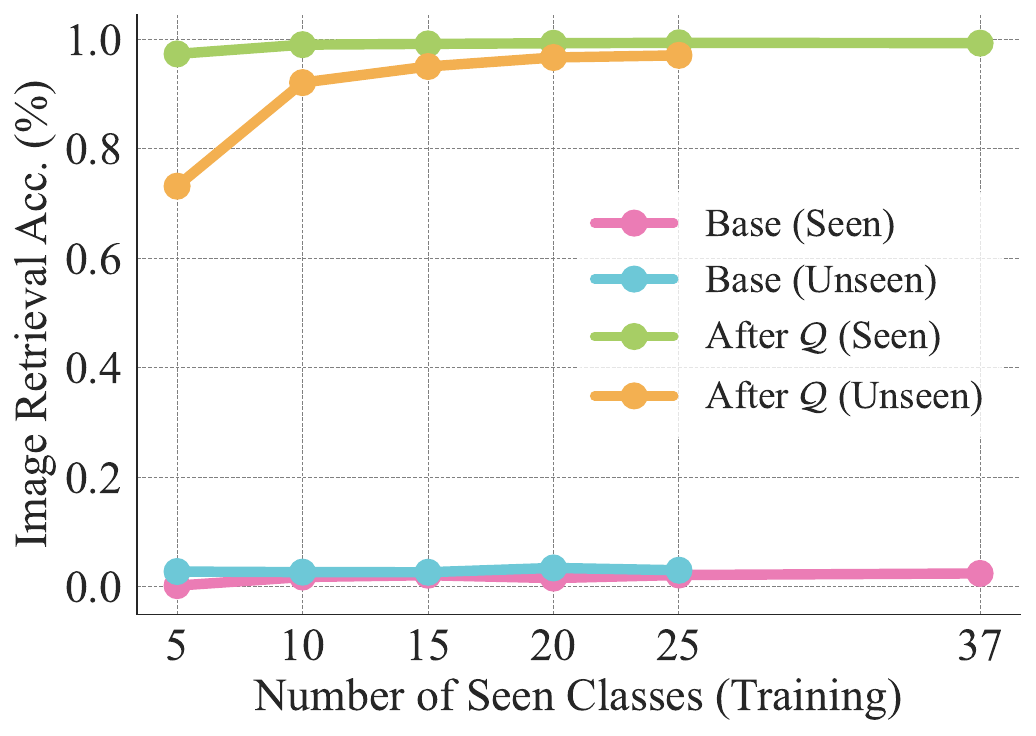}
    \caption{}
    \end{subfigure}
    \begin{subfigure}{0.32\textwidth}
    \centering
    \includegraphics[width=\textwidth]{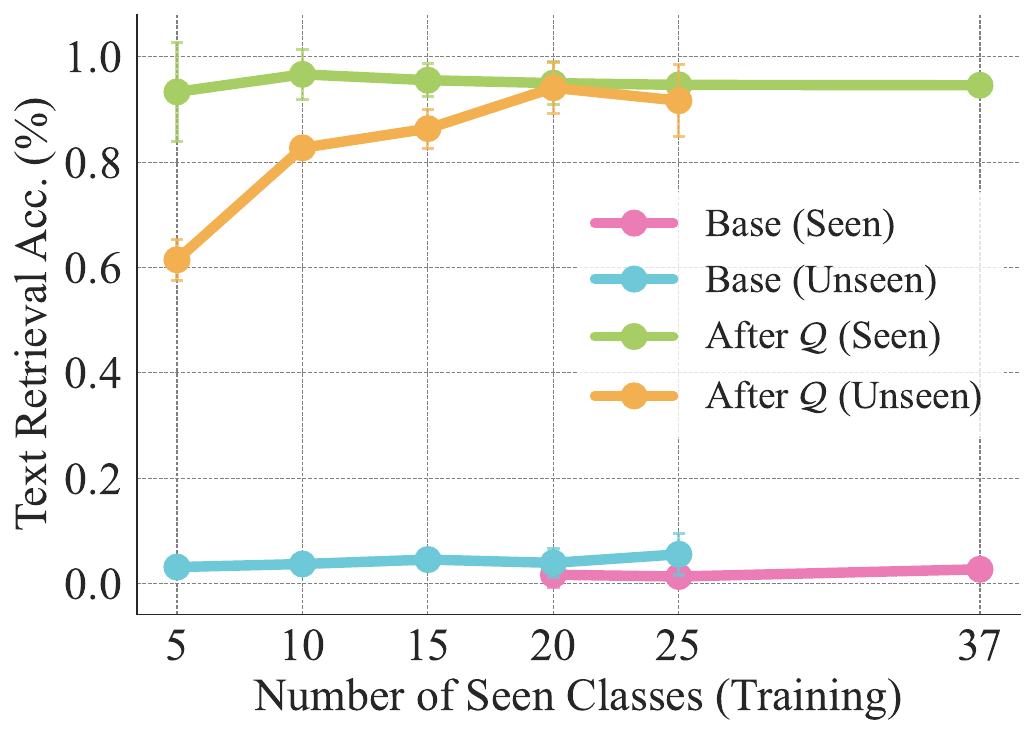}
    \caption{}
    \end{subfigure}
    \begin{subfigure}{0.32\textwidth}
    \centering
    \includegraphics[width=\textwidth]{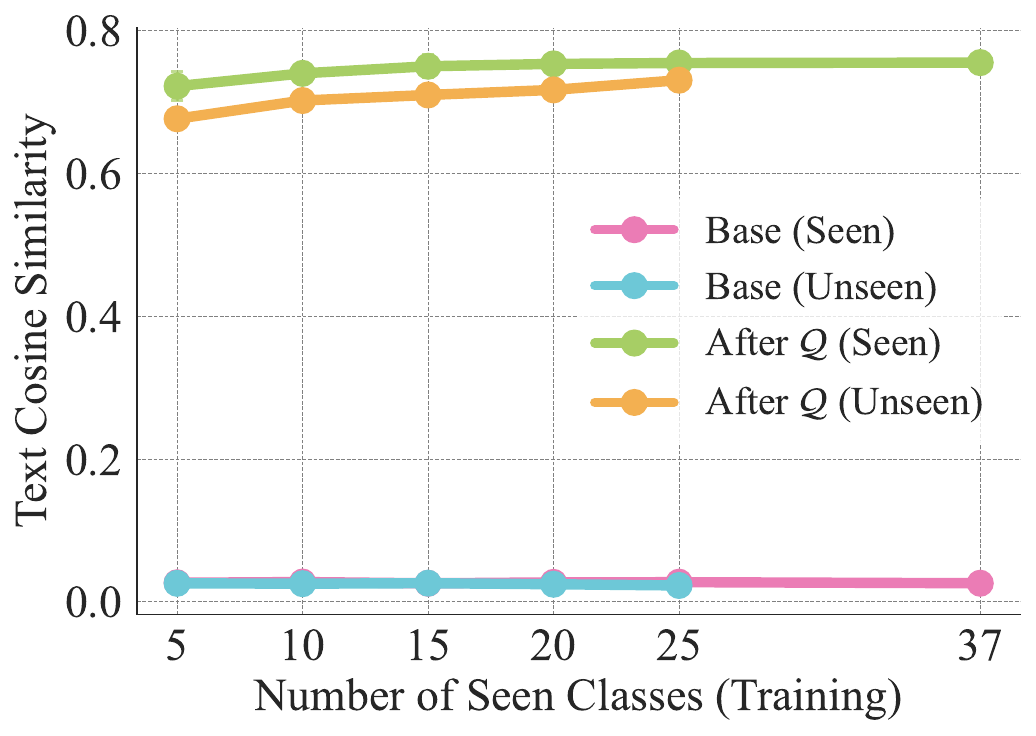}
    \caption{}
    \end{subfigure}
    \begin{subfigure}{0.32\textwidth}
    \centering
    \includegraphics[width=\textwidth]{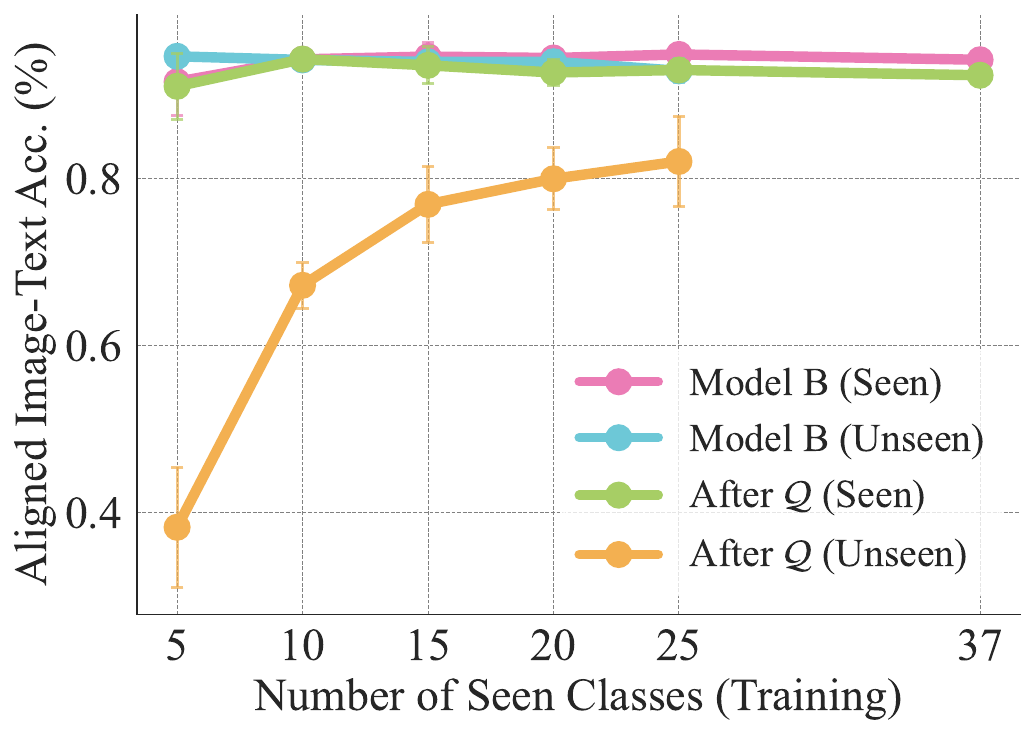}
    \caption{}
    \end{subfigure}
    \begin{subfigure}{0.32\textwidth}
    \centering
    \includegraphics[width=\textwidth]{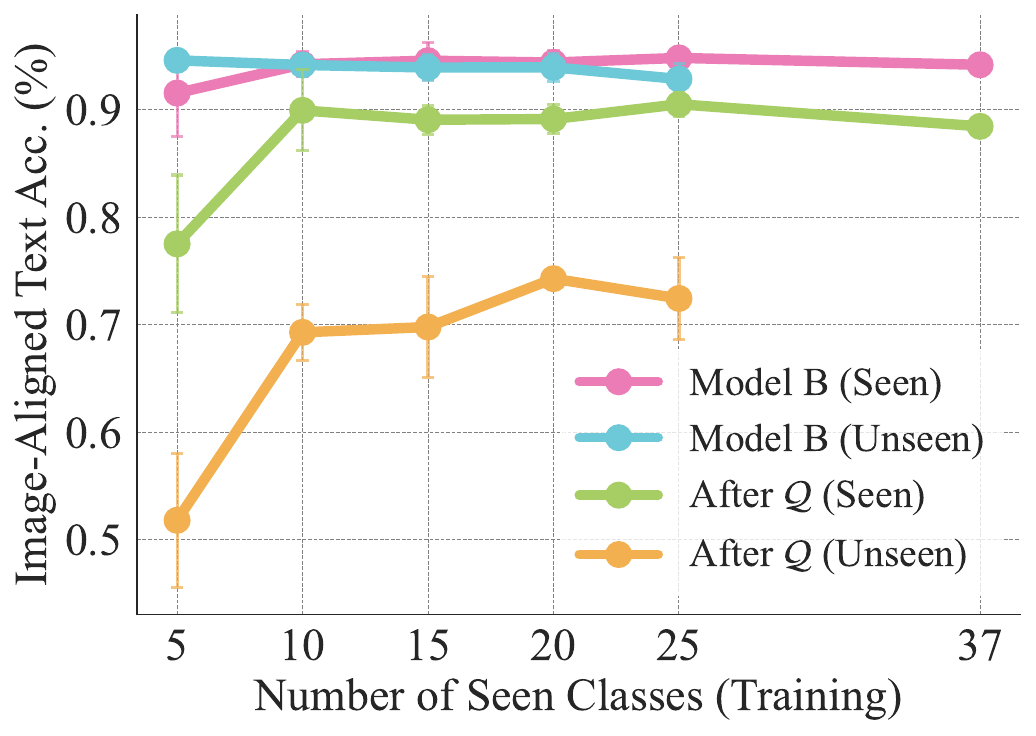}
    \caption{}
    \end{subfigure}
    \begin{subfigure}{0.32\textwidth}
    \centering
    \includegraphics[width=\textwidth]{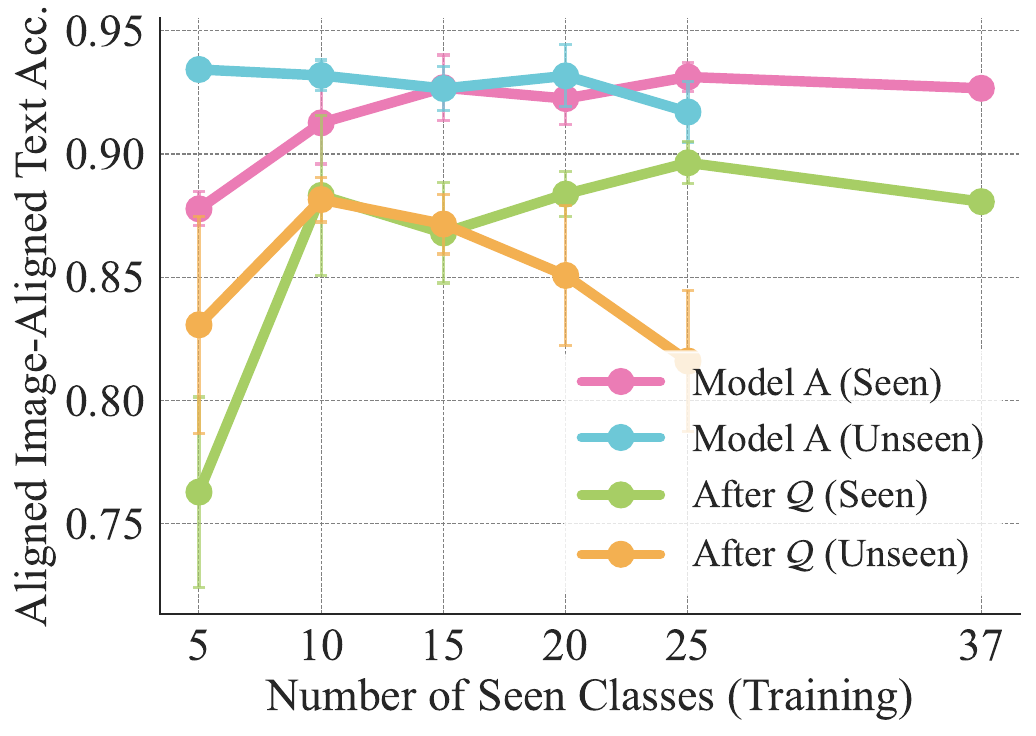}
    \caption{}
    \end{subfigure}
    \caption{\textit{Generalization of orthogonal alignment under limited supervision (Oxford Pets; CLIP ViT-L/14 (OpenAI) to SigLIP.} (a) Image-image class retrieval, (b) text-text class retrieval, and (c) mean text-text cosine similarity, each reported on seen and unseen classes. (d-f) Downstream transfer: (d) aligned images from model A with text from model B, (e) images from model B with aligned text from model A, and (f) aligned images with aligned text from model A. $\mathcal R$ learned from few classes transfers across modalities and generalizes to unseen classes, achieving near-oracle cross-model retrieval and classification.\looseness=-1}
\label{fig:app_split_oxford_openai_to_siglip}
\end{figure*}

\FloatBarrier

\subsection{The Learned Orthogonal Map Generalizes Broadly}\label{sec:app_cross_dataset}
The previous experiment shows that $\mathcal Q$ is identifiable from few anchors within a dataset. We next ask whether the same $\mathcal Q$ depends on the downstream distribution used to estimate it. We fit $\mathcal Q$ using paired images from one dataset (Oxford Pets) and evaluate on a different dataset (Caltech-101). When transferring $\mathcal Q$ to Caltech-101, we keep $\mathcal Q$ fixed and re-center using the modality-specific means computed on Caltech-101’s training split (for both source and target models), i.e., we apply $z \mapsto \mathcal Q(z- \mu_{\text{Caltech}}^{(\cdot)}) + \tilde{\mu}_{\text{Caltech}}^{(\cdot)}$, to account for cross-dataset mean shift. 

Results for transfer from \textsc{CLIP} (OpenAI) to \textsc{CLIP} (OpenAI), \textsc{CLIP} (LAION), \textsc{FLAVA}, and \textsc{SigLIP} are shown in~\Cref{fig:app_cross_dataset_B32_B16,fig:app_cross_dataset_B32_OpenAI_LAION,fig:app_cross_L14_FLAVA,fig:app_cross_L14_SigLIP} respectively. Across all figures, the orthogonal map learned on images from one dataset transfers to the text from another dataset, as shown by strong text-text cosine similarity (subplot c). Further, downstream image-text retrieval metrics (subplots d,e,f) remain strong under transfer---often closely matching
or even exceeding an in-domain fit---indicating that Q generalizes beyond the calibration dataset.

\begin{figure*}[!htb]
    \centering
    \begin{minipage}{0.85\linewidth} 
    \centering
    \begin{subfigure}{0.32\textwidth}
    \centering
    \includegraphics[width=\textwidth]{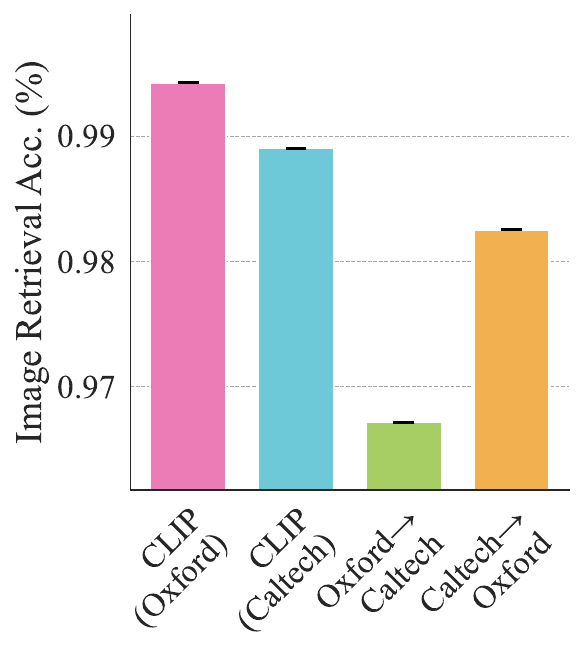}
    \caption{}
    \end{subfigure}
    \begin{subfigure}{0.32\textwidth}
    \centering
    \includegraphics[width=\textwidth]{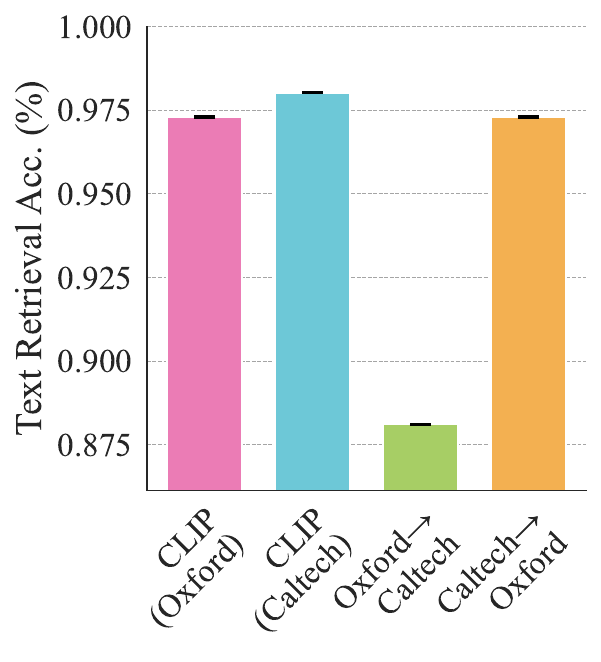}
    \caption{}
    \end{subfigure}
    \begin{subfigure}{0.32\textwidth}
    \centering
    \includegraphics[width=\textwidth]{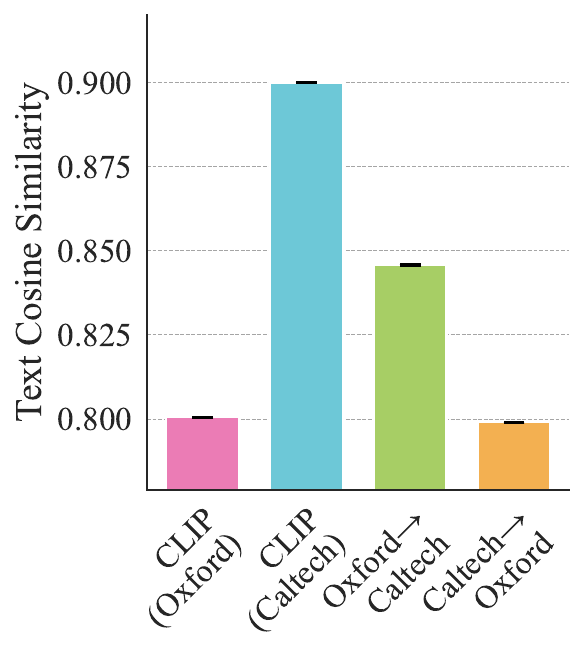}
    \caption{}
    \end{subfigure}
    \begin{subfigure}{0.32\textwidth}
    \centering
    \includegraphics[width=\textwidth]{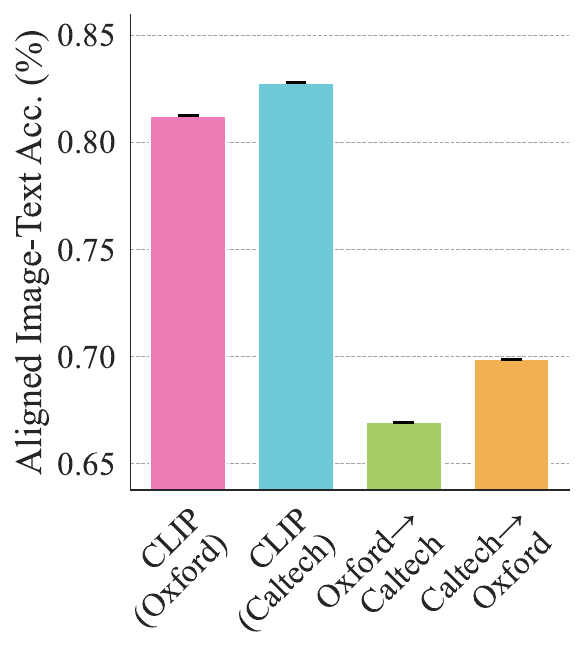}
    \caption{}
    \end{subfigure}
    \begin{subfigure}{0.32\textwidth}
    \centering
    \includegraphics[width=\textwidth]{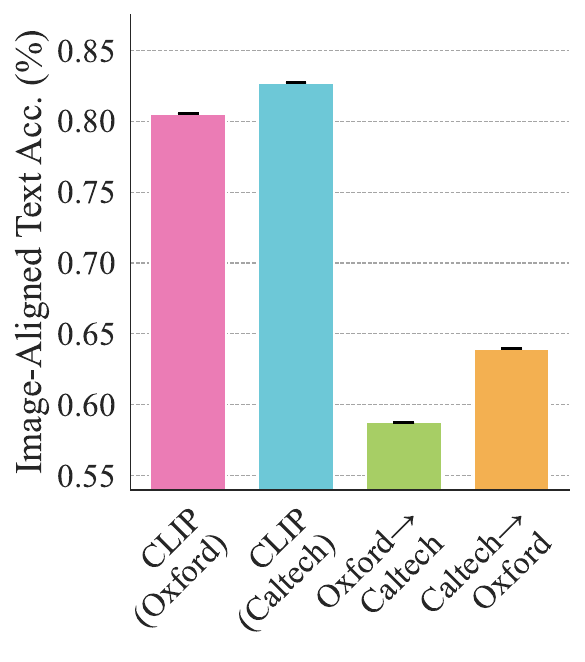}
    \caption{}
    \end{subfigure}
    \begin{subfigure}{0.32\textwidth}
    \centering
    \includegraphics[width=\textwidth]{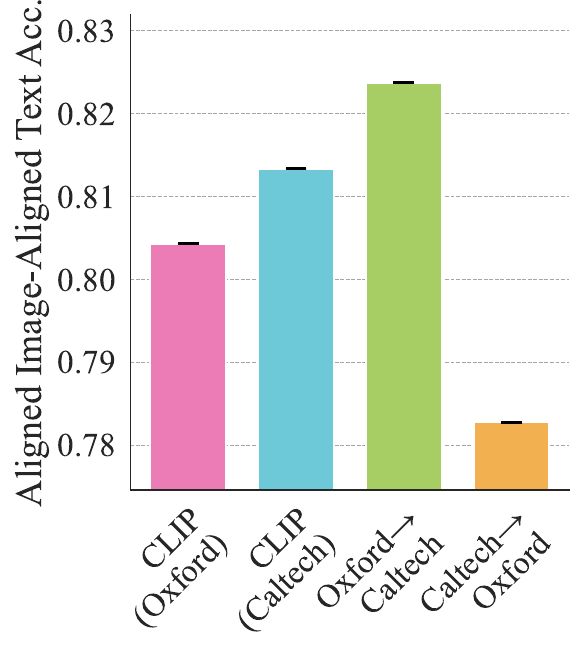}
    \caption{}
    \end{subfigure}
    \end{minipage}
    \caption{\textit{Transfer of $\mathcal Q$ learnt on Oxford Pets to Caltech-101 and vice-versa across CLIP ViT-B/32 and ViT-B/16 (OpenAI).} (a) Image–image class retrieval and (b) text–text class retrieval (c) Mean text–text cosine similarity. (d) Image–text retrieval using aligned images from model A and text from model B. (e) Image–text retrieval using images from model B and aligned text from model A. (f) Image–text retrieval using aligned images and aligned text from model A. }
    \label{fig:app_cross_dataset_B32_B16}
\end{figure*}

\begin{figure*}[!htb]
    \centering
    \begin{minipage}{0.85\linewidth}  
    \centering
    \begin{subfigure}{0.32\textwidth}
    \centering
    \includegraphics[width=\textwidth]{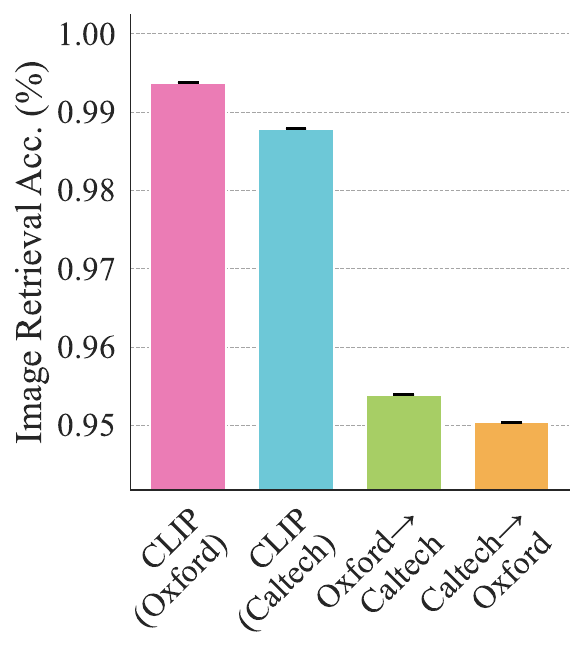}
    \caption{}
    \end{subfigure}
    \begin{subfigure}{0.32\textwidth}
    \centering
    \includegraphics[width=\textwidth]{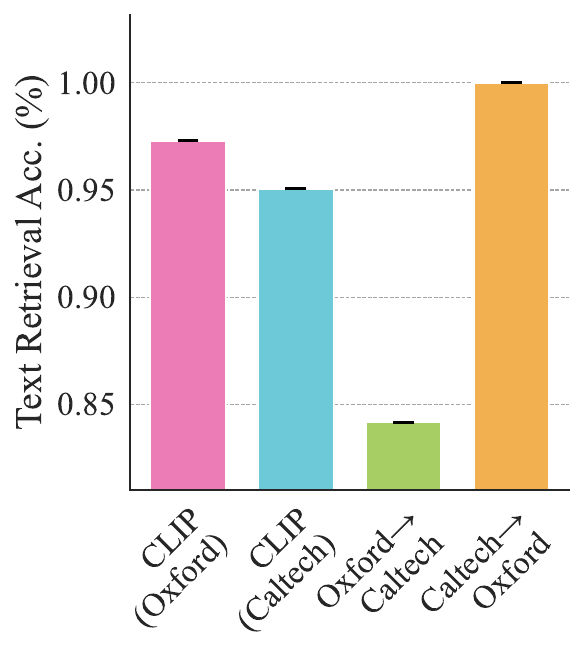}
    \caption{}
    \end{subfigure}
    \begin{subfigure}{0.32\textwidth}
    \centering
    \includegraphics[width=\textwidth]{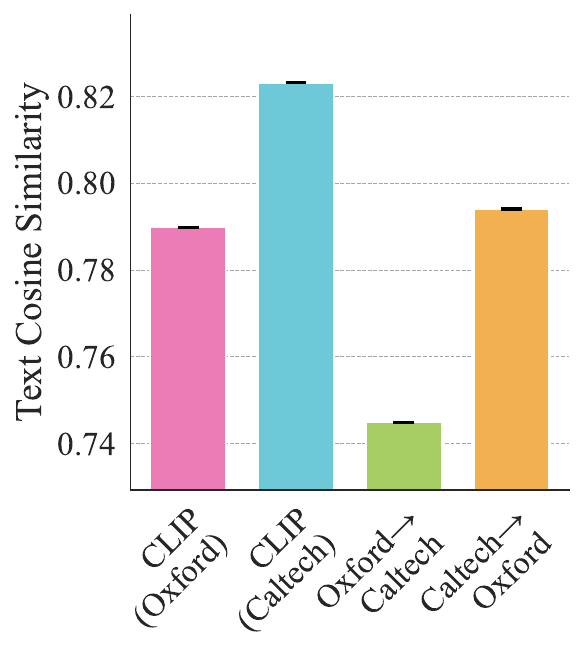}
    \caption{}
    \end{subfigure}
    \begin{subfigure}{0.32\textwidth}
    \centering
    \includegraphics[width=\textwidth]{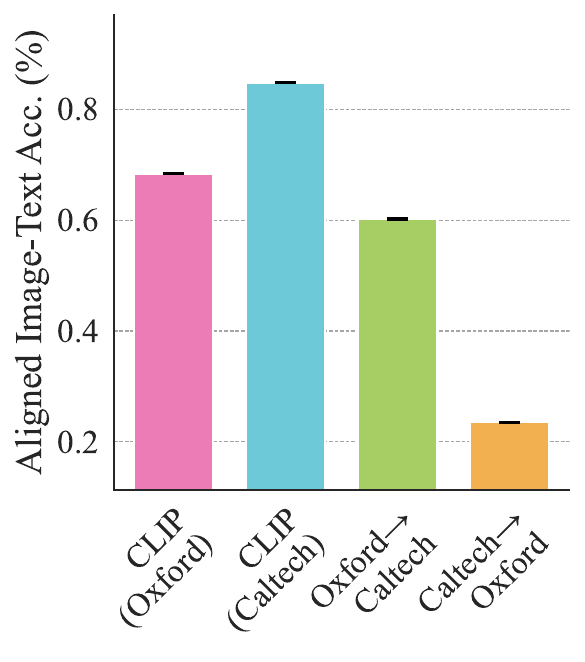}
    \caption{}
    \end{subfigure}
    \begin{subfigure}{0.32\textwidth}
    \centering
    \includegraphics[width=\textwidth]{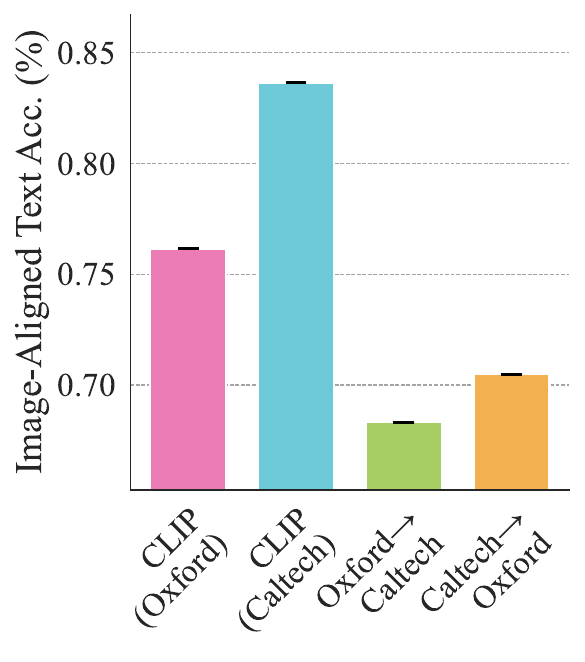}
    \caption{}
    \end{subfigure}
    \begin{subfigure}{0.32\textwidth}
    \centering
    \includegraphics[width=\textwidth]{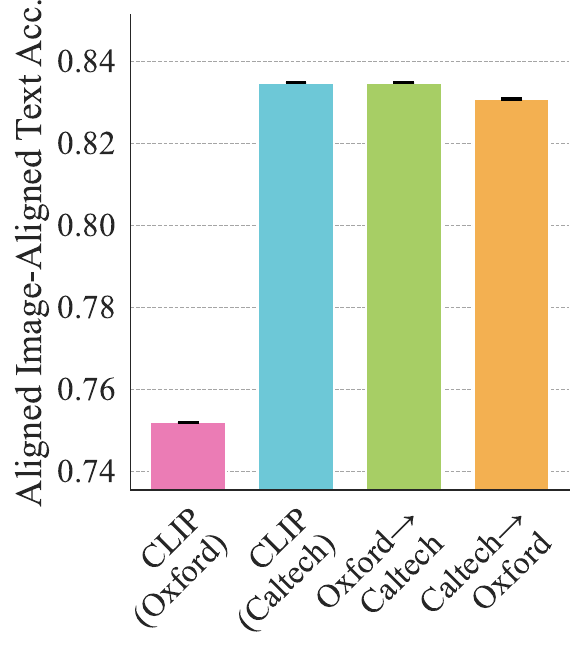}
    \caption{}
    \end{subfigure}
    \end{minipage}
    \caption{\textit{Transfer of $\mathcal Q$ learnt on Oxford Pets to Caltech-101 and vice-versa across CLIP ViT-B/32 (OpenAI) and ViT-B/32 (LAION)} (a) Image–image class retrieval and (b) text–text class retrieval (c) Mean text–text cosine similarity. (d) Image–text retrieval using aligned images from model A and text from model B. (e) Image–text retrieval using images from model B and aligned text from model A. (f) Image–text retrieval using aligned images and aligned text from model A.}
    \label{fig:app_cross_dataset_B32_OpenAI_LAION}
\end{figure*}

\begin{figure*}[!htb]
    \centering
    \begin{minipage}{0.85\linewidth}  
    \centering
    \begin{subfigure}{0.32\textwidth}
    \centering
    \includegraphics[width=\textwidth]{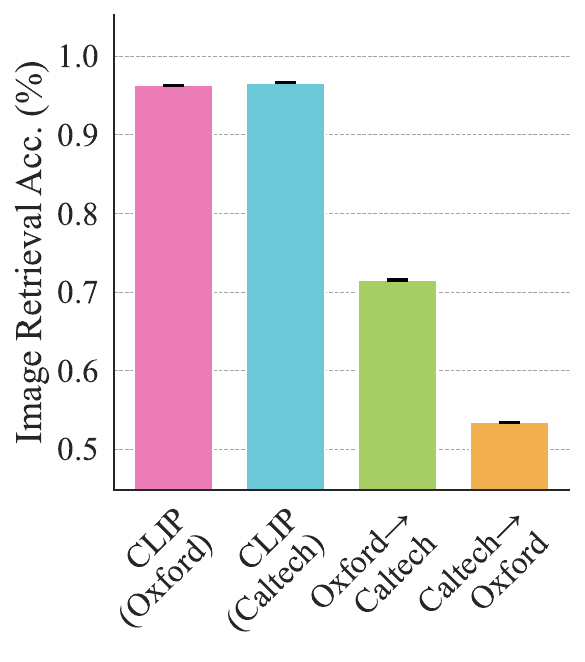}
    \caption{}
    \end{subfigure}
    \begin{subfigure}{0.32\textwidth}
    \centering
    \includegraphics[width=\textwidth]{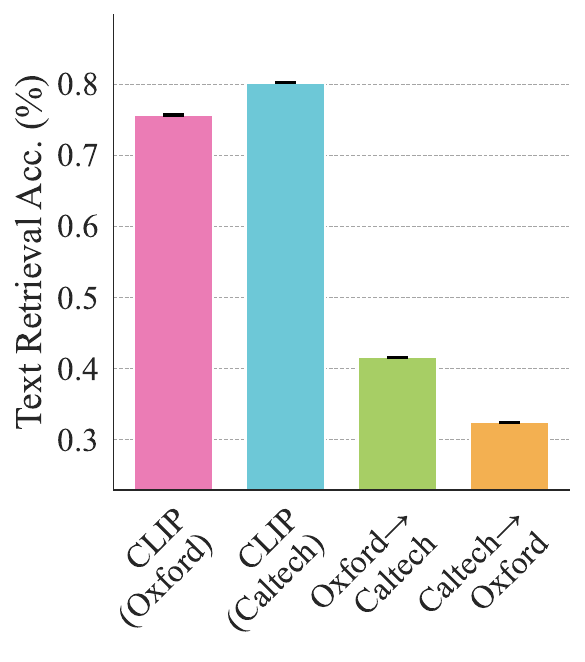}
    \caption{}
    \end{subfigure}
    \begin{subfigure}{0.32\textwidth}
    \centering
    \includegraphics[width=\textwidth]{arxiv_figs/appendix/cross_A=ViT-L-14-openai__B=flava-facebook_flava-full_cosine_text_prototypes.pdf}
    \caption{}
    \end{subfigure}
    \begin{subfigure}{0.32\textwidth}
    \centering
    \includegraphics[width=\textwidth]{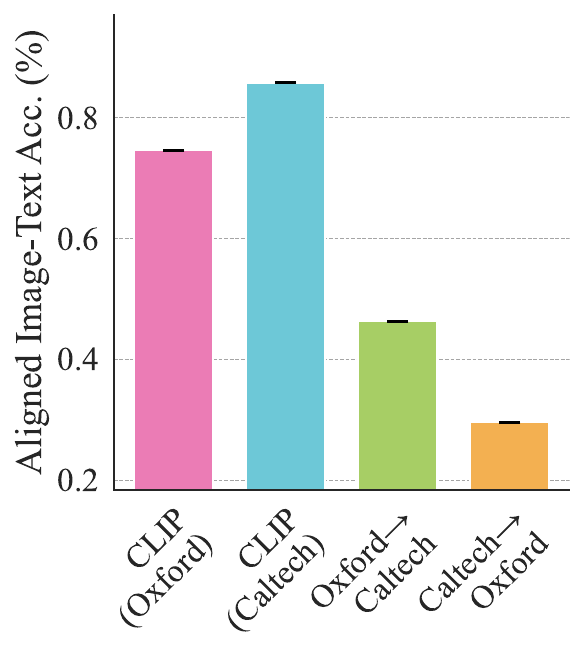}
    \caption{}
    \end{subfigure}
    \begin{subfigure}{0.32\textwidth}
    \centering
    \includegraphics[width=\textwidth]{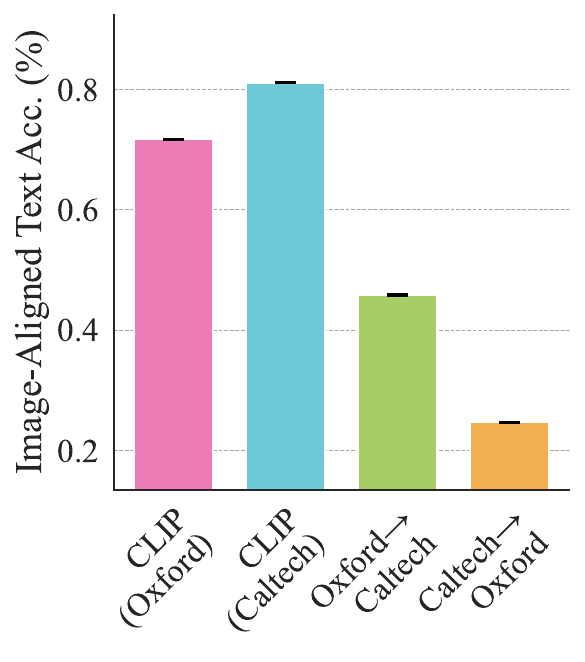}
    \caption{}
    \end{subfigure}
    \begin{subfigure}{0.32\textwidth}
    \centering
    \includegraphics[width=\textwidth]{arxiv_figs/appendix/cross_A=ViT-L-14-openai__B=flava-facebook_flava-full_aligned_img_to_aligned_text.pdf}
    \caption{}
    \end{subfigure}
    \end{minipage}
    \caption{\textit{Transfer of $\mathcal Q$ learnt on Oxford Pets to Caltech-101 and vice-versa across CLIP ViT-L/14 (OpenAI) and FLAVA} (a) Image–image class retrieval and (b) text–text class retrieval (c) Mean text–text cosine similarity. (d) Image–text retrieval using aligned images from model A and text from model B. (e) Image–text retrieval using images from model B and aligned text from model A. (f) Image–text retrieval using aligned images and aligned text from model A.}
    \label{fig:app_cross_L14_FLAVA}
\end{figure*}

\begin{figure*}[!htb]
    \centering
    \begin{minipage}{0.85\linewidth}  
    \centering
    \begin{subfigure}{0.32\textwidth}
    \centering
    \includegraphics[width=\textwidth]{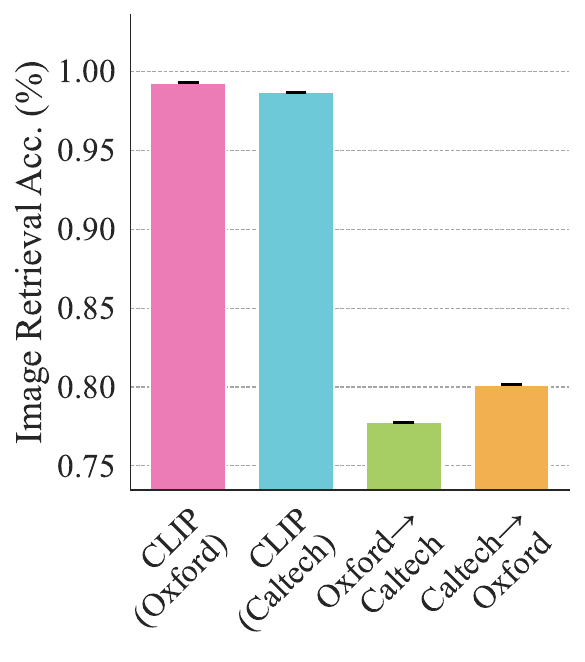}
    \caption{}
    \end{subfigure}
    \begin{subfigure}{0.32\textwidth}
    \centering
    \includegraphics[width=\textwidth]{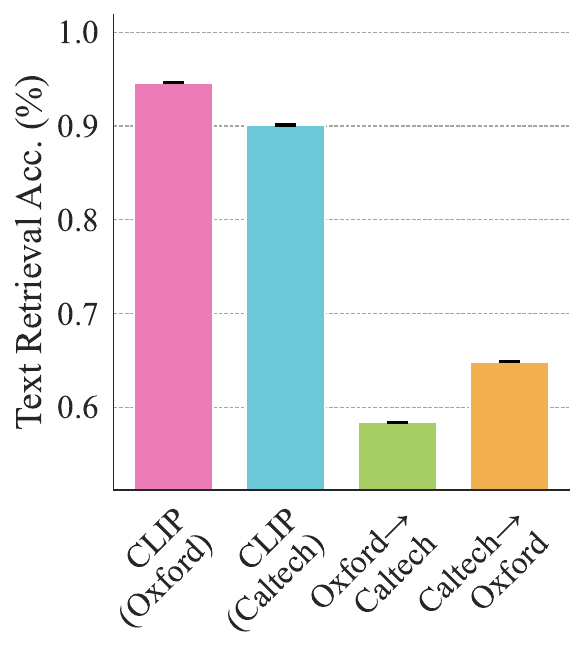}
    \caption{}
    \end{subfigure}
    \begin{subfigure}{0.32\textwidth}
    \centering
    \includegraphics[width=\textwidth]{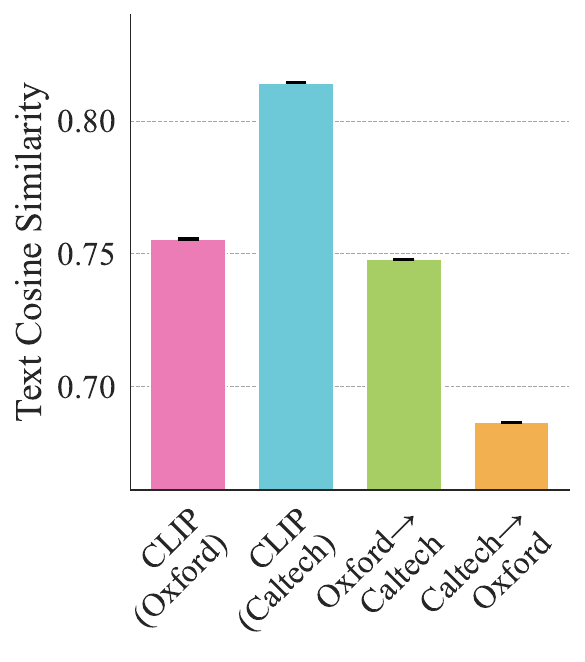}
    \caption{}
    \end{subfigure}
    \begin{subfigure}{0.32\textwidth}
    \centering
    \includegraphics[width=\textwidth]{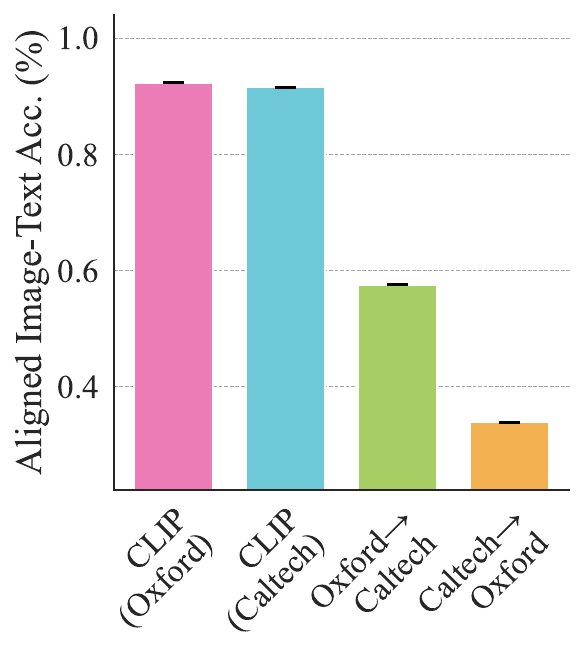}
    \caption{}
    \end{subfigure}
    \begin{subfigure}{0.32\textwidth}
    \centering
    \includegraphics[width=\textwidth]{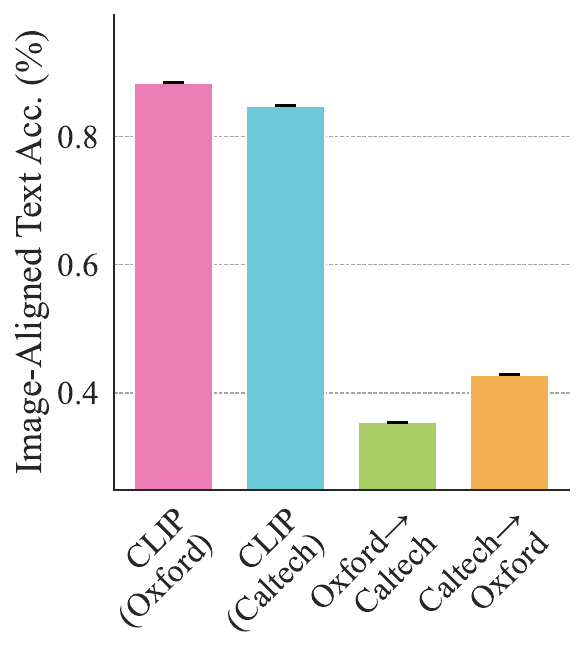}
    \caption{}
    \end{subfigure}
    \begin{subfigure}{0.32\textwidth}
    \centering
    \includegraphics[width=\textwidth]{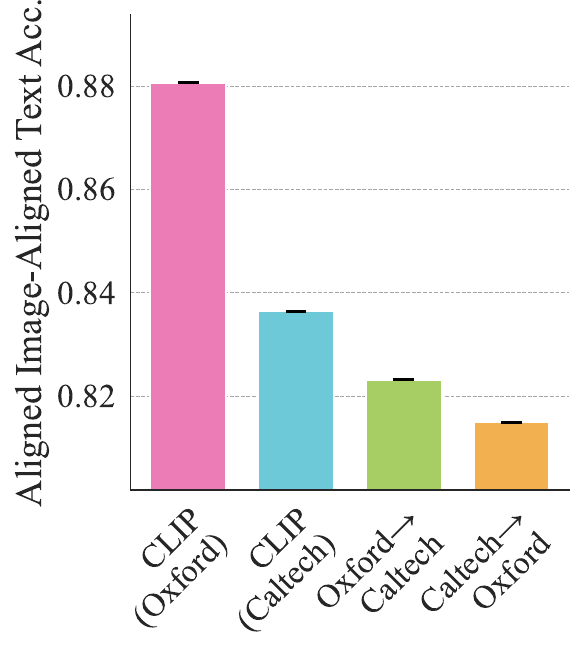}
    \caption{}
    \end{subfigure}
    \end{minipage}
    \caption{\textit{Transfer of $\mathcal Q$ learnt on Oxford Pets to Caltech-101 and vice-versa across CLIP ViT-L/14 (OpenAI) and SigLIP.} (a) Image–image class retrieval and (b) text–text class retrieval (c) Mean text–text cosine similarity. (d) Image–text retrieval using aligned images from model A and text from model B. (e) Image–text retrieval using images from model B and aligned text from model A. (f) Image–text retrieval using aligned images and aligned text from model A.}
    \label{fig:app_cross_L14_SigLIP}
\end{figure*}

\FloatBarrier

\subsection{Learning the Orthogonal Map from Text Instead of
Images Transfers to Images}\label{sec:app_text_trained}
So far, we fit $\mathcal Q$ using paired images and evaluated whether it transfers to the text modality. We now ask the converse: can we fit $\mathcal Q$ using only text and recover the same cross-model transform that governs images?~\Cref{fig:app_text_trained_caltech,fig:app_text_trained_cifar,fig:app_text_trained_oxford}, shows these results for Caltech-101, CIFAR-100 and Oxford Pets, respectively. As observed from these figures, across datasets and model pairs, this text-trained map substantially improves downstream task accuracy after $\mathcal Q$ (image-text retrieval), recovering a large fraction of the stronger model’s performance. However, compared to image-trained maps, text-trained maps yield weaker visual transfer (image-image retrieval and paired-image cosine), indicating that dense image supervision better constrains the shared orthogonal map than class-level text anchors.

\begin{figure*}[!htb]
    \centering
    \begin{subfigure}{0.32\textwidth}
    \centering
    \includegraphics[width=\textwidth]{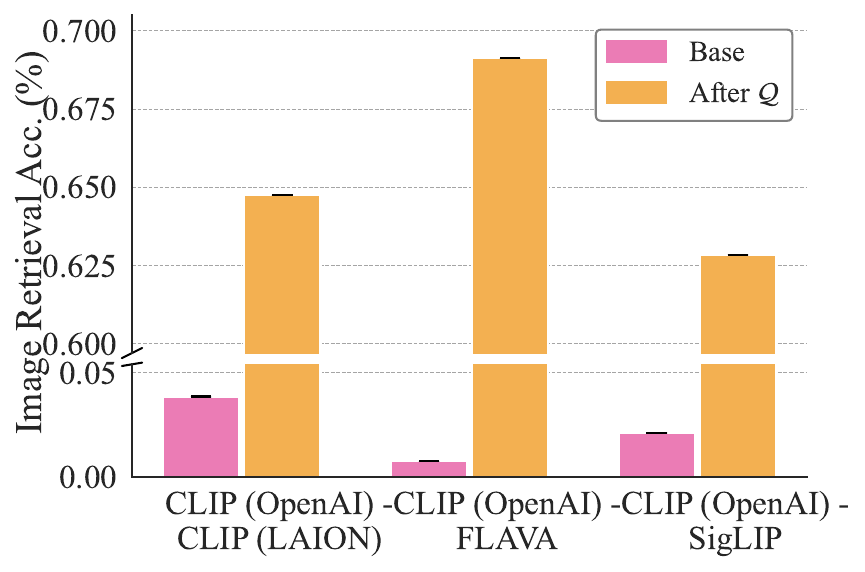}
    \caption{}
    \end{subfigure}
    \begin{subfigure}{0.32\textwidth}
    \centering
    \includegraphics[width=\textwidth]{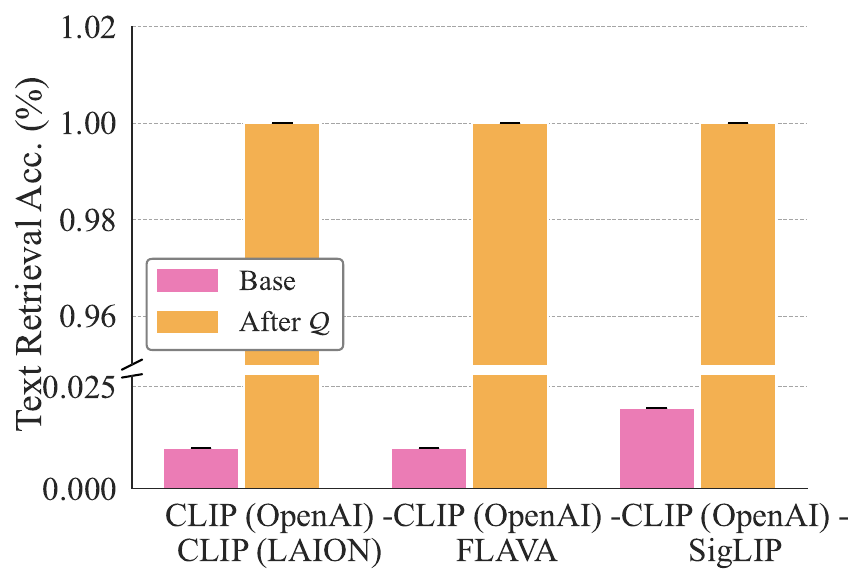}
    \caption{}
    \end{subfigure}
    \begin{subfigure}{0.32\textwidth}
    \centering
    \includegraphics[width=\textwidth]{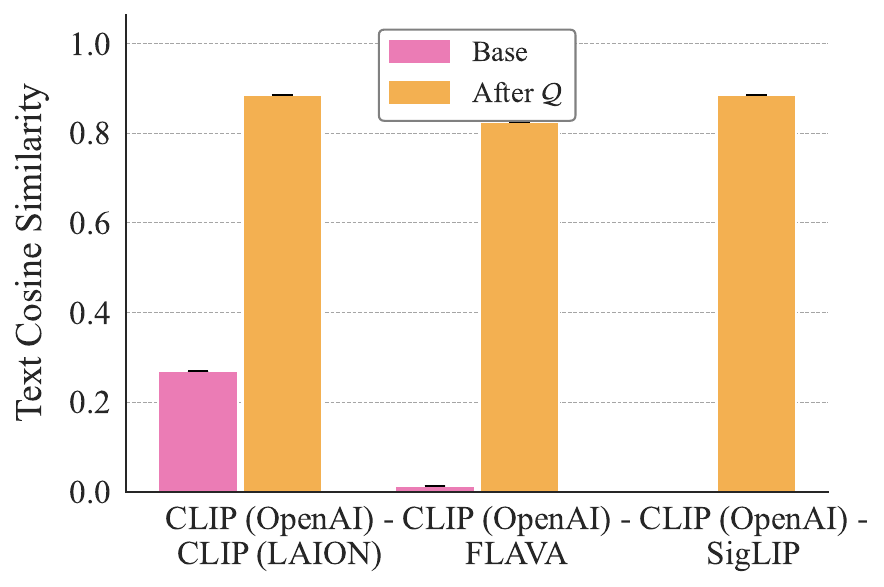}
    \caption{}
    \end{subfigure}
    \begin{subfigure}{0.32\textwidth}
    \centering
    \includegraphics[width=\textwidth]{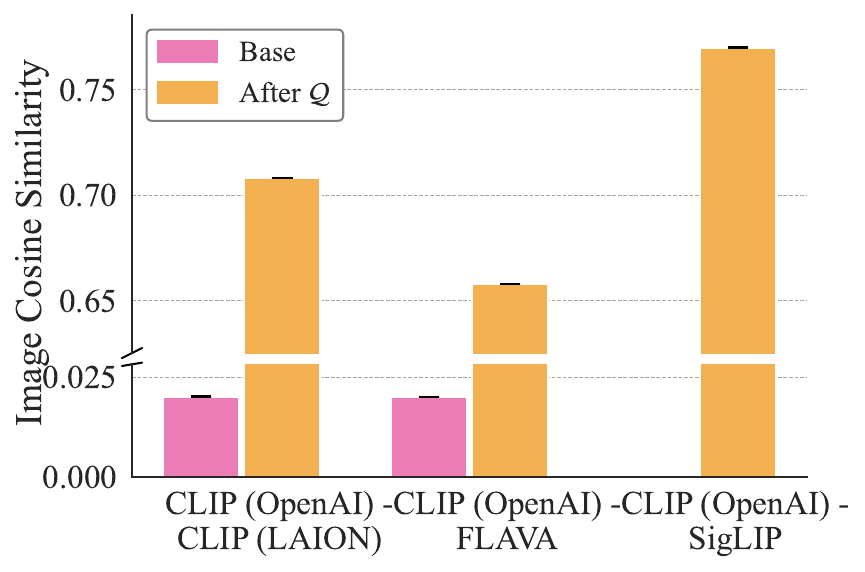}
    \caption{}
    \end{subfigure}
    \begin{subfigure}{0.32\textwidth}
    \centering
    \includegraphics[width=\textwidth]{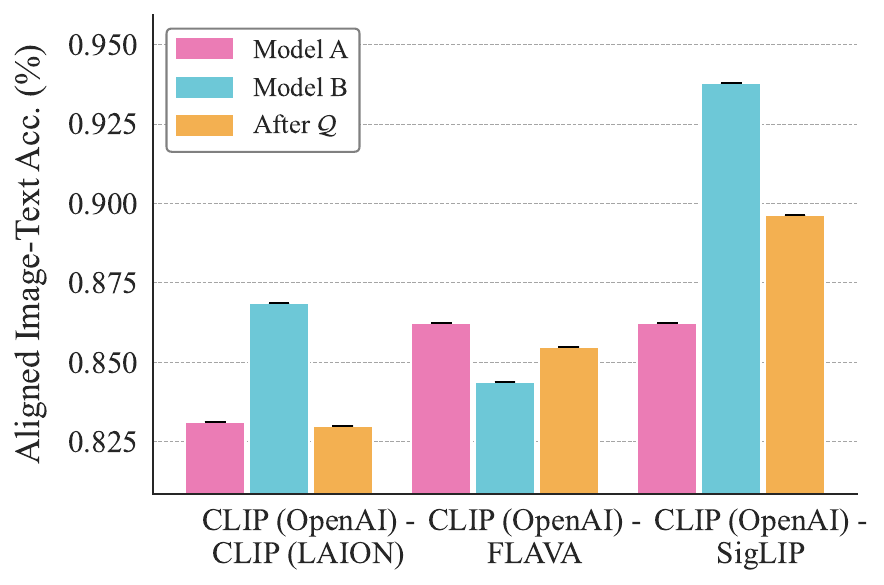}
    \caption{}
    \end{subfigure}
    \begin{subfigure}{0.32\textwidth}
    \centering
    \includegraphics[width=\textwidth]{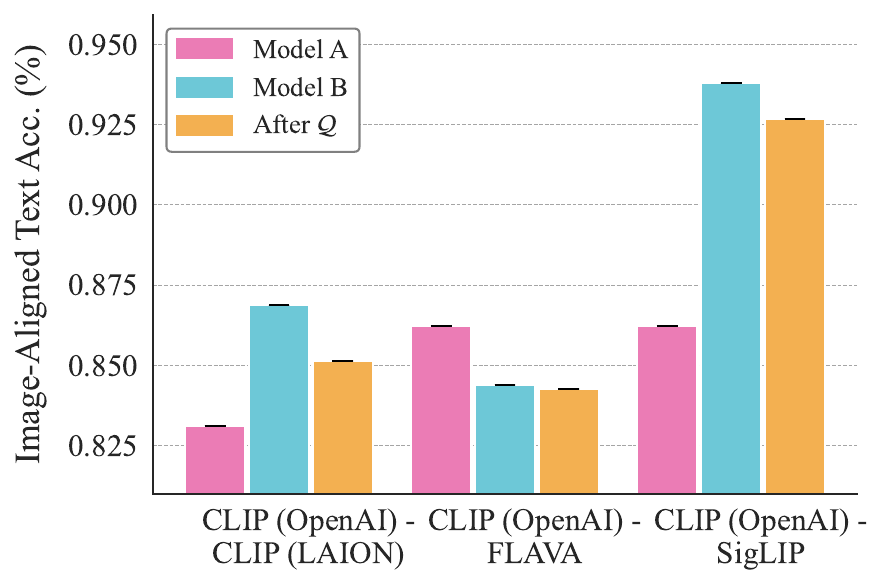}
    \caption{}
    \end{subfigure}
    \begin{subfigure}{0.32\textwidth}
    \centering
    \includegraphics[width=\textwidth]{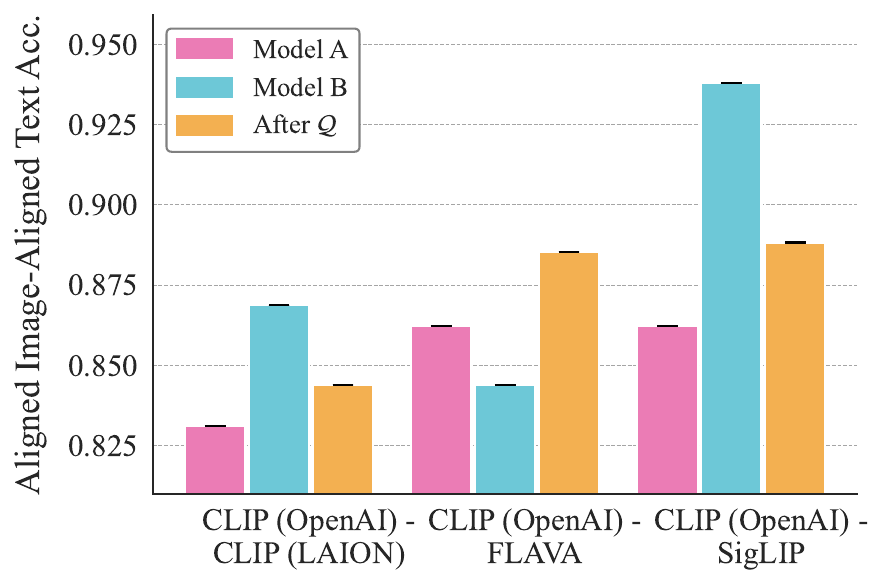}
    \caption{}
    \end{subfigure}
     \caption{\textit{Cross-model alignment on Caltech-101 learned from text prototypes instead of images, before and after fitting a single orthogonal map $\mathcal{Q}$.} (a) Image–image class retrieval and (b) text–text class retrieval (c) Mean text-text cosine similarity. (d) Mean image-image cosine similarity. (e) Image–text retrieval using aligned images from model A and text from model B. (f) Image–text retrieval using images from model B and aligned text from model A. (g) Image–text retrieval using aligned images and aligned text from model A.}
    \label{fig:app_text_trained_caltech}
\end{figure*}

\begin{figure*}[!htb]
    \centering
    \begin{subfigure}{0.32\textwidth}
    \centering
    \includegraphics[width=\textwidth]{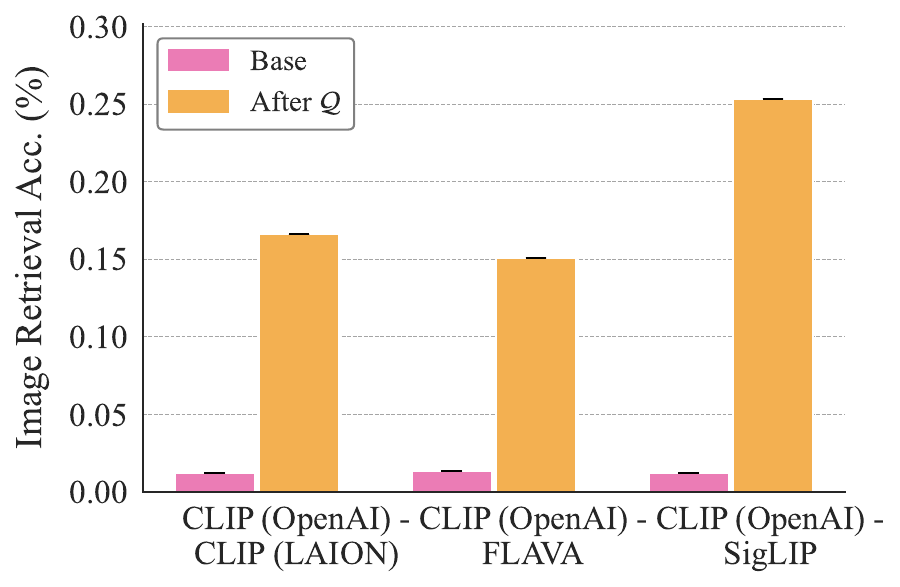}
    \caption{}
    \end{subfigure}
    \begin{subfigure}{0.32\textwidth}
    \centering
    \includegraphics[width=\textwidth]{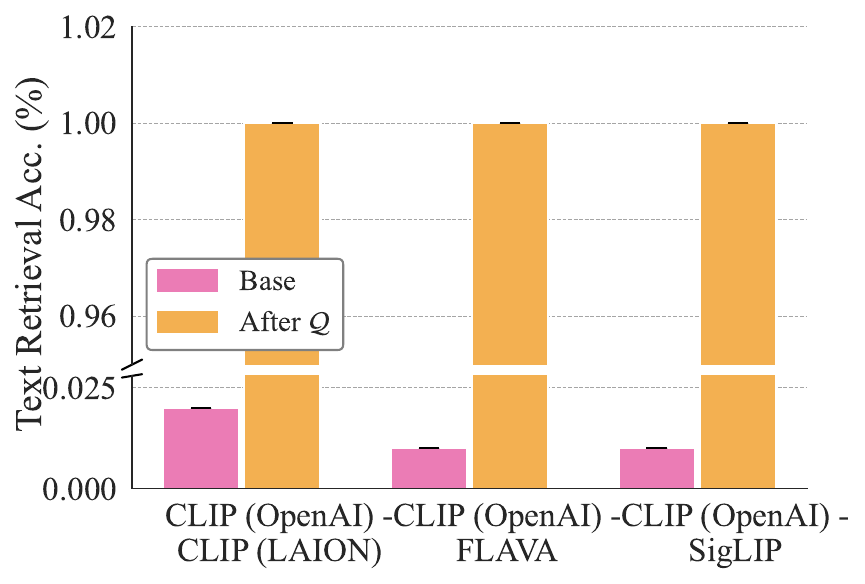}
    \caption{}
    \end{subfigure}
    \begin{subfigure}{0.32\textwidth}
    \centering
    \includegraphics[width=\textwidth]{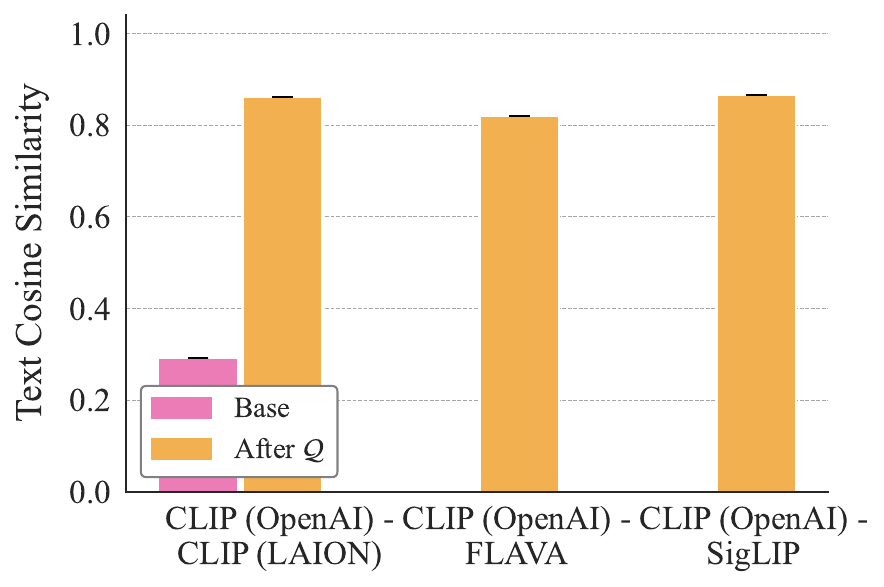}
    \caption{}
    \end{subfigure}
    \begin{subfigure}{0.32\textwidth}
    \centering
    \includegraphics[width=\textwidth]{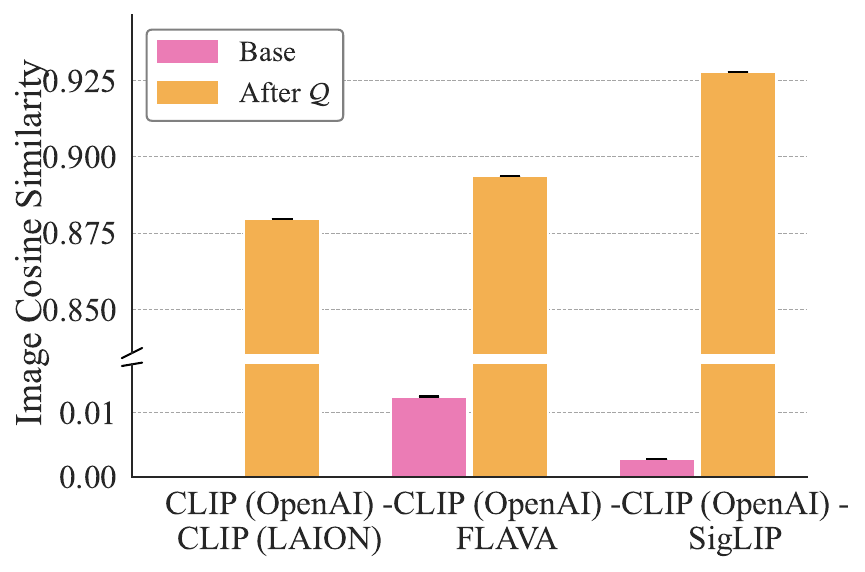}
    \caption{}
    \end{subfigure}
    \begin{subfigure}{0.32\textwidth}
    \centering
    \includegraphics[width=\textwidth]{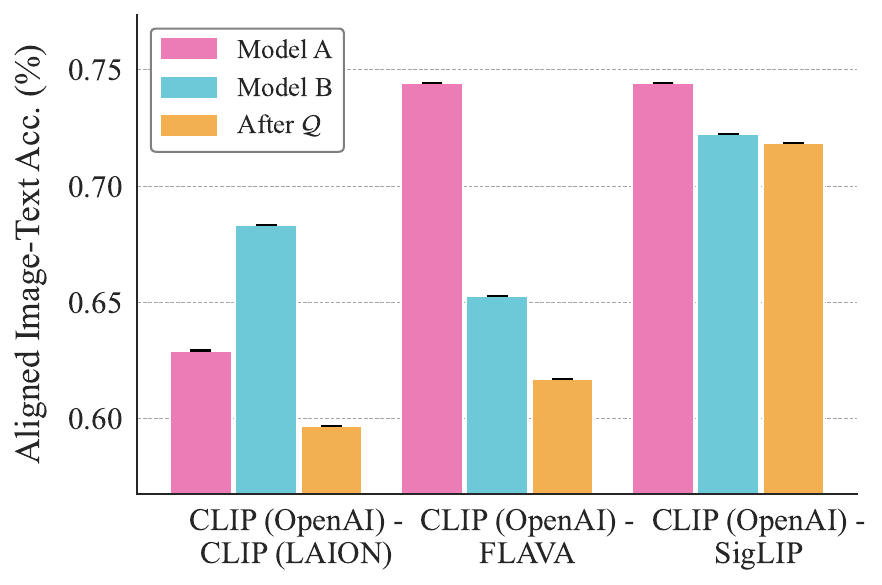}
    \caption{}
    \end{subfigure}
    \begin{subfigure}{0.32\textwidth}
    \centering
    \includegraphics[width=\textwidth]{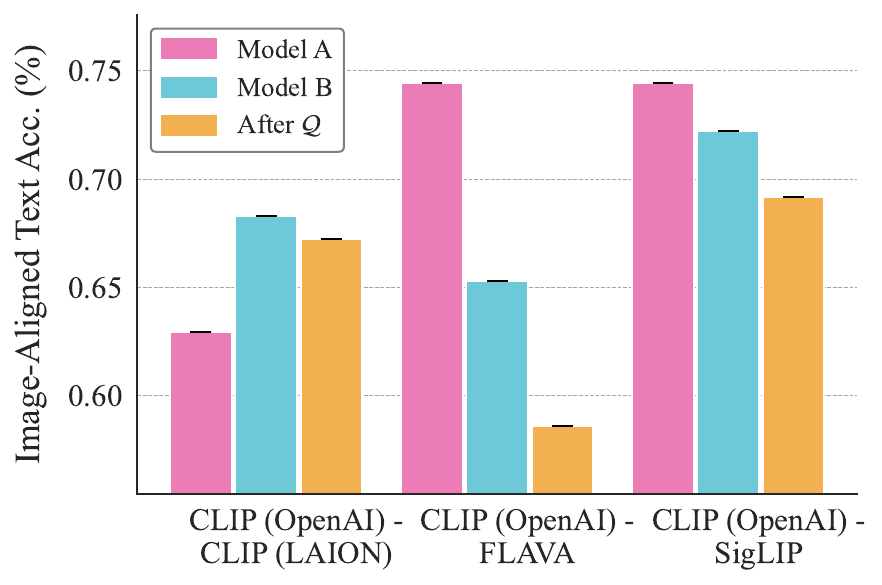}
    \caption{}
    \end{subfigure}
    \begin{subfigure}{0.32\textwidth}
    \centering
    \includegraphics[width=\textwidth]{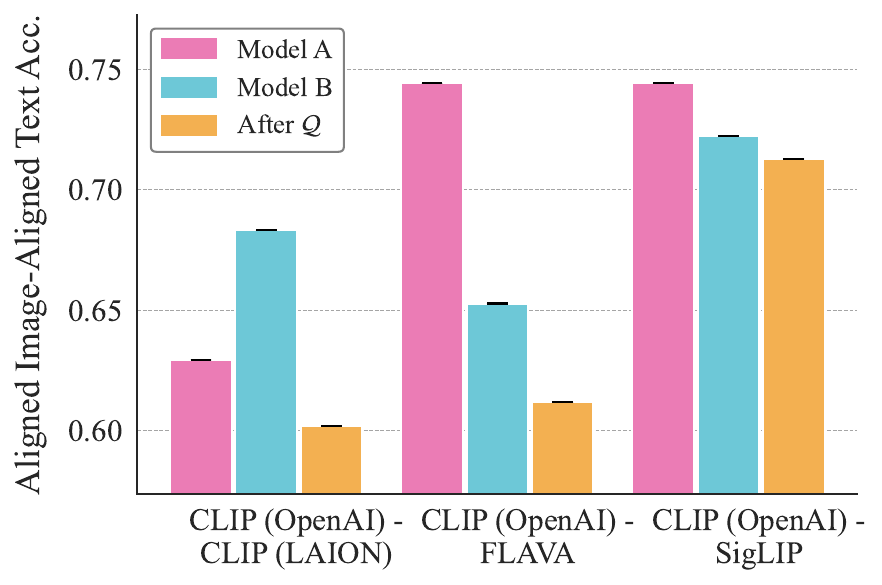}
    \caption{}
    \end{subfigure}
    \caption{\textit{Cross-model alignment on CIFAR-100 learned from text prototypes alone before and after fitting a single orthogonal map $\mathcal{Q}$.} (a) Image–image class retrieval and (b) text–text class retrieval (c) Mean text-text cosine similarity. (d) Mean image-image cosine similarity. (e) Image–text retrieval using aligned images from model A and text from model B. (f) Image–text retrieval using images from model B and aligned text from model A. (g) Image–text retrieval using aligned images and aligned text from model A.}
    \label{fig:app_text_trained_cifar}
\end{figure*}

\begin{figure*}[!htb]
    \centering
    \begin{subfigure}{0.32\textwidth}
    \centering
    \includegraphics[width=\textwidth]{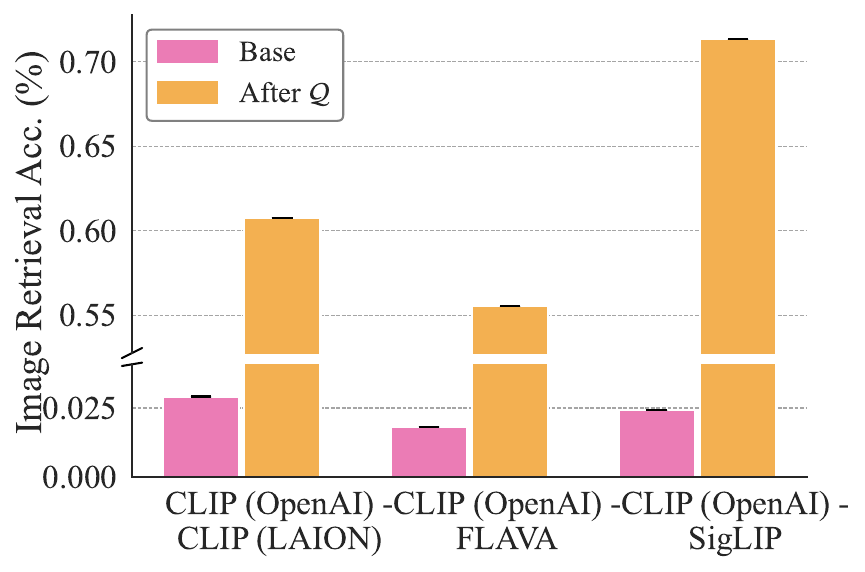}
    \caption{}
    \end{subfigure}
    \begin{subfigure}{0.32\textwidth}
    \centering
    \includegraphics[width=\textwidth]{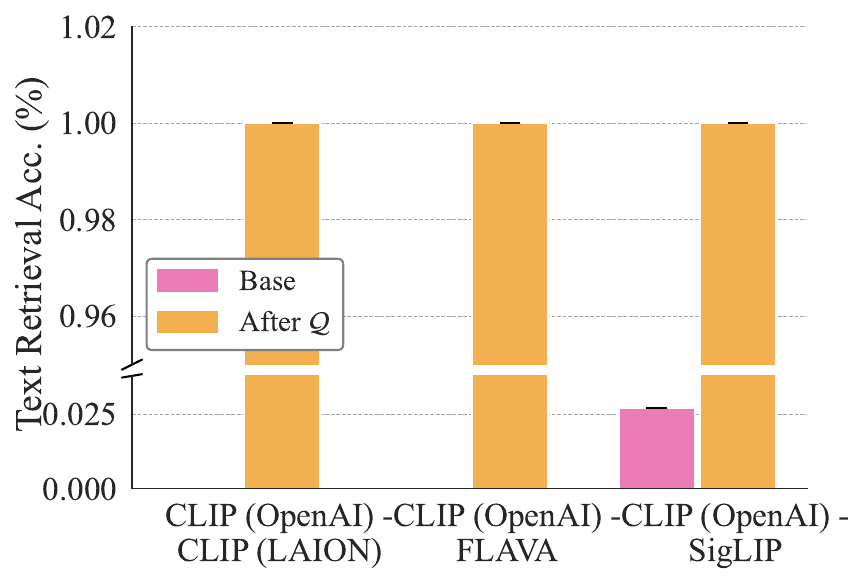}
    \caption{}
    \end{subfigure}
    \begin{subfigure}{0.32\textwidth}
    \centering
    \includegraphics[width=\textwidth]{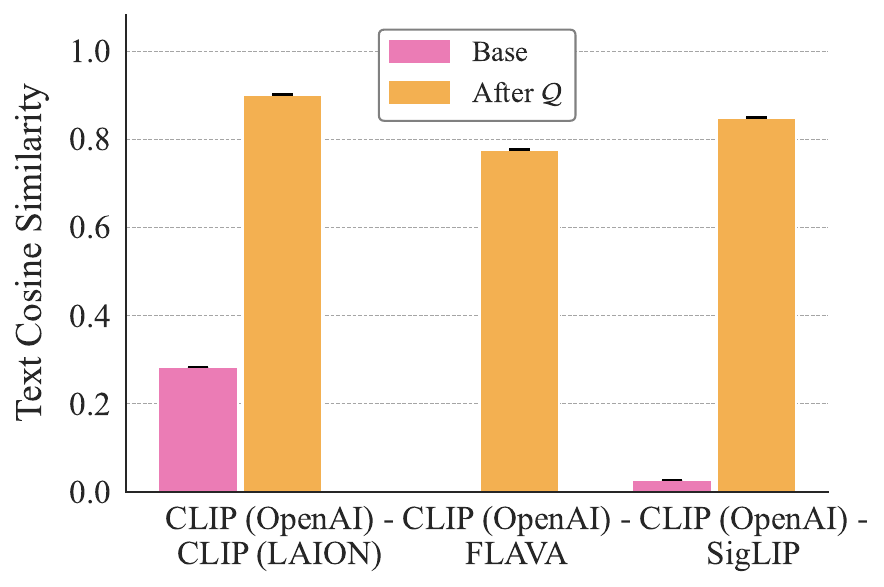}
    \caption{}
    \end{subfigure}
    \begin{subfigure}{0.32\textwidth}
    \centering
    \includegraphics[width=\textwidth]{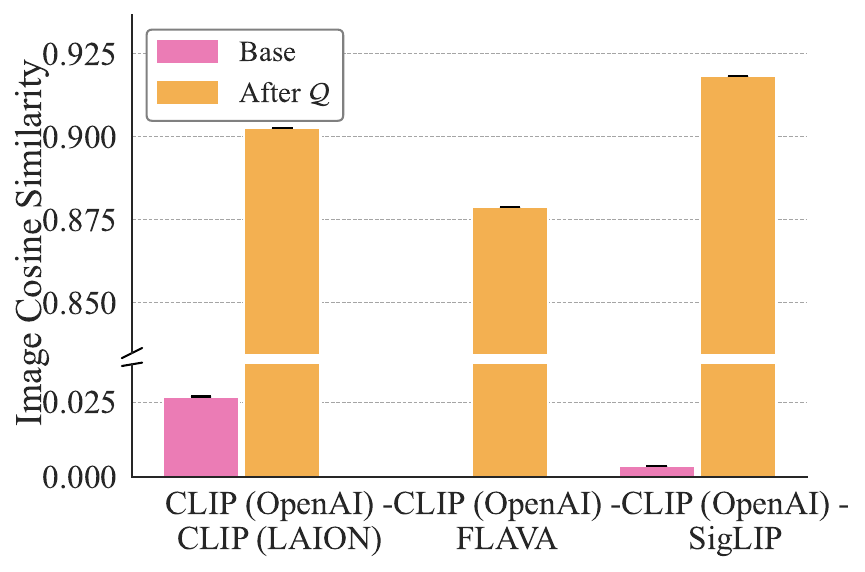}
    \caption{}
    \end{subfigure}
    \begin{subfigure}{0.32\textwidth}
    \centering
    \includegraphics[width=\textwidth]{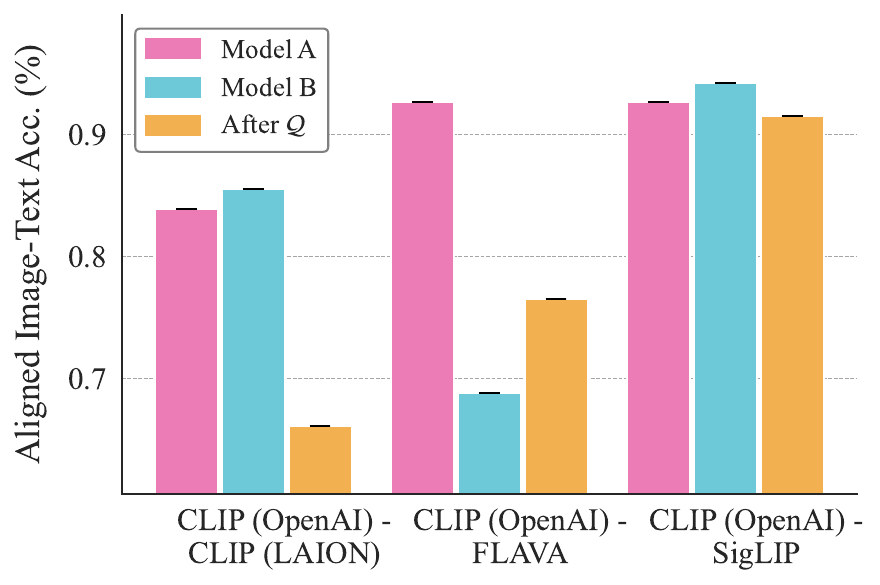}
    \caption{}
    \end{subfigure}
    \begin{subfigure}{0.32\textwidth}
    \centering
    \includegraphics[width=\textwidth]{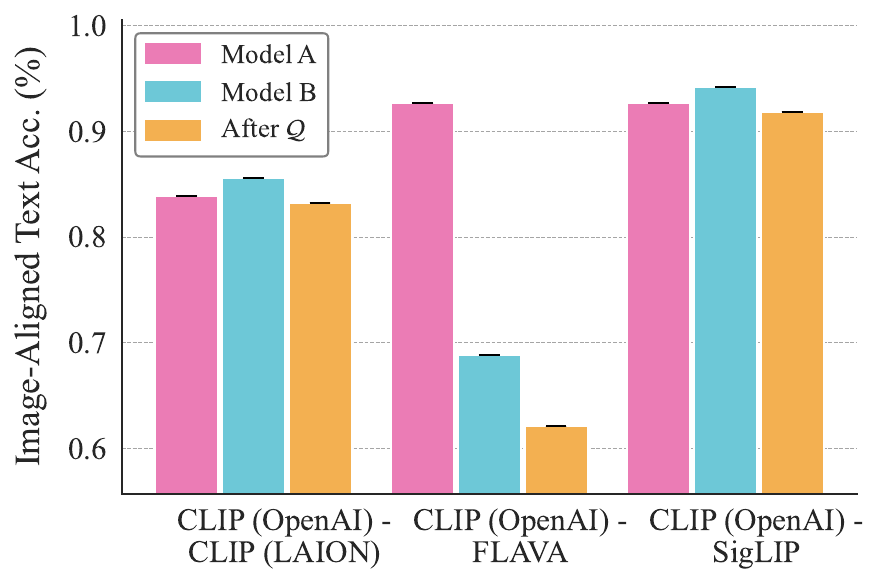}
    \caption{}
    \end{subfigure}
    \begin{subfigure}{0.32\textwidth}
    \centering
    \includegraphics[width=\textwidth]{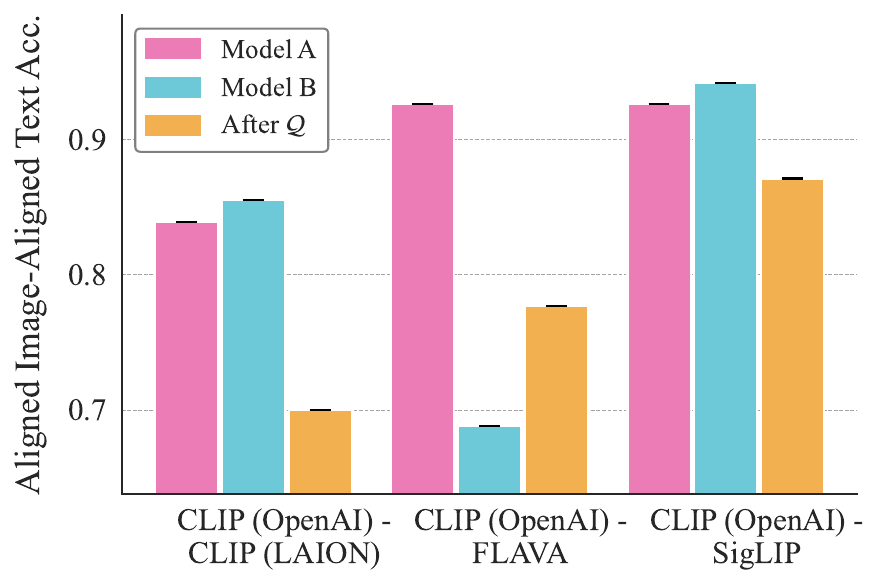}
    \caption{}
    \end{subfigure}
    \caption{\textit{Cross-model alignment on Oxford Pets learned from text prototypes alone before and after fitting a single orthogonal map $\mathcal{Q}$.} (a) Image–image class retrieval and (b) text–text class retrieval (c) Mean text-text cosine similarity. (d) Mean image-image cosine similarity. (e) Image–text retrieval using aligned images from model A and text from model B. (f) Image–text retrieval using images from model B and aligned text from model A. (g) Image–text retrieval using aligned images and aligned text from model A.}
    \label{fig:app_text_trained_oxford}
\end{figure*}

\FloatBarrier

\subsection{Cycle Consistency and Consistency Under
Composition}\label{sec:app_compositionality}
Our theory and experiment so far posit that different contrastive models represent the same underlying semantic geometry, differing only by a global orthogonal reparameterization. This would imply two forms of consistency: (1) aligning model $(f, g)$ to $(\tilde f, \tilde g)$ and then $(\tilde f, \tilde g)$ to $(\hat f, \hat g)$ should align $(f, g)$ to $(\hat f, \hat g)$ (composition), and aligning $(f, g)$ to $(\tilde f, \tilde g)$ should be undone by aligning $(\tilde f, \tilde g)$ back to $(f, g)$ (cycle). We therefore test transitivity and cycle consistency to ensure the learned maps reflect a shared geometry rather than pair-specific fitting. 
\par In this section, we present complete results for cycle consistency and composition on Caltech-101 in~\Cref{fig:app_cycle_caltech101_7} and~\Cref{fig:app_trans_caltech101} respectively, on CIFAR100 in~\Cref{fig:app_cycle_cifar100_7} and~\Cref{fig:app_trans_cifar100} respectively, and on Oxford Pets in~\Cref{fig:app_cycle_oxford_7} and~\Cref{fig:app_trans_oxford} respectively. 
\par Across all composition figures i.e.~\Cref{fig:app_trans_caltech101,fig:app_trans_cifar100,fig:app_trans_oxford}, we compare composing two learned isometries, $\mathcal Q_{BC}\circ \mathcal Q_{AB}$, with directly learning $\mathcal Q_{AC}$ between the endpoint models. Across all model triplets, the composed map almost matches the direct map. Strikingly, in some cases, the composed map outperforms the direct map, showing that the intermediate model’s stronger geometry propagates through the chain. Thus, isometries learned on disjoint model pairs compose reliably and can align a model pair that has never been calibrated together. 
\par Across all cycle-consistency figures i.e.~\Cref{fig:app_cycle_caltech101_7,fig:app_cycle_cifar100_7,fig:app_cycle_oxford_7}, we compare the downstream performance of Model A to a round trip that applies the forward map followed by the reverse map (the composition labeled $\mathcal Q_{BA}\circ \mathcal Q_{AB}$). Since each map is estimated via closed-form Procrustes, $\mathcal Q_{BA}$ is the inverse of $\mathcal Q_{AB}$ by construction; we therefore fit them on independent 95\% random subsets of the training pairs to test under finite-sample estimation. Across all pairs and datasets, image-text classification accuracies (across different aligned pairs) nearly matches the original model, indicating that the learned forward and reverse isometries behave as near-inverses and that the induced coordinate change is stable.

\begin{figure*}[!htb]
    \centering
    \begin{subfigure}{0.32\textwidth}
    \centering
    \includegraphics[width=\textwidth]{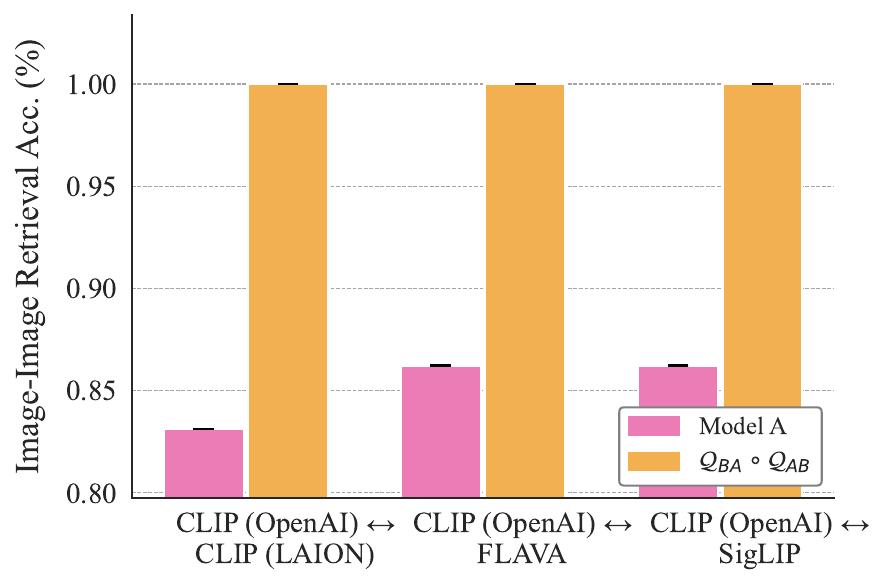}
    \caption{}
    \end{subfigure}
    \begin{subfigure}{0.32\textwidth}
    \centering
    \includegraphics[width=\textwidth]{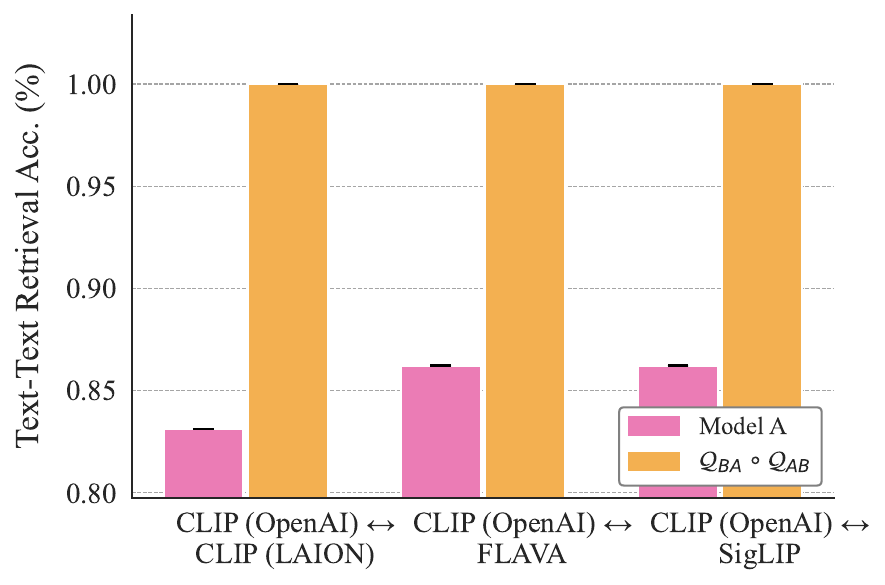}
    \caption{}
    \end{subfigure}
    \begin{subfigure}{0.32\textwidth}
    \centering
    \includegraphics[width=\textwidth]{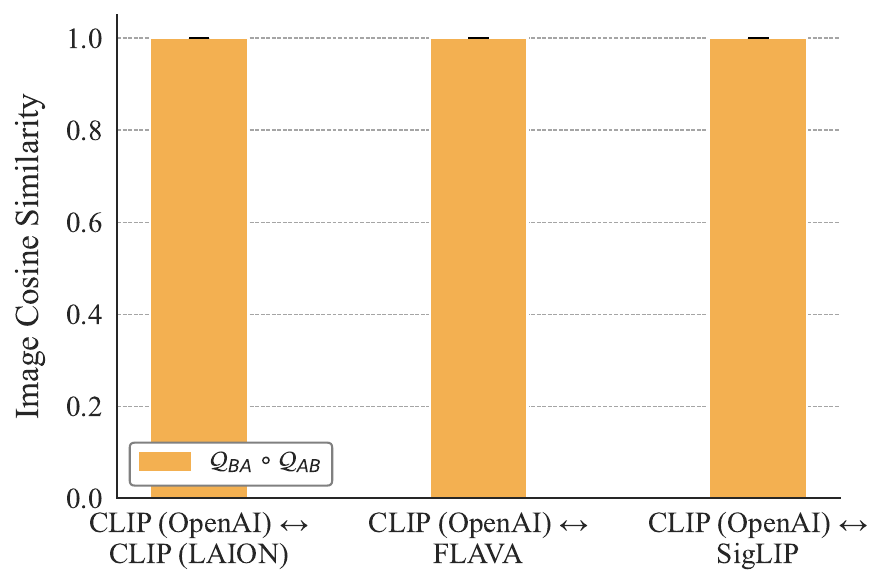}
    \caption{}
    \end{subfigure}
    \begin{subfigure}{0.32\textwidth}
    \centering
    \includegraphics[width=\textwidth]{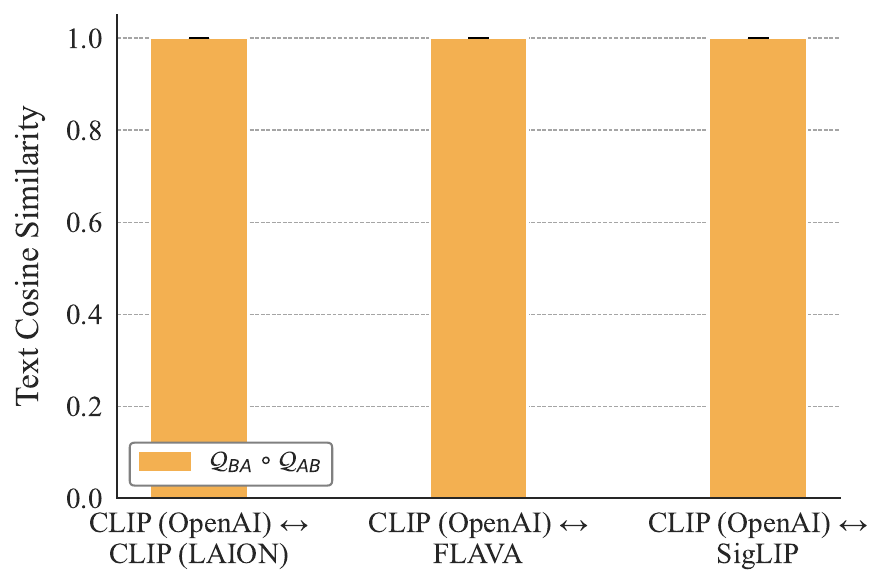}
    \caption{}
    \end{subfigure}
    \begin{subfigure}{0.32\textwidth}
    \centering
    \includegraphics[width=\textwidth]{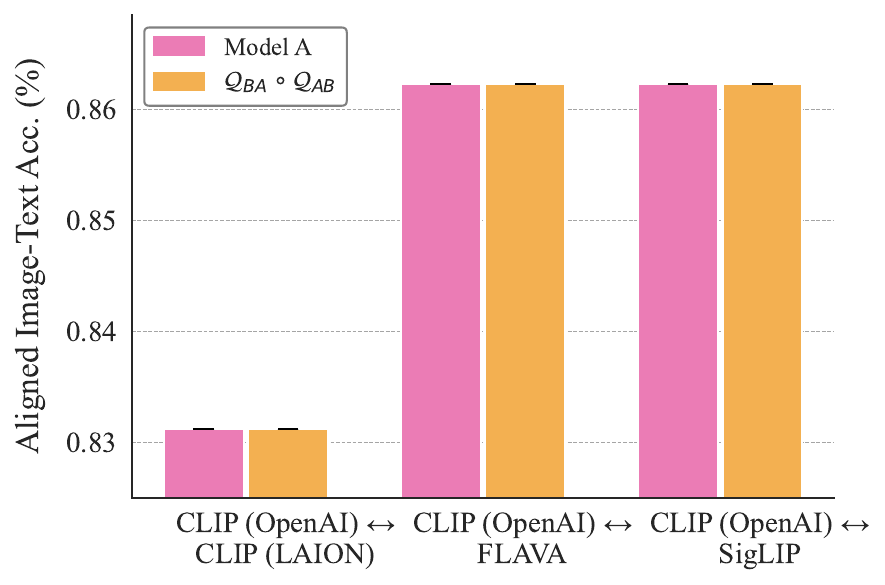}
    \caption{}
    \end{subfigure}
    \begin{subfigure}{0.32\textwidth}
    \centering
    \includegraphics[width=\textwidth]{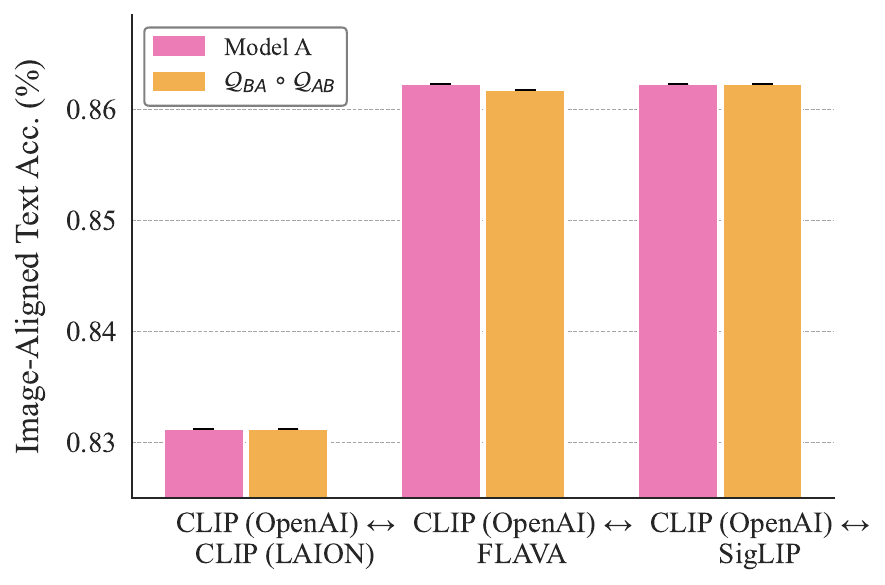}
    \caption{}
    \end{subfigure}
    \begin{subfigure}{0.32\textwidth}
    \centering
    \includegraphics[width=\textwidth]{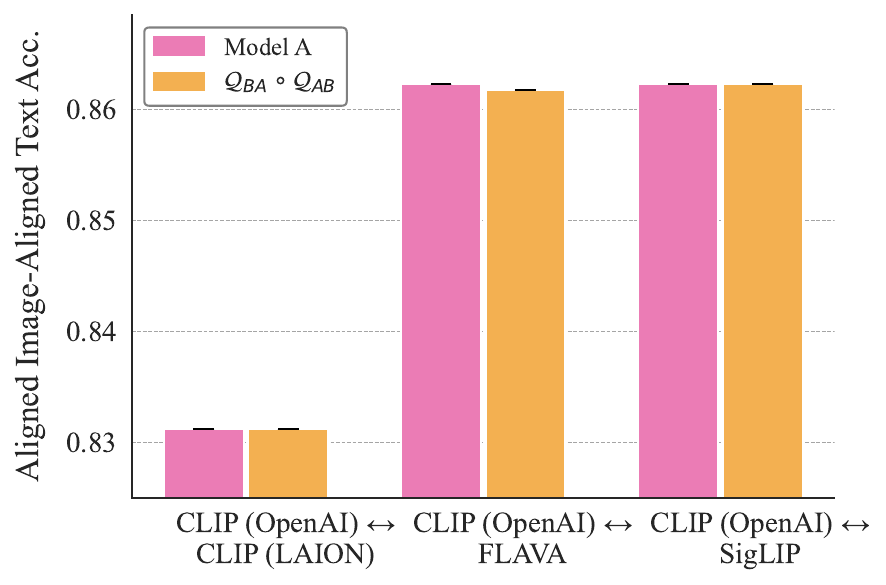}
    \caption{}
    \end{subfigure}
    \caption{\textit{Cycle consistency on Caltech-101.}
    (a) Image–image class retrieval. (b) text–text class retrieval (c) Mean image-image cosine similarity. (c) Mean text-text cosine similarity. (e) Image–text retrieval using aligned images from model A and text from model B. (f) Image–text retrieval using images from model B and aligned text from model A. (g) Image–text retrieval using aligned images and aligned text from model A.}
    \label{fig:app_cycle_caltech101_7}
\end{figure*}

\begin{figure*}[!htb]
    \centering
    \begin{subfigure}{0.32\textwidth}
    \centering
    \includegraphics[width=\textwidth]{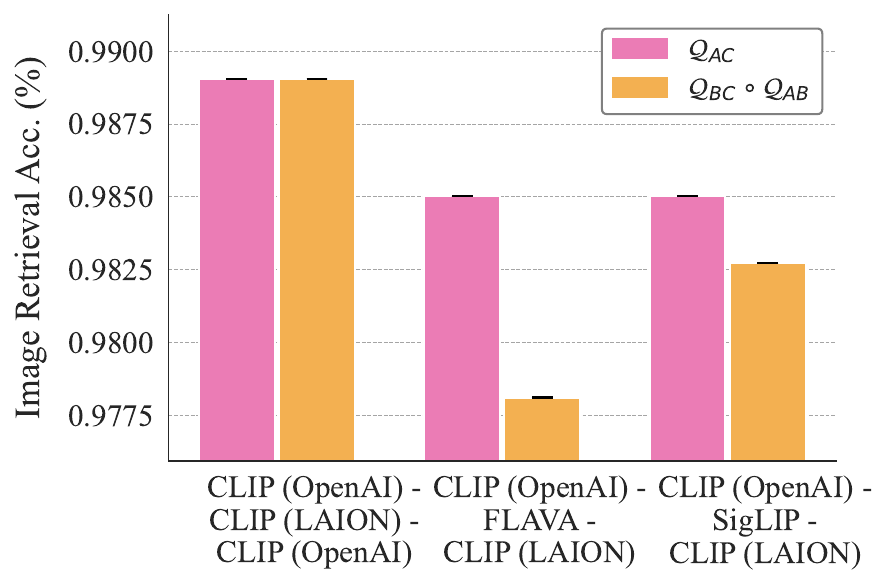}
    \caption{}
    \end{subfigure}
    \begin{subfigure}{0.32\textwidth}
    \centering
    \includegraphics[width=\textwidth]{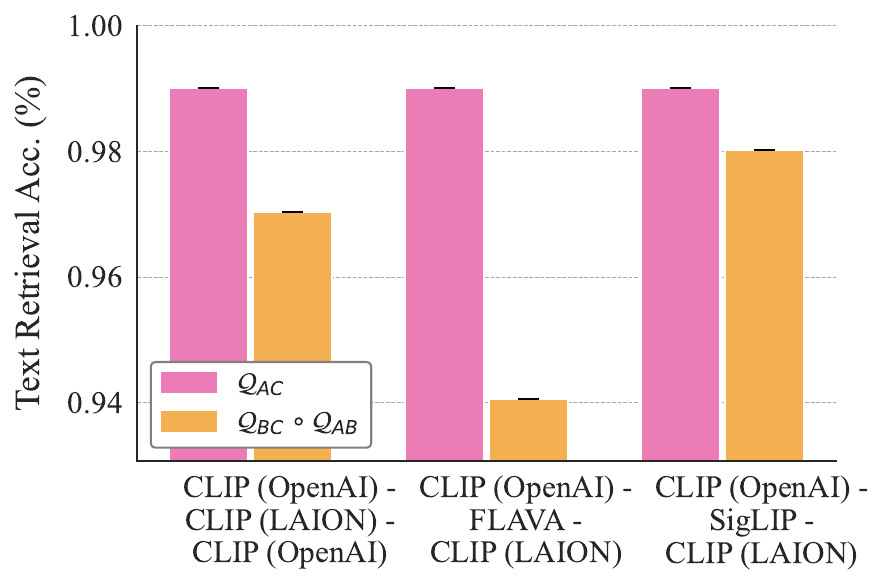}
    \caption{}
    \end{subfigure}
    \begin{subfigure}{0.32\textwidth}
    \centering
    \includegraphics[width=\textwidth]{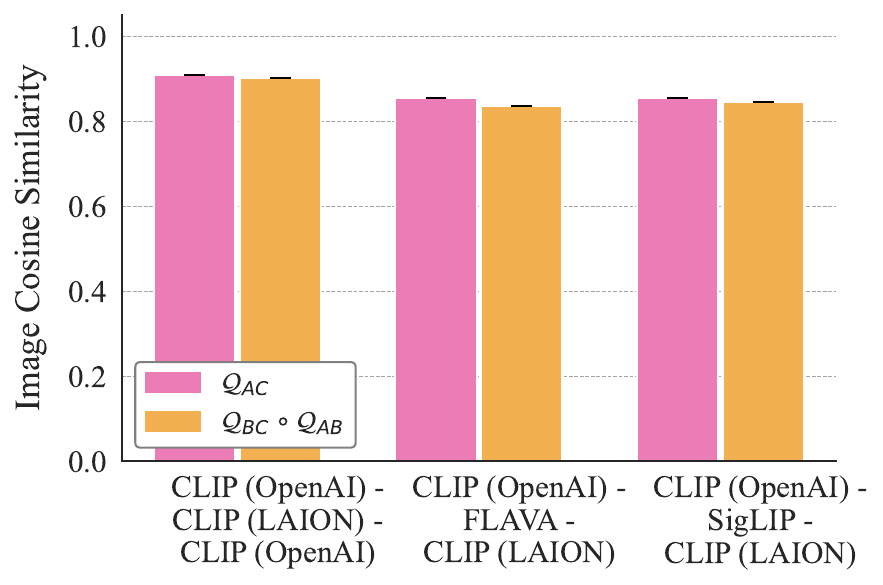}
    \caption{}
    \end{subfigure}
    \begin{subfigure}{0.32\textwidth}
    \centering
    \includegraphics[width=\textwidth]{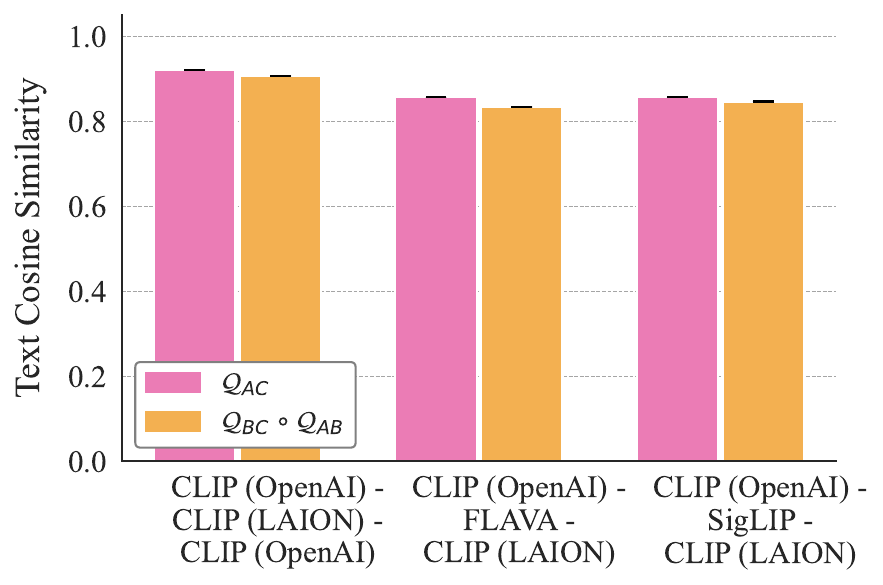}
    \caption{}
    \end{subfigure}
    \begin{subfigure}{0.32\textwidth}
    \centering
    \includegraphics[width=\textwidth]{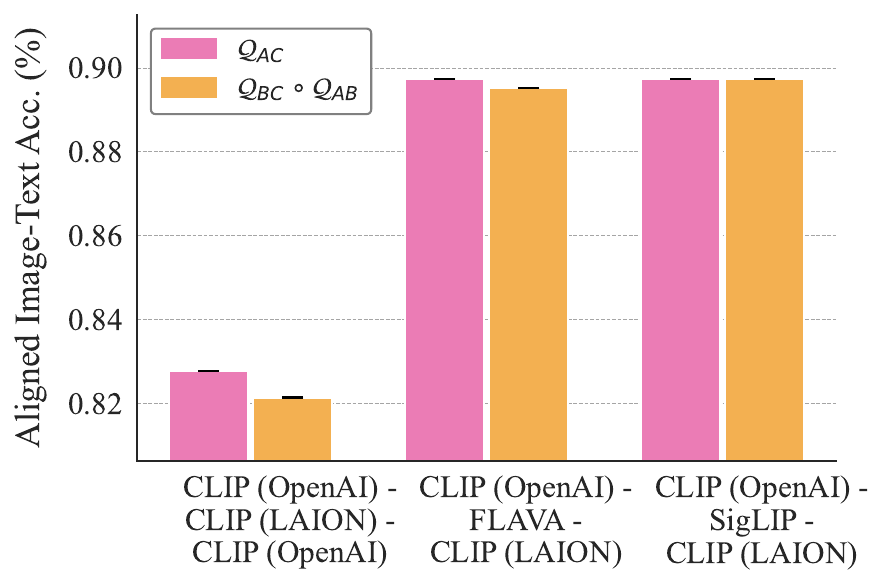}
    \caption{}
    \end{subfigure}
    \begin{subfigure}{0.32\textwidth}
    \centering
    \includegraphics[width=\textwidth]{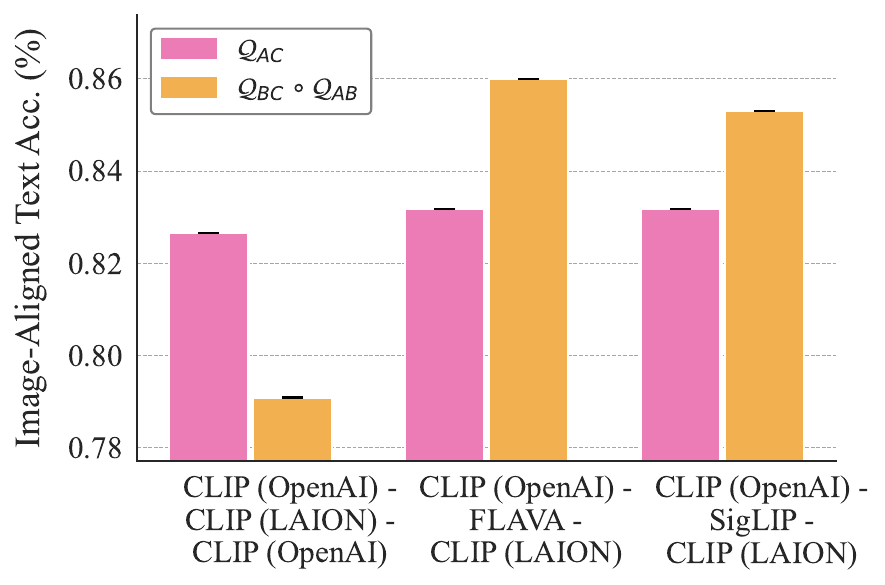}
    \caption{}
    \end{subfigure}
    \begin{subfigure}{0.32\textwidth}
    \centering
    \includegraphics[width=\textwidth]{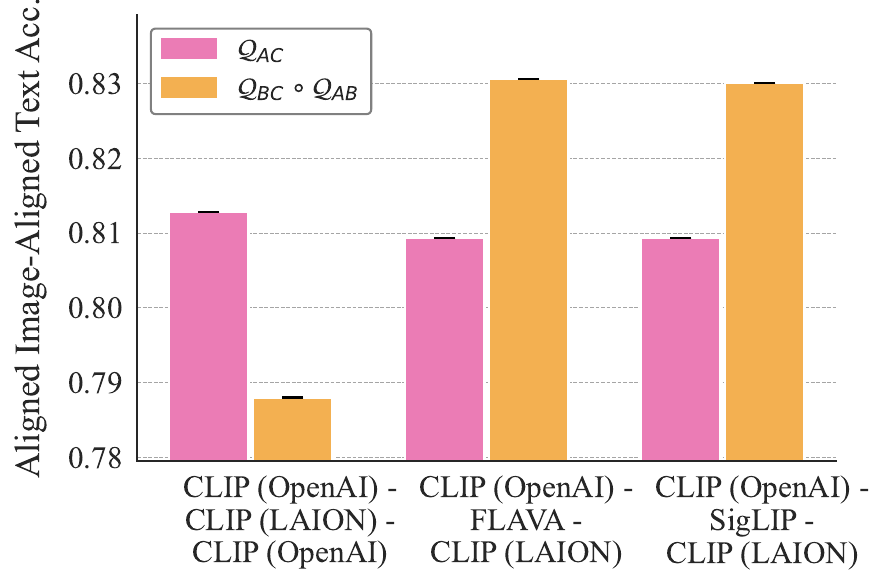}
    \caption{}
    \end{subfigure}
    \caption{\textit{Compositionality on Caltech-101 by comparing $Q_{BC}\circ Q_{AB}$ with $Q_{AC}$.}
    (a) Image–image class retrieval. (b) text–text class retrieval (c) Mean image-image cosine similarity. (c) Mean text-text cosine similarity. (e) Image–text retrieval using aligned images from model A and text from model B. (f) Image–text retrieval using images from model B and aligned text from model A. (g) Image–text retrieval using aligned images and aligned text from model A.\looseness=-1}
    \label{fig:app_trans_caltech101}
\end{figure*}

\begin{figure*}[!htb]
    \centering
    \begin{subfigure}{0.32\textwidth}
    \centering
    \includegraphics[width=\textwidth]{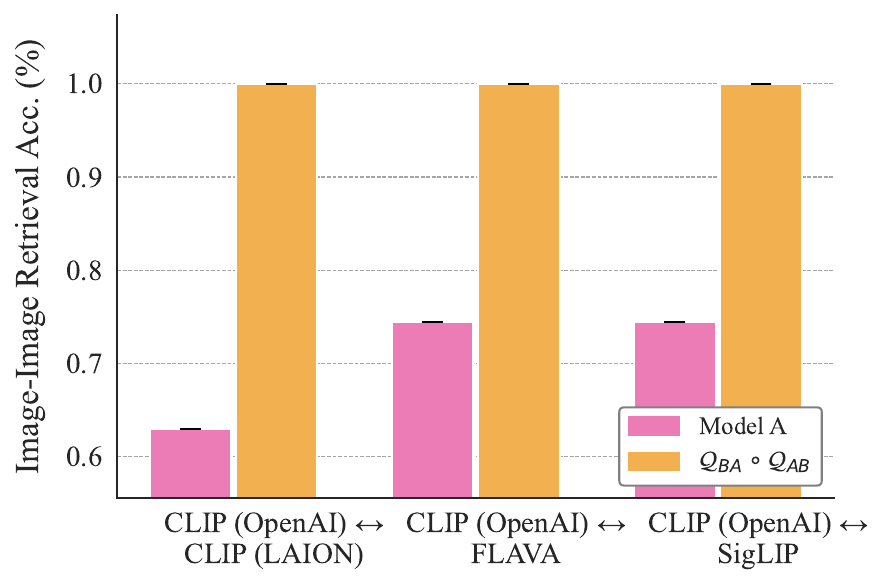}
    \caption{}
    \end{subfigure}
    \begin{subfigure}{0.32\textwidth}
    \centering
    \includegraphics[width=\textwidth]{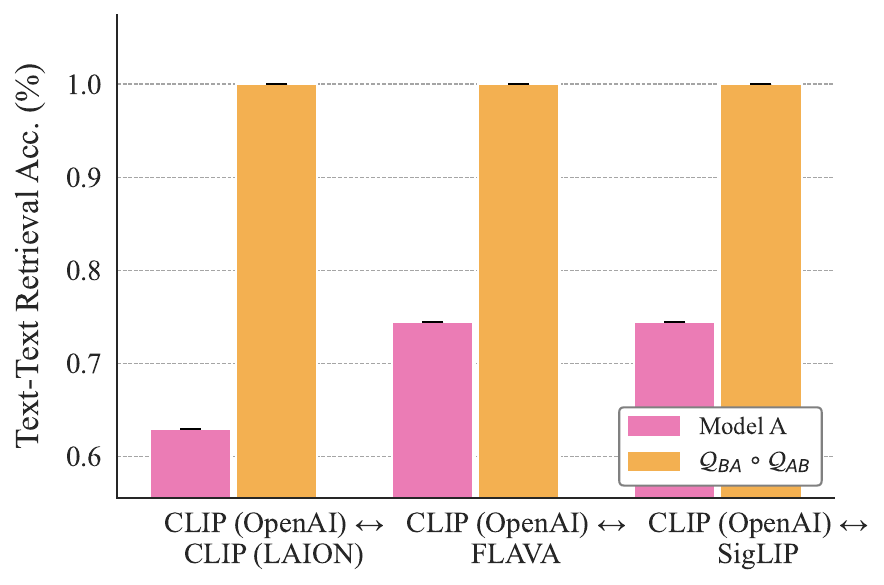}
    \caption{}
    \end{subfigure}
    \begin{subfigure}{0.32\textwidth}
    \centering
    \includegraphics[width=\textwidth]{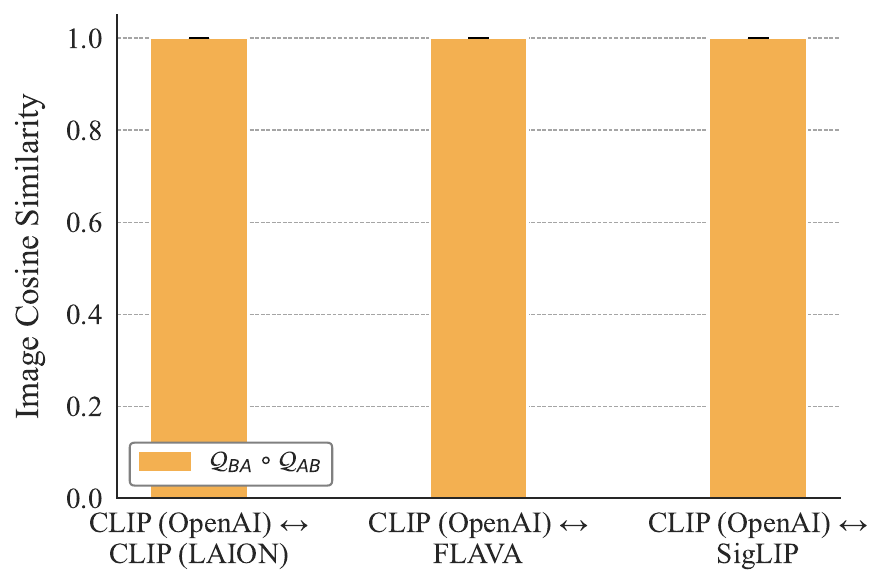}
    \caption{}
    \end{subfigure}
    \begin{subfigure}{0.32\textwidth}
    \centering
    \includegraphics[width=\textwidth]{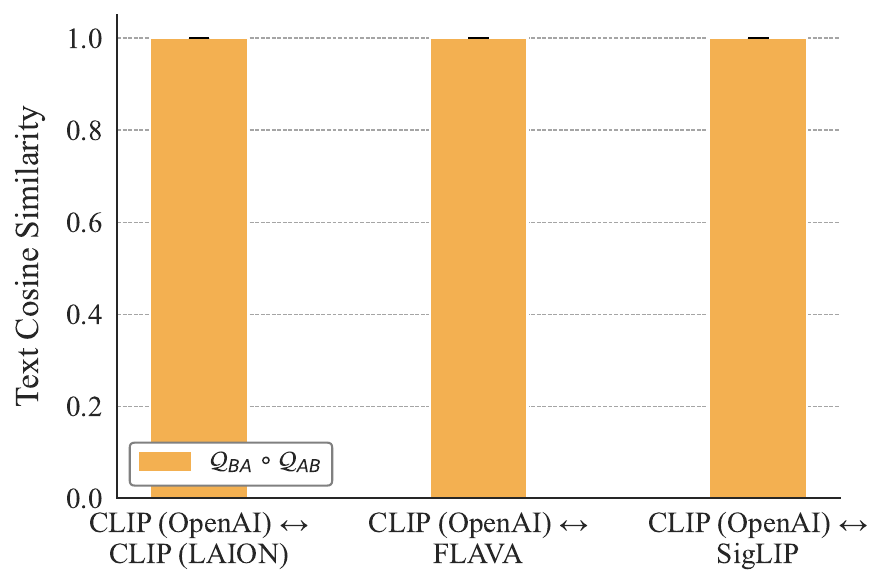}
    \caption{}
    \end{subfigure}
    \begin{subfigure}{0.32\textwidth}
    \centering
    \includegraphics[width=\textwidth]{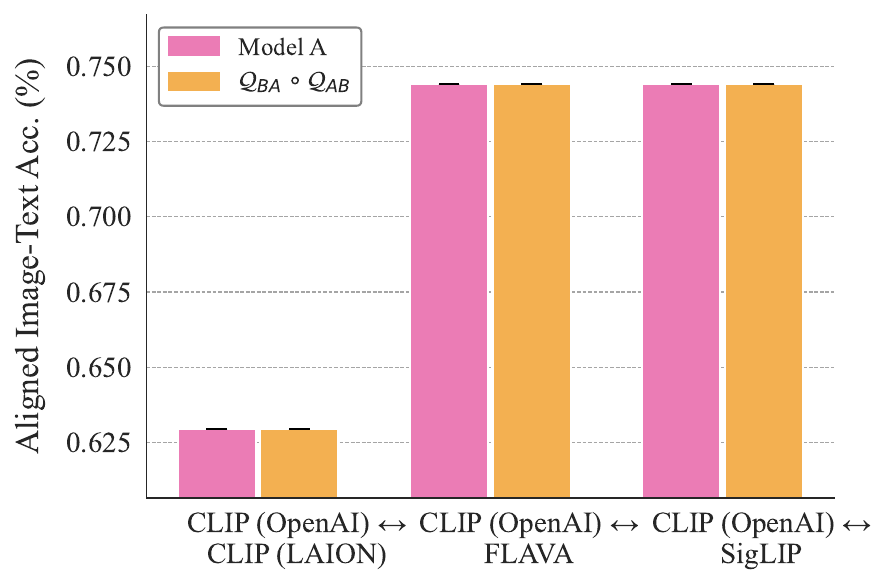}
    \caption{}
    \end{subfigure}
    \begin{subfigure}{0.32\textwidth}
    \centering
    \includegraphics[width=\textwidth]{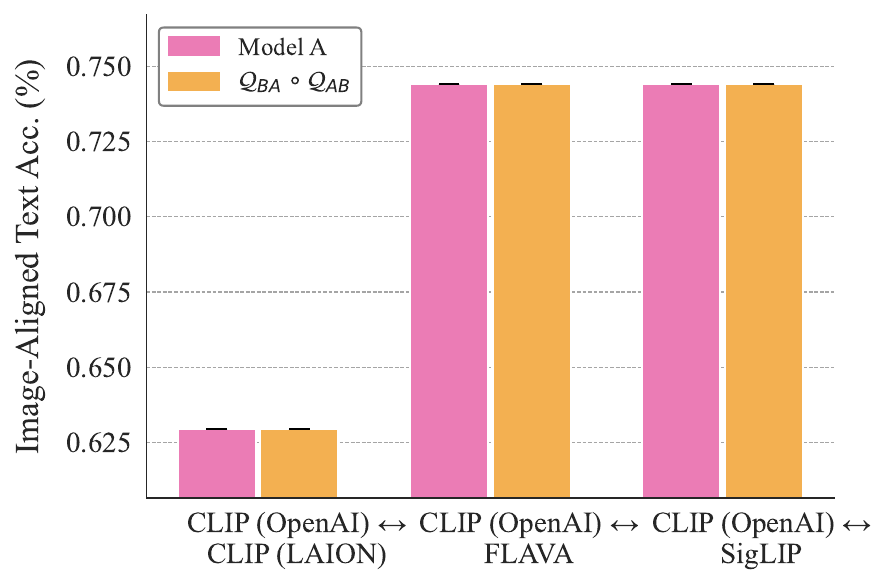}
    \caption{}
    \end{subfigure}
    \begin{subfigure}{0.32\textwidth}
    \centering
    \includegraphics[width=\textwidth]{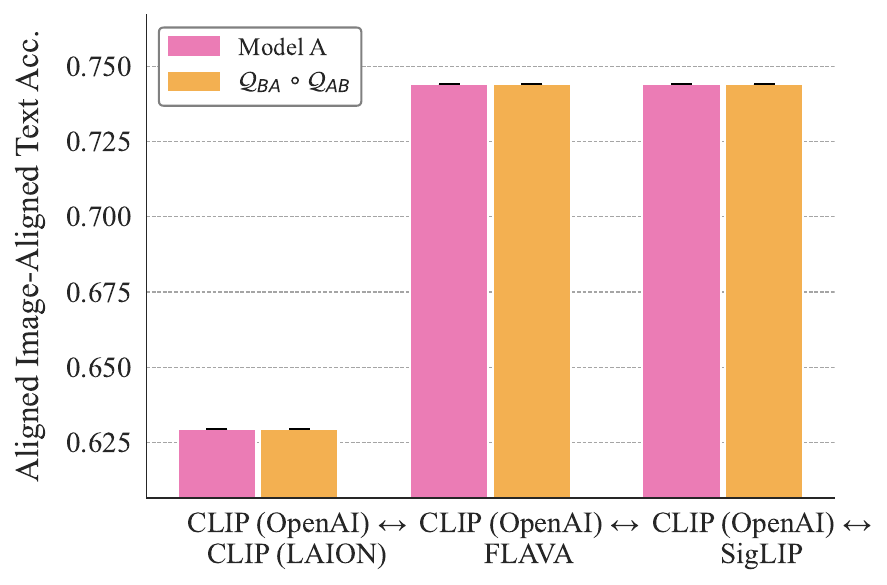}
    \caption{}
    \end{subfigure}


    \caption{\textit{Cycle consistency on CIFAR-100.}
    (a) Image–image class retrieval. (b) text–text class retrieval (c) Mean image-image cosine similarity. (c) Mean text-text cosine similarity. (e) Image–text retrieval using aligned images from model A and text from model B. (f) Image–text retrieval using images from model B and aligned text from model A. (g) Image–text retrieval using aligned images and aligned text from model A.\looseness=-1}
    \label{fig:app_cycle_cifar100_7}
\end{figure*}

\begin{figure*}[!htb]
    \centering
    \begin{subfigure}{0.32\textwidth}
    \centering
    \includegraphics[width=\textwidth]{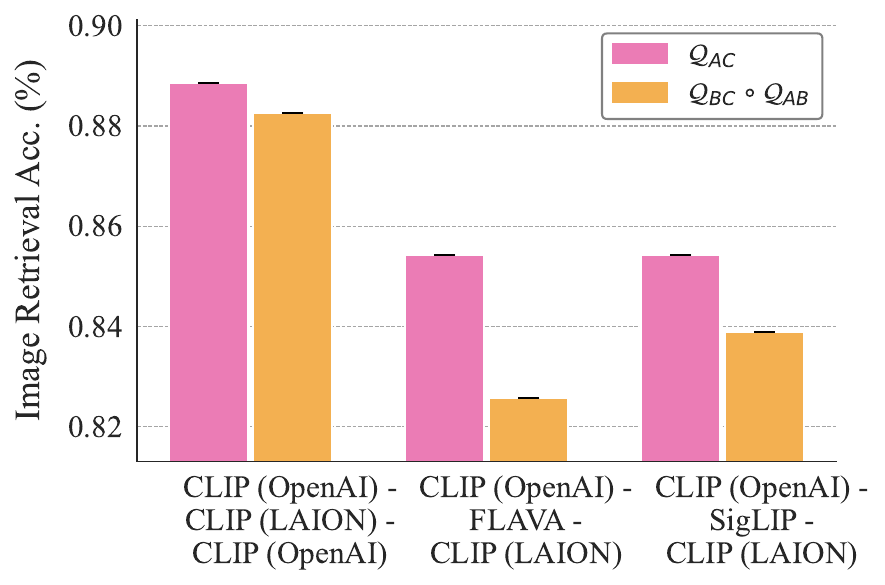}
    \caption{}
    \end{subfigure}
    \begin{subfigure}{0.32\textwidth}
    \centering
    \includegraphics[width=\textwidth]{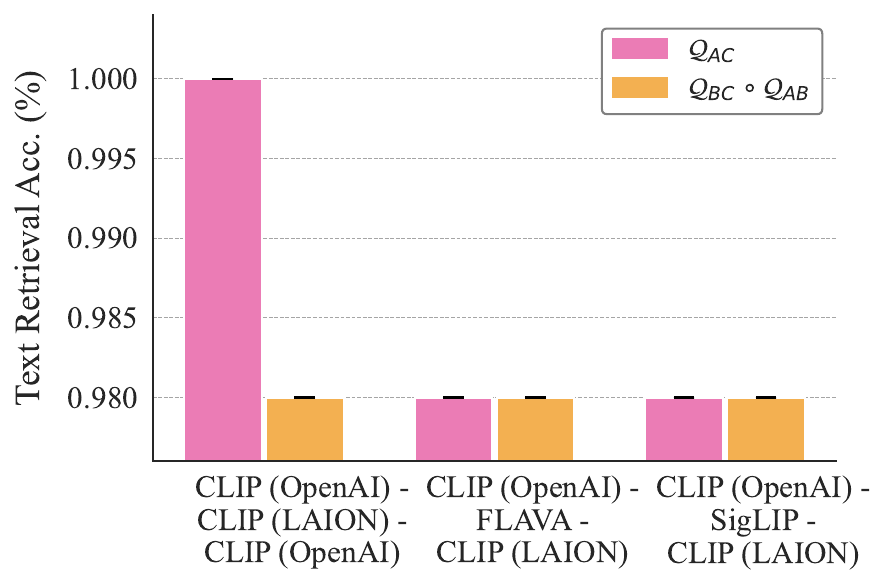}
    \caption{}
    \end{subfigure}
    \begin{subfigure}{0.32\textwidth}
    \centering
    \includegraphics[width=\textwidth]{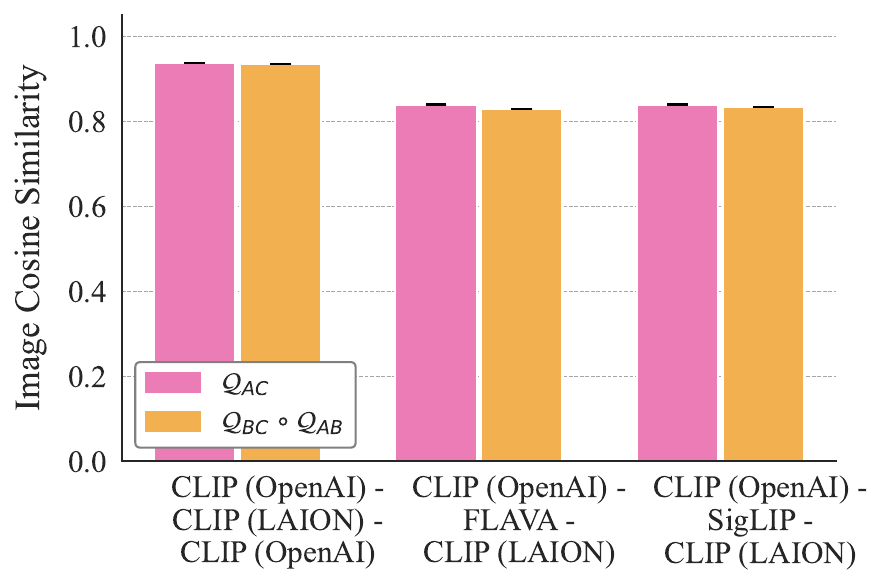}
    \caption{}
    \end{subfigure}
    \begin{subfigure}{0.32\textwidth}
    \centering
    \includegraphics[width=\textwidth]{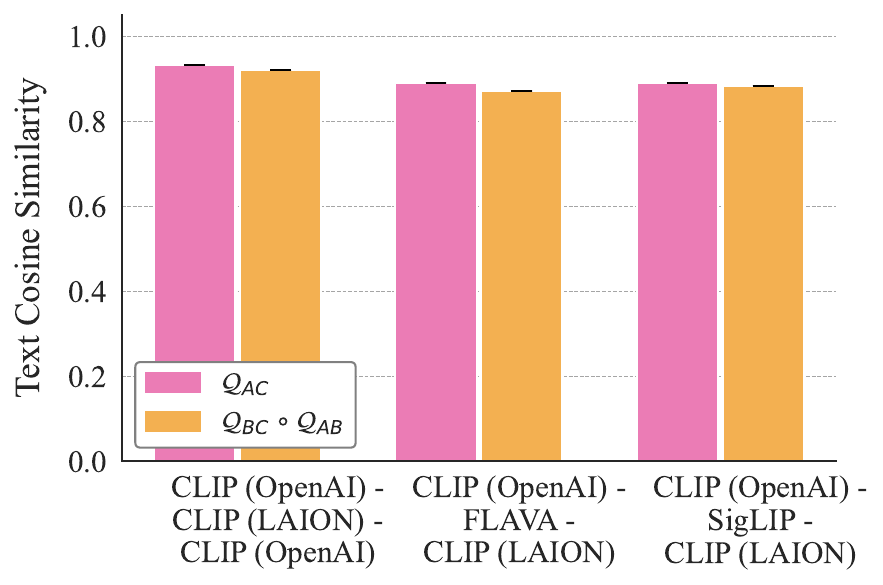}
    \caption{}
    \end{subfigure}
    \begin{subfigure}{0.32\textwidth}
    \centering
    \includegraphics[width=\textwidth]{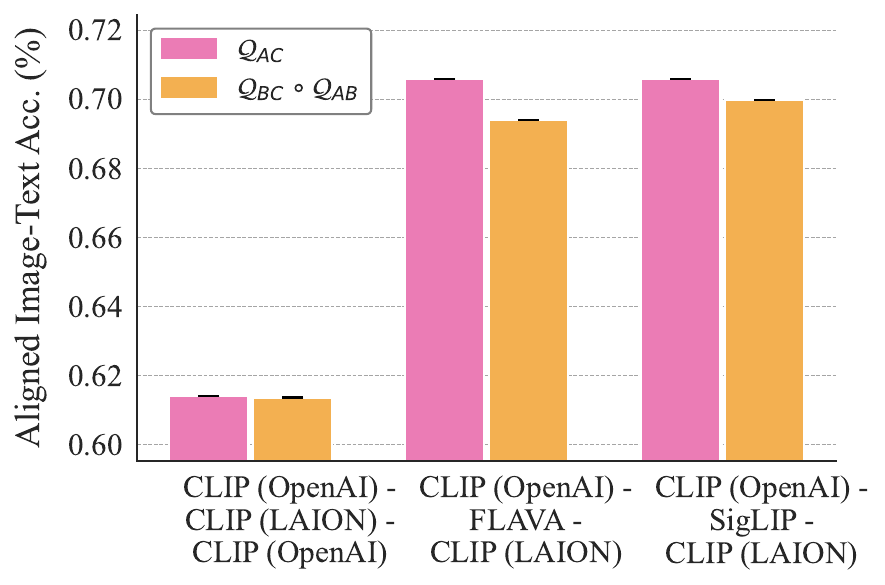}
    \caption{}
    \end{subfigure}
    \begin{subfigure}{0.32\textwidth}
    \centering
    \includegraphics[width=\textwidth]{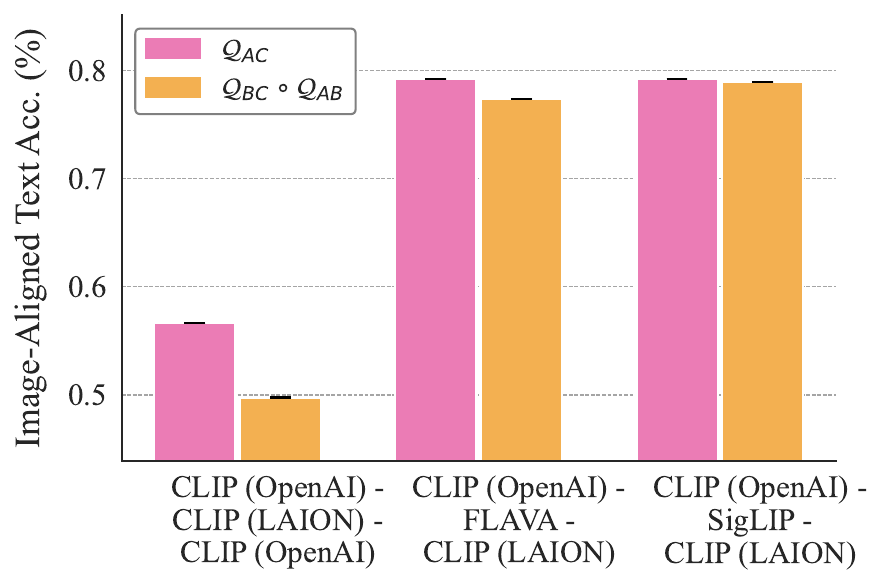}
    \caption{}
    \end{subfigure}
    \begin{subfigure}{0.32\textwidth}
    \centering
    \includegraphics[width=\textwidth]{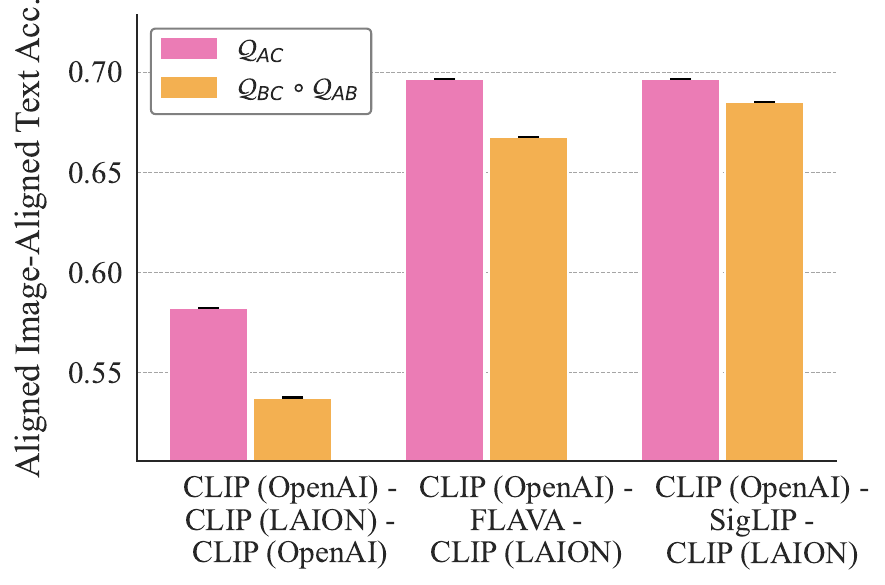}
    \caption{}
    \end{subfigure}


    \caption{\textit{Compositionality on CIFAR-100 by comparing $Q_{BC}\circ Q_{AB}$ with $Q_{AC}$.}
    (a) Image–image class retrieval. (b) text–text class retrieval (c) Mean image-image cosine similarity. (c) Mean text-text cosine similarity. (e) Image–text retrieval using aligned images from model A and text from model B. (f) Image–text retrieval using images from model B and aligned text from model A. (g) Image–text retrieval using aligned images and aligned text from model A.\looseness=-1}
    
    \label{fig:app_trans_cifar100}
\end{figure*}

\begin{figure*}[!htb]
    \centering
    \begin{subfigure}{0.32\textwidth}
    \centering
    \includegraphics[width=\textwidth]{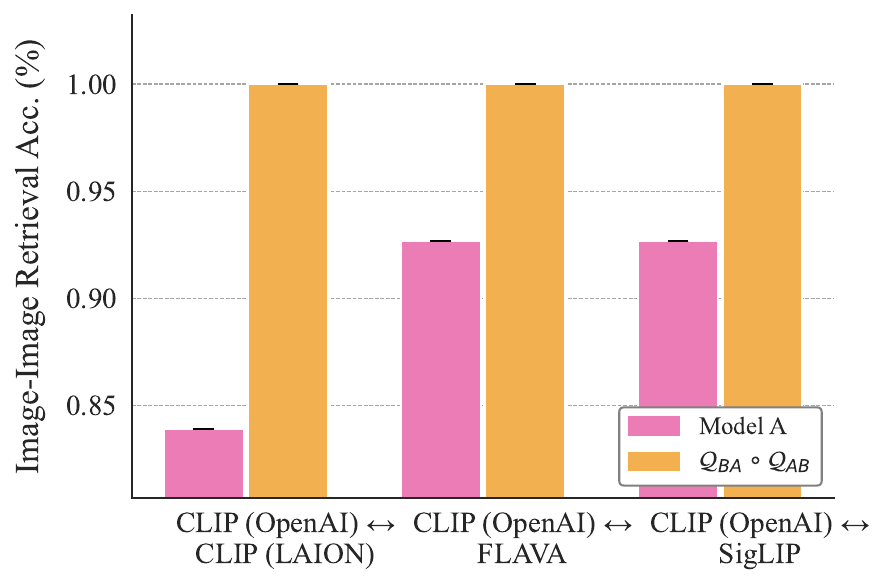}
    \caption{}
    \end{subfigure}
    \begin{subfigure}{0.32\textwidth}
    \centering
    \includegraphics[width=\textwidth]{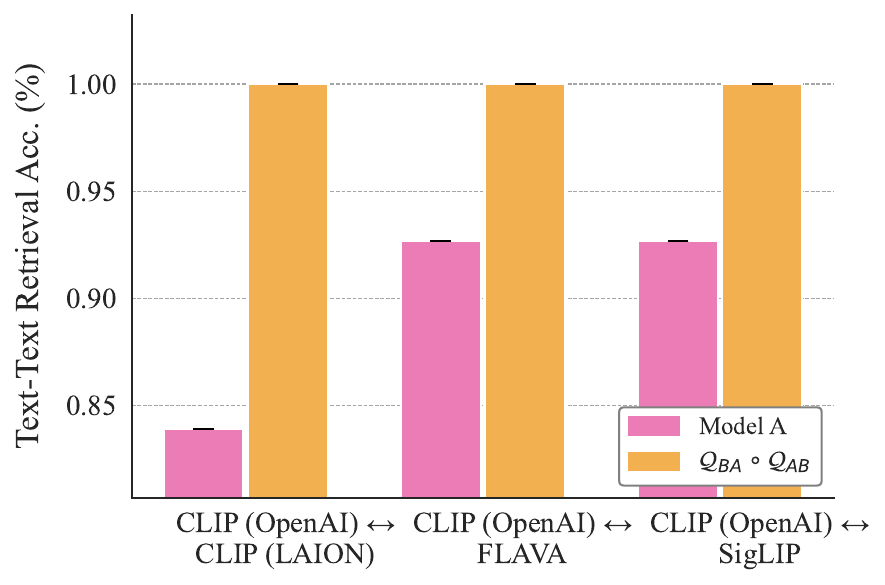}
    \caption{}
    \end{subfigure}
    \begin{subfigure}{0.32\textwidth}
    \centering
    \includegraphics[width=\textwidth]{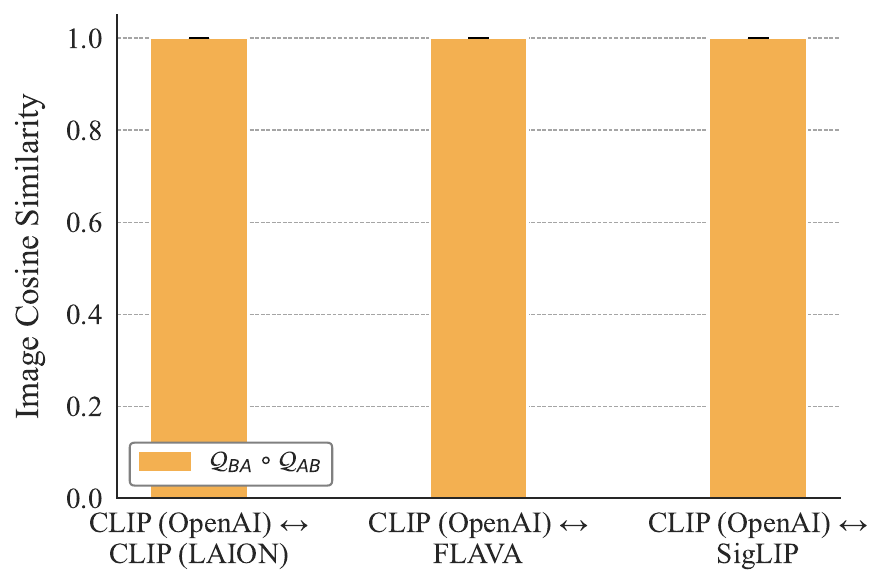}
    \caption{}
    \end{subfigure}
    \begin{subfigure}{0.32\textwidth}
    \centering
    \includegraphics[width=\textwidth]{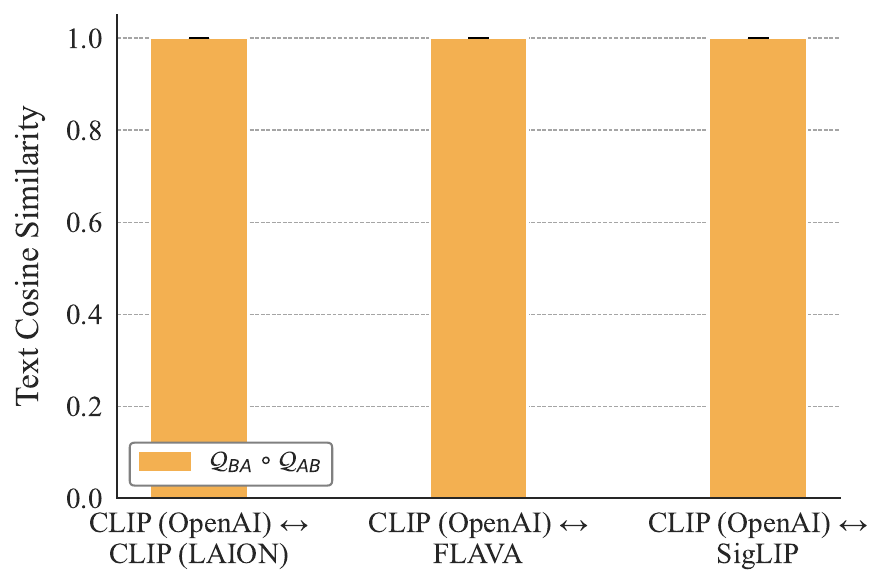}
    \caption{}
    \end{subfigure}
    \begin{subfigure}{0.32\textwidth}
    \centering
    \includegraphics[width=\textwidth]{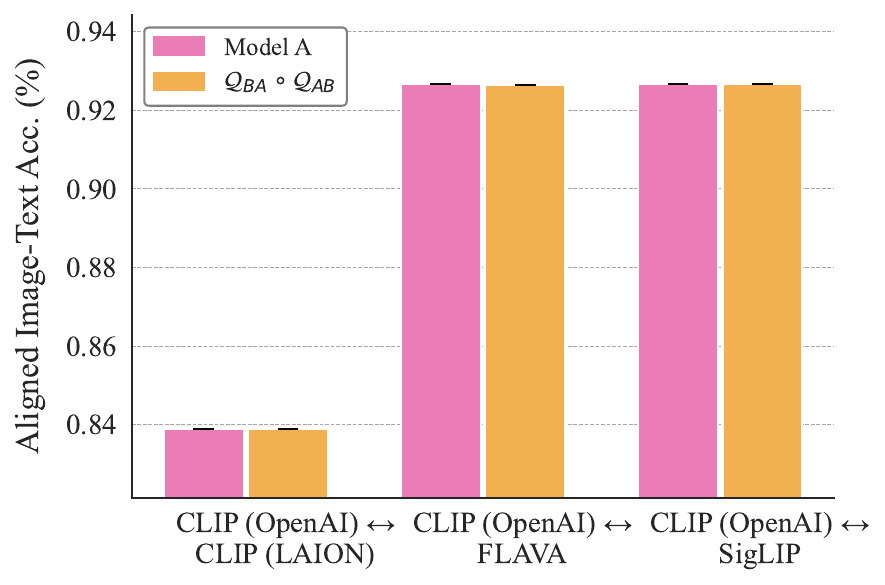}
    \caption{}
    \end{subfigure}
    \begin{subfigure}{0.32\textwidth}
    \centering
    \includegraphics[width=\textwidth]{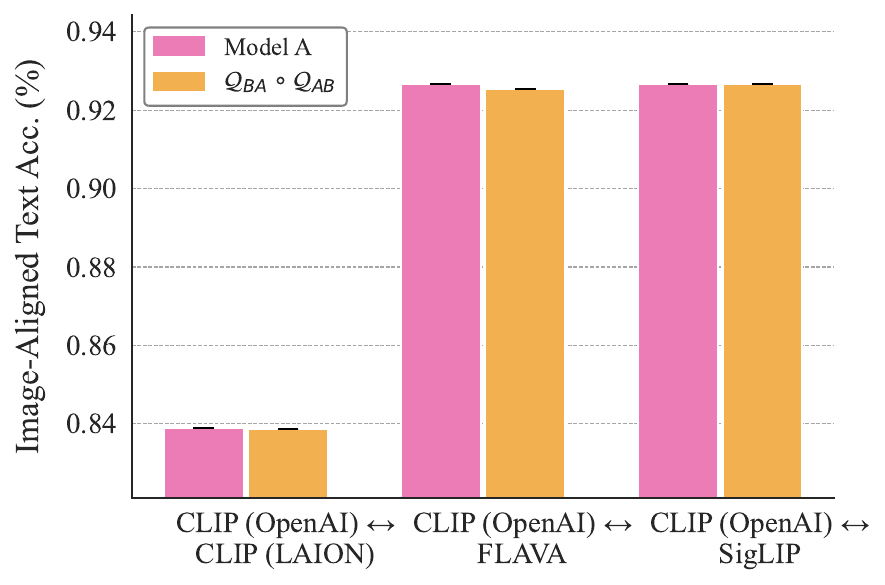}
    \caption{}
    \end{subfigure}
    \begin{subfigure}{0.32\textwidth}
    \centering
    \includegraphics[width=\textwidth]{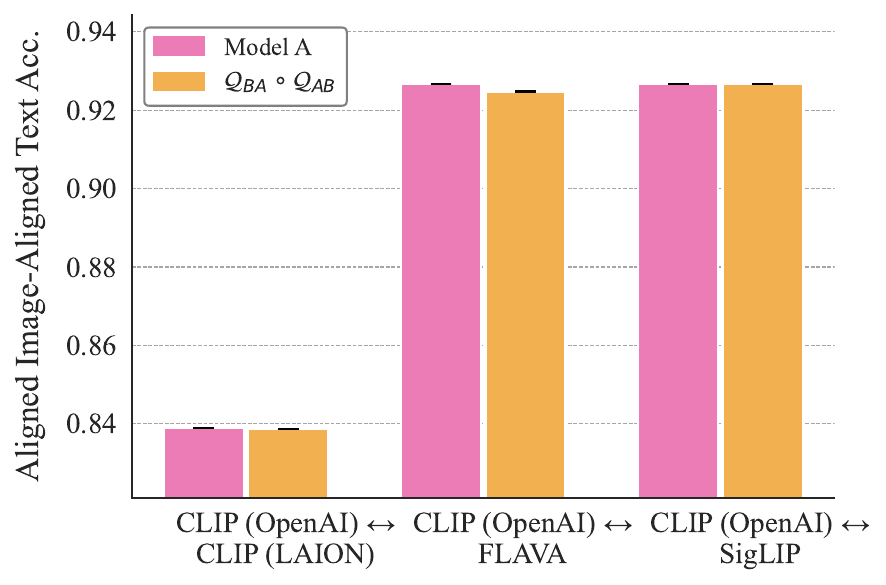}
    \caption{}
    \end{subfigure}


    \caption{\textit{Cycle consistency on Oxford Pets.}
    (a) Image–image class retrieval. (b) text–text class retrieval (c) Mean image-image cosine similarity. (c) Mean text-text cosine similarity. (e) Image–text retrieval using aligned images from model A and text from model B. (f) Image–text retrieval using images from model B and aligned text from model A. (g) Image–text retrieval using aligned images and aligned text from model A.\looseness=-1}
    \label{fig:app_cycle_oxford_7}
\end{figure*}

\begin{figure*}[!htb]
    \centering
    \begin{subfigure}{0.32\textwidth}
    \centering
    \includegraphics[width=\textwidth]{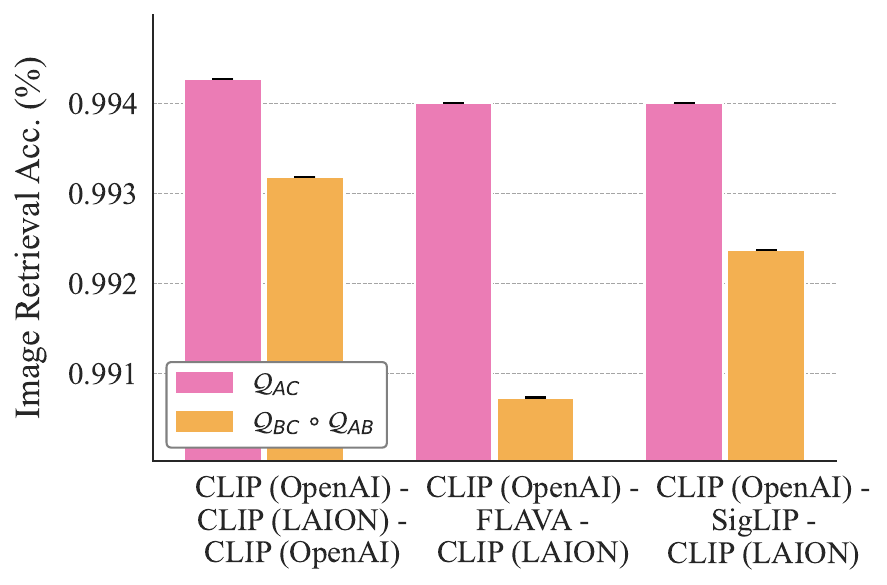}
    \caption{}
    \end{subfigure}
    \begin{subfigure}{0.32\textwidth}
    \centering
    \includegraphics[width=\textwidth]{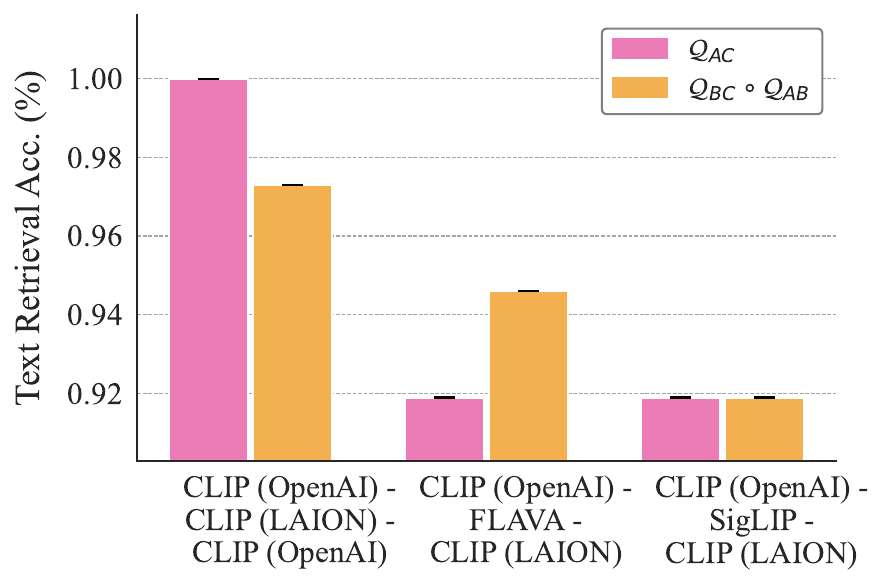}
    \caption{}
    \end{subfigure}
    \begin{subfigure}{0.32\textwidth}
    \centering
    \includegraphics[width=\textwidth]{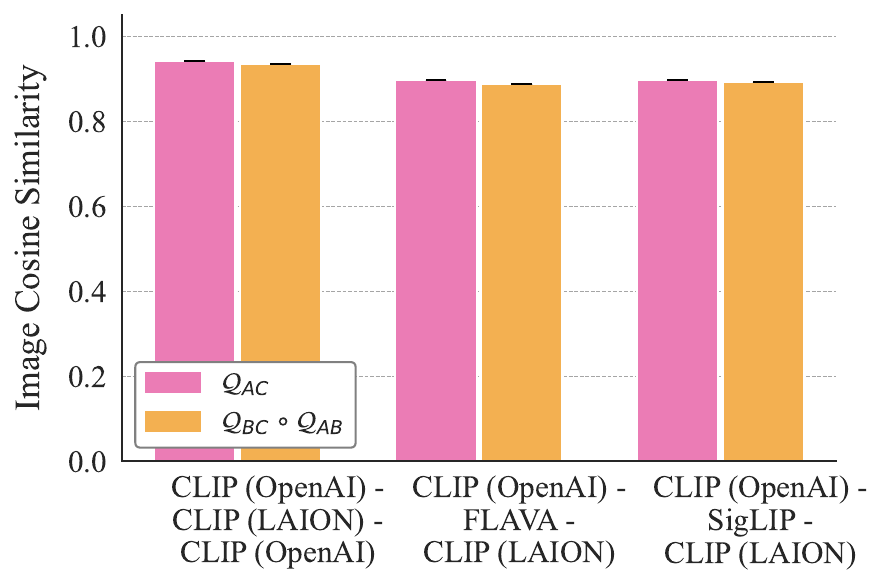}
    \caption{}
    \end{subfigure}
    \begin{subfigure}{0.32\textwidth}
    \centering
    \includegraphics[width=\textwidth]{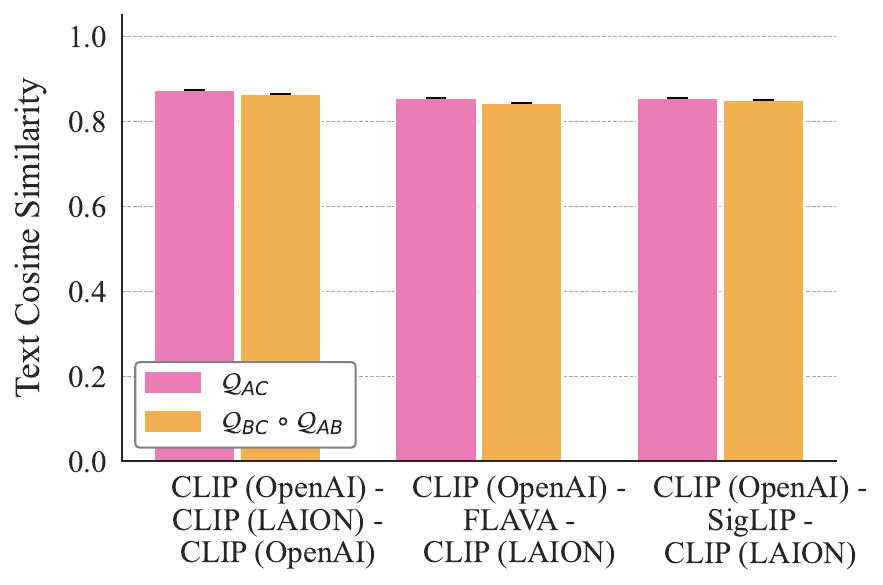}
    \caption{}
    \end{subfigure}
    \begin{subfigure}{0.32\textwidth}
    \centering
    \includegraphics[width=\textwidth]{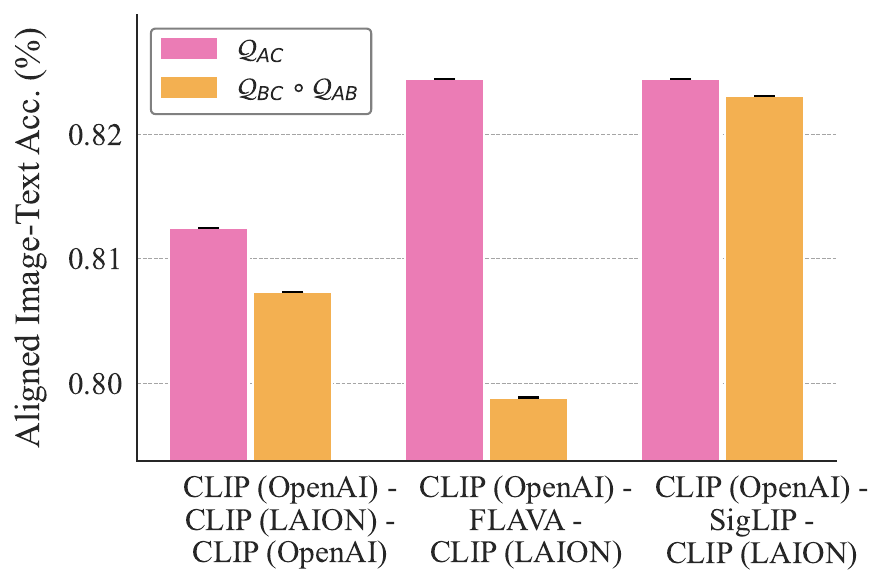}
    \caption{}
    \end{subfigure}
    \begin{subfigure}{0.32\textwidth}
    \centering
    \includegraphics[width=\textwidth]{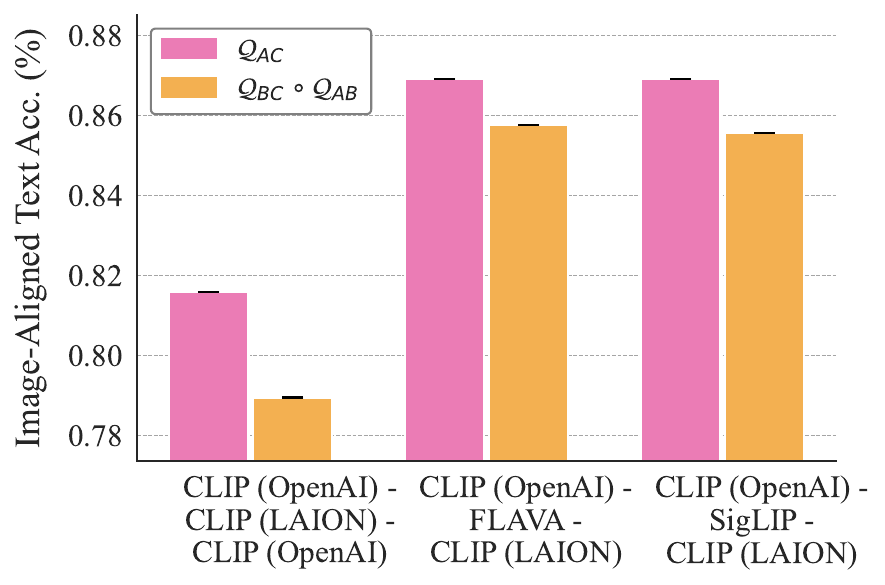}
    \caption{}
    \end{subfigure}
    \begin{subfigure}{0.32\textwidth}
    \centering
    \includegraphics[width=\textwidth]{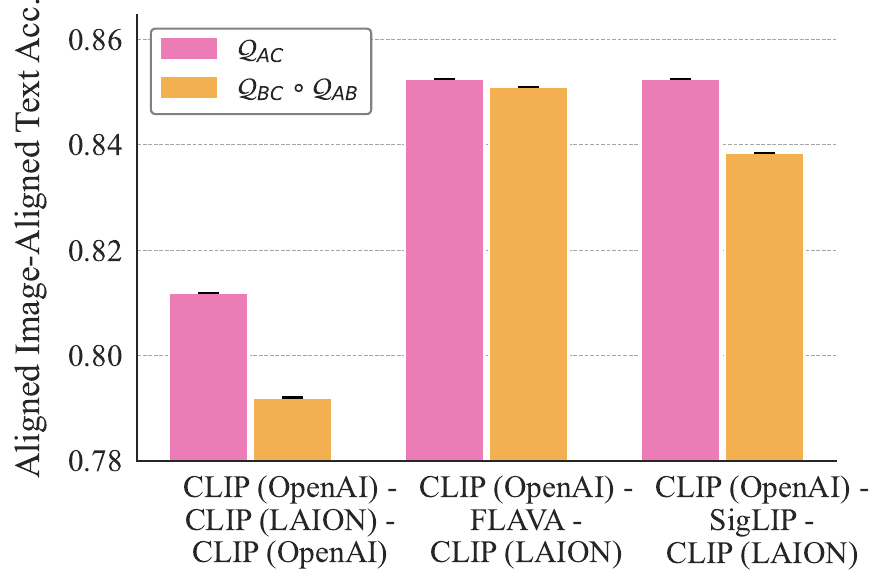}
    \caption{}
    \end{subfigure}


    \caption{\textit{Compositionality on Oxford Pets by comparing $Q_{BC}\circ Q_{AB}$ with $Q_{AC}$.}
    (a) Image–image class retrieval. (b) text–text class retrieval (c) Mean image-image cosine similarity. (c) Mean text-text cosine similarity. (e) Image–text retrieval using aligned images from model A and text from model B. (f) Image–text retrieval using images from model B and aligned text from model A. (g) Image–text retrieval using aligned images and aligned text from model A.\looseness=-1}
    
    \label{fig:app_trans_oxford}
\end{figure*}

\FloatBarrier
\subsection{Alignment Across Embedding Dimensions.}\label{sec:app_different_dims}
All results so far use matched embedding sizes, where an orthogonal map is well-defined. We next align \textsc{CLIP} ViT-B/32 ($d{=}512$) with \textsc{CLIP} ViT-L/14 ($d{=}768$) using a rectangular $\mathcal Q$ projected onto the Stiefel manifold, enforcing $\mathcal Q^\top \mathcal Q \approx I$. In~\Cref{fig:app_different_dims_caltech101,fig:app_different_dims_cifar100,fig:app_different_dims_oxford}, we report performance before and after learning $\mathcal Q$ on image embeddings, and evaluate how effectively the learned map transfers to texts. Across all figures, image-text accuracy remains high in the aligned space, whether using aligned images with native text, native images with aligned text, or aligning both. Text-text pointwise cosine and retrieval also remain strong, indicating that near-isometric maps preserve task-relevant geometry while enabling reliable transfer to the other modality (text in this case).

\begin{figure*}[!htb]
    \centering
    \begin{minipage}{0.65\linewidth}  
    \begin{subfigure}{0.32\textwidth}
    \centering
    \includegraphics[width=\textwidth]{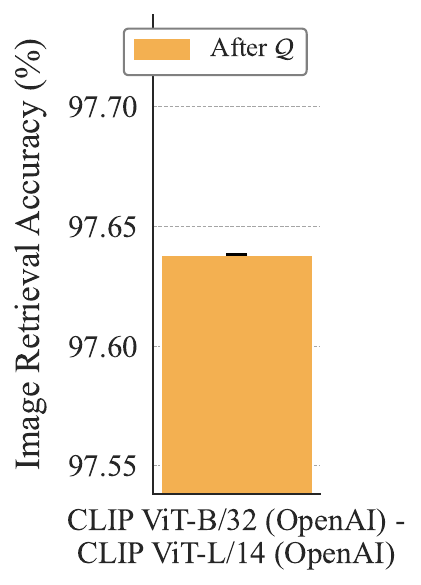}
    \caption{}
    \end{subfigure}
    \begin{subfigure}{0.32\textwidth}
    \centering
    \includegraphics[width=\textwidth]{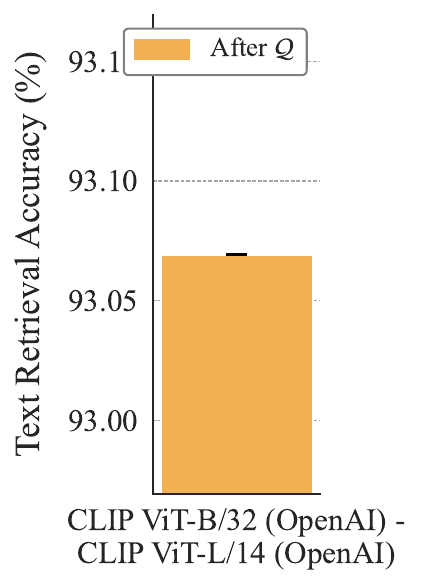}
    \caption{}
    \end{subfigure}
    \begin{subfigure}{0.32\textwidth}
    \centering
    \includegraphics[width=\textwidth]{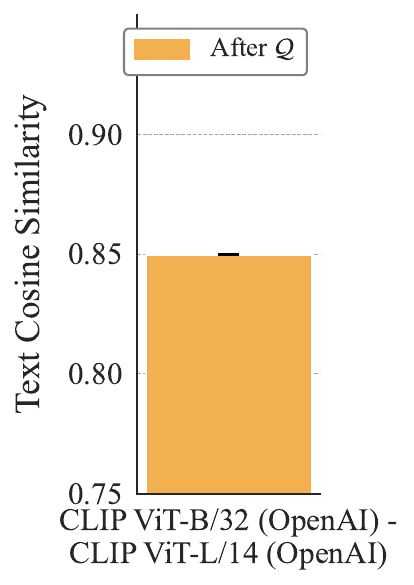}
    \caption{}
    \end{subfigure}

    \vspace{1mm}
    \begin{subfigure}{0.32\textwidth}
    \centering
    \includegraphics[width=\textwidth]{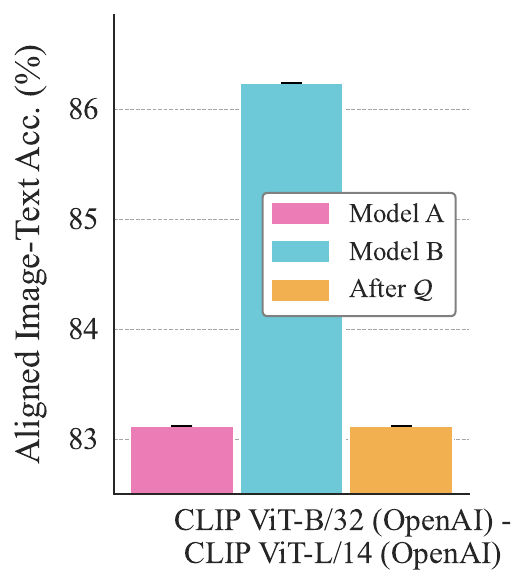}
    \caption{}
    \end{subfigure}
    \begin{subfigure}{0.32\textwidth}
    \centering
    \includegraphics[width=\textwidth]{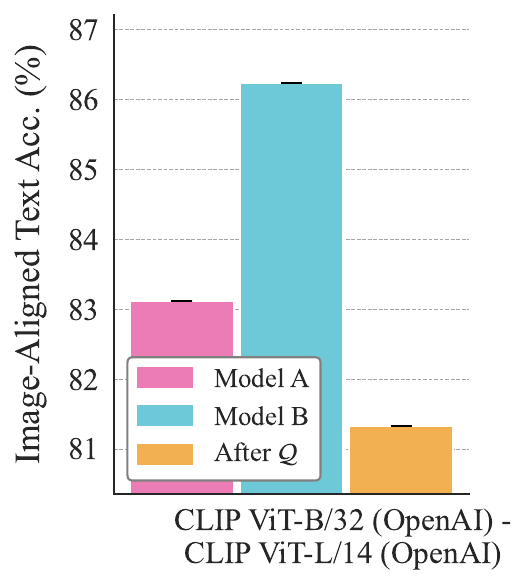}
    \caption{}
    \end{subfigure}
    \begin{subfigure}{0.32\textwidth}
    \centering
    \includegraphics[width=\textwidth]{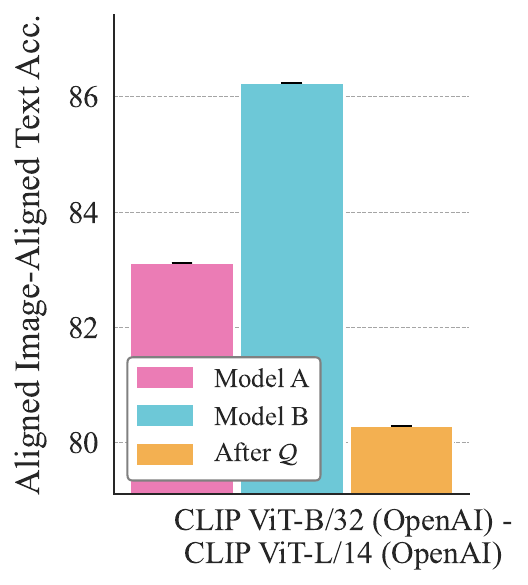}
    \caption{}
    \end{subfigure}
    \end{minipage}
    \caption{\textit{Cross-model alignment across differing embedding dimensions (Caltech-101; CLIP ViT-B/32 to CLIP ViT-L/14, OpenAI) before and after fitting an orthogonal map $\mathcal{Q}$.} (a) Image–image class retrieval and (b) text–text class retrieval. (c) Mean text–text cosine similarity. (d) Image–text retrieval using aligned images from model A and text from model B. (e) Image–text retrieval using images from model B and aligned text from model A. (f) Image–text retrieval using aligned images and aligned text from model A. Enforcing semi-orthogonality via Stiefel manifold projection allows the $512$-dimensional source manifold to be isometrically embedded into the $768$-dimensional target space, maintaining high cross-modal retrieval accuracy.}
    \label{fig:app_different_dims_caltech101}
\end{figure*}

\begin{figure*}[!htb]
    \centering
    \begin{minipage}{0.65\linewidth} 
    \begin{subfigure}{0.32\textwidth}
    \centering
    \includegraphics[width=\textwidth]{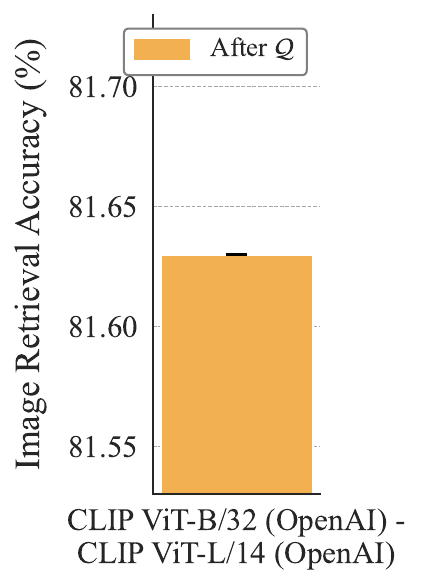}
    \caption{}
    \end{subfigure}
    \begin{subfigure}{0.32\textwidth}
    \centering
    \includegraphics[width=\textwidth]{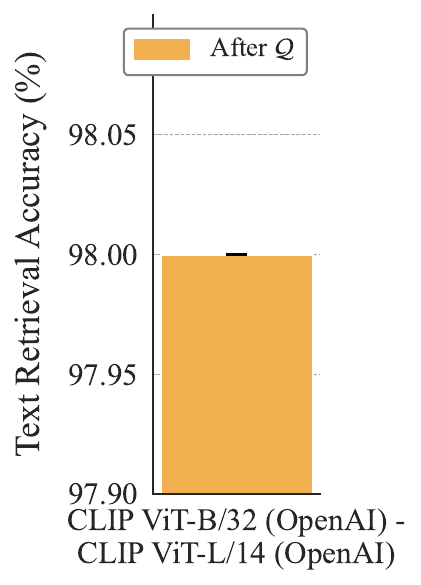}
    \caption{}
    \end{subfigure}
    \begin{subfigure}{0.32\textwidth}
    \centering
    \includegraphics[width=\textwidth]{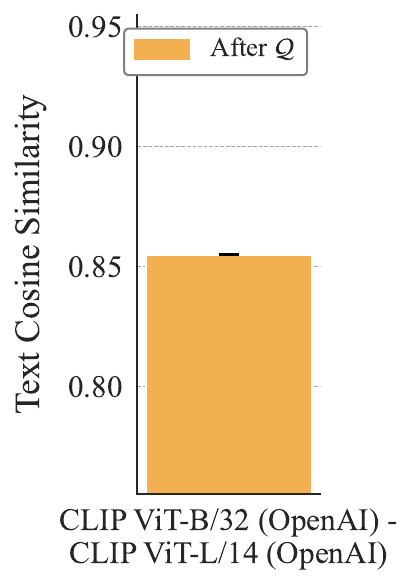}
    \caption{}
    \end{subfigure}

    \vspace{1mm}
    \begin{subfigure}{0.32\textwidth}
    \centering
    \includegraphics[width=\textwidth]{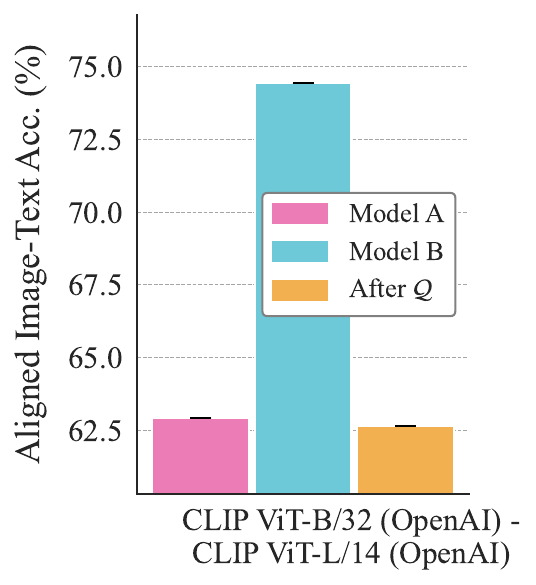}
    \caption{}
    \end{subfigure}
    \begin{subfigure}{0.32\textwidth}
    \centering
    \includegraphics[width=\textwidth]{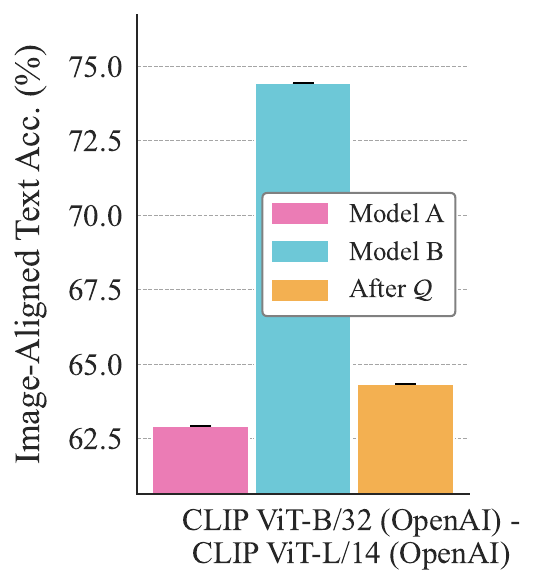}
    \caption{}
    \end{subfigure}
    \begin{subfigure}{0.32\textwidth}
    \centering
    \includegraphics[width=\textwidth]{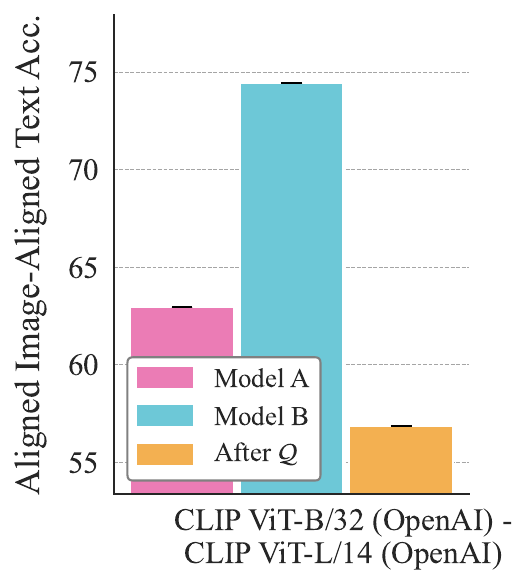}
    \caption{}
    \end{subfigure}
    \end{minipage}
    \caption{\textit{Cross-model alignment across differing embedding dimensions (CIFAR-100; CLIP ViT-B/32 to CLIP ViT-L/14, OpenAI) before and after fitting an orthogonal map $\mathcal{Q}$.} (a) Image–image class retrieval and (b) text–text class retrieval. (c) Mean text–text cosine similarity. (d) Image–text retrieval using aligned images from model A and text from model B. (e) Image–text retrieval using images from model B and aligned text from model A. (f) Image–text retrieval using aligned images and aligned text from model A. Enforcing semi-orthogonality via Stiefel manifold projection allows the $512$-dimensional source manifold to be isometrically embedded into the $768$-dimensional target space, maintaining high cross-modal retrieval accuracy.}
    \label{fig:app_different_dims_cifar100}
\end{figure*}

\begin{figure*}[!htb]
    \centering
    \begin{minipage}{0.65\linewidth} 
    \begin{subfigure}{0.32\textwidth}
    \centering
    \includegraphics[width=\textwidth]{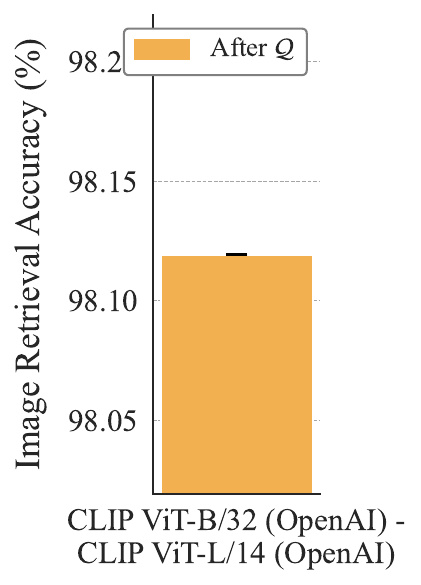}
    \caption{}
    \end{subfigure}
    \begin{subfigure}{0.32\textwidth}
    \centering
    \includegraphics[width=\textwidth]{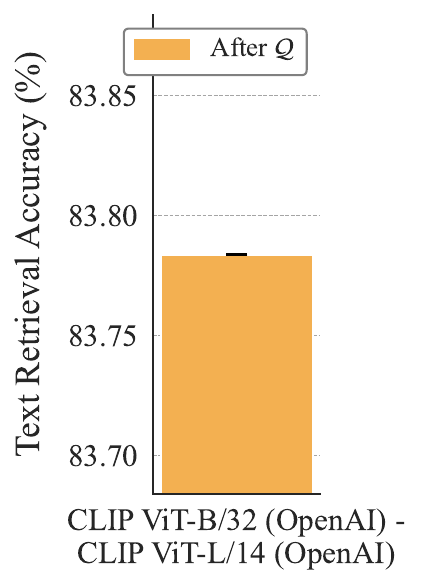}
    \caption{}
    \end{subfigure}
    \begin{subfigure}{0.32\textwidth}
    \centering
    \includegraphics[width=\textwidth]{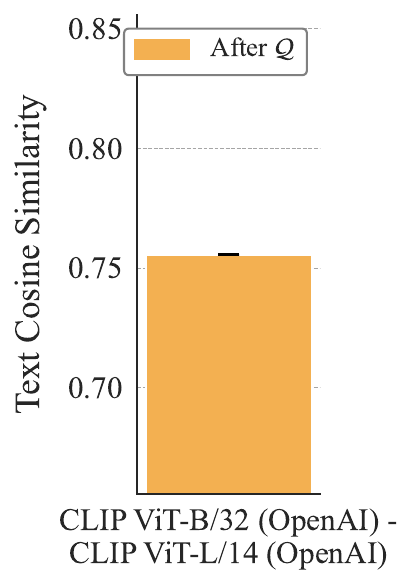}
    \caption{}
    \end{subfigure}

    \vspace{1mm}
    \begin{subfigure}{0.32\textwidth}
    \centering
    \includegraphics[width=\textwidth]{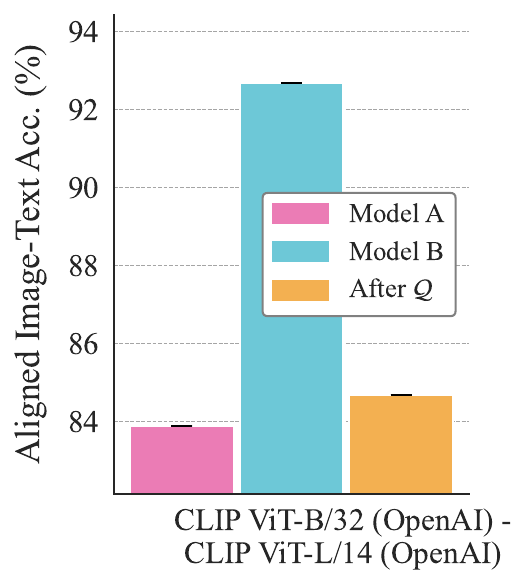}
    \caption{}
    \end{subfigure}
    \begin{subfigure}{0.32\textwidth}
    \centering
    \includegraphics[width=\textwidth]{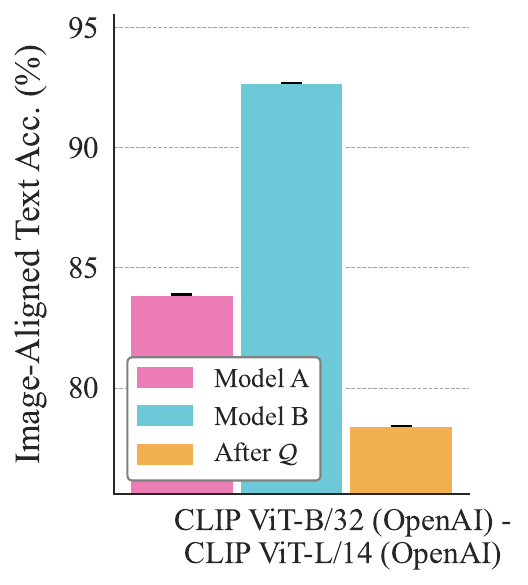}
    \caption{}
    \end{subfigure}
    \begin{subfigure}{0.32\textwidth}
    \centering
    \includegraphics[width=\textwidth]{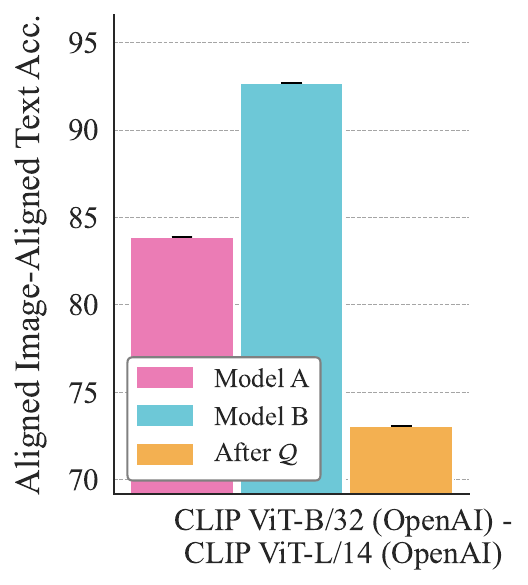}
    \caption{}
    \end{subfigure}
    \end{minipage}
    \caption{\textit{Cross-model alignment across differing embedding dimensions (Oxford Pets; CLIP ViT-B/32 to CLIP ViT-L/14, OpenAI) before and after fitting an orthogonal map $\mathcal{Q}$.} (a) Image–image class retrieval and (b) text–text class retrieval. (c) Mean text–text cosine similarity. (d) Image–text retrieval using aligned images from model A and text from model B. (e) Image–text retrieval using images from model B and aligned text from model A. (f) Image–text retrieval using aligned images and aligned text from model A. Enforcing semi-orthogonality via Stiefel manifold projection allows the $512$-dimensional source manifold to be isometrically embedded into the $768$-dimensional target space, maintaining high cross-modal retrieval accuracy.}
    \label{fig:app_different_dims_oxford}
\end{figure*}

\FloatBarrier
\subsection{Evaluating Alternative Alignment Maps Than The Orthogonal Mapping}\label{sec:app_compare_strategies}
In this section, we ablate the design choices for the alignment function and test three alignment strategies of increasing expressiveness: (i) an orthogonal Procrustes map $Q$, (ii) an unconstrained linear map, and (iii) a small MLP with a residual connection. Results across all model pairs are reported for Caltech-101 (\Cref{fig:app_compare_caltech101}), CIFAR-100 (\Cref{fig:app_compare_cifar100}), and Oxford Pets (\Cref{fig:app_compare_oxford}).
We find that more expressive maps achieve better \emph{pointwise} alignment: both the unconstrained linear map and the MLP yield higher text-text pointwise cosine similarity, reflecting their ability to approximate small distortions and scaling. However, this flexibility comes at the cost of geometric fidelity. The orthogonal map consistently outperforms the alternatives on geometry-sensitive metrics---most notably image to aligned text and aligned-image to aligned-text classification accuracy---which are the metrics that matter for downstream use. By contrast, text-text and image-image retrieval are comparable across all three approaches, with no clear winner. These results indicate that while flexible maps can improve local fit, preserving global geometry through orthogonality can be critical for reliable cross-modal transfer.

\begin{figure*}[!htb]
    \centering
    \begin{subfigure}{0.32\textwidth}
    \centering
    \includegraphics[width=\textwidth]{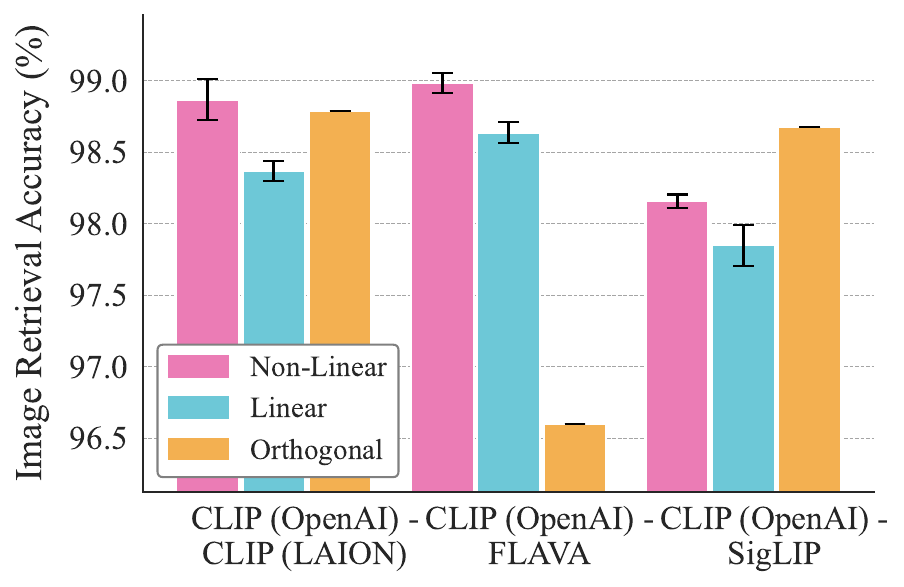}
    \caption{}
    \end{subfigure}
    \begin{subfigure}{0.32\textwidth}
    \centering
    \includegraphics[width=\textwidth]{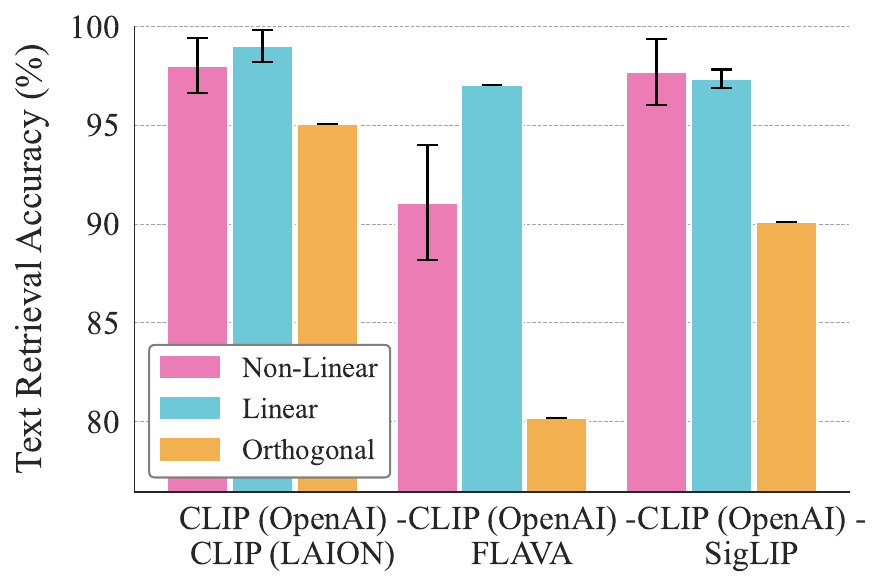}
    \caption{}
    \end{subfigure}
    \begin{subfigure}{0.32\textwidth}
    \centering
    \includegraphics[width=\textwidth]{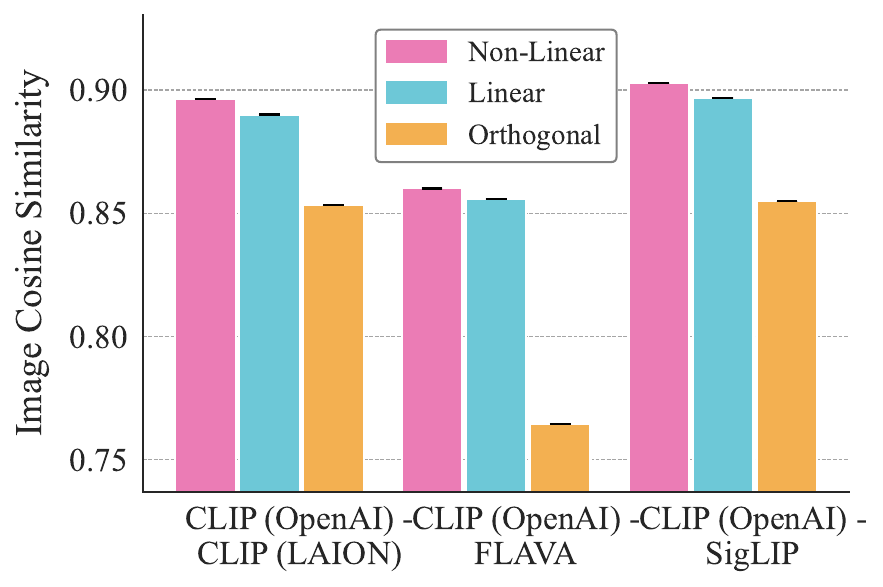}
    \caption{}
    \end{subfigure}
    \begin{subfigure}{0.32\textwidth}
    \centering
    \includegraphics[width=\textwidth]{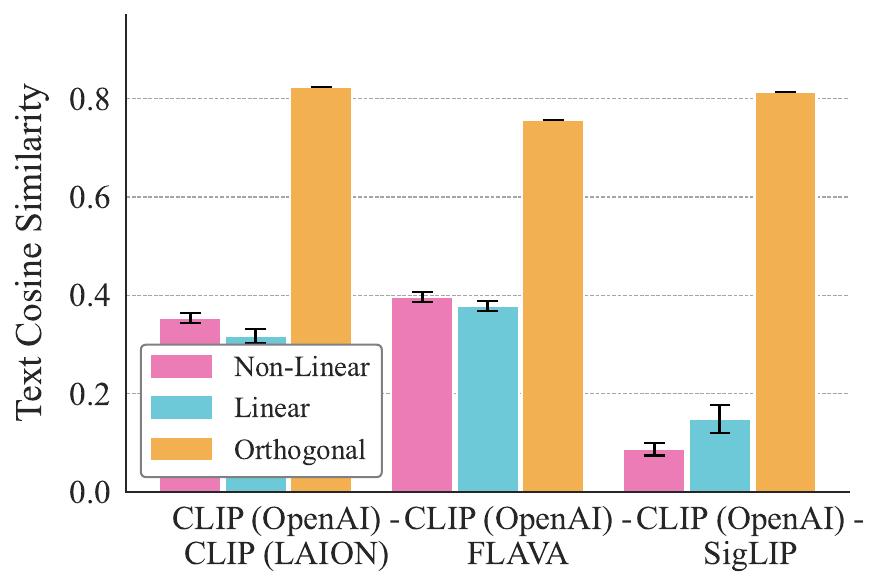}
    \caption{}
    \end{subfigure}
    \begin{subfigure}{0.32\textwidth}
    \centering
    \includegraphics[width=\textwidth]{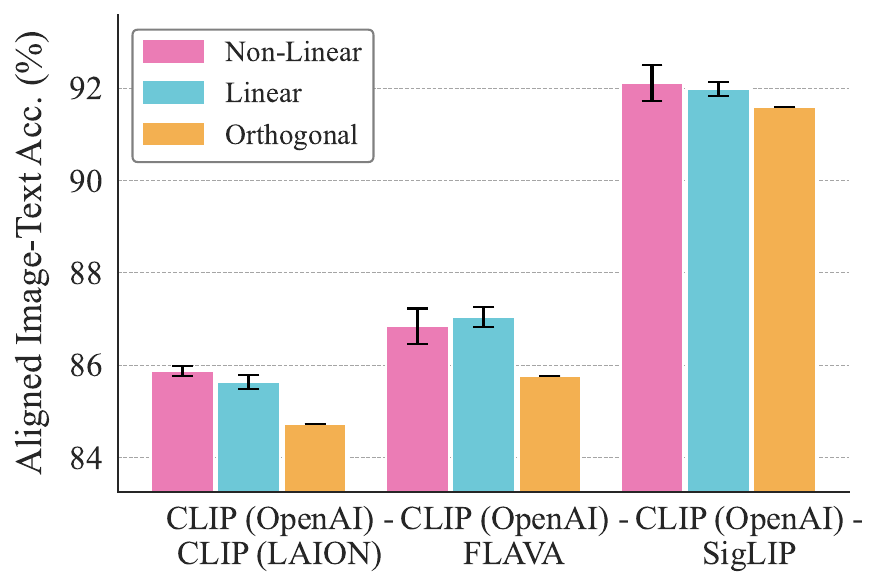}
    \caption{}
    \end{subfigure}
    \begin{subfigure}{0.32\textwidth}
    \centering
    \includegraphics[width=\textwidth]{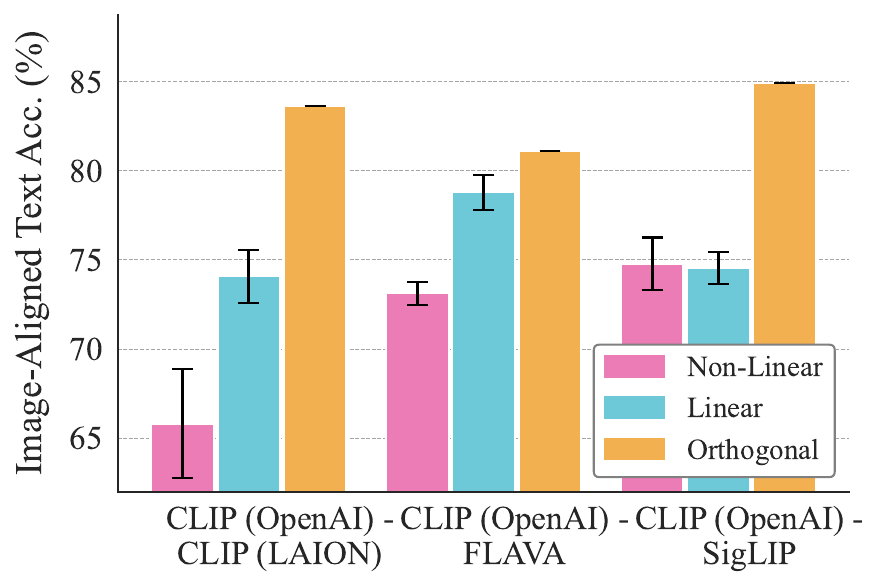}
    \caption{}
    \end{subfigure}
    \begin{subfigure}{0.32\textwidth}
    \centering
    \includegraphics[width=\textwidth]{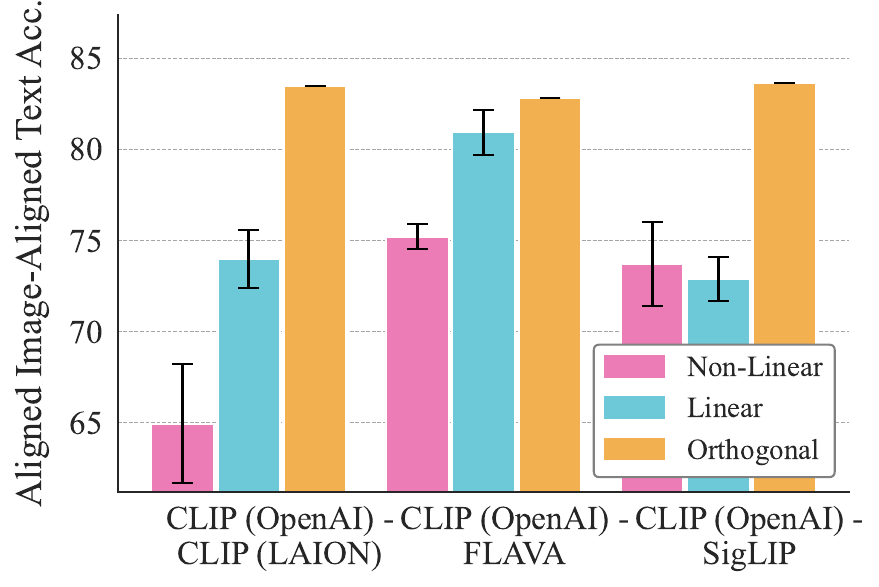}
    \caption{}
    \end{subfigure}

    \caption{\textit{Comparison of alignment strategies (Non-Linear, Linear, and Orthogonal) on Caltech-101 across model pairs.} (a) Image–image class retrieval and (b) text–text class retrieval. (c) Mean image–image cosine similarity. (d) Mean text–text cosine similarity. (e) Image–text retrieval using aligned images from model A and text from model B. (f) Image–text retrieval using images from model B and aligned text from model A. (g) Image–text retrieval using aligned images and aligned text from model A. Unlike linear and non-linear maps, orthogonal transformations preserve the underlying image-text geometry favorable for downstream tasks and transfers across modalities (text in this case).}
    \label{fig:app_compare_caltech101}
\end{figure*}

\begin{figure*}[!htb]
    \centering
    \begin{subfigure}{0.32\textwidth}
    \centering
    \includegraphics[width=\textwidth]{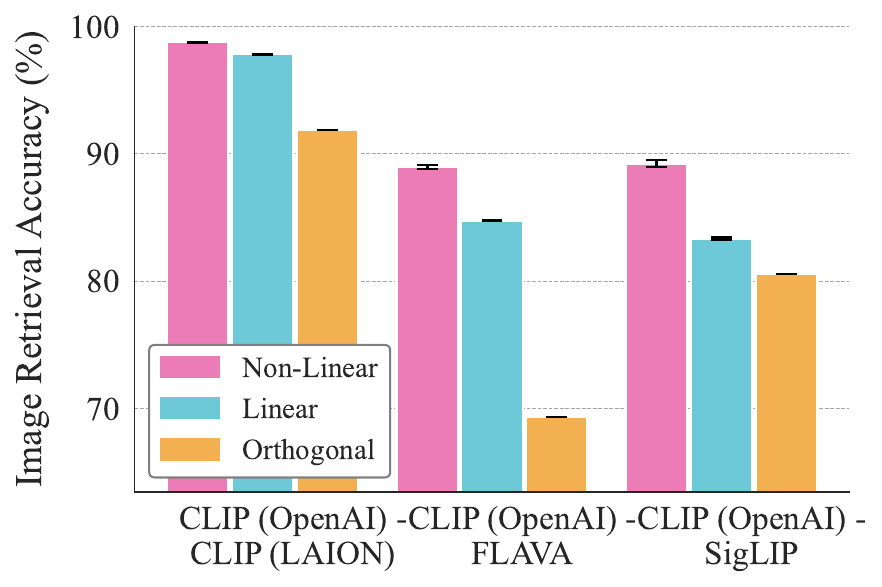}
    \caption{}
    \end{subfigure}
    \begin{subfigure}{0.32\textwidth}
    \centering
    \includegraphics[width=\textwidth]{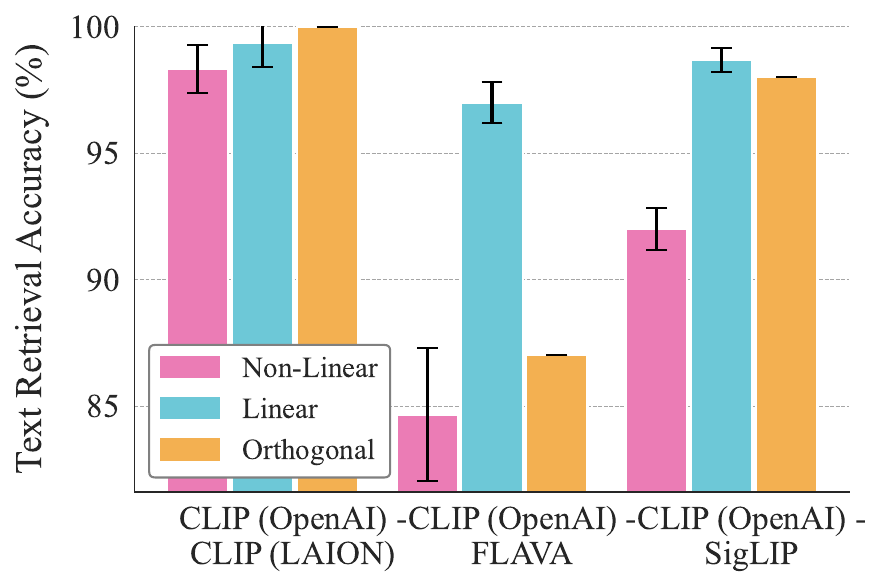}
    \caption{}
    \end{subfigure}
    \begin{subfigure}{0.32\textwidth}
    \centering
    \includegraphics[width=\textwidth]{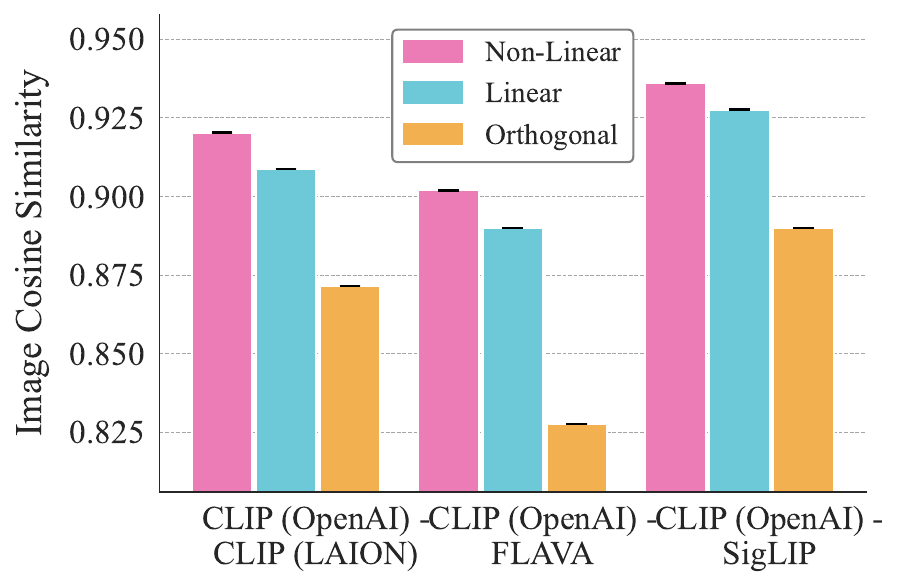}
    \caption{}
    \end{subfigure}
    \begin{subfigure}{0.32\textwidth}
    \centering
    \includegraphics[width=\textwidth]{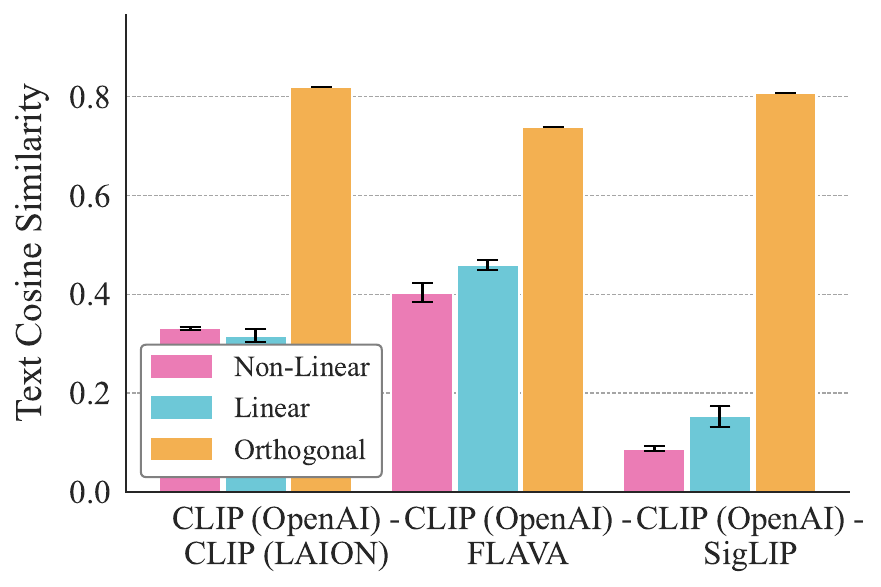}
    \caption{}
    \end{subfigure}
    \begin{subfigure}{0.32\textwidth}
    \centering
    \includegraphics[width=\textwidth]{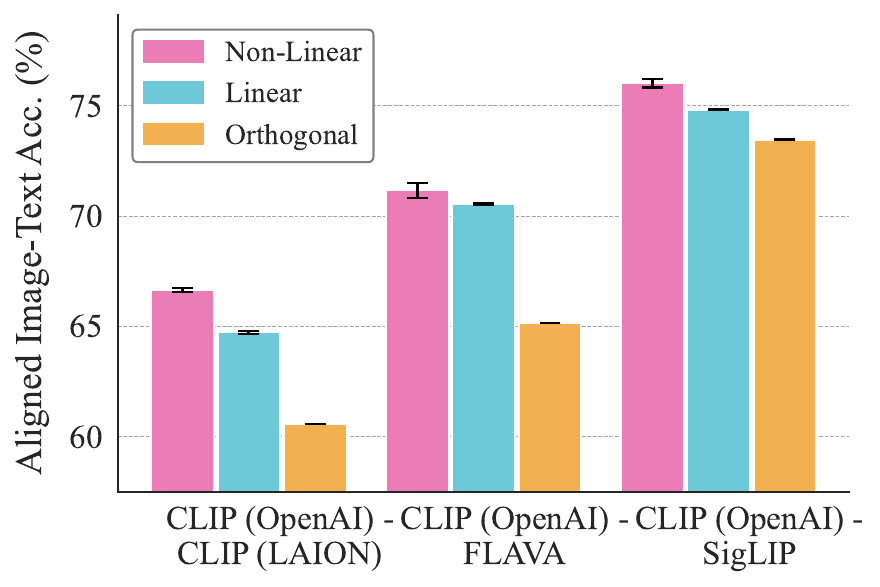}
    \caption{}
    \end{subfigure}
    \begin{subfigure}{0.32\textwidth}
    \centering
    \includegraphics[width=\textwidth]{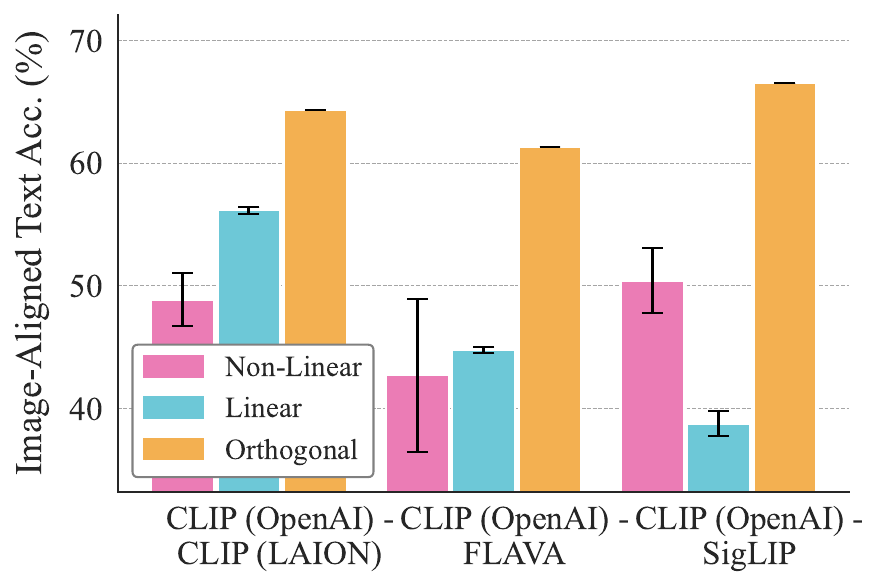}
    \caption{}
    \end{subfigure}
    \begin{subfigure}{0.32\textwidth}
    \centering
    \includegraphics[width=\textwidth]{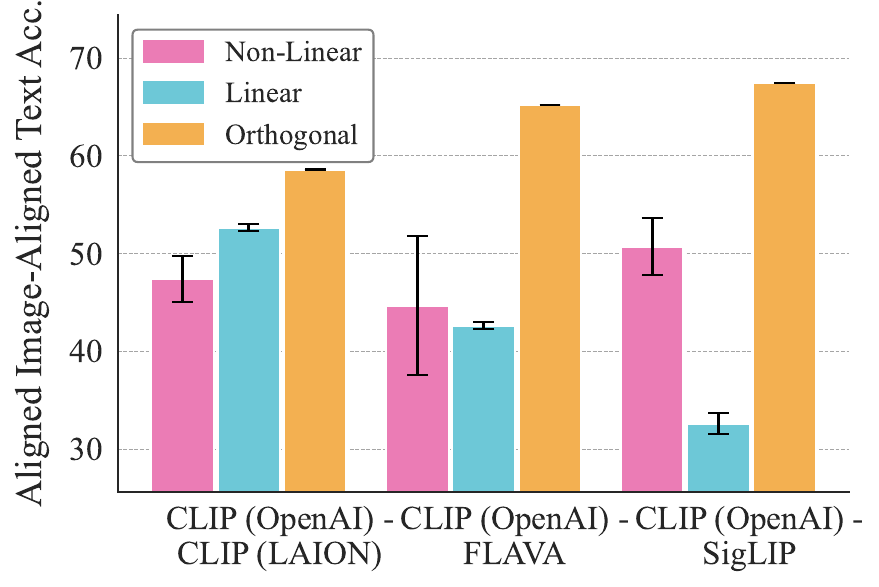}
    \caption{}
    \end{subfigure}

    \caption{\textit{Comparison of alignment strategies (Non-Linear, Linear, and Orthogonal) on CIFAR-100 across model pairs.} (a) Image–image class retrieval and (b) text–text class retrieval. (c) Mean image–image cosine similarity. (d) Mean text–text cosine similarity. (e) Image–text retrieval using aligned images from model A and text from model B. (f) Image–text retrieval using images from model B and aligned text from model A. (g) Image–text retrieval using aligned images and aligned text from model A. Unlike linear and non-linear maps, orthogonal transformations preserve the underlying image-text geometry favorable for downstream tasks and transfers across modalities (text in this case).}
    \label{fig:app_compare_cifar100}
\end{figure*}

\begin{figure*}[!htb]
    \centering
    \begin{subfigure}{0.32\textwidth}
    \centering
    \includegraphics[width=\textwidth]{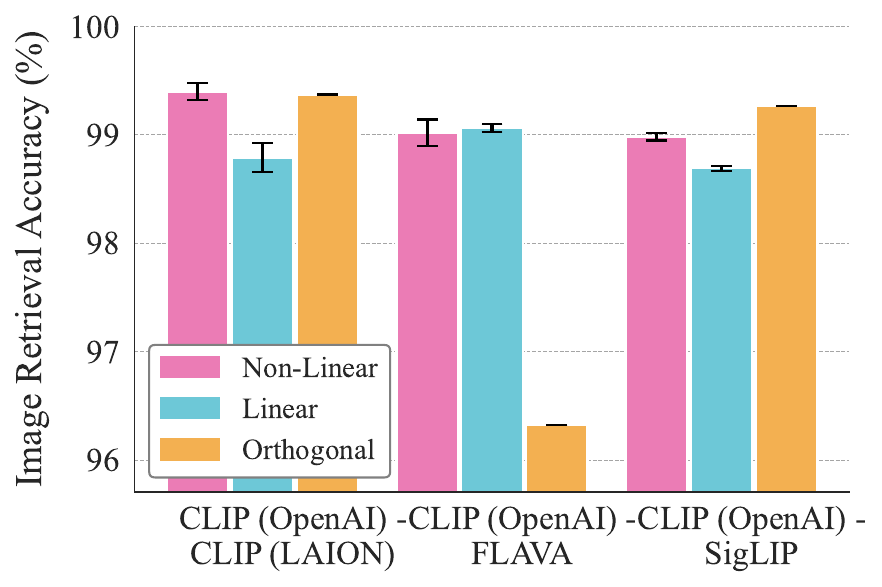}
    \caption{}
    \end{subfigure}
    \begin{subfigure}{0.32\textwidth}
    \centering
    \includegraphics[width=\textwidth]{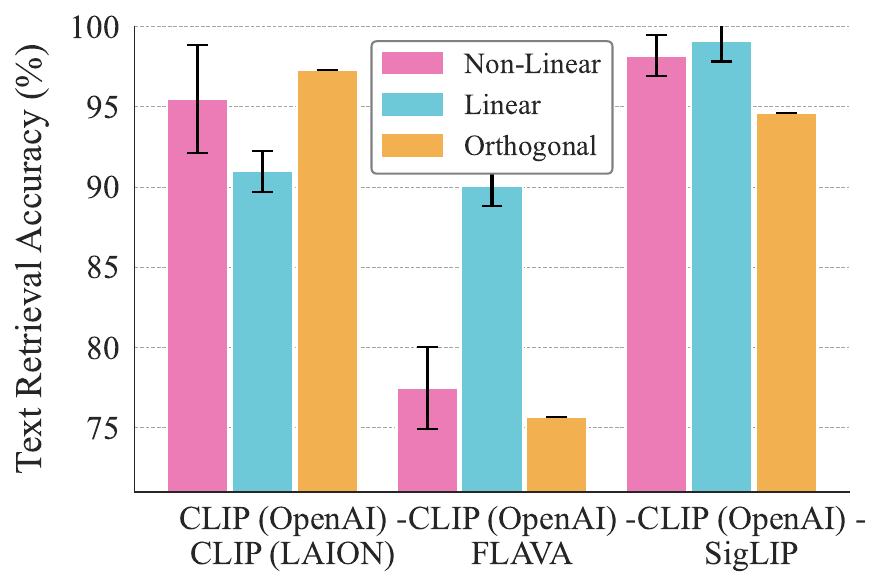}
    \caption{}
    \end{subfigure}
    \begin{subfigure}{0.32\textwidth}
    \centering
    \includegraphics[width=\textwidth]{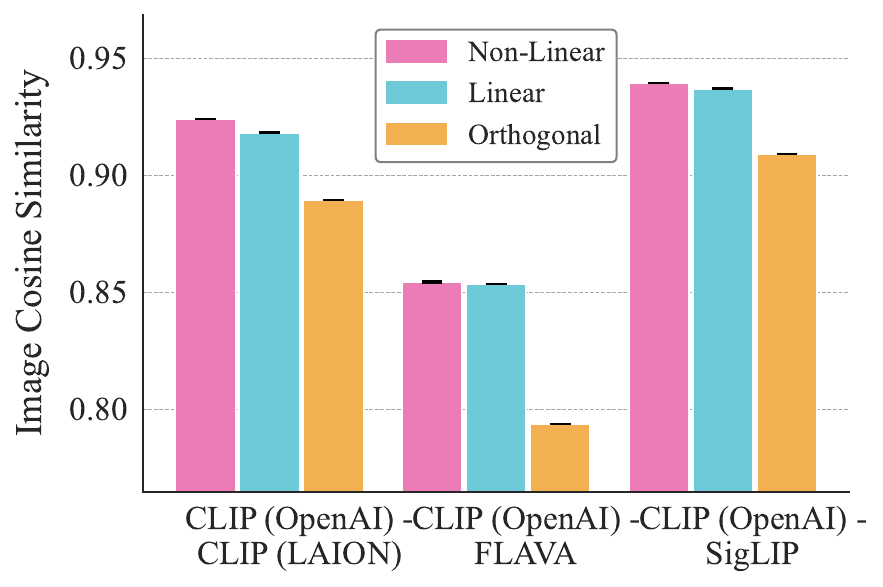}
    \caption{}
    \end{subfigure}
    \begin{subfigure}{0.32\textwidth}
    \centering
    \includegraphics[width=\textwidth]{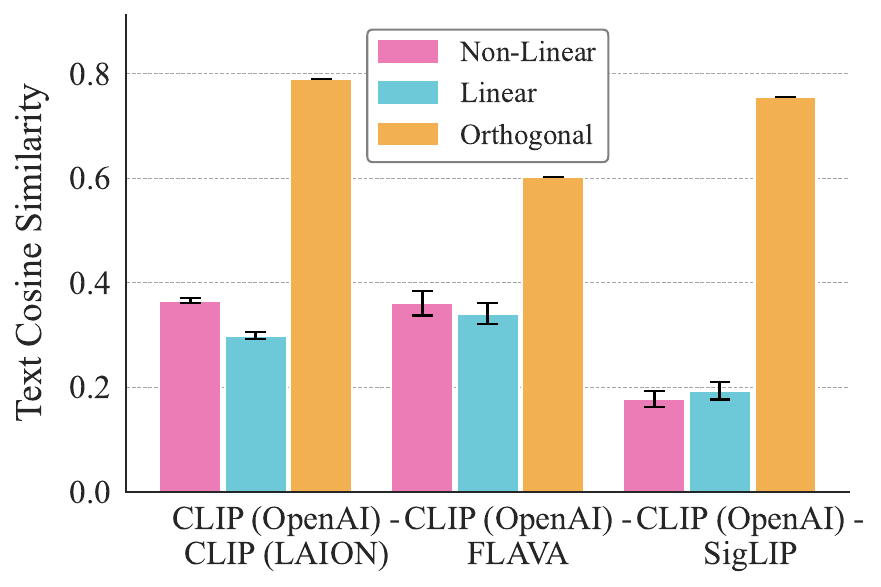}
    \caption{}
    \end{subfigure}
    \begin{subfigure}{0.32\textwidth}
    \centering
    \includegraphics[width=\textwidth]{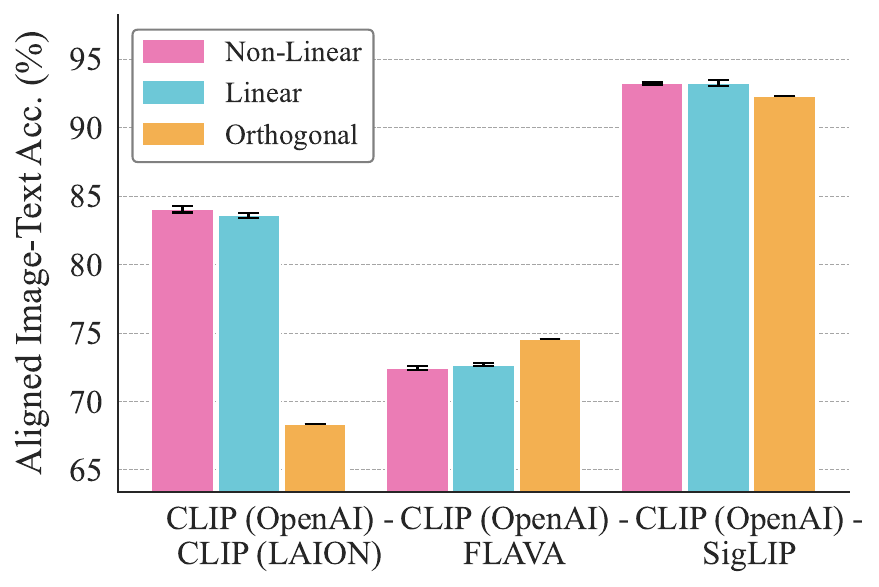}
    \caption{}
    \end{subfigure}
    \begin{subfigure}{0.32\textwidth}
    \centering
    \includegraphics[width=\textwidth]{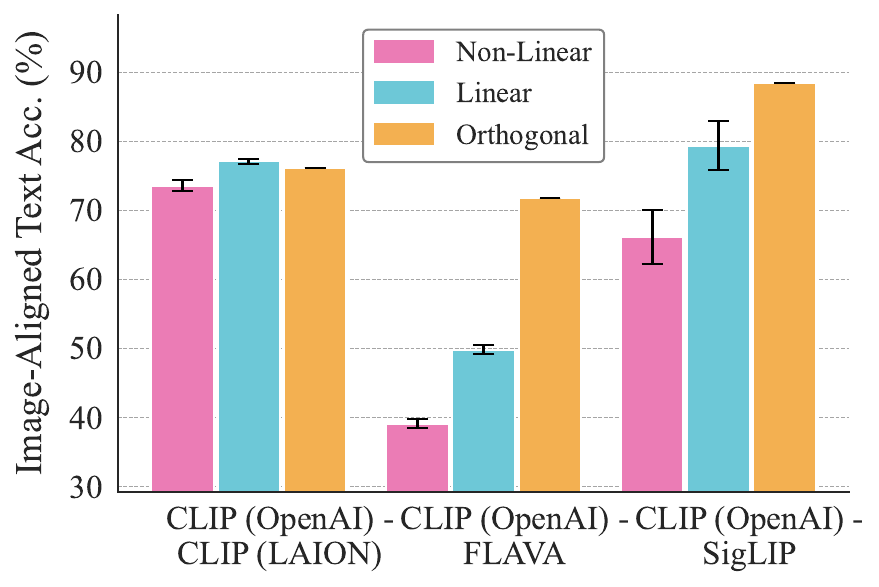}
    \caption{}
    \end{subfigure}
    \begin{subfigure}{0.32\textwidth}
    \centering
    \includegraphics[width=\textwidth]{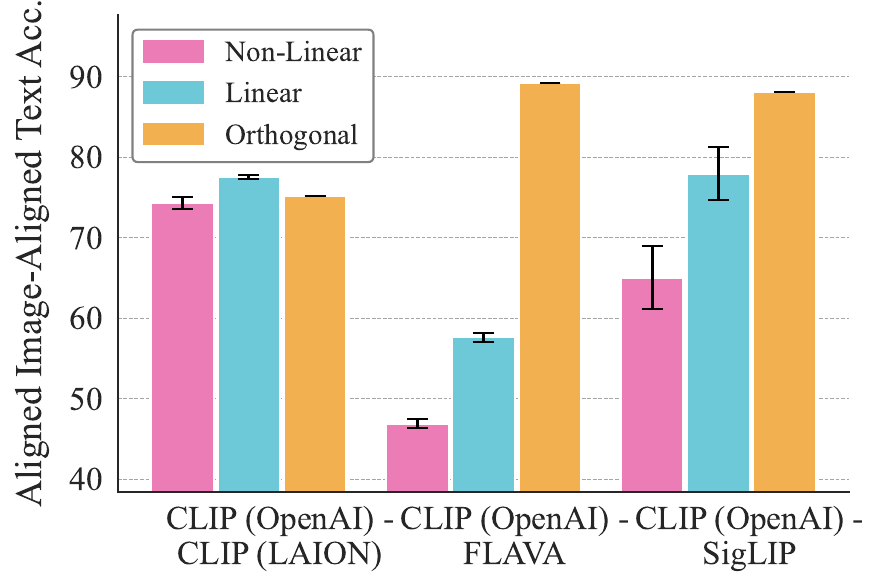}
    \caption{}
    \end{subfigure}

    \caption{\textit{Comparison of alignment strategies (Non-Linear, Linear, and Orthogonal) on Oxford Pets across model pairs (CLIP OpenAI vs. LAION, FLAVA, and SigLIP).} (a) Image–image class retrieval and (b) text–text class retrieval. (c) Mean image–image cosine similarity. (d) Mean text–text cosine similarity. (e) Image–text retrieval using aligned images from model A and text from model B. (f) Image–text retrieval using images from model B and aligned text from model A. (g) Image–text retrieval using aligned images and aligned text from model A. Unlike linear and non-linear maps, orthogonal transformations preserve the underlying image-text geometry favorable for downstream tasks and transfers across modalities (text in this case).}
    \label{fig:app_compare_oxford}
\end{figure*}

\FloatBarrier
\subsection{Orthogonal Alignment With and Without Centering}\label{sec:app_procrustes_variants}
All results thus far use our \emph{centered} orthogonal Procrustes map (i.e., we fit $Q$ on mean-centered embeddings and re-center at deployment). Here, we ablate this centering step by comparing it to a pure orthogonal map applied directly to raw embeddings. Specifically, we evaluate: (i) \emph{Orthogonal (no centering)}: $z \mapsto Qz$, a pure change of coordinates that preserves inner products exactly; and (ii) \emph{Orthogonal (with centering)}: $z \mapsto Q(z-\mu_A^{(\cdot)})+\mu_B^{(\cdot)}$, where $\mu_A^{(\cdot)}$ and $\mu_B^{(\cdot)}$ are modality-specific means of the source and target models (image or text, as appropriate). Throughout, we use ``$Q$'' to denote the centered variant, consistent with the rest of the paper. 

\par Results on Caltech-101, CIFAR-100 and Oxford Pets are presented in~\Cref{fig:app_procrustes_caltech101,fig:app_procrustes_cifar100,fig:app_procrustes_oxford} respectively. 
Across all figures, the orthogonal map with mean centering (denoted by $Q$) improves downstream metrics such as aligned-image-to-text and image-to-aligned-text retrieval, suggesting that centering helps correct first-order inter-model offsets in these comparisons. In contrast, aligned-image-to-aligned-text accuracy is typically higher without centering, since orthogonal w/o centering exactly preserves model~A’s geometry. That said, these gains are modest, and the overall retrieval performance remains largely similar with or without centering. Finally, pointwise image-image and text-text cosine similarities increase significantly under centering, as subtracting and re-adding the global centroid shift removes a global centroid shift that otherwise depresses cosine agreement.

\begin{figure*}[!htb]
    \centering
    \begin{subfigure}{0.32\textwidth}
    \centering
    \includegraphics[width=\textwidth]{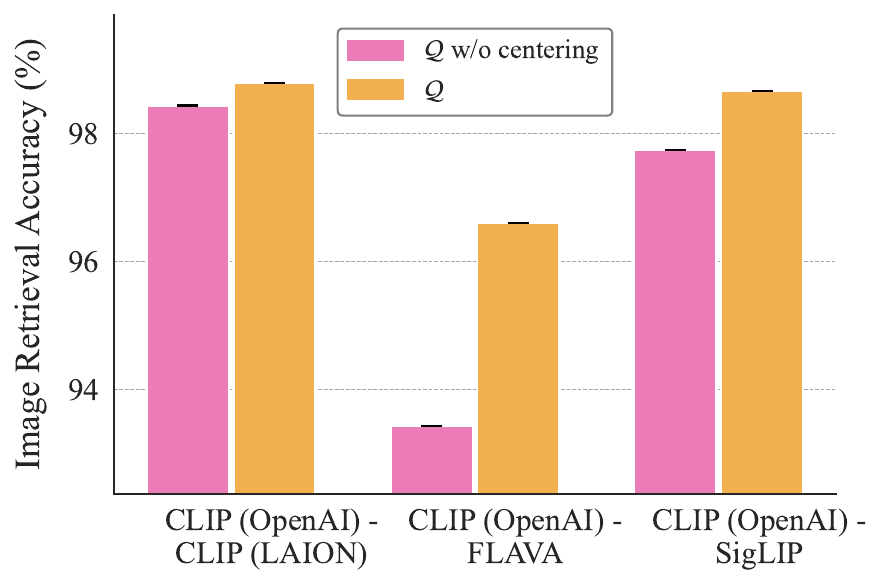}
    \caption{}
    \end{subfigure}
    \begin{subfigure}{0.32\textwidth}
    \centering
    \includegraphics[width=\textwidth]{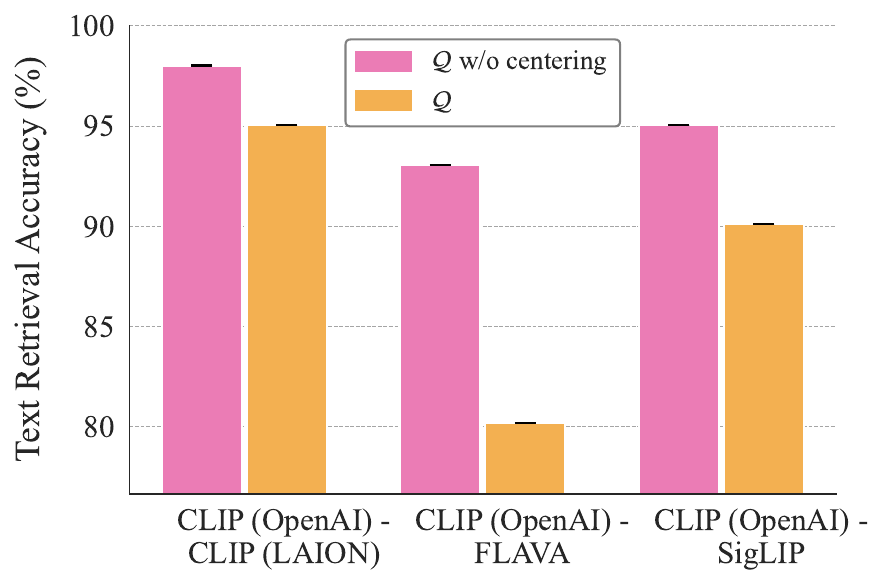}
    \caption{}
    \end{subfigure}
    \begin{subfigure}{0.32\textwidth}
    \centering
    \includegraphics[width=\textwidth]{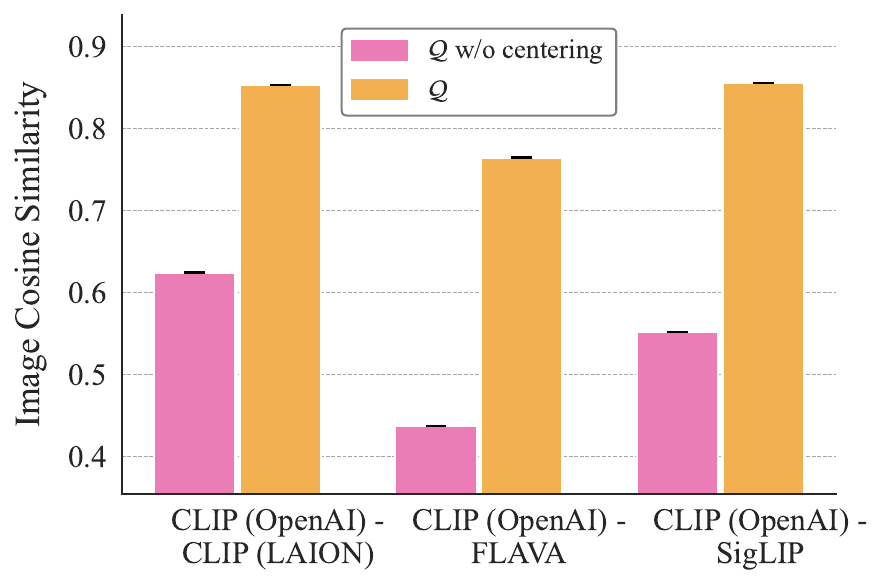}
    \caption{}
    \end{subfigure}
    \begin{subfigure}{0.32\textwidth}
    \centering
    \includegraphics[width=\textwidth]{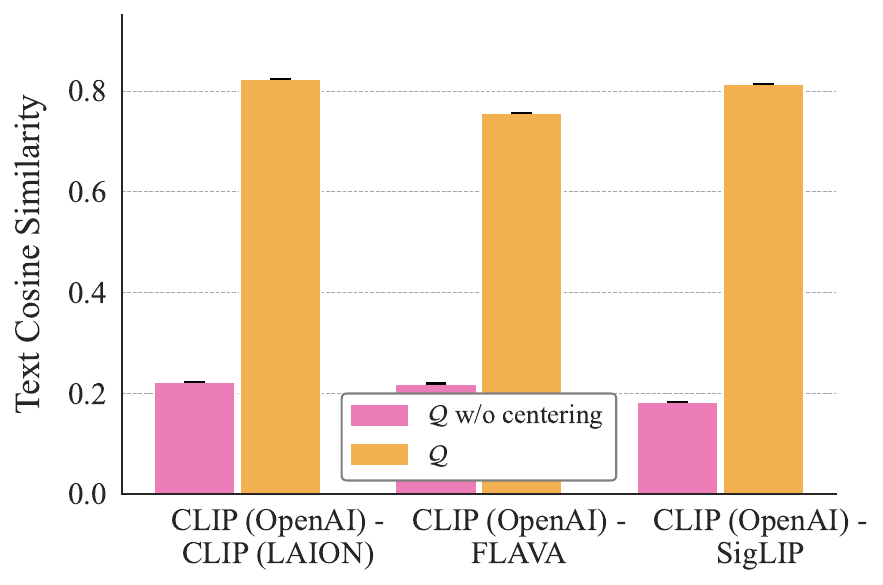}
    \caption{}
    \end{subfigure}
    \vspace{1mm}
    \begin{subfigure}{0.32\textwidth}
    \centering
    \includegraphics[width=\textwidth]{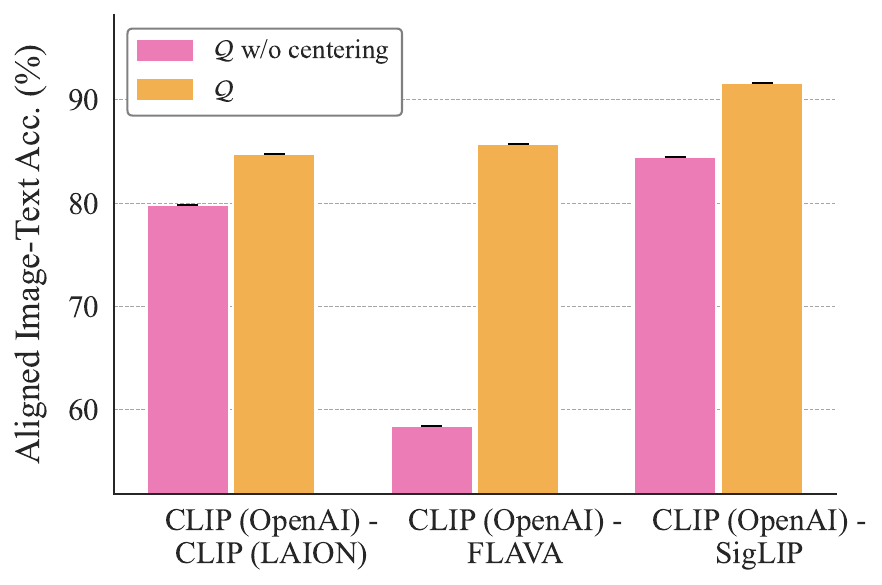}
    \caption{}
    \end{subfigure}
    \begin{subfigure}{0.32\textwidth}
    \centering
    \includegraphics[width=\textwidth]{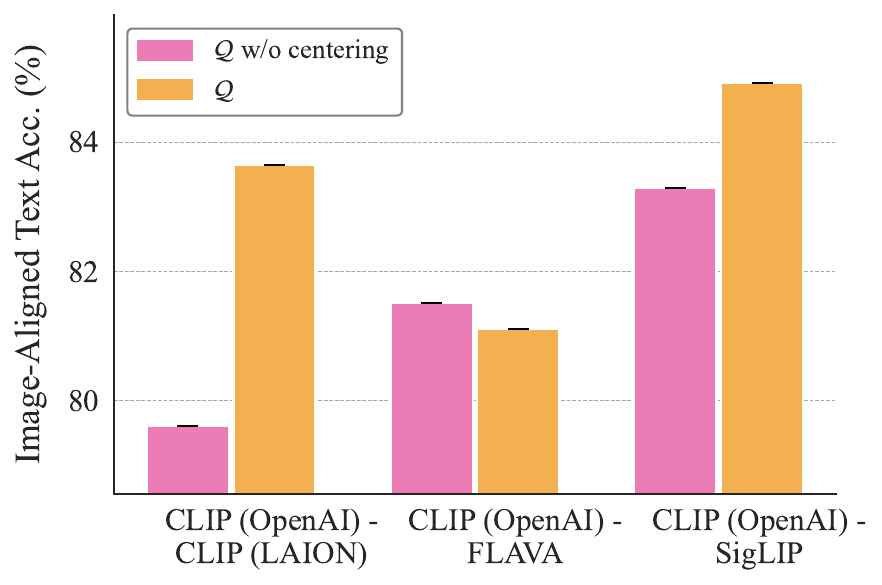}
    \caption{}
    \end{subfigure}
    \begin{subfigure}{0.32\textwidth}
    \centering
    \includegraphics[width=\textwidth]{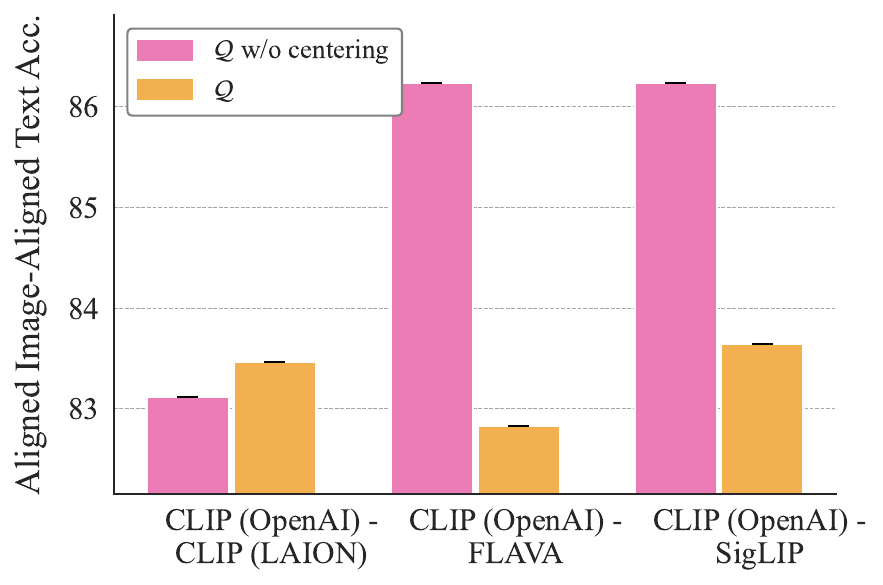}
    \caption{}
    \end{subfigure}

    \caption{\textit{Comparison between an orthogonal map with and without mean adjustment across model pairs on Caltech-101.} (a) Image–image class retrieval (b) text–text class retrieval (c) Mean image–image cosine similarity (d) Mean text–text cosine similarity. (e) Image–text retrieval using aligned images from model A and text from model B. (f) Image–text retrieval using images from model B and aligned text from model A. (g) Image–text retrieval using aligned images and aligned text from model A. Mean adjustment boosts pointwise image-image (c) and text-text (d) cosine and downstream performance (e,f), but can slightly underperform on aligned-image-to-aligned-text and text-text retrieval (b,g).}
    \label{fig:app_procrustes_caltech101}
\end{figure*}

\begin{figure*}[!htb]
    \centering
    \begin{subfigure}{0.32\textwidth}
    \centering
    \includegraphics[width=\textwidth]{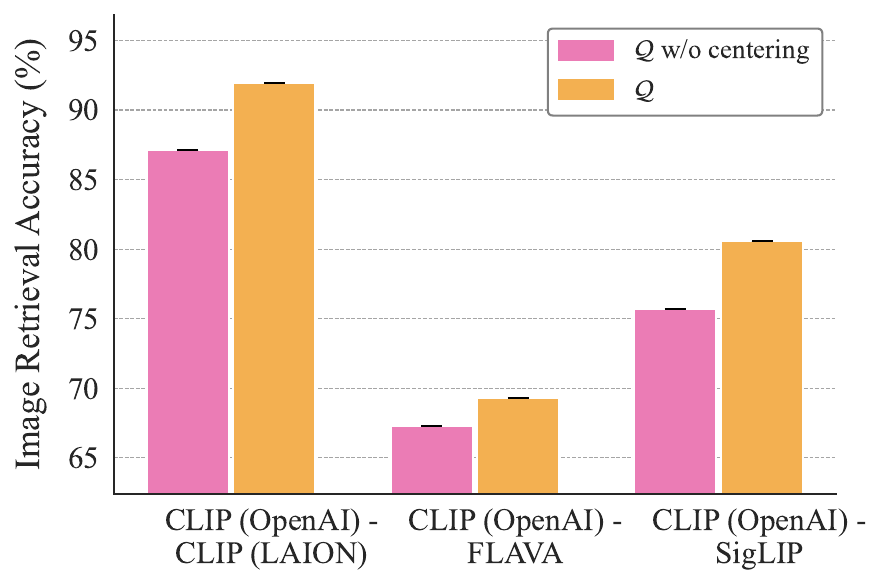}
    \caption{}
    \end{subfigure}
    \begin{subfigure}{0.32\textwidth}
    \centering
    \includegraphics[width=\textwidth]{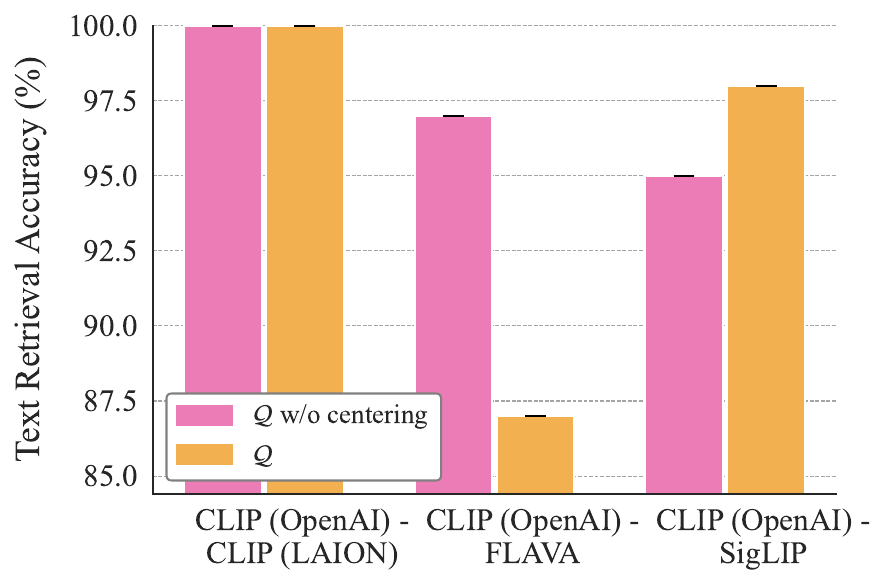}
    \caption{}
    \end{subfigure}
    \begin{subfigure}{0.32\textwidth}
    \centering
    \includegraphics[width=\textwidth]{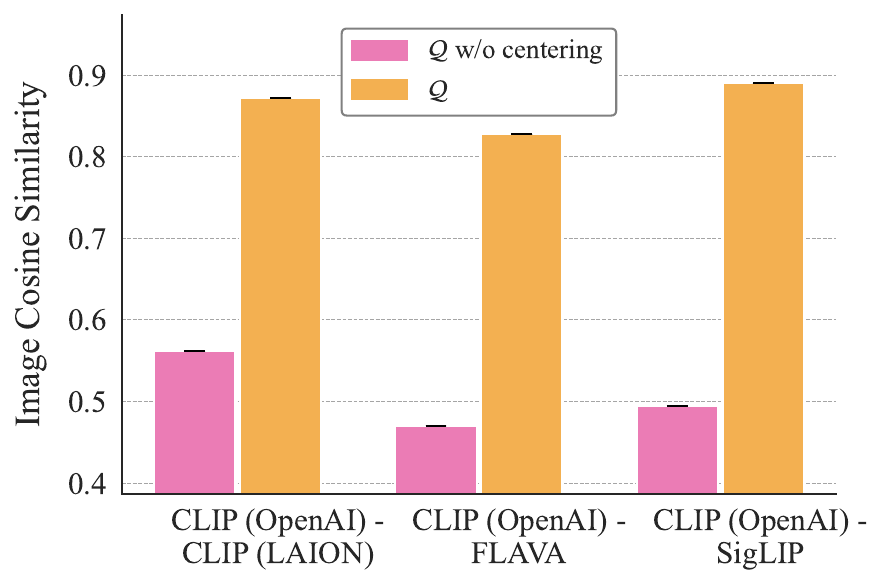}
    \caption{}
    \end{subfigure}
    \begin{subfigure}{0.32\textwidth}
    \centering
    \includegraphics[width=\textwidth]{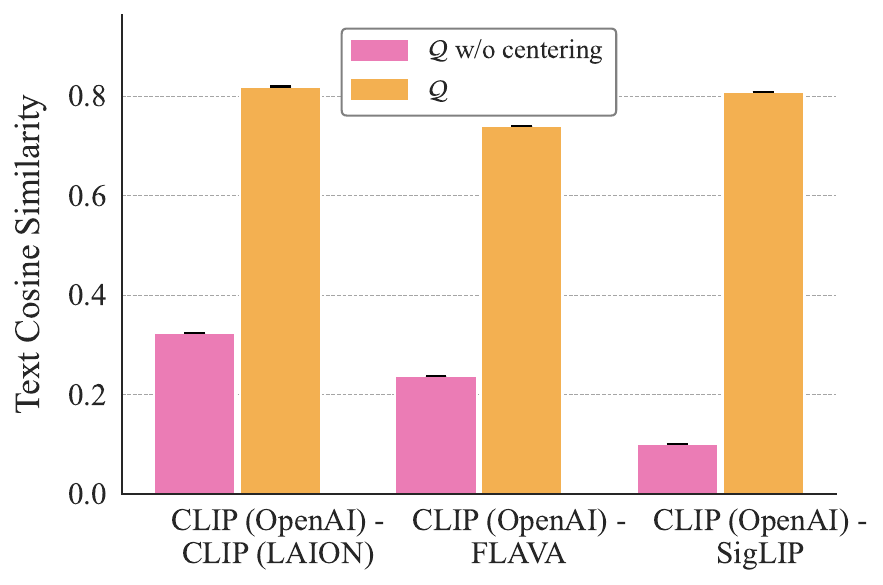}
    \caption{}
    \end{subfigure}
    \vspace{1mm}
    \begin{subfigure}{0.32\textwidth}
    \centering
    \includegraphics[width=\textwidth]{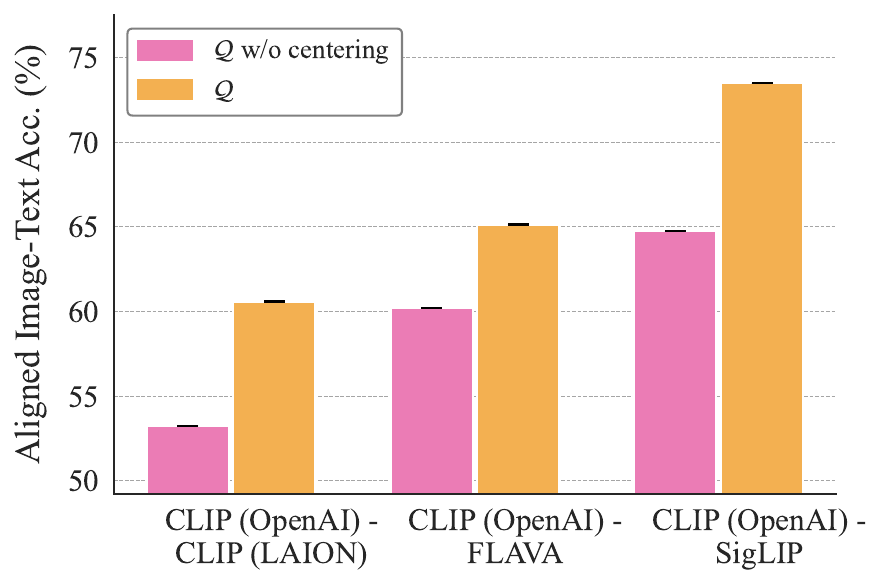}
    \caption{}
    \end{subfigure}
    \begin{subfigure}{0.32\textwidth}
    \centering
    \includegraphics[width=\textwidth]{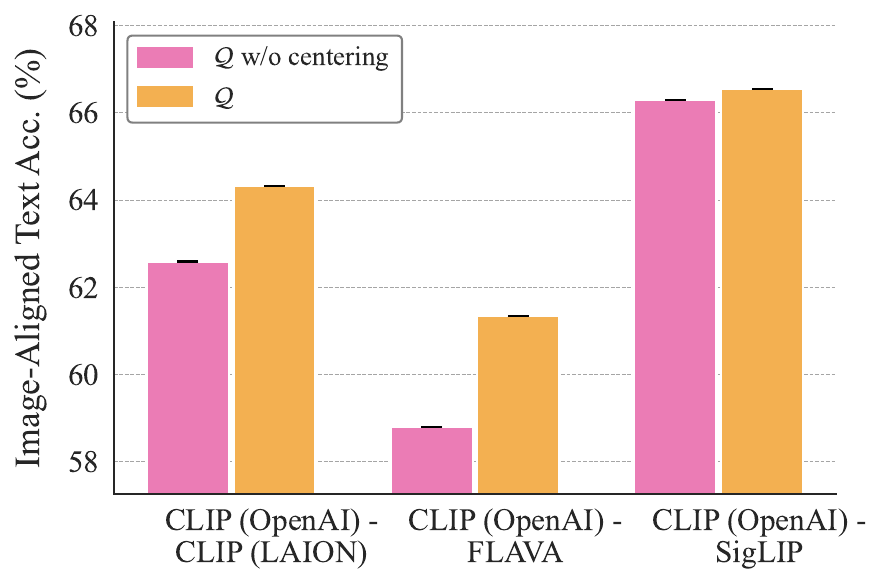}
    \caption{}
    \end{subfigure}
    \begin{subfigure}{0.32\textwidth}
    \centering
    \includegraphics[width=\textwidth]{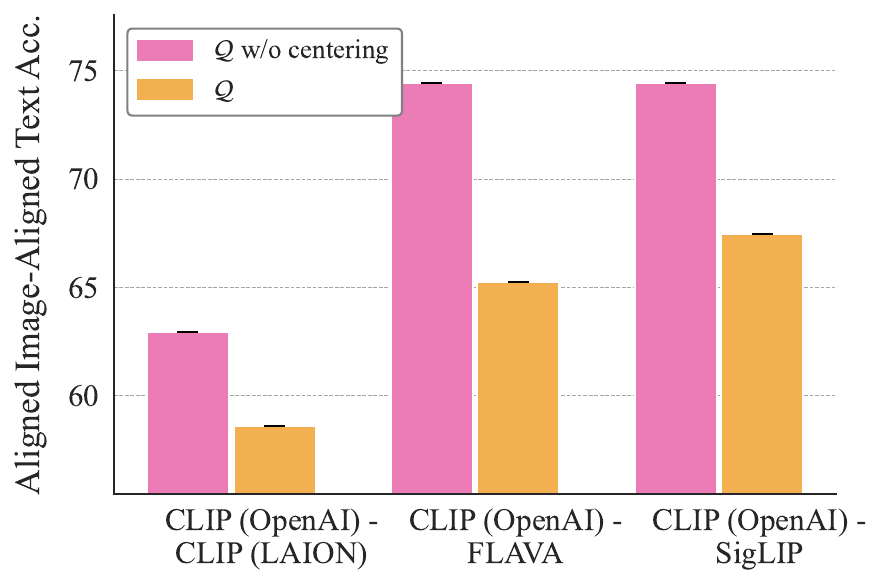}
    \caption{}
    \end{subfigure}

    \caption{\textit{Comparison between an orthogonal map with and without mean adjustment across model pairs on CIFAR-100.} (a) Image–image class retrieval (b) text–text class retrieval (c) Mean image–image cosine similarity (d) Mean text–text cosine similarity. (e) Image–text retrieval using aligned images from model A and text from model B. (f) Image–text retrieval using images from model B and aligned text from model A. (g) Image–text retrieval using aligned images and aligned text from model A. Mean adjustment boosts pointwise image-image (c) and text-text (d) cosine and downstream performance (e,f), but can slightly underperform on aligned-image-to-aligned-text and text-text retrieval (b,g).}
    \label{fig:app_procrustes_cifar100}
\end{figure*}

\begin{figure*}[!htb]
    \centering
    \begin{subfigure}{0.32\textwidth}
    \centering
    \includegraphics[width=\textwidth]{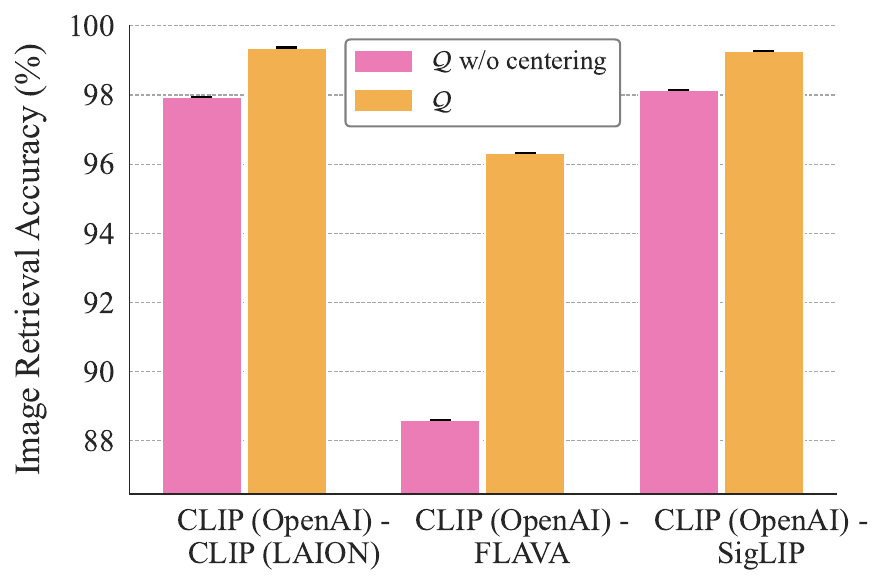}
    \caption{}
    \end{subfigure}
    \begin{subfigure}{0.32\textwidth}
    \centering
    \includegraphics[width=\textwidth]{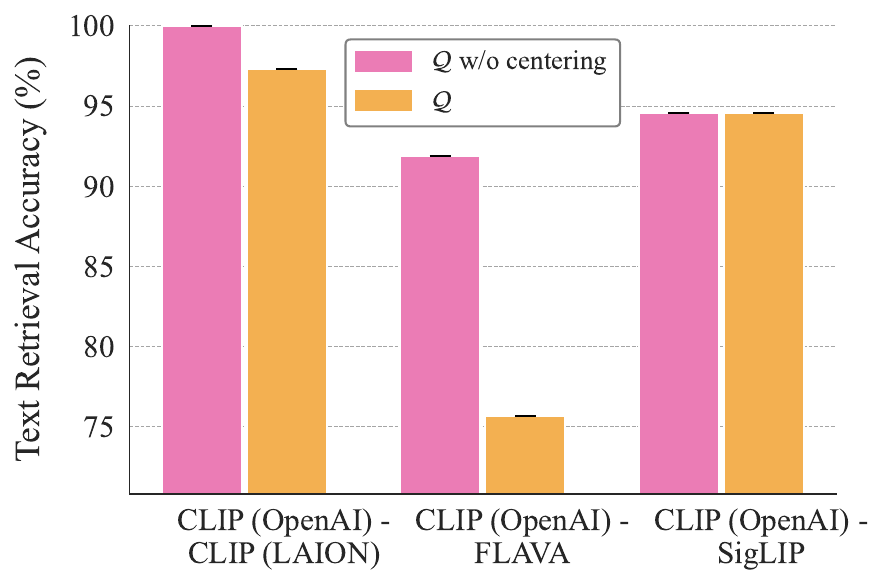}
    \caption{}
    \end{subfigure}
    \begin{subfigure}{0.32\textwidth}
    \centering
    \includegraphics[width=\textwidth]{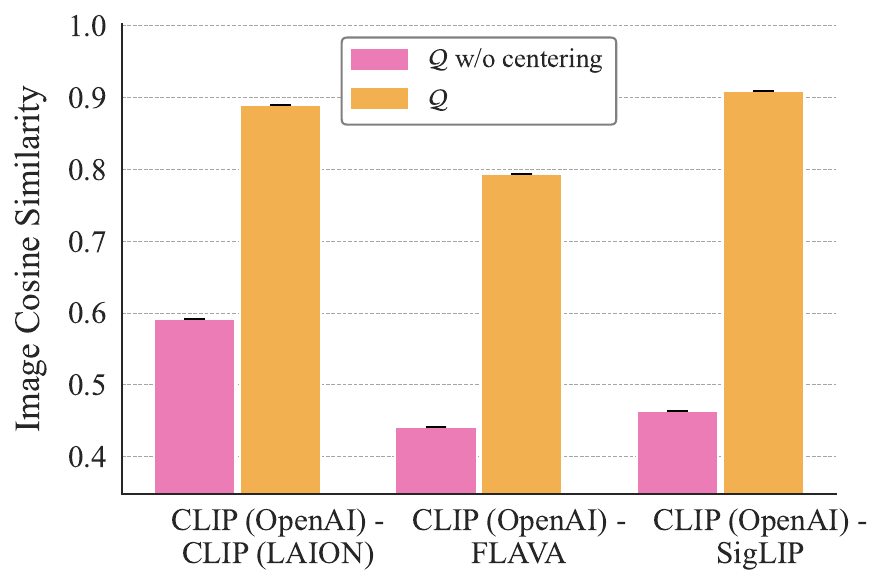}
    \caption{}
    \end{subfigure}
    \begin{subfigure}{0.32\textwidth}
    \centering
    \includegraphics[width=\textwidth]{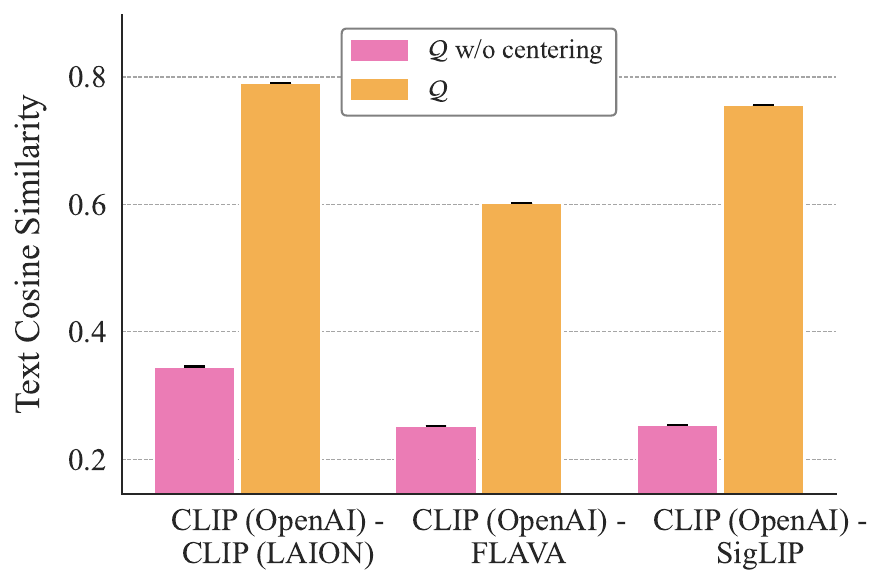}
    \caption{}
    \end{subfigure}
    \vspace{1mm}
    \begin{subfigure}{0.32\textwidth}
    \centering
    \includegraphics[width=\textwidth]{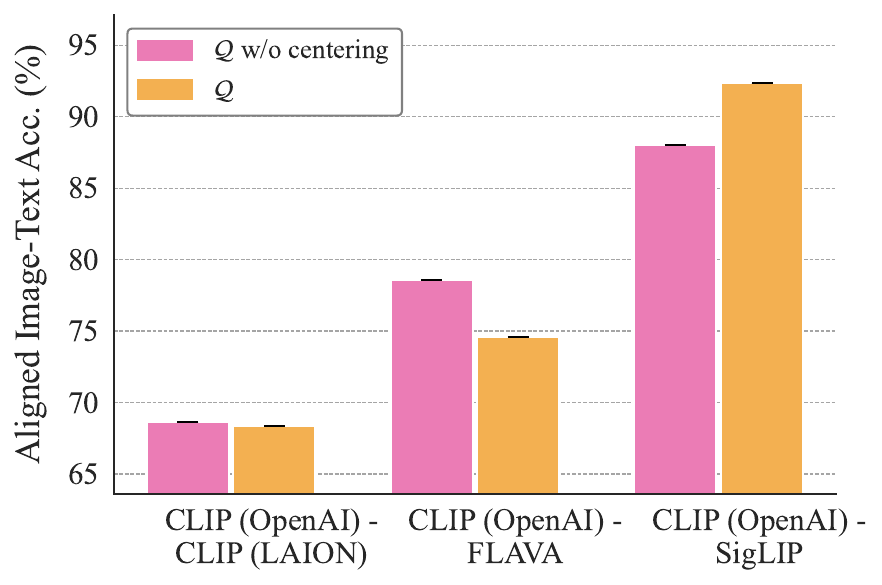}
    \caption{}
    \end{subfigure}
    \begin{subfigure}{0.32\textwidth}
    \centering
    \includegraphics[width=\textwidth]{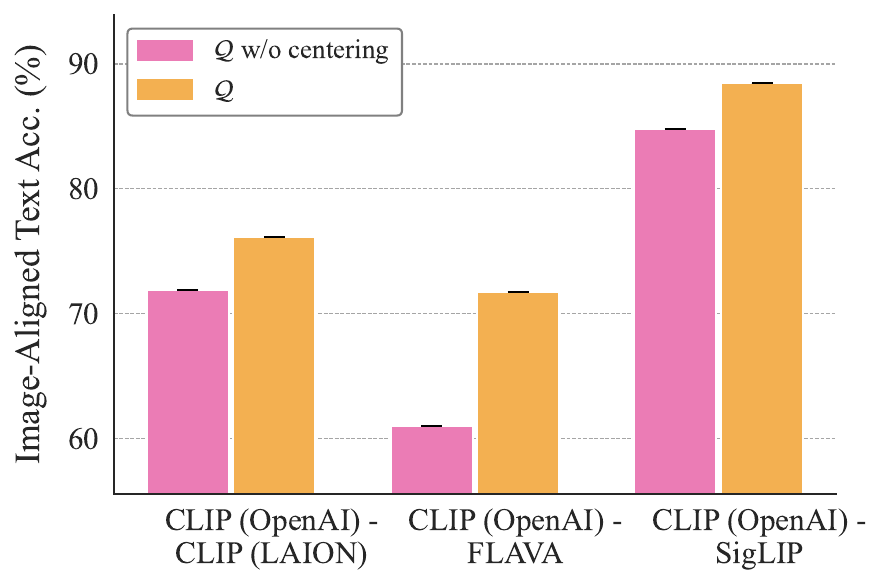}
    \caption{}
    \end{subfigure}
    \begin{subfigure}{0.32\textwidth}
    \centering
    \includegraphics[width=\textwidth]{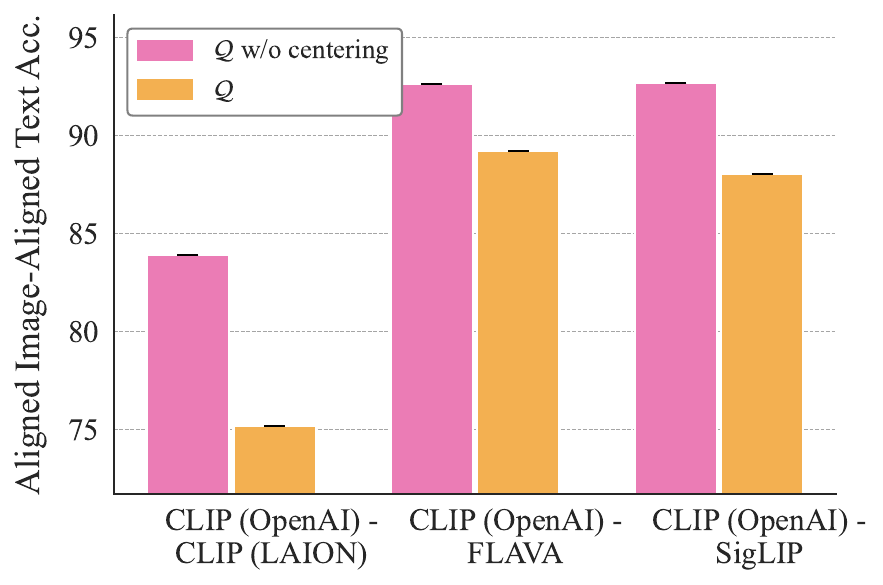}
    \caption{}
    \end{subfigure}

    \caption{\textit{Comparison between an orthogonal map with and without mean adjustment across model pairs on Oxford Pets.} (a) Image–image class retrieval (b) text–text class retrieval (c) Mean image–image cosine similarity (d) Mean text–text cosine similarity. (e) Image–text retrieval using aligned images from model A and text from model B. (f) Image–text retrieval using images from model B and aligned text from model A. (g) Image–text retrieval using aligned images and aligned text from model A. Mean adjustment boosts pointwise image-image (c) and text-text (d) cosine and downstream performance (e,f), but can slightly underperform on aligned-image-to-aligned-text and text-text retrieval (b,g).}
    \label{fig:app_procrustes_oxford}
\end{figure*}

\FloatBarrier

\subsection{Isolating Architecture Effects Under Fixed Training Data}\label{sec:app_openi_openai}
All previous experiments compare model pairs that differ in both training distribution and design choices. Here we isolate the role of the data by aligning two \textsc{CLIP} models trained on (probably) the same distribution (dataset used by OpenAI~\citep{radford2021clip}) but with different architectures (ViT-B/32 vs. ViT-B/16). Results on Caltech-101, CIFAR-100 and Oxford Pets are reported in~\Cref{fig:app_openai_caltech101,fig:app_openai_cifar100,fig:app_openai_oxford}. As observed across these figures, in this controlled setting, transfer is consistently stronger, compared to when using models trained on different distributions. This aligns with our theoretical results showing that when two contrastive models optimize the same underlying $P_{XY}$, their induced multimodal kernels agree up to an additive constant at the population optimum~\Cref{cor:unique_kernels}, tightening the conditions under which a single orthogonal map aligns both modalities.

\begin{figure*}[!htb]
    \centering
    \begin{minipage}{0.65\linewidth}
    \begin{subfigure}{0.32\textwidth}
    \centering
    \includegraphics[width=\textwidth]{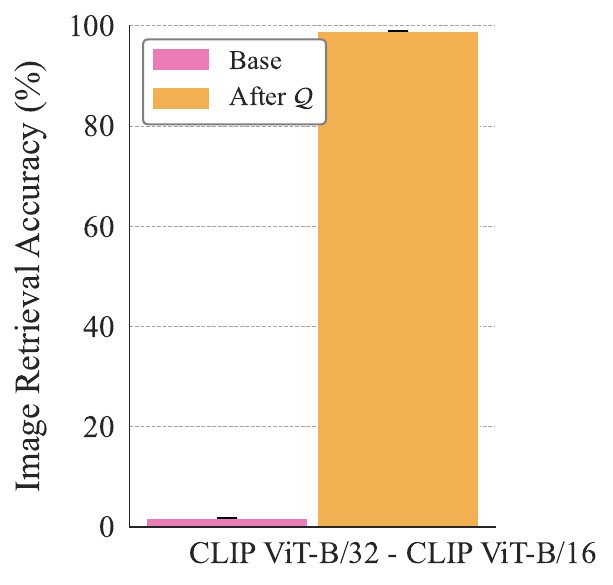}
    \caption{}
    \end{subfigure}
    \begin{subfigure}{0.32\textwidth}
    \centering
    \includegraphics[width=\textwidth]{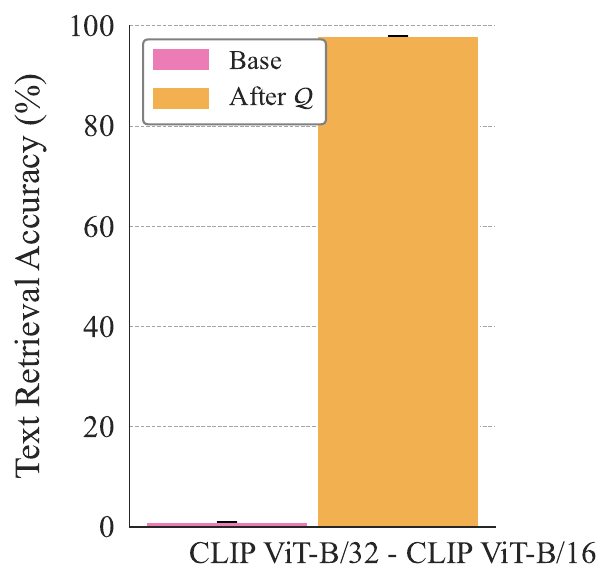}
    \caption{}
    \end{subfigure}
    \begin{subfigure}{0.32\textwidth}
    \centering
    \includegraphics[width=\textwidth]{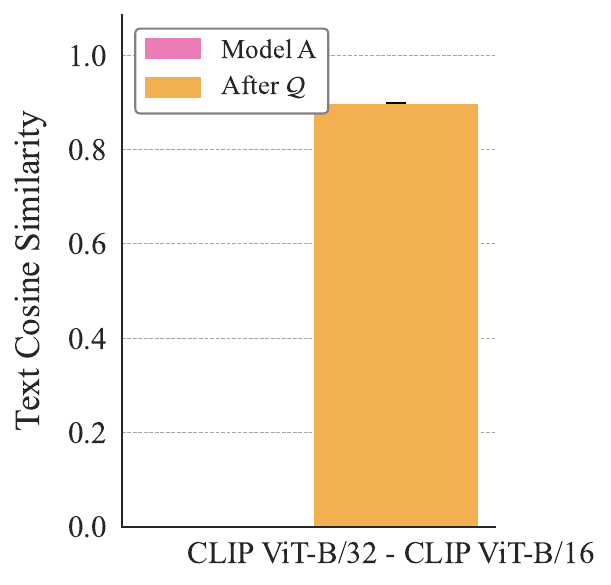}
    \caption{}
    \end{subfigure}

    \vspace{1mm}
    \begin{subfigure}{0.32\textwidth}
    \centering
    \includegraphics[width=\textwidth]{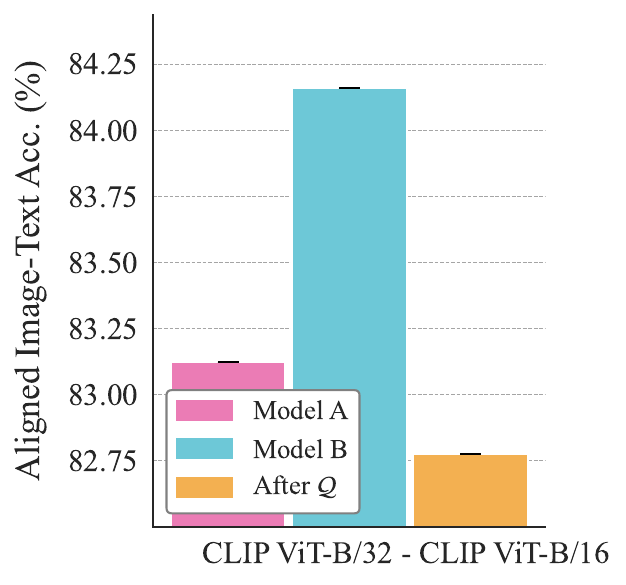}
    \caption{}
    \end{subfigure}
    \begin{subfigure}{0.32\textwidth}
    \centering
    \includegraphics[width=\textwidth]{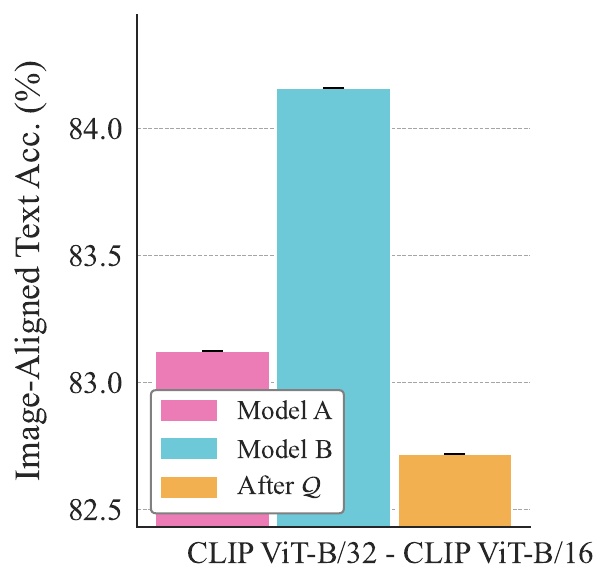}
    \caption{}
    \end{subfigure}
    \begin{subfigure}{0.32\textwidth}
    \centering
    \includegraphics[width=\textwidth]{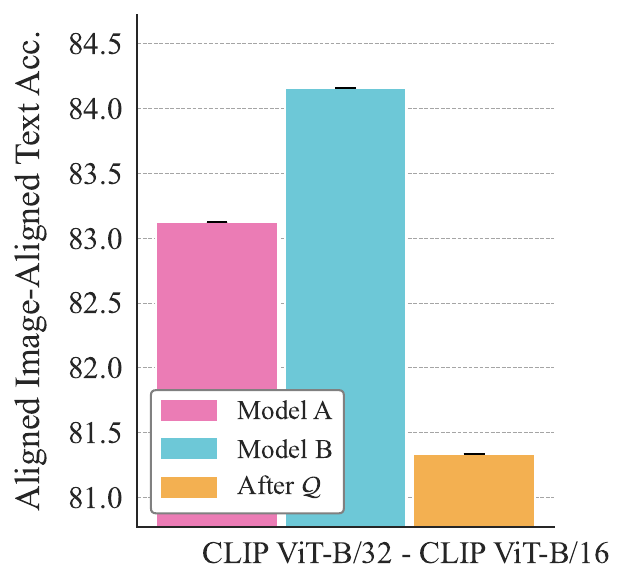}
    \caption{}
    \end{subfigure}
    \end{minipage}
    \caption{\textit{Cross-model alignment between identical training datasets with differing architectures (Caltech-101; CLIP ViT-B/32 to ViT-B/16, OpenAI) before and after fitting a single orthogonal map $\mathcal{Q}$.} (a) Image–image class retrieval and (b) text–text class retrieval. (c) Mean text–text cosine similarity. (d) Image–text retrieval using aligned images from model A and text from model B. (e) Image–text retrieval using images from model B and aligned text from model A. (f) Image–text retrieval using aligned images and aligned text from model A. When both models are trained on the same data distribution and differ only in architecture, applying $\mathcal Q$ yields substantially stronger modality transfer than in settings where the training distributions differ.}
    \label{fig:app_openai_caltech101}
\end{figure*}

\begin{figure*}[!htb]
    \centering
    \begin{minipage}{0.65\linewidth}
    \begin{subfigure}{0.32\textwidth}
    \centering
    \includegraphics[width=\textwidth]{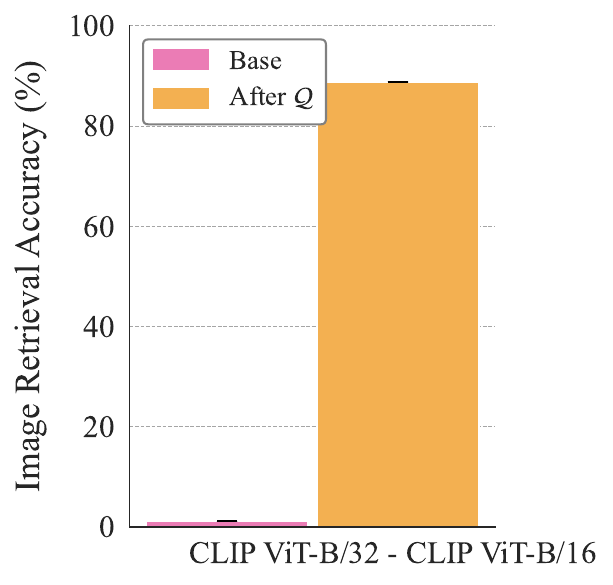}
    \caption{}
    \end{subfigure}
    \begin{subfigure}{0.32\textwidth}
    \centering
    \includegraphics[width=\textwidth]{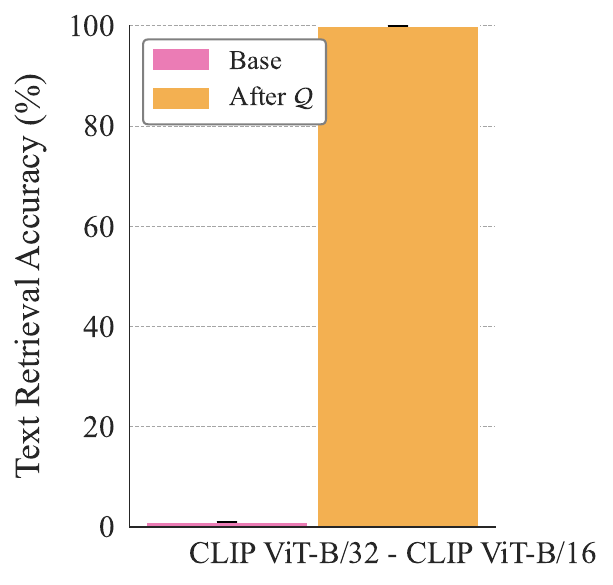}
    \caption{}
    \end{subfigure}
    \begin{subfigure}{0.32\textwidth}
    \centering
    \includegraphics[width=\textwidth]{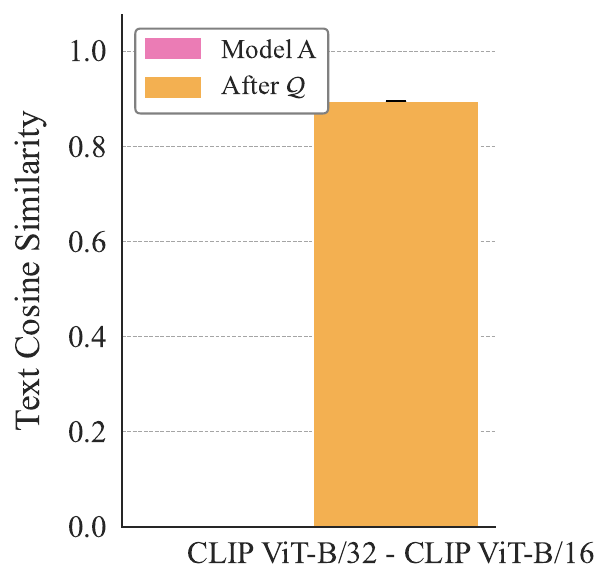}
    \caption{}
    \end{subfigure}

    \vspace{1mm}
    \begin{subfigure}{0.32\textwidth}
    \centering
    \includegraphics[width=\textwidth]{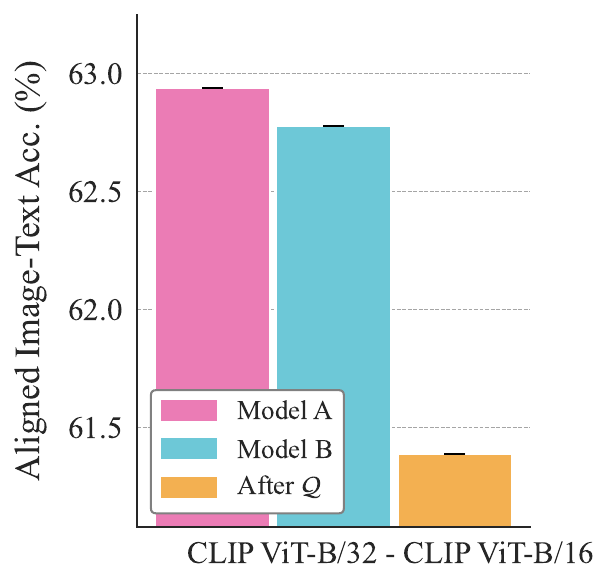}
    \caption{}
    \end{subfigure}
    \begin{subfigure}{0.32\textwidth}
    \centering
    \includegraphics[width=\textwidth]{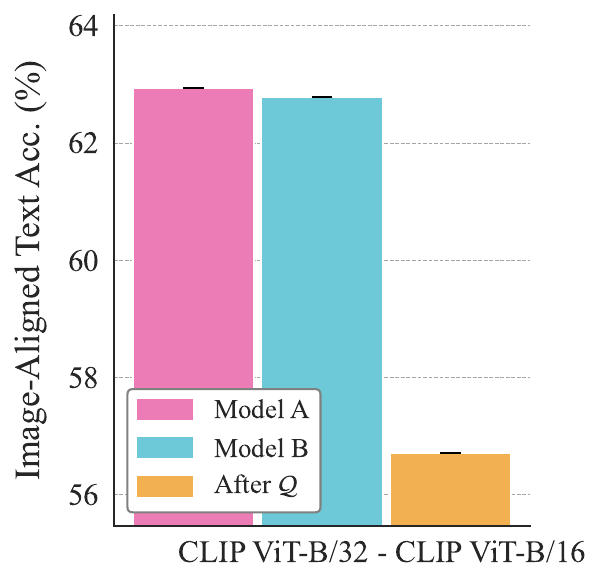}
    \caption{}
    \end{subfigure}
    \begin{subfigure}{0.32\textwidth}
    \centering
    \includegraphics[width=\textwidth]{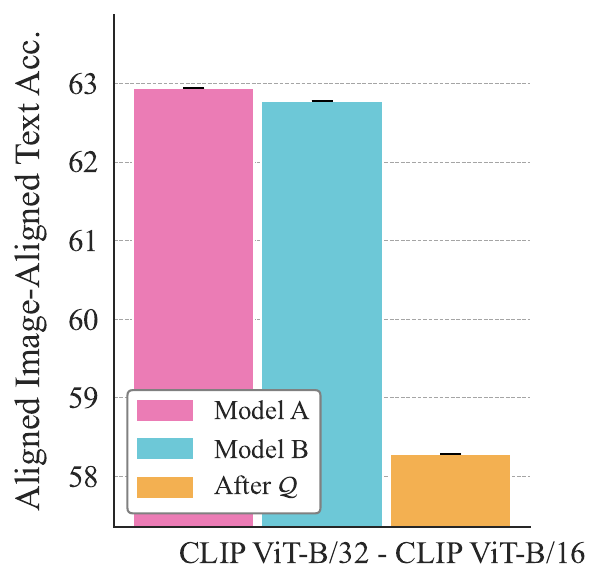}
    \caption{}
    \end{subfigure}
    \end{minipage}

   \caption{\textit{Cross-model alignment between identical training datasets with differing architectures (CIFAR-100; CLIP ViT-B/32 to ViT-B/16, OpenAI) before and after fitting a single orthogonal map $\mathcal{Q}$.} (a) Image–image class retrieval and (b) text–text class retrieval. (c) Mean text–text cosine similarity. (d) Image–text retrieval using aligned images from model A and text from model B. (e) Image–text retrieval using images from model B and aligned text from model A. (f) Image–text retrieval using aligned images and aligned text from model A. When both models are trained on the same data distribution and differ only in architecture, applying $\mathcal Q$ yields substantially stronger modality transfer than in settings where the training distributions differ.}
    \label{fig:app_openai_cifar100}
\end{figure*}

\begin{figure*}[!htb]
    \centering
    \begin{minipage}{0.65\linewidth}
    \begin{subfigure}{0.32\textwidth}
    \centering
    \includegraphics[width=\textwidth]{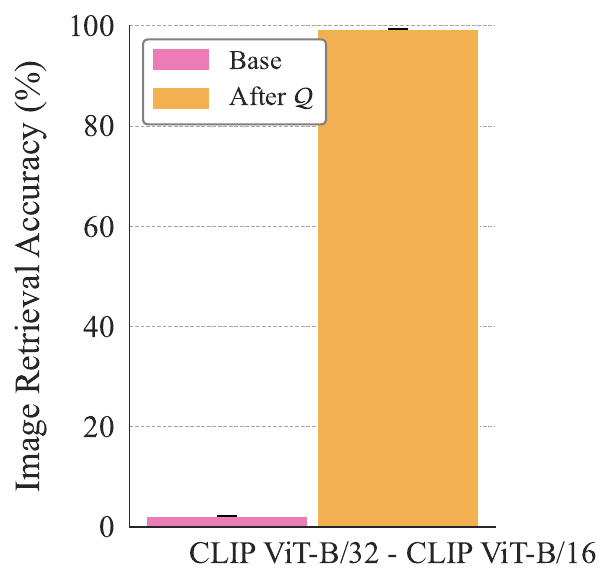}
    \caption{}
    \end{subfigure}
    \begin{subfigure}{0.32\textwidth}
    \centering
    \includegraphics[width=\textwidth]{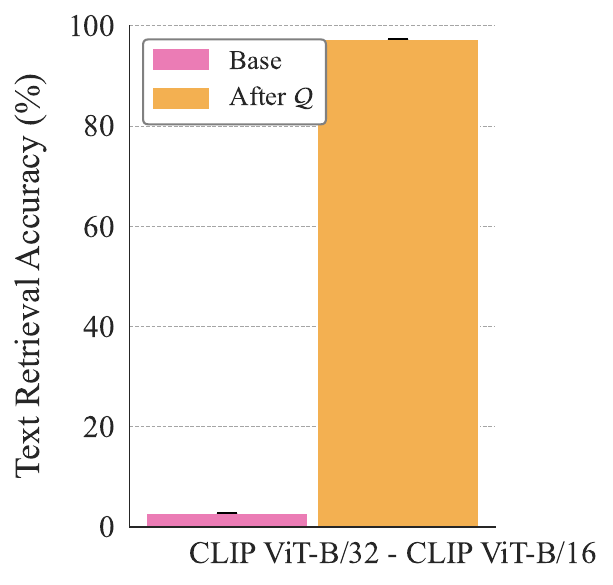}
    \caption{}
    \end{subfigure}
    \begin{subfigure}{0.32\textwidth}
    \centering
    \includegraphics[width=\textwidth]{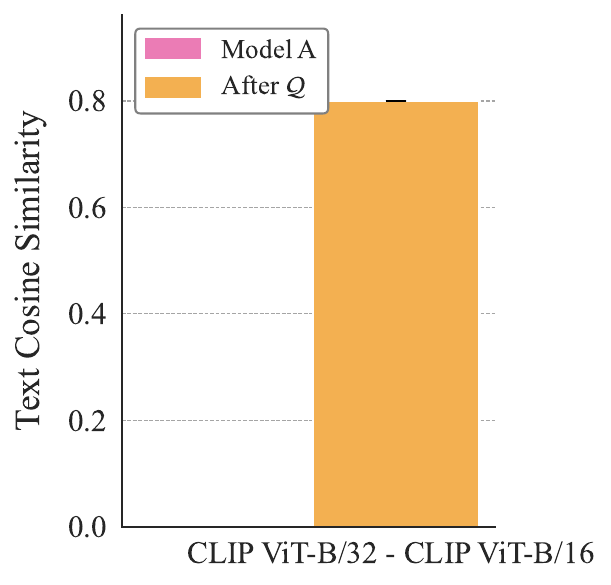}
    \caption{}
    \end{subfigure}

    \vspace{1mm}
    \begin{subfigure}{0.32\textwidth}
    \centering
    \includegraphics[width=\textwidth]{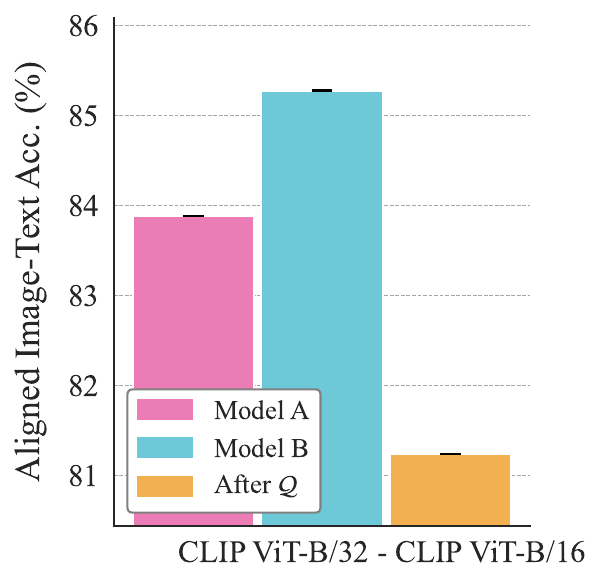}
    \caption{}
    \end{subfigure}
    \begin{subfigure}{0.32\textwidth}
    \centering
    \includegraphics[width=\textwidth]{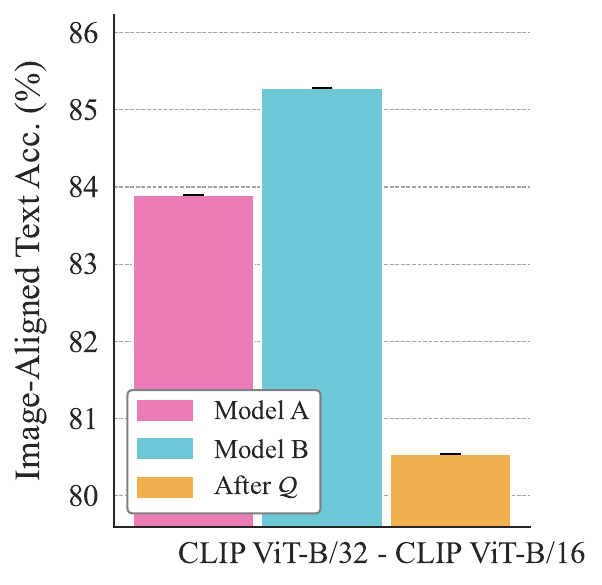}
    \caption{}
    \end{subfigure}
    \begin{subfigure}{0.32\textwidth}
    \centering
    \includegraphics[width=\textwidth]{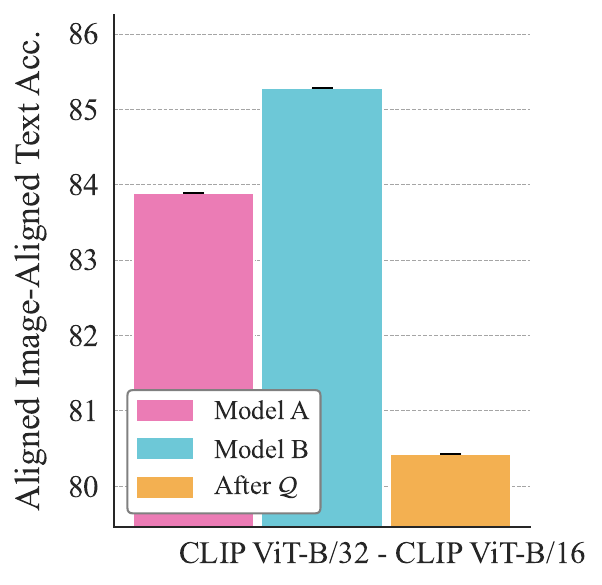}
    \caption{}
    \end{subfigure}
    \end{minipage}

    \caption{\textit{Cross-model alignment between identical training datasets with differing architectures (Oxford Pets; CLIP ViT-B/32 to ViT-B/16, OpenAI) before and after fitting a single orthogonal map $\mathcal{Q}$.} (a) Image–image class retrieval and (b) text–text class retrieval. (c) Mean text–text cosine similarity. (d) Image–text retrieval using aligned images from model A and text from model B. (e) Image–text retrieval using images from model B and aligned text from model A. (f) Image–text retrieval using aligned images and aligned text from model A. When both models are trained on the same data distribution and differ only in architecture, applying $\mathcal Q$ yields substantially stronger modality transfer than in settings where the training distributions differ.}
    \label{fig:app_openai_oxford}
\end{figure*}

\FloatBarrier

\subsection{Qualitative Evidence of Commutativity of Image-Text Alignment Paths}\label{sec:app_qualitative_knn}
In this section, we qualitatively test whether the learned map $\mathcal Q$ induces a consistent, modality-invariant geometry by comparing two alignment paths from a source image $x$. In the \emph{direct} path, we map the image embedding via $\mathcal Q$ and retrieve its top-$5$ nearest images in the target model’s space. In the \emph{text-mediated} path, we first retrieve the nearest text to $x$ in the source model, map this text embedding via $\mathcal Q$, retrieve its top-$1$ nearest text in the target model, and then retrieve the top-$5$ images associated with that text.

As shown in~\Cref{fig:knn_alignment_results_1} and~\Cref{fig:knn_alignment_results_2}, both paths yield highly consistent semantic neighborhoods: images retrieved via the text-mediated route closely match those obtained by direct image alignment. Equivalently, in terms of the commuting diagram induced by cross-modal nearest-neighbor operators, the two routes from $x$—direct image transport and transport via a retrieved caption—lead to nearly the same target-space neighborhood, suggesting that $\mathcal Q$ approximately commutes with these retrieval operators on this domain.

\begin{figure}[!htb]
    \centering
    \includegraphics[width=0.85\textwidth]{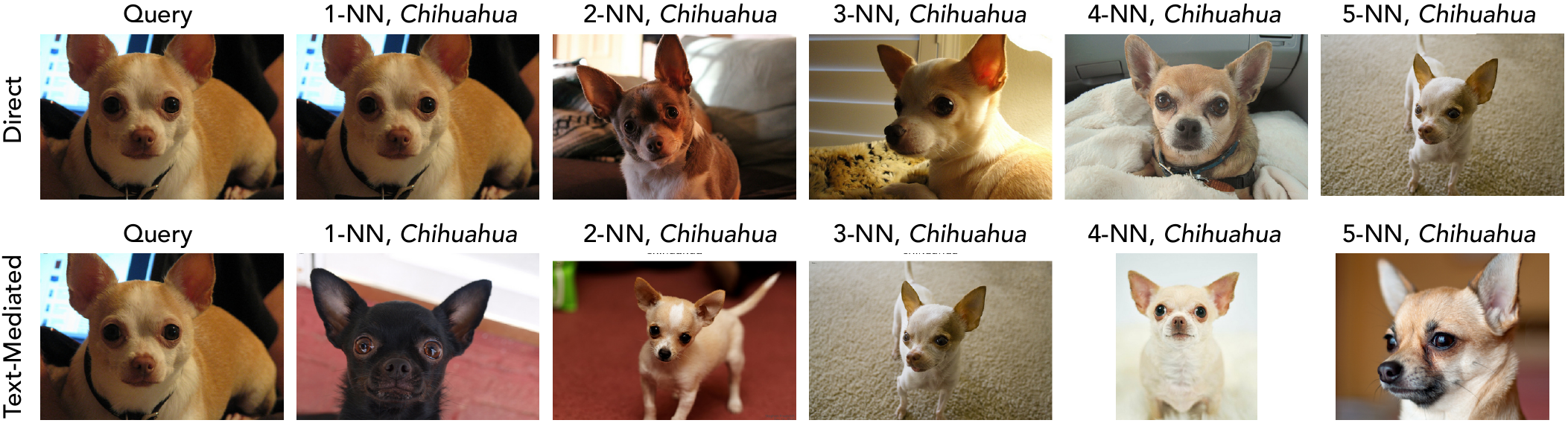} 
    \caption{\textit{Qualitative k-NN retrieval under two alignment paths for Oxford Pets.} We compare two routes from a source image $x$ into the target model’s image space after alignment by $\mathcal Q$: a direct image-alignment path (top row) and a text-mediated path (bottom row). Both routes yield highly consistent target-space neighborhoods, showing that the two alignment paths approximately close as a commuting diagram.}  
    \label{fig:knn_alignment_results_1}
\end{figure}

\begin{figure}[!htb]
    \centering
    \includegraphics[width=0.85\textwidth]{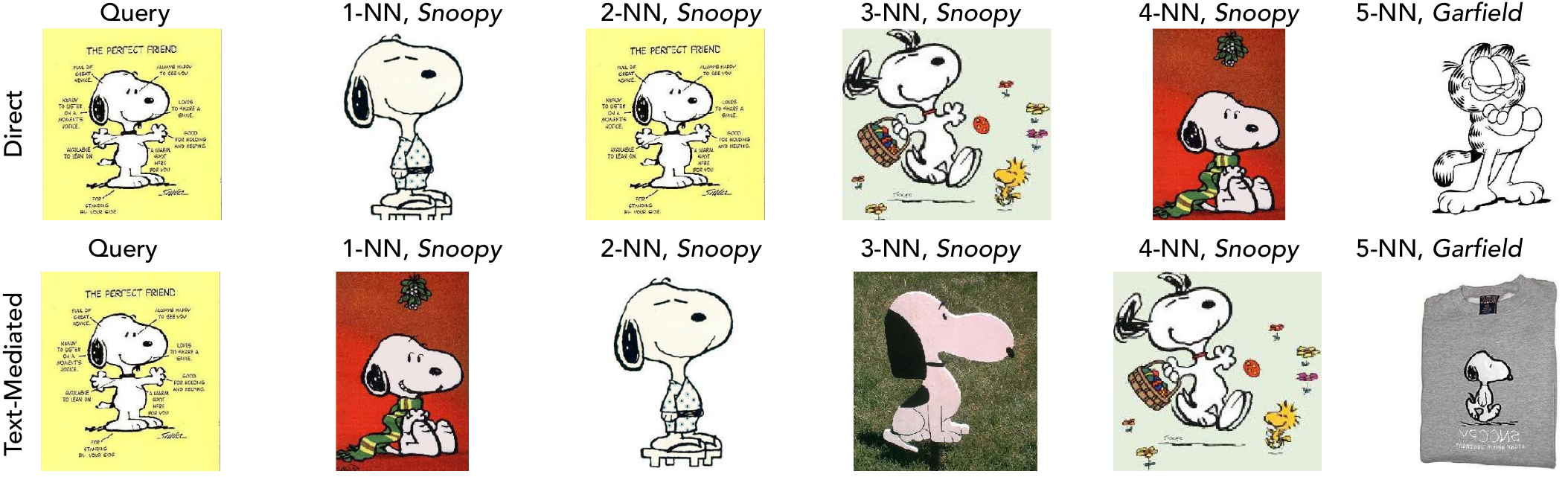} 
     \caption{\textit{Qualitative k-NN retrieval under two alignment paths for Caltech-101.} We compare two routes from a source image $x$ into the target model’s image space after alignment by $\mathcal Q$: a direct image-alignment path (top row) and a text-mediated path (bottom row). Both routes yield highly consistent target-space neighborhoods, showing that the two alignment paths approximately close as a commuting diagram.}    \label{fig:knn_alignment_results_2}
\end{figure}

\FloatBarrier

\subsection{Fine-grained Semantic Preservation Under Orthogonal Maps.}\label{sec:app_semantic}
We further probe whether the learned map $\mathcal Q$ preserves semantic structure beyond coarse class labels. Starting from a source image $x$, we retrieve its nearest caption in Model A’s text space and then construct a minimally perturbed variant by changing only the class token (e.g., \textit{Birman} to \textit{Russian Blue}) while keeping the remaining description fixed. We map this modified text via $\mathcal Q$, retrieve its nearest text in Model B’s space, and finally retrieve the nearest image in Model B’s image space.

As shown in~\Cref{fig:knn_russian_blue}, the retrieved images not only reflect the edited class (Russian Blue) but also somewhat preserve finer visual attributes implied by the original description. In particular, different captions lead to distinct Russian Blue images that emphasize different features (e.g., eye prominence vs.\ body pose), despite sharing the same class label. Similarly, in~\Cref{fig:knn_basset_hound}, the retrieved Basset Hound images emphasize body mass and posture consistent with the original Bulldog description, rather than collapsing to a generic class prototype. Together with the previous example, this indicates that alignment by $\mathcal Q$ preserves fine-grained semantic structure encoded in text and image embeddings, enabling controlled semantic edits to transfer predictably across models.

\begin{figure}[!htb]
    \centering
    \includegraphics[width=0.6\textwidth]{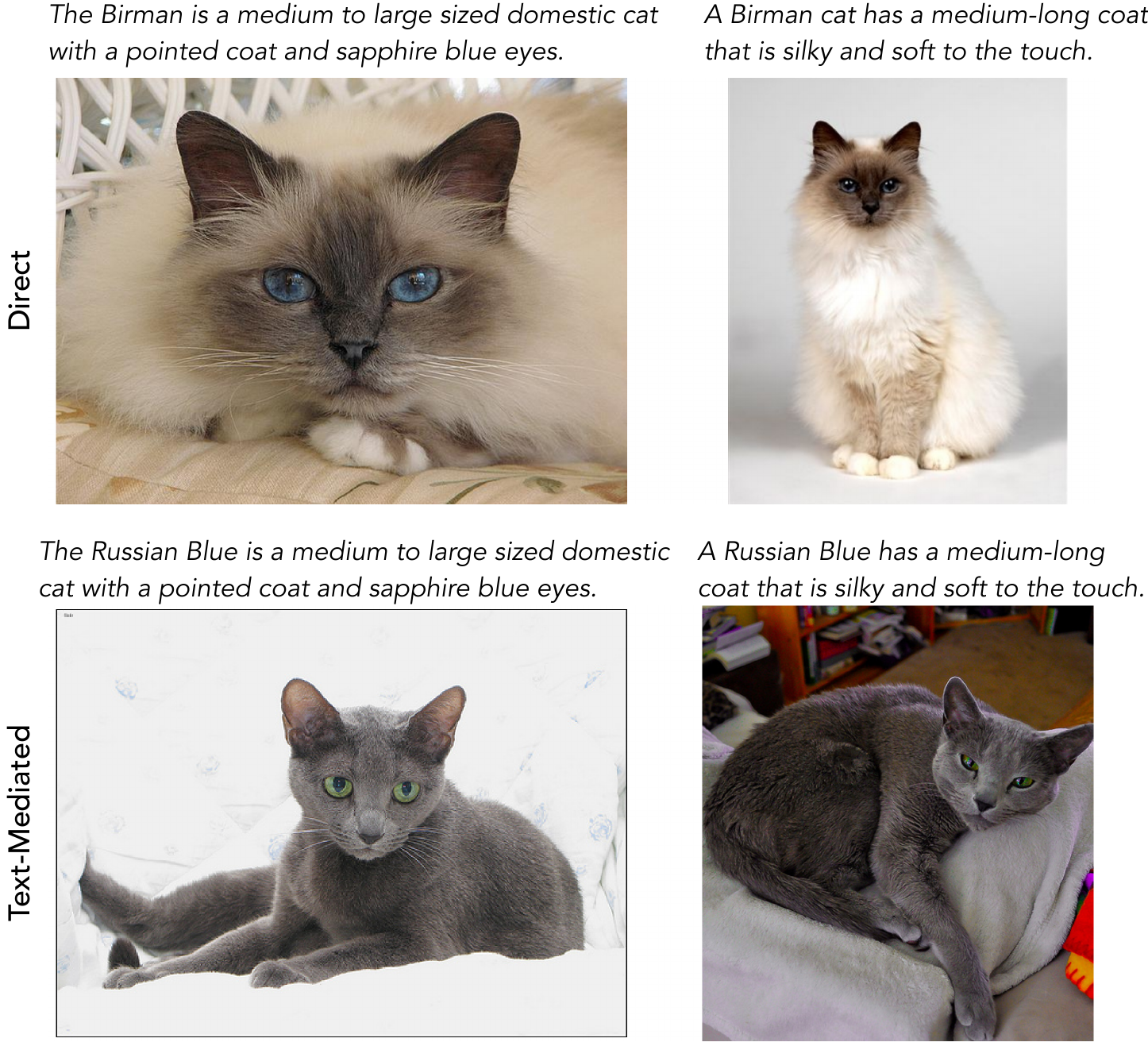} 
    \caption{\textit{Fine-grained semantic preservation under controlled text edits on Oxford Pets.} After changing only the class token in the caption (Birman to Russian Blue) while keeping descriptive attributes fixed, alignment via $\mathcal Q$ retrieves target-model images that match the edited class while preserving caption-specific visual cues (e.g., eye prominence and facial structure), indicating preservation of fine-grained semantics beyond class identity.\looseness=-1}
    \label{fig:knn_russian_blue}
\end{figure}

\begin{figure}[!htb]
    \centering
    \includegraphics[width=0.6\textwidth]{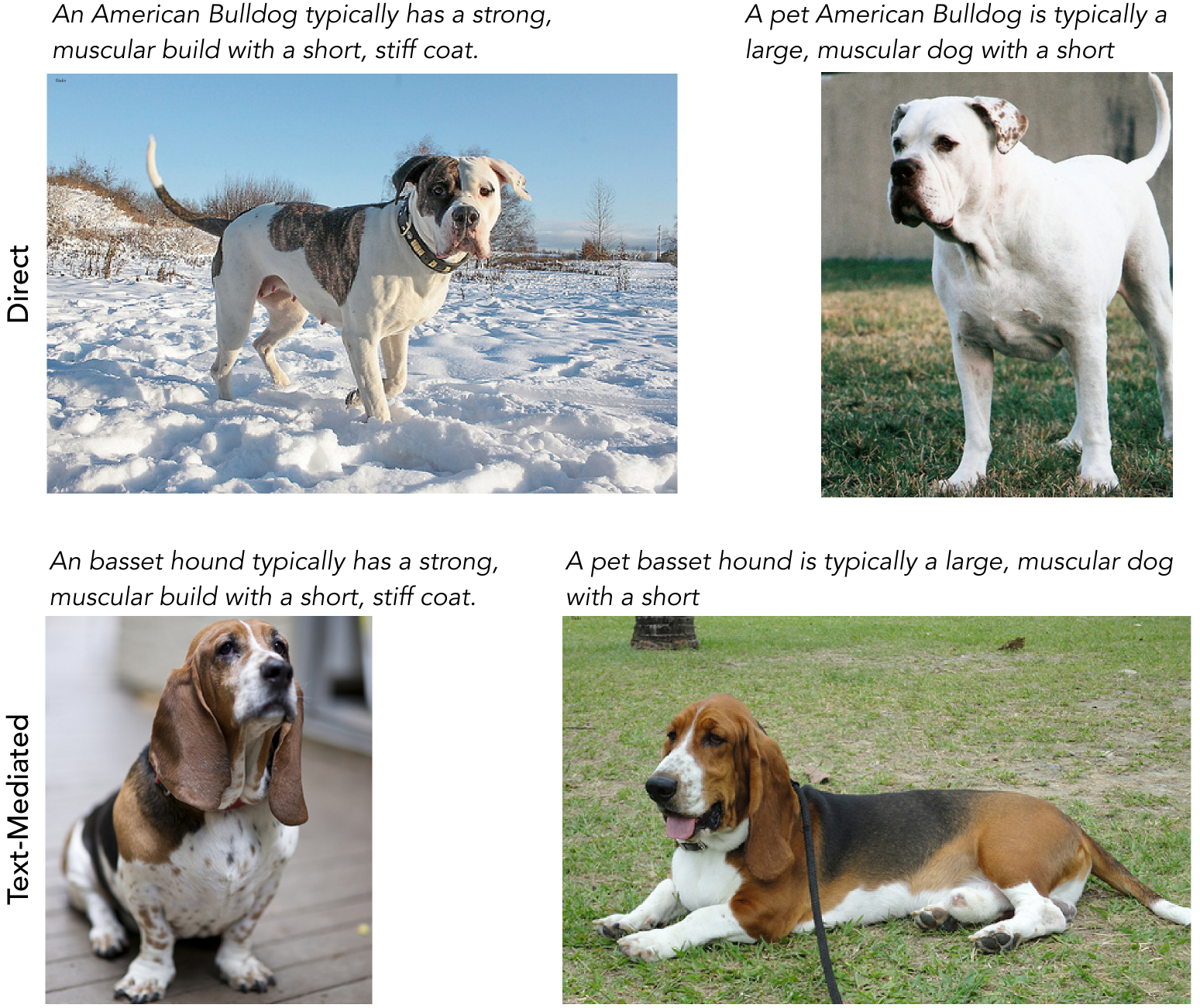} 
    \caption{\textit{Fine-grained semantic preservation under controlled text edits on Oxford Pets.} After changing only the class token in the caption (American Bulldog to Basset Hound) while keeping descriptive attributes fixed, alignment via $\mathcal Q$ retrieves target-model images that match both the edited class and the original attribute cues (e.g., muscular build and posture), indicating preservation of fine-grained semantics beyond class identity.\looseness=-1}
    \label{fig:knn_basset_hound}
\end{figure}

\end{document}